\documentclass[11pt]{article}

\usepackage{fullpage}
\usepackage{natbib}
\usepackage{algorithm}
\usepackage{algorithmic}

\usepackage{comment} 
\usepackage{lipsum} 
\usepackage{amsmath}
\usepackage{mathtools}
\usepackage[utf8]{inputenc}
\usepackage[T2A]{fontenc}
\usepackage{amssymb}  
\usepackage{algorithm}
\usepackage{algorithmic}
\usepackage{graphicx}
\usepackage[dvipsnames]{xcolor}
\usepackage{nicefrac}
\usepackage{bm}

\usepackage{threeparttable}
\usepackage{makecell}
\usepackage{multirow}
\usepackage{colortbl}
\definecolor{bgcolor}{rgb}{0.8,1,1}
\definecolor{bgcolor2}{rgb}{0.8,1,0.8}

\usepackage{graphicx} 
\graphicspath{{.}}

\newcommand{\fronorm}[1]{\left\|#1\right\|_{\text{F}}}
\newcommand{\eqdef}{\; { := }\;}

\newcommand{\R}{\mathbb{R}}
\newcommand{\B}{\mathbb{B}}
\newcommand{\bC}{\mathbb{C}}

\newcommand{\N}{\mathbb{N}}

\newcommand{\E}{\mathbb{E}}

\newcommand{\norm}[1]{\left\|#1\right\|}
\def\<#1,#2>{\langle #1,#2\rangle}

\usepackage{xspace}
\newcommand{\algname}[1]{{\sf\green#1}\xspace}
\newcommand{\algnameS}[1]{{\sf\green#1}\xspace}
\newcommand{\dataname}[1]{{\tt\color{blue}#1}\xspace}
\newcommand{\des}[1]{{\em\color{red} [#1]}}
\newcommand{\alert}[1]{{\em\color{red} #1}}

\newcommand{\HS}{L_{*}}      
\newcommand{\HF}{L_{\rm F}}
\newcommand{\HM}{L_{\infty}}

\usepackage[font=normalsize, labelfont=bf]{caption}
\newcommand{\makecellnew}[1]{{\renewcommand{\arraystretch}{0.8}\begin{tabular}{c} #1 \end{tabular}}}

\newcommand{\squeeze}{}

\def\<{\left\langle}
\def\>{\right\rangle}
\def\[{\left[}
\def\]{\right]}
\def\({\left(}
\def\){\right)}

\newcommand{\cO}{\mathcal{O}}

\usepackage{tcolorbox}
\usepackage{pifont}
\definecolor{mydarkgreen}{RGB}{39,130,67}
\definecolor{mydarkred}{RGB}{192,47,25}
\newcommand{\green}{\color{mydarkgreen}}
\newcommand{\red}{\color{mydarkred}}
\newcommand{\cmark}{\green\ding{51}}%
\newcommand{\xmark}{\red\ding{55}}%

\newcommand{\mA}{\mathbf{A}}
\newcommand{\mB}{\mathbf{B}}

\newcommand{\mH}{\mathbf{H}}

\newcommand{\mI}{\mathbf{I}}

\newcommand{\mM}{\mathbf{M}}
\newcommand{\mS}{\mathbf{S}}

\newcommand{\cC}{{\mathcal{C}}}

\newcommand{\cH}{{\mathcal{H}}}

\usepackage{amsmath,amsfonts,amssymb,amsthm,array}

\usepackage{mdframed} 
\usepackage{thmtools}
\usepackage{textcomp}

\declaretheorem[within=section]{definition}
\declaretheorem[sibling=definition]{theorem}

\declaretheorem[sibling=definition]{assumption}

\declaretheorem[sibling=definition]{lemma}

\usepackage[colorinlistoftodos,bordercolor=orange,backgroundcolor=orange!20,linecolor=orange,textsize=scriptsize]{todonotes}

\usepackage{microtype}
\usepackage{subfigure}
\usepackage{booktabs} 

\usepackage{grffile}

\usepackage{hyperref}

\usepackage[utf8]{inputenc} 
\usepackage[T1]{fontenc}    
\usepackage{hyperref}       
\usepackage{url}            
\usepackage{booktabs}       
\usepackage{amsfonts}       
\usepackage{nicefrac}       
\usepackage{microtype}      
\usepackage{xcolor}         

\title{\bf FedNL: Making Newton-Type Methods \\ \bf Applicable to Federated Learning}

\author{Mher Safaryan\thanks{King Abdullah University of Science and Technology, Thuwal, Saudi Arabia.} \and Rustem Islamov\thanks{King Abdullah University of Science and Technology (KAUST), Thuwal, Saudi Arabia, and Moscow Institute of Physics and Technology (MIPT), Dolgoprudny, Russia. This research was conducted while this author was an intern at KAUST and an undergraduate student at MIPT.} \and Xun Qian\thanks{King Abdullah University of Science and Technology, Thuwal, Saudi Arabia.} \and Peter Richt\'{a}rik\thanks{King Abdullah University of Science and Technology, Thuwal, Saudi Arabia.}}
\date{June 05, 2021}

\begin{document}

\maketitle

\begin{abstract}
Inspired by recent work of Islamov et al (2021), we propose a family of Federated Newton Learn (\algname{FedNL}) methods, which we believe is a marked step in the direction of making second-order methods applicable to FL. In contrast to the aforementioned work, \algname{FedNL} employs a different Hessian learning technique which i) enhances privacy as it does not rely on the training data to be revealed to the coordinating server, ii) makes it applicable beyond generalized linear models, and iii) provably works with general contractive compression operators for compressing the local Hessians, such as Top-$K$ or Rank-$R$, which are vastly superior in practice. Notably, we do not need to rely on error feedback for our methods to work with contractive compressors. 

Moreover, we develop \algname{FedNL-PP}, \algname{FedNL-CR} and \algname{FedNL-LS}, which are variants of \algname{FedNL} that support partial participation, and globalization via cubic regularization and line search, respectively, and \algname{FedNL-BC}, which is a variant that can further benefit from bidirectional compression of gradients and models, i.e., smart uplink gradient and smart downlink model compression. 

We prove local convergence rates that are independent of the condition number, the number of training data points, and compression variance.  Our communication efficient Hessian learning technique provably learns the Hessian at the optimum. 

Finally, we perform a variety of numerical experiments that show that our \algname{FedNL} methods have state-of-the-art communication complexity when compared to key baselines.
\end{abstract}

\clearpage 
 \tableofcontents
 
\clearpage 
\section{Introduction}

In this paper we consider the {\em federated learning} problem
\begin{equation}\label{erm-prob}
\squeeze
\min\limits_{x\in\R^d} \left\{f(x) \eqdef \frac{1}{n}\sum \limits_{i=1}^n f_i(x)\right\},
\end{equation}
where $d$ denotes dimension of the model $x\in\R^d$ we wish to train, $n$ is the total number of silos/machines/devices/clients in the distributed system, $f_i(x)$ is the loss/risk associated with the data stored on machine $i\in[n] \eqdef \{1, 2, \dots, n\}$,
and $f(x)$ is the empirical loss/risk. 

\subsection{First-order methods for FL}
The prevalent paradigm for training federated learning (FL) models \citep{FEDLEARN,FEDOPT,FL2017-AISTATS} (see also the recent surveys by \citet{FL-big, FL_survey_2020}) is to use {\em distributed first-order optimization methods} employing one or more tools for enhancing communication efficiency, which is a key bottleneck in the federated setting. 

These tools include {\em communication compression} \citep{FEDLEARN, QSGD, DCGD} and techniques for progressively reducing the variance introduced by compression \citep{DIANA, DIANA-VR, sigma_k, ADIANA, MARINA},  {\em local computation} \citep{FL2017-AISTATS,Stich-localSGD, localSGD-AISTATS2020, FedRR} and techniques for reducing the client drift introduced by local computation \citep{SCAFFOLD, Gorbunov2020localSGD},  and  {\em partial participation} \citep{FL2017-AISTATS, SGD-AS} and techniques for taming the slow-down introduced by partial participation \citep{sigma_k, OptClientSampling2020}.  

Other useful techniques for further reducing the communication complexity of FL methods include the use of {\em momentum} \citep{DIANA, ADIANA}, and {\em adaptive learning rates} \citep{MM2019, xie2019local, reddi2020adaptive, xie2019local, IntSGD}.   In addition, aspiring FL methods need to protect the {\em privacy} of the clients' data, and need to be built with {\em data  heterogeneity} in mind \citep{FL-big}.

\subsection{Towards second-order methods for FL}
While first-order methods are the methods of choice in the context of FL  at the moment, their communication complexity necessarily depends on (a suitable notion of) the condition number of the problem, which can be very large as it depends on the structure of the model being trained, on the choice of the loss function, and most importantly, on the properties of the~training~data. 

However,  in many situations when algorithm design is not constrained by the stringent requirements characterizing FL, it is very well known that carefully designed {\em second-order methods} can be vastly superior. On an intuitive level, this is mainly because these methods make an extra computational effort to estimate the local curvature of the loss landscape, which is useful in generating more powerful and adaptive update direction. However, in FL, it is often communication and not computation which forms the key bottleneck, and hence the idea of ``going second order'' looks attractive.  The theoretical benefits of using curvature information are well known. For example, the classical Newton's method, which forms the basis for most efficient second-order method in much the same way the gradient descent method forms the basis for more elaborate first-order methods, enjoys a fast {\em condition-number-independent} (local) convergence rate \citep{Beck-book-nonlinear}, which is beyond the reach of {\em all} first-order methods. However,  Newton's method does not admit an efficient distributed implementation in the heterogeneous data regime as it requires repeated communication of local  Hessian matrices $\nabla^2 f_i \in \R^{d\times d}$ to the server, which is prohibitive as this constitutes a massive burden on the communication links. 

\subsection{Desiderata for second-order methods applicable to FL} 
In this paper we take the stance that it would be highly \alert{desirable} to develop Newton-type methods for solving the federated learning problem \eqref{erm-prob} that would  
\begin{itemize}
\item [\des{hd}] work well in the truly \alert{h}eterogeneous \alert{d}ata setting (i.e., we do not want to assume that the functions $f_1,\dots,f_n$ are ``similar''), 
\item [\des{fs}] apply to the general \alert{f}inite-\alert{s}um problem \eqref{erm-prob}, without imparting strong structural assumptions on the local functions $f_1,\dots,f_n$ (e.g., we do not want to assume that the functions $f_1,\dots,f_n$ are quadratics, generalized linear models, and so on), 
\item [\des{as}] benefits from Newton-like (matrix-valued) \alert{a}daptive  \alert{s}tepsizes,
\item [\des{pe}] employ at least a rudimentary \alert{p}rivacy \alert{e}nhancement mechanism (in particular, we do not want the devices to be sending/revealing their training data to the server), 
\item [\des{uc}] enjoy, through  \alert{u}biased communication \alert{c}ompression strategies applied to the Hessian, such as Rand-$K$, the same low $\cO(d)$  communication cost per communication round as gradient descent, 
\item [\des{cc}] be able to benefit from the more aggressive \alert{c}ontractive communication \alert{c}ompression strategies  applied to the Hessian, such as Top-$K$ and Rank-$R$,
\item [\des{fr}] have \alert{f}ast local \alert{r}ates unattainable by first order methods (e.g., rates independent of the condition number), 
\item [\des{pp}] support \alert{p}artial \alert{p}articipation (this is important when the number $n$ of devices is very large), 
\item [\des{gg}] have \alert{g}lobal convergence \alert{g}uarantees, and superior global empirical behavior, when combined with a suitable globalization strategy (e.g., line search or cubic regularization), 
\item [\des{gc}] optionally be able to use, for a more dramatic communication reduction, additional smart uplink (i..e, device to server) \alert{g}radient \alert{c}ompression,
\item [\des{mc}]  optionally be able to use, for a more dramatic communication reduction, additional smart downlink (i.e., server to device) \alert{m}odel \alert{c}ompression,
\item [\des{lc}] perform provably useful \alert{l}ocal \alert{c}omputation, even in the heterogeneous data setting (it is known that local computation via gradient-type steps, which form the backbone of methods such as FedAvg and LocalSGD, provably helps under some degree of data similarity only). 
\end{itemize}

  However, to the best of our knowledge, existing Newton-type methods are not applicable to FL as they are not compatible with most of the aforementioned desiderata. 
  
  \begin{quote}
  {\em It is therefore natural and pertinent to ask whether it is possible to design  theoretically well grounded and empirically well performing Newton-type methods that would be able to conform to the  FL-specific {\color{red}desiderata} listed above. } 
  \end{quote}

  In this work, we address this challenge  {\em in the affirmative.}

\section{Contributions} \label{sec:contributions}

Before detailing our contributions,  it will be very useful to briefly outline the key elements of the recently proposed {\em Newton Learn} (\algname{NL}) framework of \citet{Islamov2021NewtonLearn}, which served as the main inspiration for our work, and which  is also the closest work to ours. 

\subsection{The Newton Learn framework of \citet{Islamov2021NewtonLearn}}
The starting point of their work is the observation that the Newton-like method  $$x^{k+1}=x^k-(\nabla^2 f(x^*))^{-1} \nabla f(x^k),$$ called Newton Star (\algname{NS}), where $x^*$ is the (unique) solution of \eqref{erm-prob}, converges to $x^*$ locally quadratically under suitable assumptions, which is a desirable property it inherits from the classical Newton method. Clearly, this method is not practical, as it relies on the knowledge of the Hessian at the optimum. 

However, under the assumption that the matrix $\nabla^2 f(x^*)$ is known to the server, \algname{NS} can be implemented with $\cO(d)$ cost in each communication round. Indeed, \algname{NS} can simply  be treated as gradient descent, albeit with a matrix-valued stepsize equal to $(\nabla^2 f(x^*))^{-1}$. 

The first key contribution of \citet{Islamov2021NewtonLearn} is the design of a strategy, for which they coined the term {\em Newton Learn}, which {\em learns the Hessians $\nabla^2 f_1(x^*), \dots, \nabla^2 f_n(x^*)$}, and hence their average, $\nabla^2 f(x^*)$, progressively throughout the iterative process, and does so in a {\em communication efficient manner,} using unbiased compression \des{uc} of Hessian information. In particular, the compression level can be adjusted so that in each communication round,  $\cO(d)$ floats need to be communicated between each device and the server only. In each iteration, the master uses the average of the current learned local Hessian matrices in place of the Hessian at the optimum, and subsequently performs a step similar to that of \algname{NS}. So, their method uses {\em adaptive matrix-valued stepsizes} \des{as}. 

\citet{Islamov2021NewtonLearn} prove that their learning procedure indeed works in the sense that the sequences of the learned local matrices converge to the local optimal Hessians $\nabla^2 f_i(x^*)$. This property leads to a Newton-like {\em acceleration}, and as a result,  their \algname{NL} methods enjoy a local linear convergence rate (for a Lyapunov function that includes Hessian convergence) and local superlinear convergence rate (for distance to the optimum) that is {\em independent of the condition number}, which is a property beyond the reach of any first-order method \des{fr}.  Moreover, all of this provably works in the heterogeneous data setting \des{hd}. 

Finally, they develop a practical and theoretically grounded globalization strategy \des{gg} based on cubic regularization, called {\em Cubic Newton Learn} (\algname{CNL}).
   
\begin{table*}
\caption{Comparison of the main features of our family of \algname{FedNL} algorithms and results  with those of \citet{Islamov2021NewtonLearn}, which we used as an inspiration. We have made numerous and significant modifications and improvements in order to obtain methods applicable to federated learning.} \label{tbl:main}
\centering
\begin{threeparttable}
\footnotesize
\begin{tabular}{|c|c|c|c|}
\hline
\bf \# & \bf Feature & \begin{tabular}{c}\bf \cite{Islamov2021NewtonLearn} \end{tabular} & \begin{tabular}{c}\bf This Work \end{tabular} \\
\hline
\des{hd} & supports \alert{h}eterogeneous \alert{d}ata setting
& \cmark  & \cmark \\
\hline
\des{fs} & applies to general \alert{f}inite-\alert{s}um problems
& \xmark  & \cmark \\
\hline
\des{as} & uses \alert{a}daptive \alert{s}tepsizes & \cmark  & \cmark \\
\hline
\des{pe} &  \alert{p}rivacy is \alert{e}nhanced (training data is not sent to the server)
&  \xmark & \cmark \\
\hline
\des{uc}  & supports \alert{u}nbiased Hessian \alert{c}ompression  (e.g., Rand-$K$)  &  \cmark & \cmark \\
\hline
\des{cc}  & supports \alert{c}ontractive Hessian \alert{c}ompression (e.g., Top-$K$) &  \xmark & \cmark \\
\hline
\des{fr} & \alert{f}ast local \alert{r}ate: independent of the condition number &  \cmark & \cmark \\
\hline
 \des{fr} & \alert{f}ast local \alert{r}ate:  independent of the \# of training data points &  \xmark & \cmark \\
\hline
\des{fr}  &  \alert{f}ast local \alert{r}ate:  independent of the compressor variance &  \xmark & \cmark \\
\hline
\des{pp} & supports \alert{p}artial \alert{p}articipation &  \xmark & \cmark (Alg~\ref{alg:FedNL-PP}) \\
\hline
\des{gg} & has \alert{g}lobal convergence \alert{g}uarantees via line search &  \xmark & \cmark (Alg~\ref{alg:FedNL-LS}) \\
\hline
\des{gg} & has \alert{g}lobal convergence \alert{g}uarantees via cubic regularization &  \cmark & \cmark (Alg~\ref{alg:FedNL-CR}) \\
\hline
\des{gc} & supports smart uplink \alert{g}radient \alert{c}ompression at the devices &  \xmark & \cmark (Alg~\ref{alg:FedNL-BC}) \\
\hline
\des{mc} & supports smart downlink \alert{m}odel \alert{c}ompression by the master &  \xmark & \cmark (Alg~\ref{alg:FedNL-BC}) \\
\hline
\des{lc}  & performs useful \alert{l}ocal \alert{c}omputation &  \cmark & \cmark \\
\hline
\end{tabular}
\end{threeparttable}
\end{table*}   
   
\subsection{Issues with the Newton Learn framework}
While the above development is clearly very promising in the context of distributed optimization, the results suffer from several limitations which prevent the methods from being applicable to FL. First, the Newton Learn strategy of \citet{Islamov2021NewtonLearn} critically depends on the assumption that the local functions are of the form \begin{equation}\squeeze f_i(x) = \frac{1}{m}\sum_{j=1}^m  \varphi_{ij}(a_{ij}^\top x),\label{eq:GLM}\end{equation} where $\varphi_{ij}:\R\to \R$ are sufficiently well behaved functions, and $a_{i1},\dots,a_{im}\in \R^d$ are the training data points owned by device $i$. As a result, their approach is limited to generalized linear models only, which violates \des{fs} from the aforementioned wish list. Second,  their communication strategy critically relies on each device $i$ sending a small subset of their private training data $\{a_{i1},\dots,a_{im}\}$ to the server in each communication round, which violates \des{pe}.  Further, while their approach supports $\cO(d)$ communication, it does not support more general contractive compressors \des{cc}, such as Top-$K$ and Rank-$R$, which have been found very useful in the context of first order methods with gradient compression.  Finally, the methods of \citet{Islamov2021NewtonLearn} do not support bidirectional compression \des{bc} of gradients and models, and do not support partial participation \des{pp}.

{
    \begin{table*}[!h]
    \scriptsize
    \addtolength{\tabcolsep}{-3pt} 
        \centering
        \caption{Theoretical comparison of 2 gradient-based (Gradient Descent and \algname{ADIANA}) and 3 second-order (\algname{Newton}, \algname{NL} and \algname{FedNL}) methods. See Section \ref{sec:more} and the extended Table \ref{tbl:comm-complexity} for more details.}
        \label{tbl:comm-complexity-sketch}
        \renewcommand{\arraystretch}{1.7}
        \begin{tabular}{|c|c|c|c|}
        \hline
        \makecellnew{Method}
        & \makecellnew{\# Communication Rounds}
        & \makecellnew{Comm. Cost \\ per Round}
        & \makecellnew{Communication Complexity} \\
        \hline
\hline
            \makecellnew{Gradient Descent$^1$}
            & $\cO(\kappa\log\frac{1}{\epsilon})$
            & $\cO(d)$
            & $\cO(d\kappa\log\frac{1}{\epsilon})$\\
\hline
            \makecellnew{ADIANA$^1$ \\ \cite{ADIANA}}
            & $\cO\(\(d + \sqrt{\kappa} + \sqrt{\(\frac{d}{n} + \sqrt{\frac{d}{n}}\)d\kappa} \)\log\frac{1}{\epsilon}\)$
            & $\cO(1)$
            & $\cO\(\(d + \sqrt{\kappa} + \sqrt{\(\frac{d}{n} + \sqrt{\frac{d}{n}}\)d\kappa} \)\log\frac{1}{\epsilon}\)$\\
\hline
            \makecellnew{Newton}
            & $\cO(\log\log\frac{1}{\epsilon})$
            & $\cO(d^2)$
            & $\cO(d^2\log\log\frac{1}{\epsilon})$\\
\hline
            \makecellnew{NL \\ \citep{Islamov2021NewtonLearn}}
            & $\cO\(\sqrt{\#\text{data}} \sqrt{\log\frac{1}{\epsilon}}\)$
            & $\cO(d)$
            & $\cO\(d\sqrt{\#\text{data}} \sqrt{\log\frac{1}{\epsilon}}\)$\\
\hline
            \makecellnew{\bf FedNL \\ \bf (this work; \eqref{rate:local-linear-iter})}
            & $\cO\(\log\frac{1}{\epsilon}\)$
            & $\cO(d)$
            & $\cO\(d \log\frac{1}{\epsilon}\)$\\
\hline
            \makecellnew{\bf FedNL \\ \bf (this work; \eqref{rate:local-superlinear-iter})}
            & $\cO\(\sqrt{d} \sqrt{\log\frac{1}{\epsilon}}\)$
            & $\cO(d)$
            & $\cO\(d\sqrt{d} \sqrt{\log\frac{1}{\epsilon}}\)$\\
\hline
        \end{tabular}   
        \begin{tablenotes}
        {\scriptsize     
        \item ${}^1$ These methods have global rates. $\kappa$ is the condition number: $\kappa = \frac{L}{\mu}$ where $L$ is a smoothness constant and $\mu$ is the strong convexity constant.
        \item ${}^2$ The last column (communication complexity) is the product of the previous two columns and is the key quantity to be compared.
        }
        \end{tablenotes}   
    \end{table*}
}

\subsection{Our \algname{FedNL} framework}
We propose a family of five {\em Federated} Newton Learn methods (Algorithms~\ref{alg:FedNL}--\ref{alg:FedNL-BC}), which we believe constitutes a marked step in the direction of making second-order methods applicable to FL.  

In contrast to the work of \citet{Islamov2021NewtonLearn} (see Table~\ref{tbl:main}), our vanilla method \algname{FedNL} (Algorithm~\ref{alg:FedNL}) employs a {\em different Hessian learning technique}, which makes it applicable beyond generalized linear models \eqref{eq:GLM} to general finite-sum problems \des{fs}, enhances privacy as it does not rely on the training data to be revealed to the coordinating server \des{pe}, and provably works with general contractive compression operators for compressing the local Hessians, such as Top-$K$ or Rank-$R$, which are vastly superior in practice \des{cc}. Notably, we do not need to rely on error feedback \citep{1bit, StichNIPS2018-memory, SQSM, EC-SGD}, which is essential to prevent divergence in first-order methods employing such compressors \citep{biased2020}, for our methods to work with contractive compressors. We prove that our communication efficient Hessian learning technique provably learns the Hessians at the optimum. 

Like \citet{Islamov2021NewtonLearn}, we prove local convergence rates that are independent of the condition number \des{fr}. However, unlike their rates, some of our  rates are also independent of number training data points, and of compression variance \des{fr}. All our complexity results are summarized in Table~\ref{tbl:rates}.

Moreover, we show that our approach works in the partial participation \des{pp} regime by developing the \algname{FedNL-PP} method (Algorithm~\ref{alg:FedNL-PP}), and devise methods employing globalization strategies: \algname{FedNL-LS} (Algorithm~\ref{alg:FedNL-LS}), based on backtracking line search, and \algname{FedNL-CR} (Algorithm~\ref{alg:FedNL-CR}), based on cubic regularization \des{gg}. We show through experiments that the former is much more efficient in practice than the latter. Hence, the proposed line search globalization is superior to the cubic regularization  approach employed by \citet{Islamov2021NewtonLearn}. 

Our approach can further benefit from smart uplink gradient compression \des{gc} and smart downlink model compression \des{mc} -- see \algname{FedNL-BC} (Algorithm~\ref{alg:FedNL-BC}).   

Finally, we perform a variety of numerical experiments that show that our \algname{FedNL} methods have state-of-the-art communication complexity when compared to~key~baselines.

\begin{table*}[t]
    \caption{Summary of algorithms proposed and convergence results proved in this paper.}
    \label{tbl:rates}
    \vspace{-12pt}
    \begin{center}
        \scriptsize
            \begin{tabular}{| c | ccc | cc |}
                \hline
                  & \multicolumn{3}{c|}{\bf Convergence} & {\bf Rate independent of}  &   \\
                  {\bf Method} &  {\bf result} ${}^\dagger$   &  {\bf type}  & {\bf rate}  & \begin{tabular}{c} {\bf the condition \# (left)} \\ {\bf \# training data  (middle)} \\ {\bf compressor  (right)}  \end{tabular} & {\bf Theorem }\\              
                \hline
                \begin{tabular}{c} Newton Zero \\ \algnameS{N0} (Equation \eqref{eq:N0}) \end{tabular} & $r_{k}\leq \tfrac{1}{2^k} r_0$   & local & linear & \cmark\quad\cmark\quad\cmark &  \ref{th:NLU}\\ \hline  
             \multirow{3}{*}{ \begin{tabular}{c} \algnameS{FedNL}  (Algorithm~\ref{alg:FedNL}) \end{tabular}  } 
             &  $r_{k}\leq \tfrac{1}{2^k} r_0$ & local & linear & \cmark\quad\cmark\quad\cmark &  \ref{th:NLU}\\ 
       &  $\Phi_1^k \leq \theta^k \Phi_1^0$ & local & linear & \cmark\quad\cmark\quad\xmark &  \ref{th:NLU}\\ 
             &  $r_{k+1} \leq c \theta^k  r_k$ & local & superlinear & \cmark\quad\cmark\quad\xmark &  \ref{th:NLU}\\ \hline    
             \multirow{3}{*}{  \begin{tabular}{c}   Partial Participation \\ \algnameS{FedNL-PP} (Algorithm~\ref{alg:FedNL-PP})\end{tabular}  } 
             &   ${\cal W}^k \leq \theta^k {\cal W}^0$  & local & linear & \cmark\quad\cmark\quad\cmark &  \ref{th:NL-pp}\\ 
             &  $\Phi_2^k \leq \theta^k  \Phi_2^0$ & local & linear & \cmark\quad\cmark\quad\xmark &  \ref{th:NL-pp}\\  
       &  $r_{k+1} \leq c \theta^k  {\cal W}_k$ & local & linear & \cmark\quad\cmark\quad\xmark &  \ref{th:NL-pp}\\  
             \hline 
                           \begin{tabular}{c} Line Search \\ \algnameS{FedNL-LS} (Algorithm~\ref{alg:FedNL-LS}) \end{tabular} 
                & $\Delta_k \leq \theta^k \Delta_0$   & global & linear & \xmark\quad\cmark\quad\cmark &  \ref{th:NL-ls}\\ 
            \hline               
 \multirow{4}{*}{ \begin{tabular}{c}Cubic Regularization \\ \algnameS{FedNL-CR} (Algorithm~\ref{alg:FedNL-CR}) \end{tabular}  } 
                 & $\Delta_k \leq c/k$   & global & sublinear & \xmark\quad\cmark\quad\cmark &  \ref{th:NL-cr}\\ 
                 & $\Delta_k \leq \theta^k \Delta_0$ & global & linear & \xmark\quad\cmark\quad\cmark &  \ref{th:NL-cr}\\ 
                 & $\Phi_1^k \leq \theta^k \Phi_1^0$ & local & linear & \cmark\quad\cmark\quad\xmark &  \ref{th:NL-cr}\\                  
                 & $r_{k+1} \leq c \theta^k r_k$ & local & superlinear & \cmark\quad\cmark\quad\xmark &  \ref{th:NL-cr}\\ 
 \hline                          
       \begin{tabular}{c} Bidirectional Compression \\ \algnameS{FedNL-BC} (Algorithm~\ref{alg:FedNL-BC}) \end{tabular} & $\Phi_3^k \leq \theta^k \Phi_3^0$  & local & linear & \cmark\quad\cmark\quad\xmark &  \ref{th:FedNL-BC}\\ 
       \hline
       \begin{tabular}{c} Newton Star \\ \algnameS{NS} (Equation~\eqref{eq:newton-star}) \end{tabular} & $r_{k+1} \leq c r_k^2$   & local & quadratic & \cmark\quad\cmark\quad\cmark &  \ref{th:N*}\\ 
       \hline               
            \end{tabular}
    \begin{tablenotes}
      {\scriptsize      
          \item Quantities for which we prove convergence:
          (i) distance to solution
               $r_k \eqdef \|x^k-x^*\|^2$;\;
               ${\cal W}^k \eqdef \frac{1}{n} \sum_{i=1}^n \|w_i^k-x^*\|^2$ (ii) Lyapunov functions
               $\Phi_1^k \eqdef c \|x^k-x^*\|^2 +  \frac{1}{n} \sum_{i=1}^n \|\mH_i^k - \nabla^2 f_i(x^*)\|^2_{\rm F}$;
               \;
               $\Phi_2^k \eqdef c{\cal W}^k +  \frac{1}{n} \sum_{i=1}^n \|\mH_i^k - \nabla^2 f_i(x^*)\|^2_{\rm F}$;
               \;
               $\Phi_3^k \eqdef \|z^k - x^*\|^2 + c\|w^k - x^*\|^2$.
          (iii) Function value suboptimality  $\Delta_k \eqdef f(x^k) - f(x^*)$
        \item ${}^\dagger$ constants $c>0$ and $\theta\in(0,1)$ are possibly different each time they appear. Refer to the precise statements of the theorems for the exact values.
        }
    \end{tablenotes}            
    \end{center}
\end{table*}

\section{The Vanilla Federated Newton Learn}\label{sec:FedNL-main-paper}

We start the presentation of our algorithms with the vanilla \algname{FedNL} method,  commenting on the intuitions and technical novelties. The method is formally described\footnote{For all our methods, we describe the steps constituting a single communication round only. To get an iterative method, one simply needs to repeat provided steps in an iterative fashion.} in Alg. \ref{alg:FedNL}.

\subsection{New Hessian learning technique}
The first key technical novelty in \algname{FedNL} is the new mechanism for learning the Hessian $\nabla^2 f(x^*)$ at the (unique) solution $x^*$ in a communication efficient manner. This is achieved by maintaining and progressively updating local Hessian estimates $\mH_i^k$ of $\nabla^2 f_i(x^*)$ for all devices $i\in[n]$ and the global Hessian estimate $$\mH^k = \frac{1}{n}\sum_{i=1}^n \mH_i^k$$ of $\nabla^2 f(x^*)$ for the central server. Thus, the goal is to induce $\mH_i^k\to\nabla^2 f_i(x^*)$ for all $i\in[n]$, and as a consequence, $\mH^k\to\nabla^2 f(x^*)$, throughout the training process.

A {\em naive} choice for the local estimates $\mH_i^k$ would be the exact local Hessians $\nabla^2 f_i(x^k)$, and consequently the global estimate $\mH^k$ would be the exact global Hessian $\nabla^2 f(x^k)$. While this naive approach learns the global Hessian at the optimum, it needs to communicate the entire matrices $\nabla^2 f_i(x^k)$ to the server in each iteration, which is extremely costly.
Instead, in \algname{FedNL} we aim to {\em reuse} past Hessian information and build the next estimate $\mH_i^{k+1}$ by updating the current estimate $\mH_i^k$. Since all devices have to be synchronized with the server, we also need to make sure the update from $\mH_i^k$ to $\mH_i^{k+1}$ is easy to communicate. With this intuition in mind, we propose to update the local Hessian estimates via the rule $$\squeeze \mH_i^{k+1} = \mH_i^k + \alpha\mS_i^k,$$ where $$\mS_i^k = \cC_i^k(\nabla^2 f_i(x^k) - \mH_i^k),$$ and $\alpha>0$ is the learning rate. Notice that we reduce the communication cost by explicitly requiring all devices $i\in[n]$ to send compressed matrices $\mS_i^k$ to the server only.

\begin{algorithm}[H]
\caption{\algname{FedNL} (Federated Newton Learn) }
\label{alg:FedNL}
\begin{algorithmic}[1]
\STATE \textbf{Parameters:} Hessian learning rate $\alpha\ge0$; compression operators $\{\cC_1^k, \dots,\cC_n^k\}$
\STATE \textbf{Initialization:} $x^0\in\R^d$; $\mH_1^0, \dots, \mH_n^0 \in \R^{d\times d}$ and $\mH^0 \eqdef \frac{1}{n}\sum_{i=1}^n \mH_i^0$
\FOR{each device $i = 1, \dots, n$ in parallel} 
\STATE Get $x^k$ from the server and compute local gradient $\nabla f_i(x^k)$ and local Hessian $\nabla^2 f_i(x^k)$
\STATE Send $\nabla f_i(x^k)$,\; $\mS_i^k \eqdef \cC_i^k(\nabla^2 f_i(x^k) - \mH_i^k)$ and $l_i^k \eqdef \|\mH_i^k - \nabla^2 f_i(x^k)\|_{\rm F}$ to the server
\STATE Update local Hessian shift to $\mH_i^{k+1} = \mH_i^k + \alpha\mS_i^k$
\ENDFOR
\STATE \textbf{on} server
\STATE \quad Get $\nabla f_i(x^k),\; \mS_i^k$ and $l_i^k$ from each node $i\in [n]$
\STATE \quad $\mS^k = \frac{1}{n}\sum\limits_{i=1}^n \mS_i^k,\; l^k = \frac{1}{n}\sum\limits_{i=1}^n l_i^k,\; \mH^{k+1} = \mH^k + \alpha\mS^k$
\STATE \quad \textit{Option 1:} $x^{k+1} = x^k - \[\mH^{k}\]_{\mu}^{-1} \nabla f(x^k)$
\STATE \quad \textit{Option 2:} $x^{k+1} = x^k - \[\mH^{k} + l^k\mI\]^{-1} \nabla f(x^k)$
\end{algorithmic}
\end{algorithm}

The Hessian learning technique employed in the Newton Learn framework of \cite{Islamov2021NewtonLearn} is critically different to ours as it heavily depends on the structure \eqref{eq:GLM} of the local functions. Indeed, the local optimal Hessians $$\squeeze \nabla^2 f_i(x^*) = \frac{1}{m}\sum_{j=1}^m \varphi''_{ij}(a_{ij}^\top x^*) a_{ij} a_{ij}^\top$$ are learned via the proxy of learning the optimal scalars $\varphi''_{ij}(a_{ij}^\top x^*)$ for all local data points $j\in[m]$, which also requires the transmission of the active data points $a_{ij}$ to the server in each iteration. This makes their method inapplicable to the general finite sum problems \des{fs}, and incapable of securing even the most rudimentary privacy enhancement \des{pe} mechanism.

We do not make any structural assumption on the problem \eqref{erm-prob}, and rely on the following general conditions to prove effectiveness of our Hessian learning technique:

\begin{assumption}\label{asm:main}
The average loss $f$ is $\mu$-strongly convex, and all local losses $f_i(x)$ have Lipschitz continuous Hessians. Let $\HS$, $\HF$ and $\HM$ be the Lipschitz constants with respect to three different matrix norms: spectral, Frobenius and infinity norms, respectively. Formally,  we require 
\begin{eqnarray*}
\squeeze
\|\nabla^2 f_i(x) - \nabla^2 f_i(y)\| &\leq & \HS \|x-y\| \\
\|\nabla^2 f_i(x) - \nabla^2 f_i(y)\|_{\rm F} & \leq & \HF \|x-y\| \\
\max_{j,l}| (\nabla^2 f_i(x) - \nabla^2 f_i(y))_{jl}| & \leq & \HM \|x-y\|
\end{eqnarray*} 
to hold for all $i\in[n]$ and $x,y\in\R^d$.
\end{assumption}

\subsection{Compressing matrices}
 In the literature on first-order compressed methods, compression operators are typically applied to vectors (e.g., gradients, gradient differences, models). As our approach is based on second-order information, we apply compression operators to $d\times d$ matrices of the form $\nabla^2 f_i(x^k) - \mH_i^k$ instead. For this reason, we adapt two popular classes of compression operators used in first-order methods to act on $d\times d$ matrices by treating them as vectors of dimension $d^2$. 

\begin{definition}[Unbiased Compressors]\label{def:class-unbiased}
By $\B(\omega)$ we denote the class of (possibly randomized) unbiased compression operators $\mathcal{C}\colon\R^{d\times d}\to\R^{d\times d}$ with  variance parameter $\omega\geq 0$ satisfying 
\begin{equation}\label{class-unbiased}
\squeeze
\E\[\cC(\mM)\] = \mM, \; \E\[\|\cC(\mM)-\mM\|_{\rm F}^2\] \le \omega\|\mM\|_{\rm F}^2
\end{equation}
for all matrices $\mM\in\R^{d\times d}$.
\end{definition}

Common choices of unbiased compressors are random sparsification and quantization (see Appendix).

\begin{definition}[Contractive Compressors]\label{def:class-contractive}
By $\bC(\delta)$ we denote the class of deterministic contractive compression operators $\mathcal{C}\colon\R^{d\times d}\to\R^{d\times d}$ with contraction parameter $\delta\in[0,1]$ satisfying
\begin{equation}\label{class-contractive}
\squeeze
\|\cC(\mM)\|_{\rm F} \le \|\mM\|_{\rm F}, \; \|\cC(\mM)-\mM\|^2_{\rm F} \le (1-\delta)\|\mM\|_{\rm F}^2
\end{equation}
for all matrices $\mM\in\R^{d\times d}$.
\end{definition}

The first condition of \eqref{class-contractive} can be easily removed by scaling the operator $\cC$ appropriately. Indeed, if for some $\mM\in\R^{d\times d}$ we have $\|\cC(\mM)\|_{\rm F} > \|\mM\|_{\rm F}$, then we can use the scaled compressor $\widetilde{\cC}(\mM) \eqdef \tfrac{\|\mM\|_{\rm F}}{\|\cC(\mM)\|_{\rm F}}\cC(\mM)$ instead, as this satisfies \eqref{class-contractive} with the same parameter $\delta$. Common examples of contractive compressors are Top-$K$ and Rank-$R$ operators (see Appendix).

From the theory of first-order methods employing compressed communication, it is known that handling contractive biased compressors is much more challenging than handling unbiased compressors. In particular, a popular mechanism for preventing first-order methods utilizing biased compressors from divergence is the {\em error feedback} framework. However, contractive compressors often perform much better empirically than their unbiased counterparts. To highlight the strength of our new Hessian learning technique, we develop our theory in a flexible way, and handle both families of compression operators. Surprisingly, we do not need to use error feedback for contractive compressors for our methods to work.

Compression operators are used in \citep{Islamov2021NewtonLearn} in a fundamentally different way. First, their theory supports unbiased compressors only, and does not cover the practically favorable contractive compressors \des{cc}. More importantly, compression is applied within the representation \eqref{eq:GLM} as an operator acting on the space $\R^m$. In contrast to our strategy of using compression operators, this brings the necessity to reveal, in each iteration,  the training data $\{a_{i1},\dots,a_{im}\}$ whose corresponding coefficients in \eqref{eq:GLM} are not zeroed out after the compression step \des{pe}. Moreover, when $\cO(d)$ communication cost per communication round is achieved, the variance of the compression noise depends on the number of data points $m$, which then negatively affects the local convergence rates. As the amount of training data can be huge, our convergence rates provide stronger guarantees by not depending on the size of the training dataset \des{fr}.

\subsection{Two options for updating the global model}
Finally, we offer two options for updating the global model at the server.
\begin{itemize}
\item The first option assumes that the server knows the strong convexity parameter $\mu>0$ (see Assumption \ref{asm:main}), and that it is powerful enough to compute the projected Hessian estimate $\[\mH^k\]_{\mu}$, i.e., that it is able to project the current global Hessian estimate $\mH^k$ onto the set $$\left\{\mM\in\R^{d\times d} \colon \mM^\top = \mM,\; \mu\mI \preceq \mM \right\}$$ in each iteration (see the Appendix).

\item Alternatively, if $\mu$ is unknown, all devices send the compression errors $$l_i^k \eqdef \|\mH_i^k - \nabla^2 f_i(x^k)\|_{\rm F}$$ (this extra communication is extremely cheap as all $l_i^k$ variables are floats) to the server, which then computes the corrected Hessian estimate $\mH^k + l^k\mI$ by adding the average error $l^k = \frac{1}{n}\sum_{i=1}^n l_i^k$ to the global Hessian estimate $\mH^k$. 
\end{itemize}

Both options require the server  in each iteration to solve a linear system to invert either the projected, or the corrected, global Hessian estimate. The purpose of these options is quite simple: unlike the true Hessian, the compressed local Hessian estimates $\mH_i^k$, and also the global Hessian estimate $\mH^k$, might not be positive definite, or might even not be of full rank. Further importance of the errors $l_i^k$ will be discussed when we consider extensions of \algname{FedNL} to partial participation and globalization via cubic regularization.

\subsection{Local convergence theory}
Note that  \algname{FedNL}  includes two parameters, compression operators $\cC_i^k$ and Hessian learning rate $\alpha>0$, and two options to perform global updates by the master.
To provide theoretical guarantees, we need one of the following two assumptions.

\begin{assumption}\label{asm:comp-1}
 $\cC_i^k\in\bC(\delta)$ for all $i\in [n]$ and $k\ge0$. Moreover, (i) $\alpha=1-\sqrt{1-\delta}$, or (ii) $\alpha=1$.
\end{assumption}

\begin{assumption}\label{asm:comp-2}
$\cC_i^k\in\B(\omega)$ for all $i\in [n]$ and $k\ge0$ and  $0< \alpha \le \frac{1}{\omega+1}$. Moreover, for all $i\in[n]$ and $j, l\in [d]$,  each entry $(\mH_i^k)_{jl}$ is a convex combination of $\{  (\nabla^2 f_i(x^t))_{jl}  \}_{t=0}^k$ for any $k\geq 0$.
\end{assumption}

To present our results in a unified manner, we define some constants depending on what parameters and which option is used in \algname{FedNL}. Below, constants $A$ and $B$ depend on the choice of the compressors $\cC_i^k$ and the learning rate $\alpha$, while $C$ and $D$ depend on which option is chosen for the global update.
\begin{eqnarray}\label{ABCD}
(A, B) &\eqdef&
\begin{cases}
(\alpha^2, \alpha)   & \;\text{if Assumption}\; \ref{asm:comp-1} \text{(i) holds} \phantom{i}  \\
(\nicefrac{\delta}{4}, \nicefrac{6}{\delta} - \nicefrac{7}{2})   & \;\text{if Assumption}\; \ref{asm:comp-1} \text{(ii) holds}  \\
(\alpha, \alpha) & \;\text{if Assumption}\; \ref{asm:comp-2}\; \text{holds} \phantom{00}
\end{cases} \phantom{~~~} \\
(C, D) &\eqdef&
\begin{cases}
(2, \HS^2) & \;\text{if {\it Option 1} is used}  \\
(8, (\HS+2\HF)^2)   & \;\text{if {\it Option 2} is used} \\
\end{cases}
\end{eqnarray}

We prove three local rates for \algname{FedNL}: for the squared distance to the solution $\|x^k-x^*\|^2$, and for the Lyapunov function $$\Phi^k \eqdef {\cal H}^k + 6B\HF^2 \|x^k-x^*\|^2,$$ where $${\cal H}^k \eqdef \frac{1}{n} \sum_{i=1}^n \|\mH_i^k - \nabla^2 f_i(x^*)\|^2_{\rm F}.$$

\begin{theorem}\label{th:NLU}
    Let Assumption \ref{asm:main} hold. Assume $\|x^0-x^*\| \leq \frac{\mu}{\sqrt{2D}}$ and ${\cal H}^k \leq \frac{\mu^2}{4C}$ for all $k\geq 0$. Then, \algname{FedNL} (Algorithm~\ref{alg:FedNL}) converges linearly with the  rate
    \begin{equation}\label{rate:local-linear-iter}
    \squeeze
    \|x^k - x^*\|^2 \leq  \frac{1}{2^k}   \|x^0-x^*\|^2. 
    \end{equation}
  Moreover, depending on the choice \eqref{ABCD} of the compressors $\cC_i^k$ (Assumption \ref{asm:comp-1} or \ref{asm:comp-2}), learning rate $\alpha$, and which option is used for global model updates, we have the following linear and superlinear rates:
    \begin{equation}\label{rate:local-linear-Lyapunov}
    \squeeze
  \mathbb{E}[\Phi^k] \leq  \left(  1 - \min\left\{  A, \frac{1}{3}  \right\}  \right)^k\Phi^0,
    \end{equation}
  \begin{equation}\label{rate:local-superlinear-iter}
  \E\[\frac{\|x^{k+1}-x^*\|^2}{\|x^k-x^*\|^2}\] \leq  \left(  1 - \min\left\{A, \frac{1}{3}\right\}  \right)^k \left(  C + \frac{D}{12B\HF^2}  \right) \frac{\Phi^0}{\mu^2}.
  \end{equation}    
\end{theorem}

Let us comment on these rates. 
\begin{itemize}
\item First, the local linear rate \eqref{rate:local-linear-iter} with respect to iterates is based on a universal constant, i.e., it does not depend on the condition number of the problem, the size of the training data, or the dimension of the problem. Indeed, the squared distance to the optimum is halved in each iteration. 

\item Second, we have linear rate \eqref{rate:local-linear-Lyapunov} for the Lyapunov function $\Phi^k$, which implies the  linear convergence of all local Hessian estimates $\mH_i^k$ to the local optimal Hessians $\nabla^2 f_i(x^*)$.  Thus, our initial goal to progressively learn the local optimal Hessians in a communication efficient manner is achieved, justifying the effectiveness of the new Hessian learning technique. 

\item Finally, our Hessian learning process accelerates the convergence of iterates to a superlinear rate \eqref{rate:local-superlinear-iter}. Both rates \eqref{rate:local-linear-Lyapunov} and \eqref{rate:local-superlinear-iter} are independent of the condition number of the problem, or the number of data points. However, they do depend on the compression variance (since $A$ depends on $\delta$ or $\omega$), which, in case of $\cO(d)$ communication constraints, depend on the dimension $d$ only.
\end{itemize}

For clarity of exposition,  in Theorem \ref{th:NLU} we assumed ${\cal H}^k \leq \frac{\mu^2}{4C}$ for all iterations $k\geq 0$. Below, we prove that this inequality holds, using the initial conditions only.

\begin{lemma}\label{lm:boundforbiased}
Let Assumption \ref{asm:comp-1} hold, and  assume $\|x^0 - x^*\| \le e_1 \eqdef \min\{ \frac{\mu}{2\HF} \sqrt{\frac{A}{BC}}, \frac{\mu}{\sqrt{2D}}  \}$ and $\|\mH_i^0 - \nabla^2 f_i(x^*)\|_{\rm F} \leq \frac{\mu}{2\sqrt{C}}$.
Then $\|x^k-x^*\| \leq e_1$ and $\|\mH_i^k - \nabla^2 f_i(x^*)\|_{\rm F}  \leq \frac{\mu}{2\sqrt{C}}$ for all $k\geq 0$. 
\end{lemma}

\begin{lemma}\label{lm:boundforspar}
Let Assumption \ref{asm:comp-2} hold, and assume $\|x^0 - x^*\| \le e_2 \eqdef \frac{\mu}{\sqrt{D + 4Cd^2\HM^2}}$. Then $\|x^k - x^*\| \leq e_2$ and ${\cal H}^k \leq \frac{\mu^2}{4C}$ for all $k\geq 0$. 
\end{lemma}

\subsection{\algname{FedNL} and the ``Newton Triangle''}
One implication of Theorem \ref{th:NLU} is that the local rate $\frac{1}{2^k}$ (see \eqref{rate:local-linear-iter})  holds even when we specialize \algname{FedNL} to $\cC_i^k\equiv\bm{0}$, $\alpha=0$ and $\mH_i^0=\nabla^2 f_i(x^0)$ for all $i\in[n]$. These parameter choices give rise to the following simple method, which we call Newton Zero (\algname{N0}):
\begin{equation} \label{eq:N0}
\squeeze
x^{k+1} = x^k - \[\nabla^2 f(x^0)\]^{-1}\nabla f(x^k), \quad k\ge0.
\end{equation}

Interestingly, \algname{N0} only  needs {\em initial second-order information}, i.e., Hessian at the zeroth iterate, and the same first-order information as Gradient Descent (\algname{GD}), i.e., $\nabla f(x^k)$ in each iteration. Moreover, unlike \algname{GD}, whose rate depends on a condition number, the local rate $\frac{1}{2^k}$ of  \algname{N0} does not. Besides, \algname{FedNL} includes \algname{NS} (when $\cC_i^k\equiv\bm{0}$, $\alpha=0$, $\mH_i^0=\nabla^2 f_i(x^*)$) and classical Newton (\algname{N}) (when $\cC_i^k\equiv\mI$, $\alpha=1$, $\mH_i^0=\bm{0}$)  as special cases.

It can be helpful to visualize the three special Newton-type methods---\algname{N}, \algname{NS} and \algname{N0} ---as the vertices of a triangle capturing a subset of two of these three requirements: 1) $\cO(d)$ communication cost per round, 2) implementability in practice, and 3) local quadratic rate. Indeed, each of these three methods satisfies {\em two} of these requirements only: \algname{N} (2+3), \algname{NS} (1+3) and \algname{N0} (1+2). Finally, \algname{FedNL}  interpolates between these requirements. See Figure~\ref{fig:FedNL-NT}.

\begin{figure*}[h]
  \begin{center}
    \includegraphics[width=0.6\linewidth]{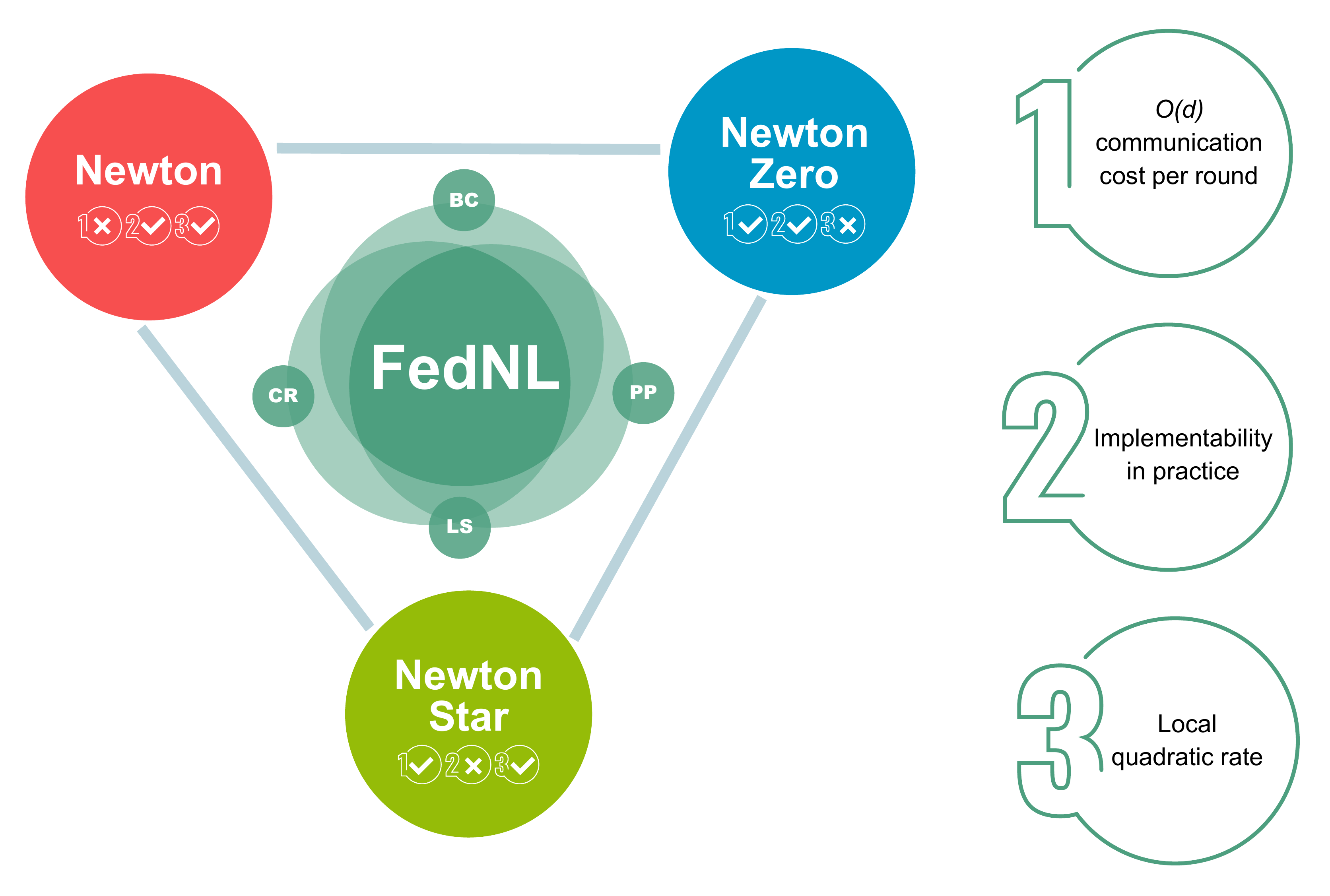}
  \end{center}
  \caption{Visualization of the three special Newton-type methods---Newton (\algname{N}), Newton Star (\algname{NS}) and Newton Zero (\algname{N0})---as the vertices of a triangle capturing a subset of two of these three requirements: 1) $\cO(d)$ communication cost per round, 2) implementability in practice, and 3) local quadratic rate. Indeed, each of these three methods satisfies {\em two} of these requirements only: \algname{N} (2+3), \algname{NS} (1+3) and \algname{N0} (1+2).
  Finally, the proposed \algname{FedNL} framework with its four extensions to Partial Participation (\algname{FedNL-PP}), globalization via Line Search (\algname{FedNL-LS}), globalization via Cubic Regularization (\algname{FedNL-CR}) and Bidirectional Compression (\algname{FedNL-BC}) interpolates between these requirements.
  } 
  \label{fig:FedNL-NT}
\end{figure*}

\section{\algname{FedNL} with Partial Participation, Globalization and Bidirectional Compression}

Here we briefly describe four extensions to \algname{FedNL} and the key technical contributions. Detailed sections for each extension are deferred to the Appendix.

\subsection{Partial Participation (see Section \ref{apx:FedNL-PP})}
In \algname{FedNL-PP} (Algorithm \ref{alg:FedNL-PP}), the server selects a subset $S^k\subseteq[n]$ of $\tau$ devices, uniformly at random, to participate in each iteration. As devices might be inactive for several iterations, the same local gradient and local Hessian used in \algname{FedNL} does not provide convergence in this case. To guarantee convergence, devices need to compute {\em Hessian corrected local gradients} $$g_i^k = (\mH_i^k + l_i^k\mI)w_i^k - \nabla f_i(w_i^k),$$ where $w_i^k$ is the last global model that device $i$ received from the server. This is an innovation which also requires a different analysis.

\subsection{Globalization via Line Search (see Section \ref{apx:FedNL-LS})}
Our first globalization strategy, \algname{FedNL-LS} (Algorithm \ref{alg:FedNL-LS}), which performs {\em significantly better in practice} than \algname{FedNL-CR} (described next), is based on a backtracking line search procedure. The idea is to fix the search direction $$d^k = -\[\mH^k\]_\mu^{-1}\nabla f(x^k)$$ by the server and find the smallest integer $s\ge0$ which leads to a sufficient decrease in the loss $$f(x^k+\gamma^sd^k) \le f(x^k) + c\gamma^s \<\nabla f(x^k), d^k\>$$ with some parameters $c\in(0,\nicefrac{1}{2}]$ and $\gamma\in(0,1)$.

\subsection{Globalization via Cubic Regularization (see Section \ref{apx:FedNL-CR})}
Our next globalization strategy, \algname{FedNL-CR} (Algorithm~\ref{alg:FedNL-CR}), is to use a cubic regularization term $\frac{\HS}{6}\|h\|^3$, where $\HS$ is the Lipschitz constant for Hessians and $h$ is the direction to the next iterate. However, to get a global upper bound, we had to {\em correct the global Hessian estimate} $\mH^k$ via compression error $l^k$. Indeed, since $\nabla^2 f(x^k) \preceq \mH^k + l^k\mI$, we deduce $$f(x^{k+1}) \le f(x^k) + \<\nabla f(x^k), h^k\> + \frac{1}{2}\<(\mH^k+l^k\mI)h^k, h^k\> + \frac{\HS}{6}\|h^k\|^3$$ for all $k\ge0$. This leads to theoretical challenges and necessitates a new analysis.

\subsection{Bidirectional Compression (see Section \ref{apx:FedNL-BC})}
Finally, we modify  \algname{FedNL} to allow for an even more severe level of compression that can't be attained by compressing the Hessians only. This is achieved by compressing the gradients (uplink) and the model (downlink), in a ``smart'' way. In \algname{FedNL-BC} (Alg.~\ref{alg:FedNL-BC}), the server operates its own compressors $\cC^k_{\rm M}$ applied to the model, and uses an additional ``smart'' {\em model learning technique} similar to the proposed Hessian learning technique. Besides, all devices compress their local gradients via a Bernoulli compression scheme, which necessitates the use of another ``smart'' strategy using {\em Hessian corrected local gradients} $$g_i^k = \mH_i^k(z^k - w^k) + \nabla f_i(w^k),$$ where $z^k$ is the current learned global model and $w^k$ is the last learned global model when local gradients are sent to the server. These changes are substantial and require~novel~analysis.

\section{Experiments}

We carry out numerical experiments to study the performance of \algname{FedNL}, and compare it with various state-of-the-art methods in federated learning. We consider the problem \eqref{erm-prob} with local loss functions
\begin{equation}\label{prob:log-reg}
\squeeze
\min\limits_{x\in\R^d}\left\{f(x)\eqdef \frac{1}{n}\sum\limits_{i=1}^n f_i(x) +\frac{\lambda}{2}\|x\|^2\right\}, \qquad f_i(x) = \frac{1}{m}\sum \limits_{j=1}^m\log\(1+\exp(-b_{ij}a_{ij}^\top x)\),
\end{equation}
where $\{a_{ij},b_{ij}\}_{j\in [m]}$ are data points at the $i$-th device and $\lambda>0$ is a regularization parameter. 
The datasets were taken from LibSVM library \citep{chang2011libsvm}: \dataname{a1a}, \dataname{a9a}, \dataname{w7a}, \dataname{w8a}, and \dataname{phishing}.

\begin{figure*}[ht]
    \begin{center}
        \begin{tabular}{cccc}
            \includegraphics[width=0.23\linewidth]{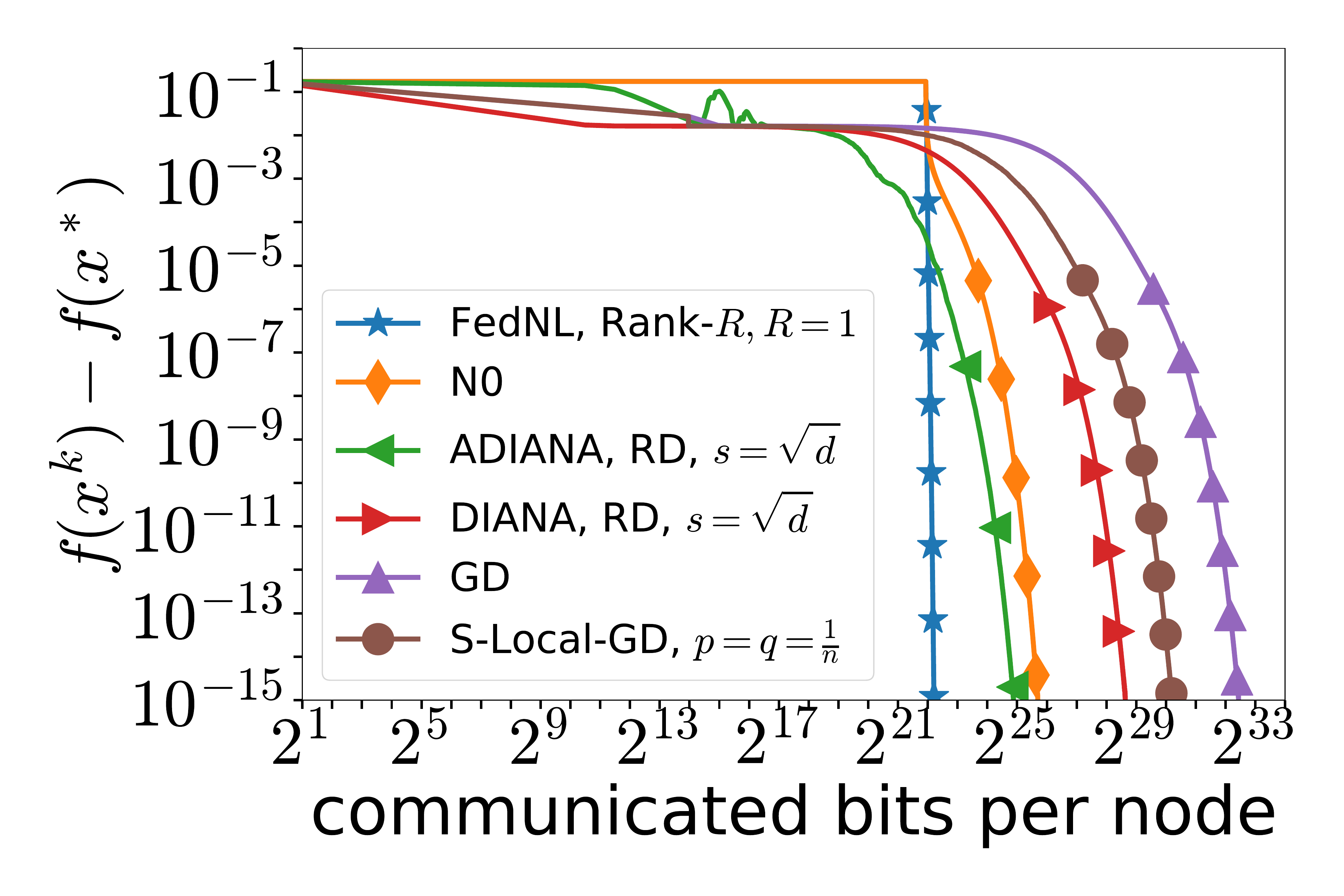}
            & 
            \includegraphics[width=0.23\linewidth]{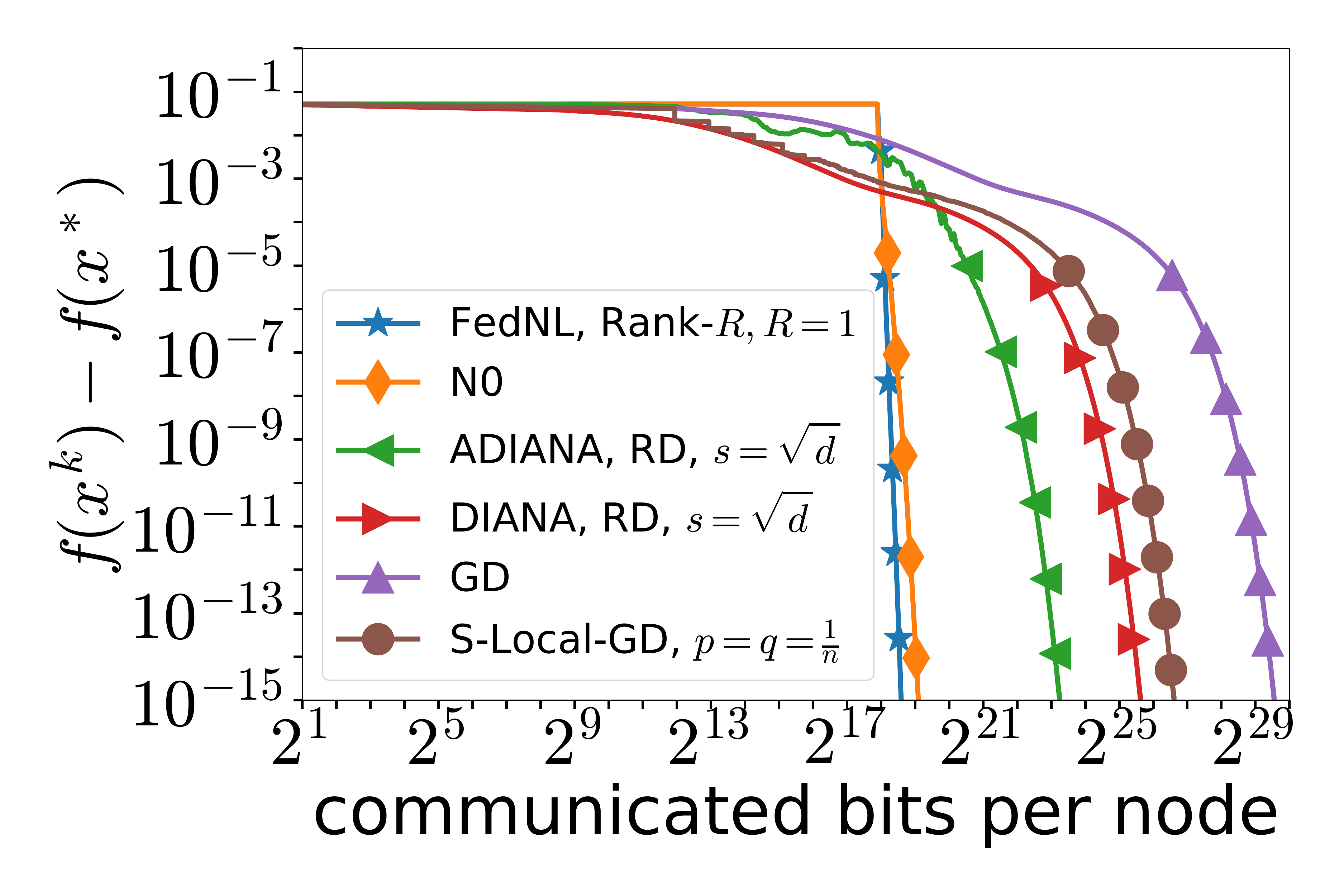} &
            \includegraphics[width=0.23\linewidth]{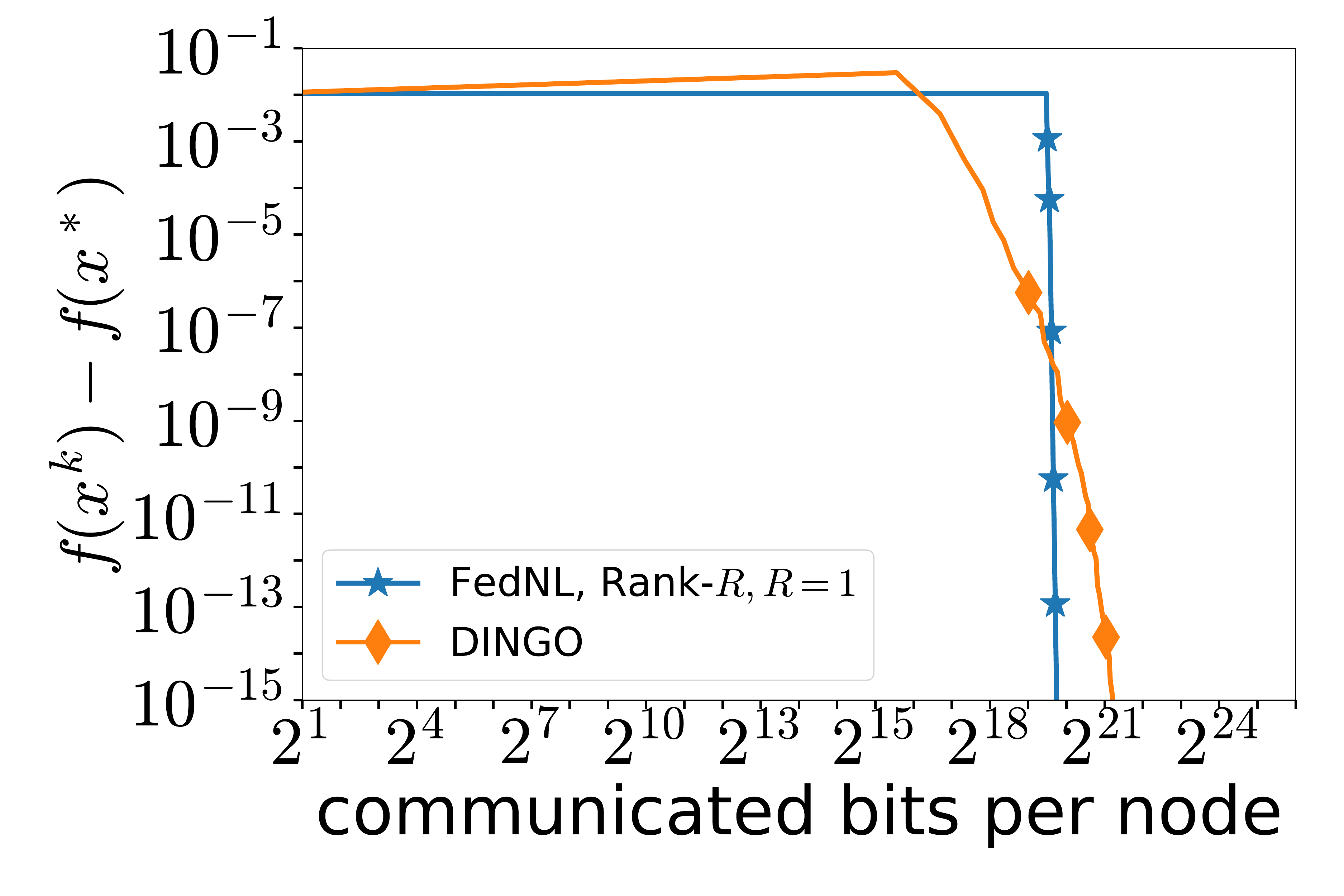} &
            \includegraphics[width=0.23\linewidth]{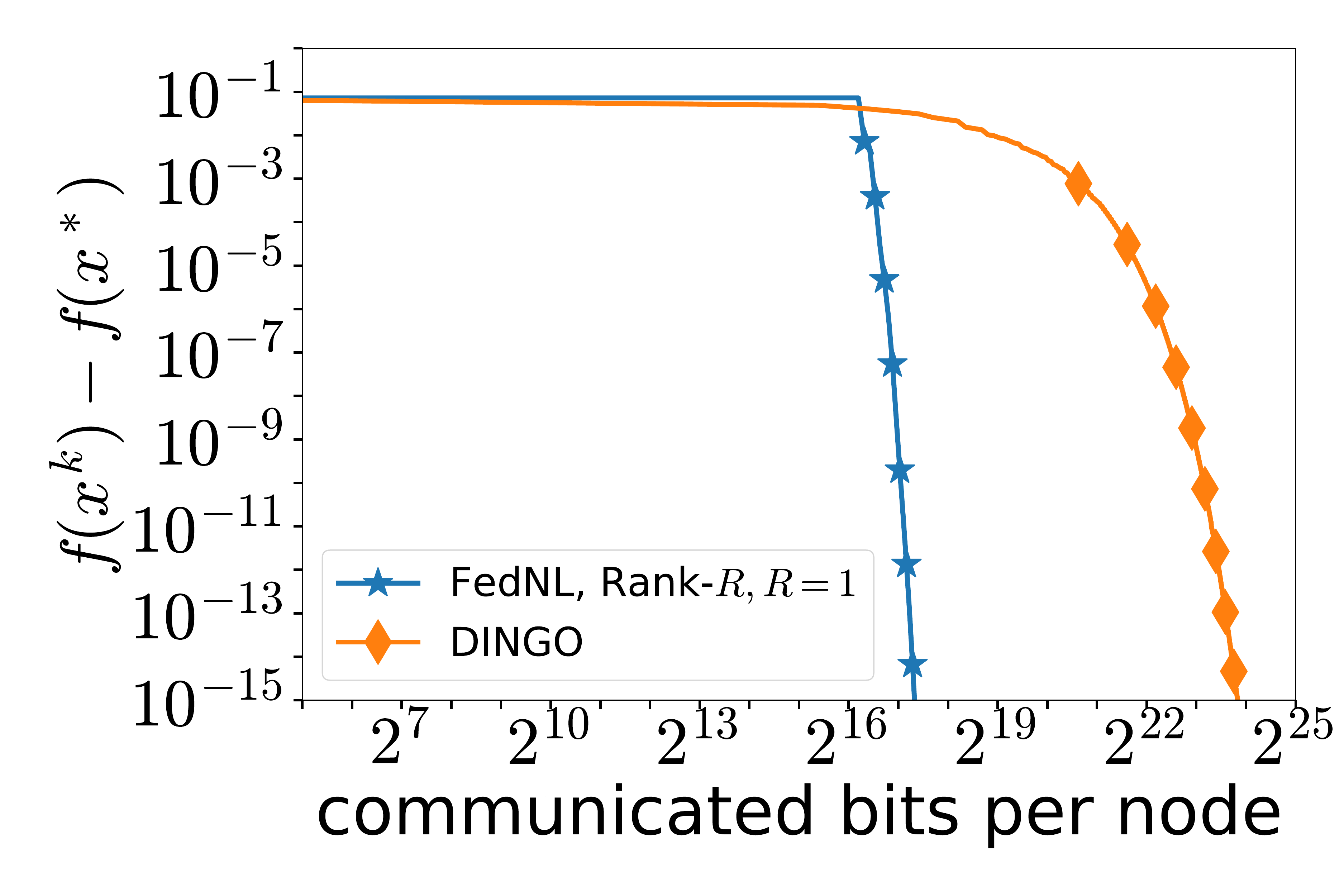}\\
            (a) \dataname{madelon}, {\scriptsize  $\lambda=10^{-3}$} &
            (b) \dataname{a1a}, {\scriptsize  $\lambda=10^{-4}$} &
            (c) \dataname{w8a}, {\scriptsize  $\lambda=10^{-3}$} &
            (d) \dataname{phishing}, {\scriptsize  $\lambda=10^{-4}$}
        \end{tabular}
        \\
        \begin{tabular}{cccc}
            \includegraphics[width=0.23\linewidth]{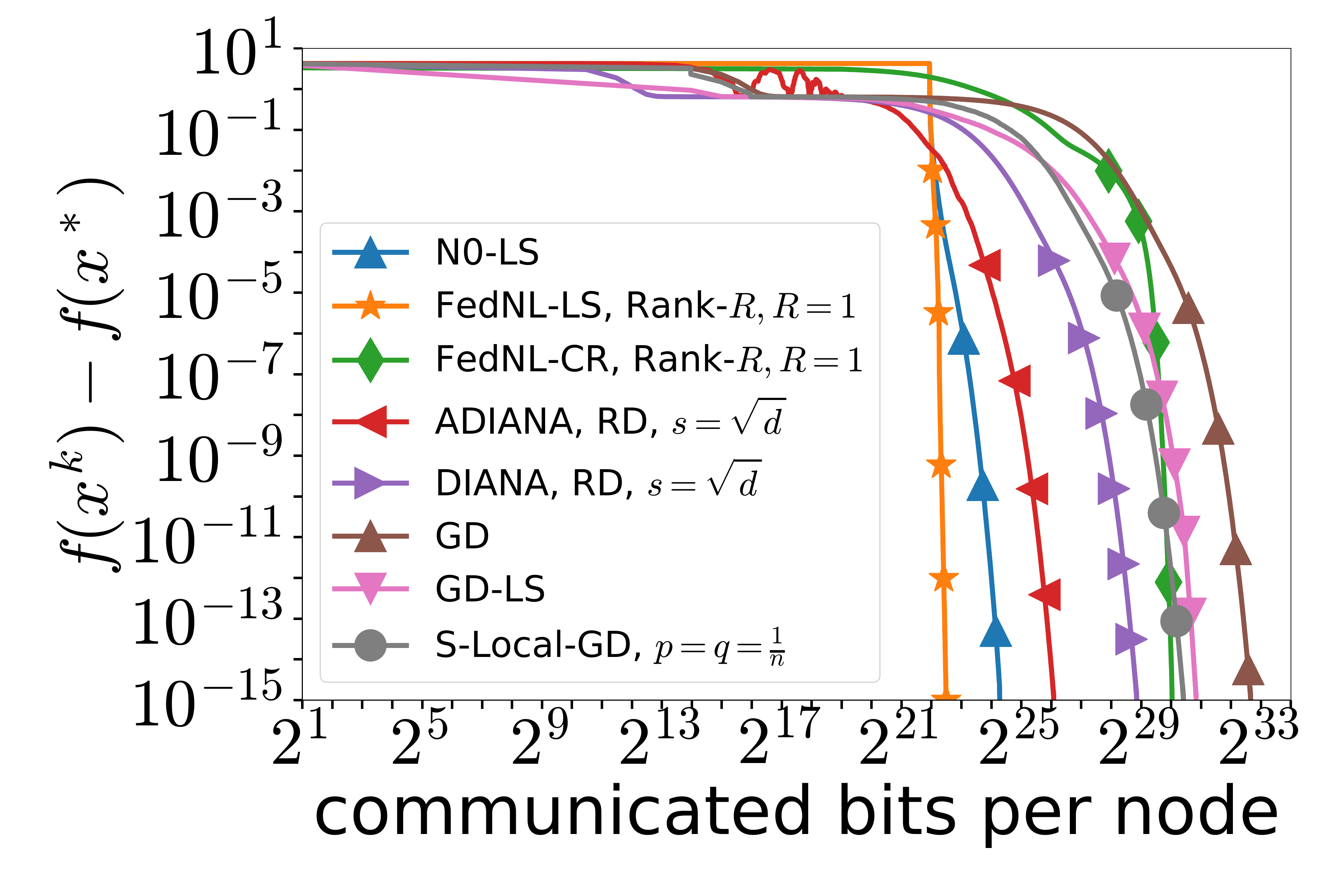} & 
            \includegraphics[width=0.23\linewidth]{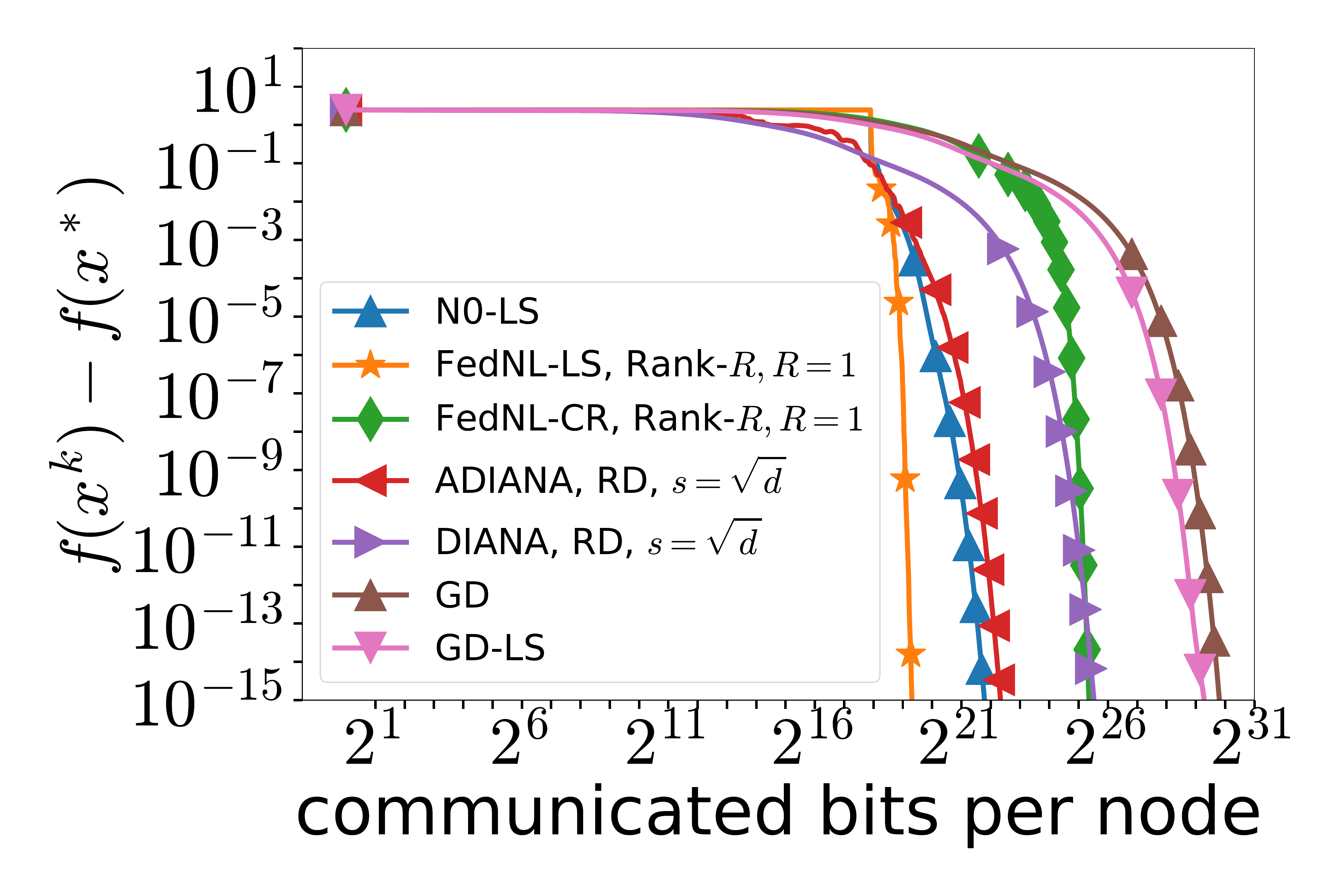} &
            \includegraphics[width=0.23\linewidth]{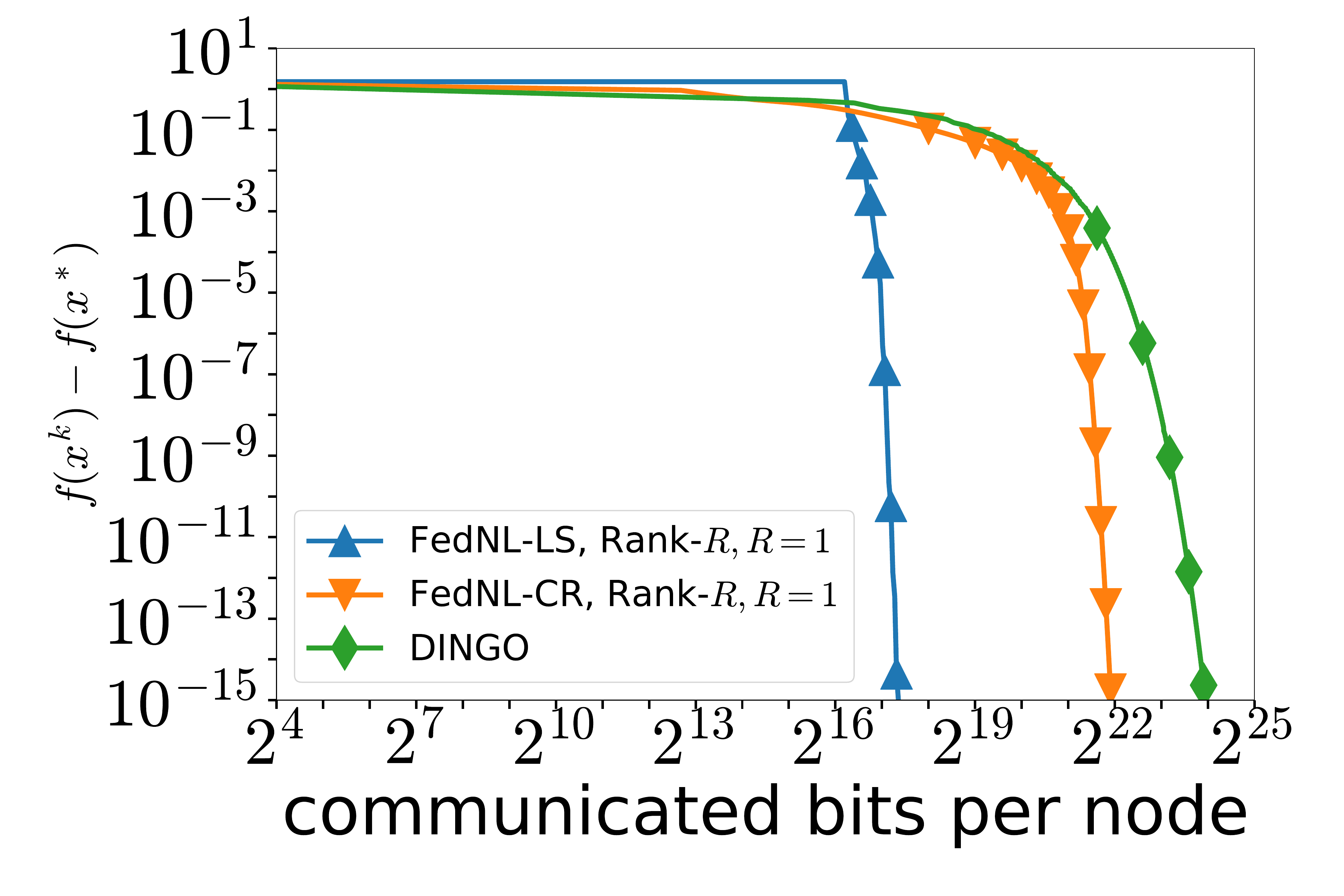} 
            &
            \includegraphics[width=0.23\linewidth]{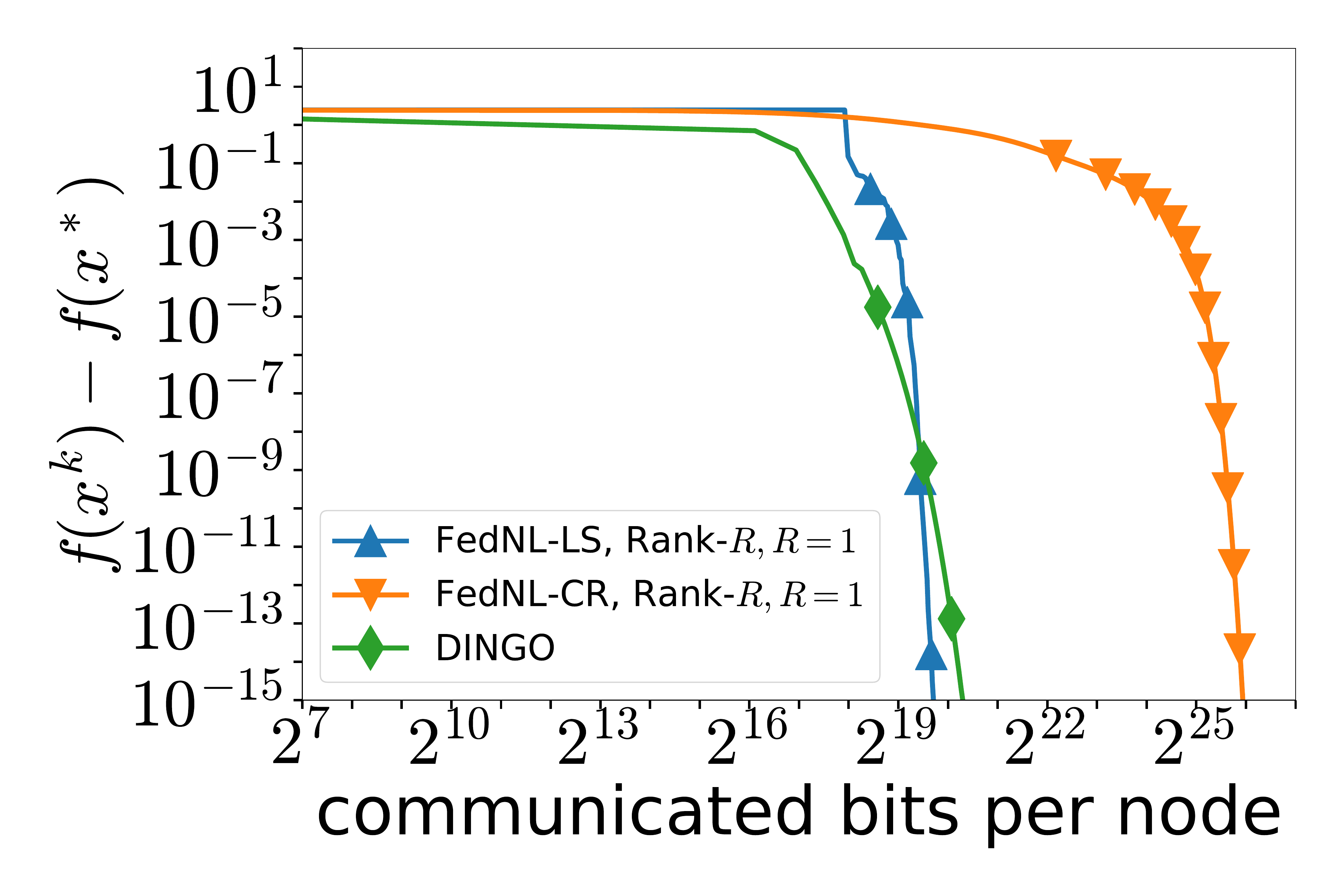} \\
            (a) \dataname{madelon}, {\scriptsize  $\lambda=10^{-3}$} &
            (b) \dataname{a1a}, {\scriptsize  $\lambda=10^{-4}$} &
            (c) \dataname{phishing}, {\scriptsize$ \lambda=10^{-3}$} &
            (d) \dataname{a9a}, {\scriptsize$ \lambda=10^{-4}$}
        \end{tabular}
        \\
        \begin{tabular}{cccc}
            \includegraphics[width=0.23\linewidth]{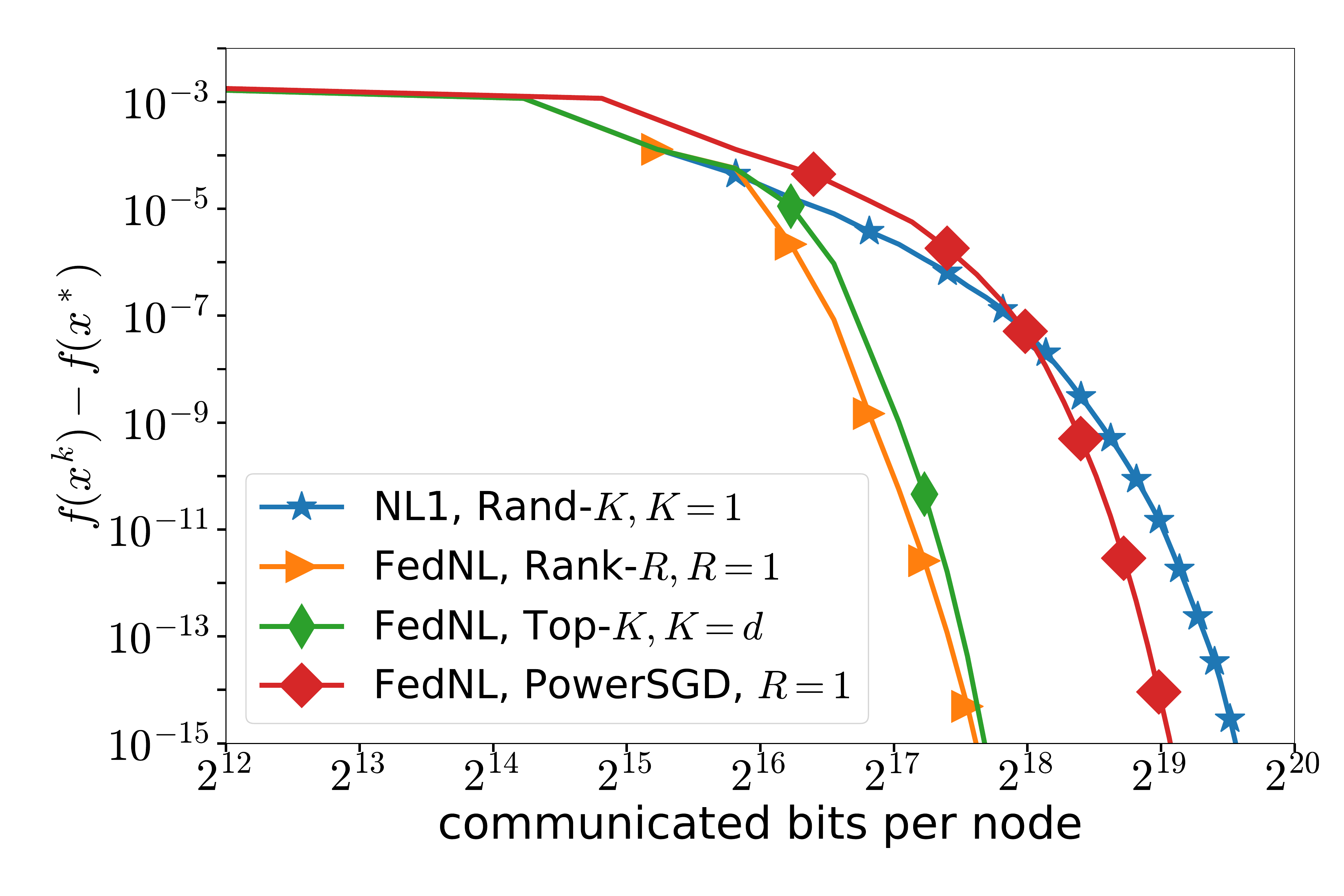} & 
            \includegraphics[width=0.23\linewidth]{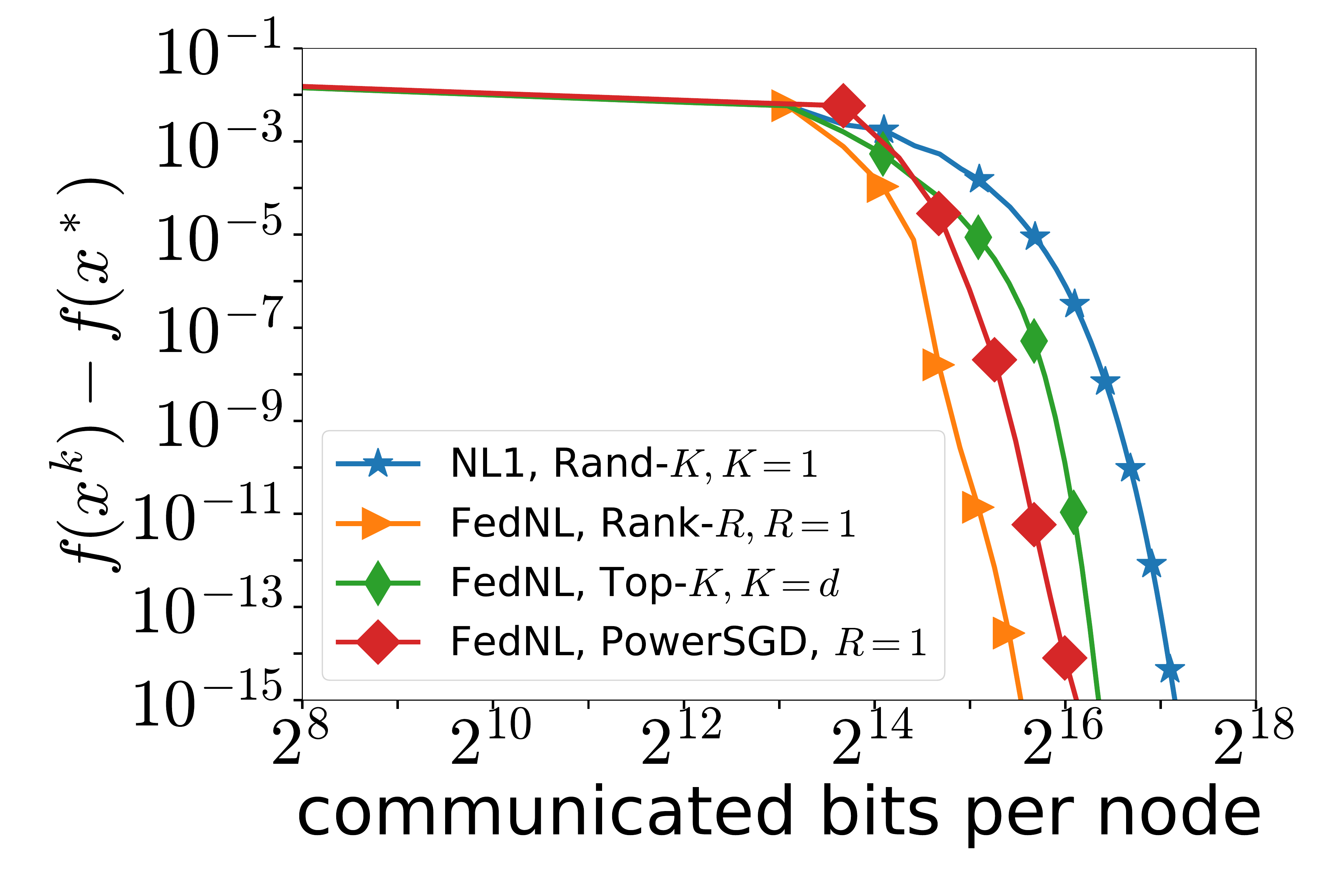} &
            \includegraphics[width=0.23\linewidth]{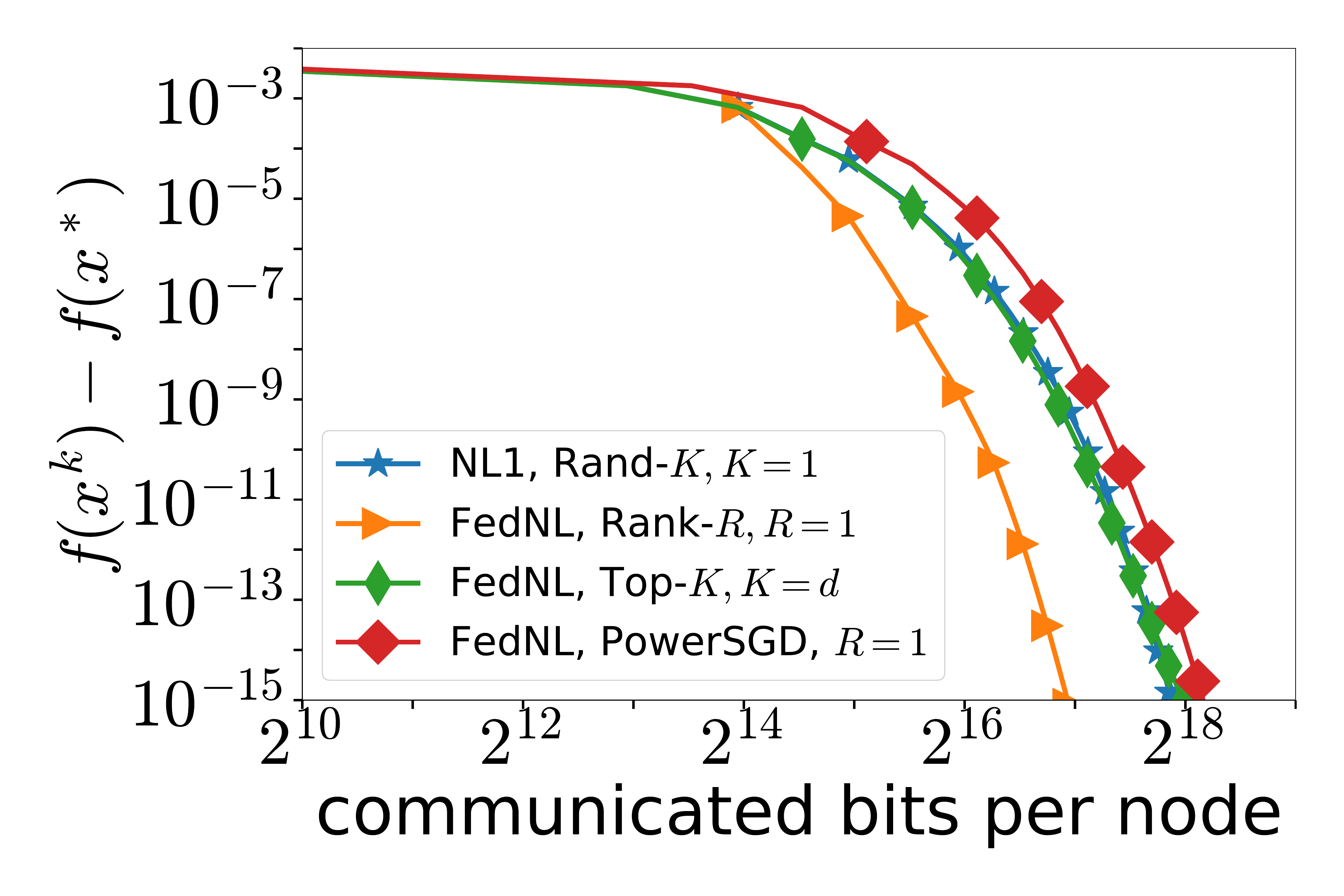} &
            \includegraphics[width=0.23\linewidth]{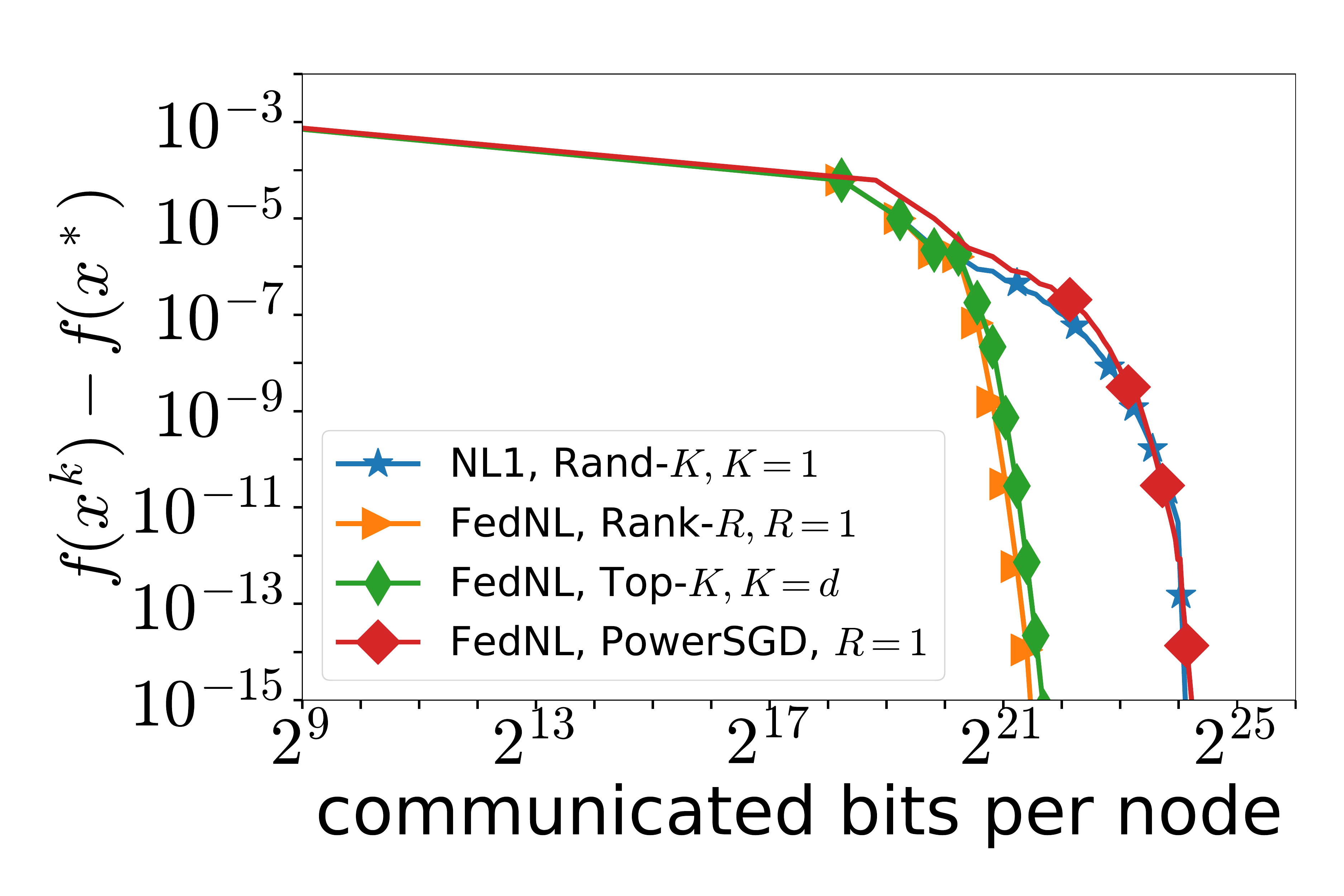} \\
            (a) \dataname{w8a}, {\scriptsize$ \lambda=10^{-3}$} &
            (b) \dataname{phishing}, {\scriptsize $\lambda=10^{-3}$} &
            (c) \dataname{a1a}, {\scriptsize$ \lambda=10^{-4}$} &
            (d) \dataname{w7a}, {\scriptsize$ \lambda=10^{-4}$} 
        \end{tabular}               
    \end{center}
    \caption{{\bf First row:} Local comparison of \algname{FedNL} and \algname{N0} with (a), (b) \algname{ADIANA}, \algname{DIANA}, \algname{GD}; with (c), (d) \algname{DINGO} in terms of communication complexity. {\bf Second row:} Global comparison of \algname{FedNL-LS}, \algname{N0-LS} and \algname{FedNL-CR} with (a), (b) \algname{ADIANA}, \algname{DIANA}, \algname{GD}, and \algname{GD} with line search; with (c), (d) \algname{DINGO} in terms of communication complexity. {\bf Third row:} Local comparison of \algname{FedNL} with $3$ types of compression operators and \algname{NL1} in terms of communication complexity.}
    \label{fig:FedNL-three-in-one}
\end{figure*}

\subsection{Parameter setting}
In all experiments we use the theoretical parameters for gradient type methods (except those using line search): vanilla gradient descent \algname{GD}, \algname{DIANA} \citep{DIANA}, \algname{ADIANA} \citep{ADIANA}, and Shifted Local gradient descent, \algname{S-Local-GD} \citep{Gorbunov2020localSGD}. For \algname{DINGO} \citep{DINGO} we use the authors' choice: $\theta=10^{-4}, \phi=10^{-6}, \rho=10^{-4}$. Backtracking line search for \algname{DINGO} selects the largest stepsize from $\{1,2^{-1},\dots,2^{-10}\}.$ The initialization of $\mH^0_i$ for \algname{NL1} \citep{Islamov2021NewtonLearn}, \algname{FedNL} and \algname{FedNL-LS} is $\nabla^2 f_i(x^0)$, and for \algname{FedNL-CR} is $\mathbf{0}$. For \algname{FedNL}, \algname{FedNL-LS}, and \algname{FedNL-CR} we use Rank-$1$ compression operator and stepsize $\alpha=1$. We use two values of the regularization parameter: $\lambda \in \{10^{-3}, 10^{-4}\}$. In the figures we plot the relation of the optimality gap $f(x^k)-f(x^*)$ and the number of communicated bits per node, or the number of communication rounds. The optimal value $f(x^*)$ is chosen as the function value at the $20$-th iterate of standard Newton's method. 

\subsection{Local convergence}
In our first experiment we compare \algname{FedNL} and \algname{N0} with gradient type methods: \algname{ADIANA} with random dithering (\algname{ADIANA}, RD, $s=\sqrt{d}$), \algname{DIANA} with random dithering (\algname{DIANA}, RD, $s=\sqrt{d}$), Shifted Local gradient descent (\algname{S-Local-GD}, $p=q=\frac{1}{n}$), vanilla gradient descent (\algname{GD}), and \algname{DINGO}. {\em According to the results summarized in Figure~\ref{fig:FedNL-three-in-one} (first row), we conclude that \algname{FedNL} outperforms all gradient type methods and \algname{DINGO}, locally, by many orders in magnitude.} We want to note that we include the communication cost of the initialization for \algname{FedNL} and \algname{N0} in order to make a fair comparison (this is why there is a straight line for these methods initially).

\subsection{Global convergence}
We now compare  \algname{FedNL-LS}, \algname{N0-LS}, and \algname{FedNL-CR} with the first-order methods  \algname{ADIANA} and \algname{DIANA} with random dithering, Shifted Local gradient descent \algname{S-Local-GD}, gradient descent (\algname{GD}), and \algname{GD} with line search (\algname{GD-LS}). Besides, we compare \algname{FedNL-LS} and \algname{FedNL-CR} with \algname{DINGO}. In this experiment we choose $x^0$ far from the solution $x^*$, i.e., we test the global convergence behavior; see Figure~\ref{fig:FedNL-three-in-one} (second row). We observe that \algname{FedNL-LS} is more communication efficient  than all first-order methods and \algname{DINGO}. However, \algname{FedNL-CR} is better  than \algname{GD} and \algname{GD-LS} only. In these experiments we again include the communication cost of initialization for \algname{FedNL-LS} and \algname{N0-LS}.

\subsection{Comparison with \algname{NL1}}
Next, we compare \algname{FedNL} with three type of compression operators: Rank-$R$ ($R=1$), Top-$K$ ($K=d$), and PowerSGD \citep{vogels} ($R=1$) against \algname{NL1} with the Rand-$K$ ($K=1$) compressor. The results, presented in Figure~\ref{fig:FedNL-three-in-one} (third row), 
show that \algname{FedNL} with Rank-$1$ compressor performs the best.

\bibliography{references}
\bibliographystyle{plainnat}

\clearpage

\appendix
\part*{Appendix}

\section{Theoretical Comparisons with Related Works}\label{sec:more}

In this part, we compare our results with the most relevant prior works in the literature. We start comparing our work with several recently proposed second order distributed optimization methods to the following criterias: problem structure, assumptions on the loss functions, communication complexity (the number of encoding bits sent from client to server in each communication round), theoretical convergence rate and other aspects of the method (such as local computation and privacy). Table \ref{tbl:summary} below provides the summary.

{
    \begin{table*}[ht]
    \scriptsize
    \addtolength{\tabcolsep}{-3pt} 
        \centering
        \caption{Theoretical comparison of 7 second order methods (including ours). Advantages are written in {\color{mydarkgreen} green}, while limitations are colored in {\color{mydarkred} red}.}
        \label{tbl:summary}
        \renewcommand{\arraystretch}{1.7}
        \begin{tabular}{|@{\hskip 0.001in} c @{\hskip 0.001in}|@{\hskip 0.001in} c @{\hskip 0.001in}|@{\hskip 0.001in}c@{\hskip 0.001in}|@{\hskip 0.001in}c@{\hskip 0.001in}|@{\hskip 0.001in}c@{\hskip 0.001in}|@{\hskip 0.001in}c@{\hskip 0.001in}|}
        \hline
        \makecellnew{Method}
        & \makecellnew{Problem}
        & \makecellnew{Assumptions}
        & \makecellnew{Comm. Cost \\ per Round}
        & \makecellnew{Rate}
        & \makecellnew{Comments} \\
        \hline
            \hline
            \makecellnew{GIANT \\ \citep{GIANT2018}}
            & {\color{mydarkred} GLM}$^2$
            & \makecellnew{LipC$^{1}$ Hessian, \\ convex $+ \, l_2$ reg., \\ {\color{mydarkred} $\approx$ i.i.d. data}}
            & $\cO(d)$
            & \makecellnew{{\color{mydarkred} Local $\kappa$-dependent linear.} \\ {\color{mydarkgreen} Global $\cO(\log\nicefrac{\kappa}{\epsilon})$, quadratics}}
            & \makecellnew{Big data regime \\ (\#data$\,\gg d$)} \\
\hline
            \makecellnew{DINGO \\ \citep{DINGO}}
            & {\color{mydarkgreen} GFS}$^3$
            & \makecellnew{Moral Smoothness$^{4}$, \\ {\color{mydarkgreen} $\approx$ strong convexity}$^5$}
            & $\cO(d)$
            & \makecellnew{ {\color{mydarkgreen} Global linear rate.} \\ {\color{mydarkred} No fast local rate.} }
            & \makecellnew{ Operates {\color{mydarkred} full gradients,} \\ {\color{mydarkgreen} Hessian-vector products,} \\ {\color{mydarkgreen} Hessian pseudo-inverse} \\ {\color{mydarkgreen} and vector products.} }\\
\hline
            \makecellnew{DAN \\ \citep{DAN-LA2020}}
            & {\color{mydarkgreen} GFS}
            & \makecellnew{LipC Hessian, \\ strong convexity}
            & {\color{mydarkred} $\cO(n d^2)$}
            & \makecellnew{ {\color{mydarkgreen} Global quadratic rate} \\ {\color{mydarkred} after $\cO(\nicefrac{L}{\mu^2})$ iterations.}}
            & \makecellnew{ Operates {\color{mydarkred} full gradients} \\ {\color{mydarkred} and Hessian matrices.}} \\
\hline
            \makecellnew{DAN-LA \\ \citep{DAN-LA2020}}
            & {\color{mydarkgreen} GFS}
            & \makecellnew{LipC Hessian, \\ {\color{mydarkred} LipC gradient}, \\ strong convexity}
            & {\color{mydarkred} $\cO(n d)$}
            & \makecellnew{ {\color{mydarkred} Asymptotic and implicit} \\ {\color{mydarkgreen} global superlinear rate.} }
            & \makecellnew{$\lim_{k\to\infty} \frac{\|x_{k+1}-x^*\|}{\|x_{k}-x^*\|} = 0$ \\ {\color{mydarkred} Independent of $\kappa$ ?} \\ {\color{mydarkred} Better non-asymptotic} \\ {\color{mydarkred} complexity over linear rate ?}} \\
\hline
            \makecellnew{NL \\ \citep{Islamov2021NewtonLearn}}
            & {\color{mydarkred} GLM}
            & \makecellnew{LipC Hessian, \\ convex $+ \, l_2$ reg.}
            & $\cO(d)$
            & \makecellnew{ {\color{mydarkgreen} Local superlinear rate} \\ {\color{mydarkgreen} independent of $\kappa$}, \\ {\color{mydarkred} but dependent on \#data.} \\ {\color{mydarkgreen} Global linear rate.}}
            & \makecellnew{\color{mydarkred} reveals local data to server} \\
\hline
            \makecellnew{Quantized Newton \\ \\ \citep{Alimisis2021QNewton}}
            & {\color{mydarkgreen} GFS}
            & \makecellnew{LipC Hessian, \\ {\color{mydarkred} LipC gradient,} \\ strong convexity$^5$}
            & {\color{mydarkred} $\widetilde{\cO}(d^2)$}
            & \makecellnew{ {\color{mydarkgreen} Local (fixed) linear rate.} \\ {\color{mydarkred} No global rate.} }
            & \makecellnew{ Operates {\color{mydarkred} full gradients} \\ {\color{mydarkred} and Hessian matrices.}} \\
\hline
            \makecellnew{\bf FedNL \bf (this work)}
            & {\color{mydarkgreen} GFS}
            & \makecellnew{LipC Hessian, \\ strong convexity}
            & $\cO(d)$
            & \makecellnew{ {\color{mydarkgreen} Local (fixed) linear rate.} \\ {\color{mydarkgreen} Local superlinear rate} \\ {\color{mydarkgreen} independent of $\kappa$,} \\ {\color{mydarkgreen} independent of \#data.} \\ {\color{mydarkgreen} Global linear rate.}}
            & \makecellnew{ Operates {\color{mydarkred} full gradients} \\ {\color{mydarkred} and Hessian matrices}. \\ {\color{mydarkgreen} Supports contractive} \\ {\color{mydarkgreen} Hessian compression}. \\ {\color{blue} Extensions}$^\dagger$ } \\
\hline
        \end{tabular}   
        \begin{tablenotes}
      {\scriptsize 
        \item ${}^1$ {\bf LipC} = {\bf Lip}schitz {\bf C}ontinuous.
        \item ${}^2$ {\color{mydarkred} GLM} = {\color{mydarkred} G}eneralized {\color{mydarkred} L}inear {\color{mydarkred} M}odel, e.g. $loss_j(x;a_j) = \phi_j(a_j^\top x) + \lambda\|x\|^2.$
        \qquad ${}^3$ {\color{mydarkgreen} GFS} = {\color{mydarkgreen} G}eneral {\color{mydarkgreen} F}inite {\color{mydarkgreen} S}um.
        \item ${}^4$ Moral Smoothness: $\|\nabla^2 f(x)\nabla f(x) - \nabla^2 f(y)\nabla f(y)\| \le L \|x-y\|$.
        \qquad ${}^5$ Applies to local loss functions for all clients.
        \item ${}^\dagger$ {\color{blue} Partial Participation, Globalization (via Line Search and Cubic Regularization) and Bidirectional Compression.}
        }
    \end{tablenotes}             
    \end{table*}
}

As we can see from the table, in contrast to \algname{FedNL}, the other methods suffer at least one of the following issues:
\begin{itemize}
\item Theoretical analysis does not cover general finite sum problems (\algname{GIANT} and \algname{NL}).
\item Communication cost per client/iteration is high (\algname{DAN} and Quantized Newton).
\item Convergence rate either depends on condition number (\algname{GIANT} and \algname{DINGO}) or the number of data points (NL) or is not explicit/clear (\algname{DAN-LA}).
\item Privacy is broken by directly revealing local training data (\algname{NL}).
\end{itemize}

We do not compare with algorithms \algname{DANE} \citep{DANE} and its accelerated variant \algname{AIDE} \citep{Reddi:2016aide} since they are first-order methods. This means that convergence rates depend on the conditioning of the problem and hence are worse than what we prove for \algname{FedNL}. Moreover, \algname{DANE} does not work well for heterogeneous datasets - the analysis and experimental evidence of \algname{DANE} only shows benefits in a sufficiently homogeneous data regime. On the other hand, our concern is the heterogeneous data regime typical to FL.
We also omit \algname{DiSCO} \citep{DiSCO} from our empirical study because the problem setup is restricted to homogeneous data distribution regime, generalized linear models and convergence rates depend on the conditioning of the problem. Furthermore, the authors of \algname{DINGO} experimentally showed that \algname{DINGO} outperforms methods like \algname{DiSCO} and \algname{GIANT}, and this is why we focused on comparing to \algname{DINGO}.

Next, we compare several first and second order methods based on their communication complexity, defined as the total number of bits sent from a client to the server to achieve some prescribed accuracy $\epsilon$. For this purpose, we use sparsification as an example of a compressor in most cases. For all methods supporting sparsification, we have used the sparsification, which reduces the number of communicated floats by the factor of $d$ compared to the non-compressed variant of the method. That is, for gradient based methods DCGD, DIANA and ADIANA, we have used the Rand-1 sparsifier, which compresses $\cO(d)$ gradient to $\cO(1)$.

To transform the local superlinear convergence rate \eqref{rate:local-superlinear-iter} of FedNL into an iteration complexity, we proceed as follows. Let $r_{k+1} \le C (1-\rho)^k r_k$, where $r_k = \|x^k - x^*\|^2$, $\rho\in(0,1)$ and $C>0$ is some constant. Note that if $k\ge\frac{2}{\rho}\log\frac{1}{C}$, then $(1-\rho)^{\nicefrac{k}{2}}C \le 1$. Hence, after $\cO(\frac{1}{\rho})$ iteration we have $r_{k+1} \le (1-\rho)^{\nicefrac{k}{2}} r_k$. Unraveling the recursion we get
$$
r_k \le (1-\rho)^{\frac{k-1}{2}} (1-\rho)^{\frac{k-2}{2}} \dots (1-\rho)^{\frac{1}{2}} r_0 = (1-\rho)^{\frac{k(k-1)}{2}} r_0.
$$
Therefore, FedNL needs $\cO\(\sqrt{\frac{1}{\rho}\log\frac{1}{\epsilon}}\)$ number of iterations to achieve $\epsilon$-accuracy. For FedNL, we used step-size $\alpha=1$ (see Assumption \ref{asm:comp-1}(ii) and also \eqref{ABCD}) and matrix sparsification described in Appendix \ref{apx:extra-exp-topk}, which compresses $\cO(d^2)$ Hessian down to $\cO(d)$ (i.e., $\delta = \frac{1}{d}$). With this choice we get $\frac{1}{\rho} = \cO(d)$ and the iteration of FedNL becomes $\cO\(\sqrt{d\log\frac{1}{\epsilon}}\)$. For Newton Learn (NL), $\frac{1}{\rho} = \cO(\#\text{data})$, where $\#\text{data}$ is the number of data points in each device. DAN and Quantized Newton use their own bespoke ways of compressing communication. Table \ref{tbl:comm-complexity} provides the details, from which we make the following observations:

{
    \begin{table*}[!h]
    \scriptsize
    \addtolength{\tabcolsep}{-3pt} 
        \centering
        \caption{Theoretical comparison of 3 gradient-based and 5 second-order methods. The last column (communication complexity) is the product of the previous two columns and is the key quantity to be compared.}
        \label{tbl:comm-complexity}
        \renewcommand{\arraystretch}{1.7}
        \begin{tabular}{|c|c|c|c|}
        \hline
        \makecellnew{Method}
        & \makecellnew{\# Communication Rounds}
        & \makecellnew{Comm. Cost \\ per Round}
        & \makecellnew{Communication \\ Complexity} \\
        \hline
\hline
            \makecellnew{Gradient Descent$^1$}
            & $\cO(\kappa\log\frac{1}{\epsilon})$
            & $\cO(d)$
            & $\cO(d\kappa\log\frac{1}{\epsilon})$\\
\hline
            \makecellnew{DCGD$^1$ \\ \cite{DCGD}}
            & $\cO\(\frac{d\sigma^*}{n\mu^2}\frac{1}{\epsilon} \log\frac{1}{\epsilon} \)$
            & $\cO(1)$
            & $\cO\(\frac{d\sigma^*}{n\mu^2}\frac{1}{\epsilon} \log\frac{1}{\epsilon} \)$\\
\hline
            \makecellnew{DIANA$^1$ \\ \cite{DIANA}}
            & $\cO\(\(d+\kappa+\kappa\frac{d}{n}\)\log\frac{1}{\epsilon}\)$
            & $\cO(1)$
            & $\cO\(\(d+\kappa+\kappa\frac{d}{n}\)\log\frac{1}{\epsilon}\)$\\
\hline
            \makecellnew{ADIANA$^1$ \\ \cite{ADIANA}}
            & $\cO\(\(d + \sqrt{\kappa} + \sqrt{\(\frac{d}{n} + \sqrt{\frac{d}{n}}\)d\kappa} \)\log\frac{1}{\epsilon}\)$
            & $\cO(1)$
            & $\cO\(\(d + \sqrt{\kappa} + \sqrt{\(\frac{d}{n} + \sqrt{\frac{d}{n}}\)d\kappa} \)\log\frac{1}{\epsilon}\)$\\
\hline
            \makecellnew{Newton}
            & $\cO(\log\log\frac{1}{\epsilon})$
            & $\cO(d^2)$
            & $\cO(d^2\log\log\frac{1}{\epsilon})$\\
\hline
            \makecellnew{DAN$^1$ \\ \cite{DAN-LA2020}}
            & $\cO\(\frac{\HS}{\mu^2} + \log\log\frac{1}{\epsilon}\)$
            & $\cO(d^2)$
            & $\cO\(d^2\(\frac{\HS}{\mu^2} + \log\log\frac{1}{\epsilon}\)\)$\\
\hline
            \makecellnew{Quantized Newton \\ \citep{Alimisis2021QNewton}}
            & $\cO(\log\frac{1}{\epsilon})$
            & $\widetilde{\cO}(d^2)$
            & $\widetilde{\cO}(d^2\log\frac{1}{\epsilon})$\\
\hline
            \makecellnew{NL \\ \citep{Islamov2021NewtonLearn}}
            & $\cO\(\sqrt{\#\text{data}} \sqrt{\log\frac{1}{\epsilon}}\)$
            & $\cO(d)$
            & $\cO\(d\sqrt{\#\text{data}} \sqrt{\log\frac{1}{\epsilon}}\)$\\
\hline
            \makecellnew{\bf FedNL \\ \bf (this work; \eqref{rate:local-linear-iter})}
            & $\cO\(\log\frac{1}{\epsilon}\)$
            & $\cO(d)$
            & $\cO\(d \log\frac{1}{\epsilon}\)$\\
\hline
            \makecellnew{\bf FedNL \\ \bf (this work; \eqref{rate:local-superlinear-iter})}
            & $\cO\(\sqrt{d} \sqrt{\log\frac{1}{\epsilon}}\)$
            & $\cO(d)$
            & $\cO\(d\sqrt{d} \sqrt{\log\frac{1}{\epsilon}}\)$\\
\hline
        \end{tabular}   
        \begin{tablenotes}
        {\scriptsize     
        \item ${}^1$ These methods have global rates.
        \quad ${}^2$ DCGD, DIANA and ADIANA are first order methods.
        \item ${}^3$ Newton, DAN, Quantized Newton, NL and FedNL are second order methods.
        \item ${}^4$ $\kappa$ is the condition number: $\kappa = \frac{L}{\mu}$ where $L$ is a smoothness constant and $\mu$ is the strong convexity constant.
        }
        \end{tablenotes}             
    \end{table*}
}


\begin{itemize}

\item {\bf FedNL achieves better communication complexity than Newton whenever $d > \frac{\log\frac{1}{\epsilon}}{\(\log\log\frac{1}{\epsilon}\)^2}$.} For example, if we set $\epsilon = 10^{-10}$, then this requirement means $d>10$, and hence is not restrictive. So, virtually in all situations of practical interest, FedNL is better than Newton. The improvement is more pronounced with larger $d$, and is approximately of the size $\cO(\sqrt{d})$. So, for 
$d=10^6$, for example, FedNL finds the solution using approximately 1000 times less communicated bits than Newton.

\item {\bf FedNL achieves better communication complexity than Gradient Descent whenever $\kappa > \frac{\sqrt{d}}{\sqrt{\log\frac{1}{\epsilon}}}$.} So, FedNL is better when the condition number $\kappa$ is large enough. This is expected, since FedNL complexity does not depend on the condition number. The advantage of FedNL grows if $d$
or $\epsilon$ are smaller.

\item ADIANA is known to have the state of the art complexity (in the strongly convex regime) among all first order method, and hence we do not need to compare FedNL to DCGD and DIANA, which are both inferior to ADIANA. It is clear that FedNL can beat ADIANA as well since the complexity of ADIANA depends on $\kappa$. {\bf So, for large enough $\kappa$, FedNL is better than ADIANA. For example, a simple sufficient condition for this to happen is to require $\kappa > d^3$} (this can be refined, but the expression will become uglier). Likewise, FedNL has square root dependence on $\log\frac{1}{\epsilon}$, and hence it becomes better than ADIANA if $\epsilon$ is sufficiently small (and other terms are kept constant).

\item Neither Quantized Newton nor DAN improve on Newton in communication complexity (but may be better in practice). We already explained that FedNL improves on Newton.

\item We do not include GIANT in the table since GIANT does not work in the heterogeneous data regime, which is critical to FL and our paper. We do not include DINGO in the table since its rate depends on various iterate-dependent assumptions which make the analysis convoluted. It is not clear that such assumptions can actually be satisfied. Their rates are not explicit - it is not possible to compare to them.

\end{itemize}

It worths noting that our work is not just about communication complexity. In fact, our contributions go far beyond this, and we make it clear in the paper. Our work is the first serious attempt to make second order methods applicable to federated learning in the sense that we address many issues which previously made second order methods inapplicable to FL. We support compression of matrices, rudimentary privacy protection (by not revealing data), partial participation, compression of (Hessian corrected) gradients, compression of model (at the master), arbitrary (strongly convex) finite sum problems rather than generalized linear models only, arbitrary contractive compressors, two globalization strategies and more. The best way to judge our contribution to the literature is via comparison to the NewtonLearn work of \cite{Islamov2021NewtonLearn} as their work is the closest work to ours and was the SOTA second order method supporting communication compression before our work. We have made a very detailed comparison to their work, including tables.

\section{Extra Experiments}

We carry out numerical experiments to study the performance of \algname{FedNL}, and compare it with various state-of-the-art methods in federated learning. We consider the following problem
\begin{equation}
    \min\limits_{x\in\R^d}\left\{\frac{1}{n}\sum\limits_{i=1}^n f_i(x) +\frac{\lambda}{2}\|x\|^2\right\}, \quad f_i(x) = \frac{1}{m}\sum_{j=1}^m\log\(1+\exp(-b_{ij}a_{ij}^\top x)\),
\end{equation}
where $\{a_{ij},b_{ij}\}_{j\in [m]}$ are data points at the $i$-th device. 

\subsection{Data sets}
The datasets were taken from LibSVM library \citep{chang2011libsvm}: \dataname{a1a}, \dataname{a9a}, \dataname{w7a}, \dataname{w8a}, \dataname{phishing}. We partitioned each data set across several nodes to capture a variety of scenarios. See Table \ref{tab:datasets} for more detailed description of data sets settings.

\begin{table}[h]
    \caption{Data sets used in the experiments with the number of worker nodes $n$ used in each case.}
    \label{tab:datasets}
    \centering
    \begin{tabular}{|l|r|r|r|}
        \toprule
        {\bf Data set} & {\# workers} $n$ & {\# data points} ($=nm$) & {\# features} $d$                 \\
        \midrule
        {\tt a1a} & $16$ & $1600$ & $123$\\ \hline
        {\tt a9a} & $80$ & $32560$ & $123$\\ \hline
        {\tt w7a} & $50$ & $24600$ & $300$\\ \hline
        {\tt w8a} & $142$ & $49700$ & $300$\\ \hline
        {\tt phishing} & $100$ & $110$ & $68$\\
        {\tt madelon} & $10$ & $2000$ & $500$\\
        \bottomrule
    \end{tabular}
\end{table}

\subsection{Parameters setting}
In all experiments we use theoretical parameters for gradient type methods (except those with line search procedure): vanilla gradient descent, \algname{DIANA} \citep{DIANA}, \algname{ADIANA} \citep{ADIANA}, and Shifted Local gradient descent \citep{Gorbunov2020localSGD}. The constants for \algname{DINGO} \citep{DINGO} are set as the authors did: $\theta=10^{-4}, \phi=10^{-6}, \rho=10^{-4}$. Backtracking line search for \algname{DINGO} selects the largest stepsize from $\{1,2^{-1},\dots,2^{-10}\}.$ The initialization of $\mH^0_i$ for \algname{NL1} \citep{Islamov2021NewtonLearn}, \algname{FedNL}, \algname{FedNL-LS}, and \algname{FedNL-PP} is $\nabla^2 f_i(x^0)$, and for \algname{FedNL-CR} is $\mathbf{0}$.

We conduct experiments for two values of regularization parameter $\lambda\in \{10^{-3}, 10^{-4}\}$. In the figures we plot the relation of the optimality gap $f(x^k)-f(x^*)$ and the number of communicated bits per node or the number of communication rounds. The optimal value $f(x^*)$ is chosen as the function value at the $20$-th iterate of standard Newton's method.

\subsection{Compression operators}\label{sec:comp-op}

Here we describe four compression operators that are used in our experiments.

\subsubsection{Random dithering for vectors}

For first order methods \algname{ADIANA} and \algname{DIANA} we use random dithering operator \citep{QSGD, DIANA-VR}. This compressor with $s$ levels is defined via the following formula
\begin{equation}
    \cC(x) \eqdef \text{sign}(x) \cdot \|x\|_q \cdot \frac{\xi_s}{s},
\end{equation}
where $\|x\|_q \eqdef \(\sum_i |x_i|^q\)^{1/q}$ and $\xi_s \in \R^d$ is a random vector with $i$-th element defind as follows
\begin{equation}
    (\xi_s)_i = \begin{cases}
        l+1 & \text{with probability } \frac{|x_i|}{\|x\|_q}s-l,\\
        l & \text{otherwise}.
    \end{cases}
\end{equation}
Here $s \in \N_+$ denotes the levels of rounding, and $l$ satisfies $\frac{|x_i|}{\|x\|_q} \in \[\frac{l}{s}, \frac{l+1}{s}\]$. According to \citep{DIANA-VR}, this compressor has the variance parameter $\omega \le 2+\frac{d^{1/2}+d^{1/q}}{s}.$ However, for standard euclidean norm ($q=2$) one can improve the bound by $\omega \leq \min\left\{\frac{d}{s^2}, \frac{\sqrt{d}}{s}\right\}$ \citep{QSGD}.

\subsubsection{Rank-$R$ compression operator for matrices}

Our theory supports contractive compression operators; see Definition~\ref{def:class-contractive}. In the experiments for \algname{FedNL} we use Rank-$R$ compression operator. Let $\mathbf{X}\in \mathbb{R}^{d\times d}$ and $\mathbf{U} \Sigma \mathbf{V}^\top$ be the singular value decomposition of $\mathbf{X}$:
\begin{equation}\label{eq:b97gs_9098fd}
    \mathbf{X} = \sum\limits_{i=1}^d \sigma_i u_iv_i^\top,
\end{equation}
where the singular values $\sigma_i$ are sorted in non-increasing order: $\sigma_1\geq \sigma_2 \geq \cdots \geq \sigma_d$. Then, the Rank-$R$ compressor, for $R \le d$, is  defined by
\begin{equation}\label{eq:n988fdgfd}
    \cC(\mathbf{X}) \eqdef \sum\limits_{i=1}^{R} \sigma_i u_iv_i^\top.
\end{equation}
Note that
$$\|\mathbf{X} \|^2_{\rm F} \overset{\eqref{eq:b97gs_9098fd}}{=} \left\| \sum\limits_{i=1}^{d} \sigma_i u_iv_i^\top  \right\|^2_{\rm F}   = \sum_{i=1}^d \sigma_i^2$$
and
$$\|\cC(\mathbf{X})-\mathbf{X} \|^2_{\rm F}  \overset{\eqref{eq:b97gs_9098fd}+\eqref{eq:n988fdgfd}}{=}  \left\| \sum\limits_{i=R+1}^{d} \sigma_i u_iv_i^\top  \right\|^2_{\rm F}   = \sum_{i=R+1}^d \sigma_i^2 .$$
Since 
$\frac{1}{d-R}\sum_{i=R+1}^d \sigma_i^2 \leq \frac{1}{d}\sum_{i=1}^d \sigma_i^2$, we have 
$$\|\cC(\mathbf{X})-\mathbf{X} \|^2_{\rm F}  \leq \frac{d-R}{d} \| \mathbf{X} \|^2_{\rm F} = \left(1 - \frac{R}{d}\right) \| \mathbf{X} \|^2_{\rm F} ,$$
and hence the Rank-$R$ compression operator belongs to $\mathbb{C}(\delta)$ with $\delta = \frac{R}{d}.$ In case when $\mathbf{X} \in \mathbb{S}^d$, we have $u_i=v_i$ for all $i\in [d]$, and Rank-$R$ compressor on matrix $\mathbf{X}$ transforms to $\sum_{i=1}^R \sigma_i u_i u_i^\top$, i.e., the output of Rank-$R$ compressor is automatically a symmetric matrix, too.

\subsubsection{Top-$K$ compression operator for matrices}\label{apx:extra-exp-topk}

Another example of contractive compression operators is Top-$K$ compressor for matrices. For arbitrary matrix $\mathbf{X} \in \R^{d\times d}$ let sort its entires in non-increasing order by magnitude, i.e., $X_{i_k, j_k}$ is the $k$-th maximal element of $\mathbf{X}$ by magnitude. Let's  $\{\mathbf{E}_{ij}\}_{i,j=1}^d$ me matrices for which 
\begin{equation}
    \(\mathbf{E}_{ij}\)_{ps} \eqdef \begin{cases}1, & \text{if } (p,s) = (i,j),\\
    0, & \text{otherwise.}  \end{cases}
\end{equation}  
Then, the Top-$K$ compression operator can be defined via
\begin{equation}
    \cC(\mathbf{X}) \eqdef \sum\limits_{k=1}^K X_{i_k,j_k}E_{i_k,j_k}.
\end{equation}
This compression operator belongs to $\mathbb{C}(\delta)$ with $\delta = \frac{K}{d^2}$. If we need to keep the output of Top-$K$ on symmetric matrix $\mathbf{X}$ to be symmetric matrix too, then we apply Top-$K$ compressor only on lower triangular part of $\mathbf{X}$.

\subsubsection{Rand-$K$ compression operator for matrices}\label{apx:extra-exp-randk}

Our theory also supports unbiased compression operators; see Definition~\ref{def:class-unbiased}. One of the examples is Rand-$K$. For arbitrary matrix $\mathbf{X} \in \R^{d\times d}$ we choose a set $\mathcal{S}_K$ of indexes $(i,j)$  of cardinality $K$ uniformly at random. Then Rand-$K$ compressor can be defined via
\begin{equation}
    \cC\(\mathbf{X}\)_{ij} \eqdef \begin{cases}\frac{d^2}{K}X_{ij} & \text{if } (i,j) \in \mathcal{S}_K,\\
    0 & \text{otherwise}. \end{cases}
\end{equation}
This compression operator belongs to $\mathbb{B}(\omega)$ with $\omega = \frac{d^2}{K}-1$. If we need to make the output of this compressor to be symmetric matrix, then we apply this compressor only on lower triangular part of the input.

\subsection{Projection onto the cone of positive definite matrices}\label{sec:cone-proj}

If one uses \algname{FedNL} with Option~$1$, then we need to project onto the cone of symmetric and positive definite matrices with constant $\mu$: $$\{\mathbf{M}\in \R^{d\times d}: \mM^\top = \mM,\; \mathbf{M} \succeq \mu\mI\}.$$ The projection of symmetric matrix $\mathbf{X}$ onto the cone of positive semidefinite matrices can by computed via
\begin{equation}
    \[\mathbf{X}\]_0 \eqdef \sum\limits_{i=1}^d \max\{\lambda_i,0\}u_iu_i^\top,
\end{equation}
where $\sum_i \lambda_iu_iu_i^\top$ is an eigenvalue decomposition of $\mathbf{X}$. Using the projection onto the cone of positive semidefinite matrices we can define the projection onto the cone of positive definite matrices with constant $\mu$ via
\begin{equation}
    \[\mathbf{X}\]_\mu \eqdef \[\mathbf{X}-\mu\mI\]_0 + \mu\mI.
\end{equation}

\subsection{The effect of compression}

First, we investigate how the level of compression influences the performance of \algname{FedNL}; see Figure~\ref{fig:comp_operators}. Here we study the performance for three types of compression operators: Rank-$R$, Top-$K$, and PowerSGD of rank $R$. According to numerical experiments, the smaller parameter is, the better performance of \algname{FedNL} is. This statement is true for all three types of compressors.

\begin{figure}[t]
    \begin{center}
        \begin{tabular}{cccc}
            \includegraphics[width=0.22\linewidth]{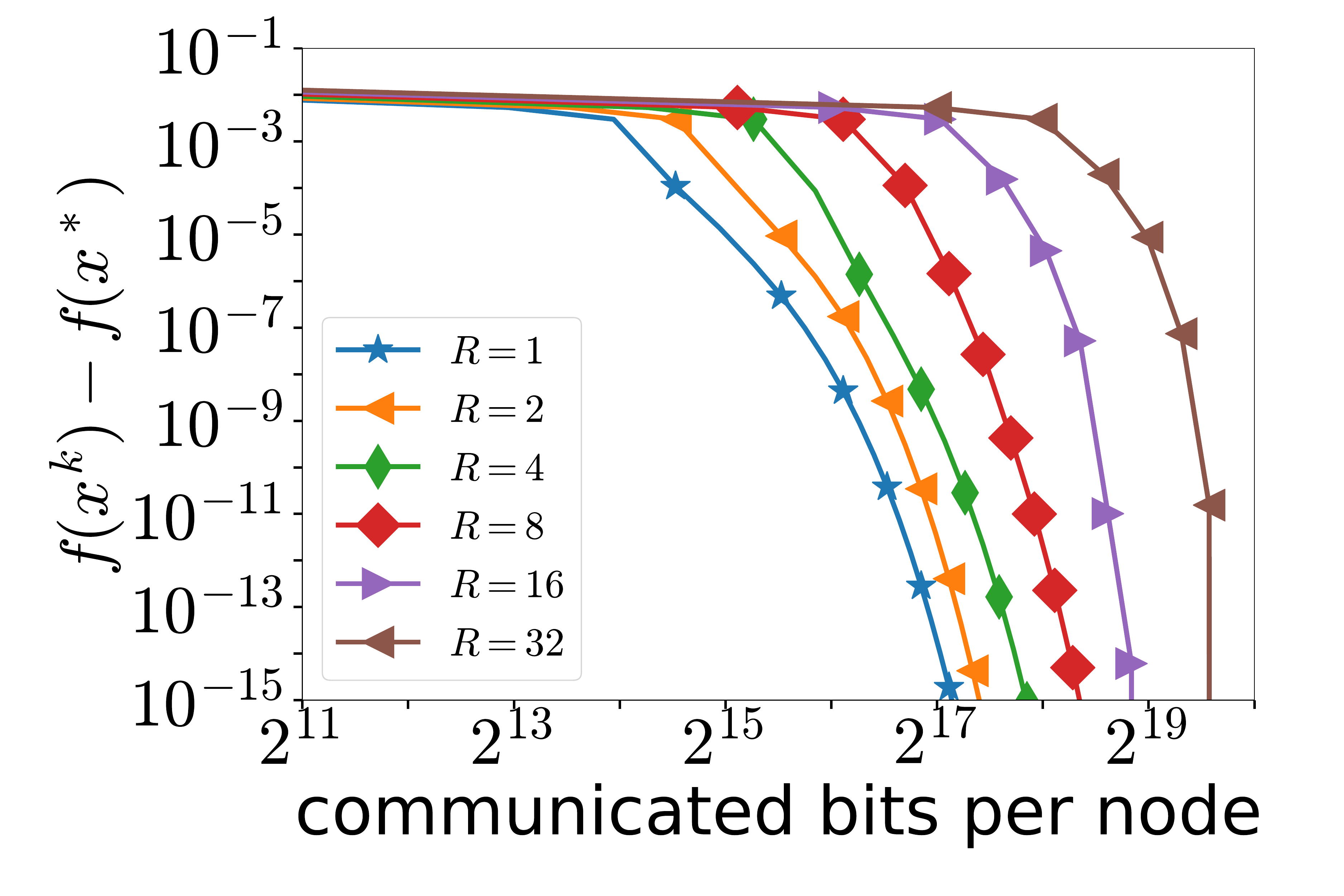} &
            \includegraphics[width=0.22\linewidth]{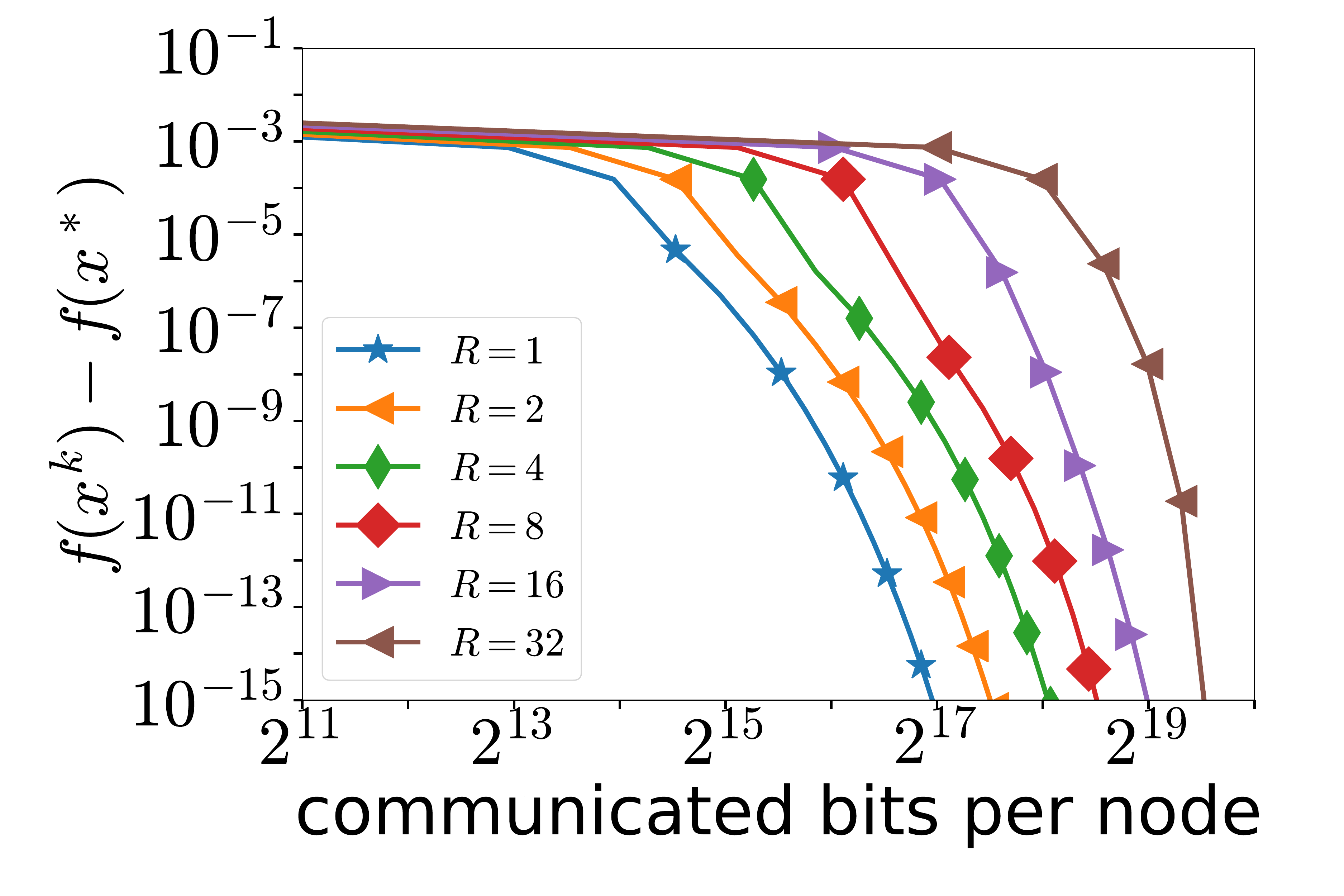} &
            \includegraphics[width=0.22\linewidth]{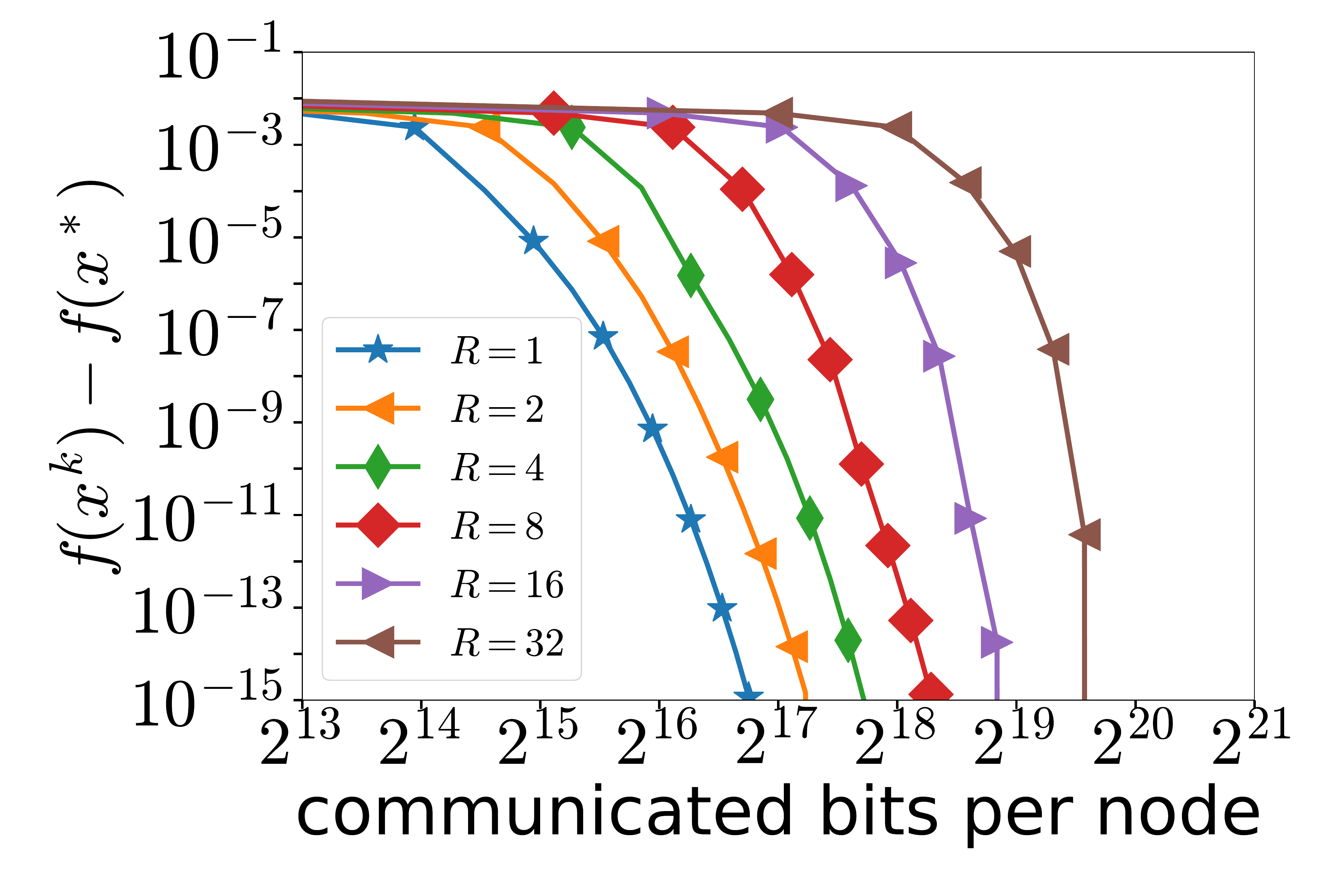} &
            \includegraphics[width=0.22\linewidth]{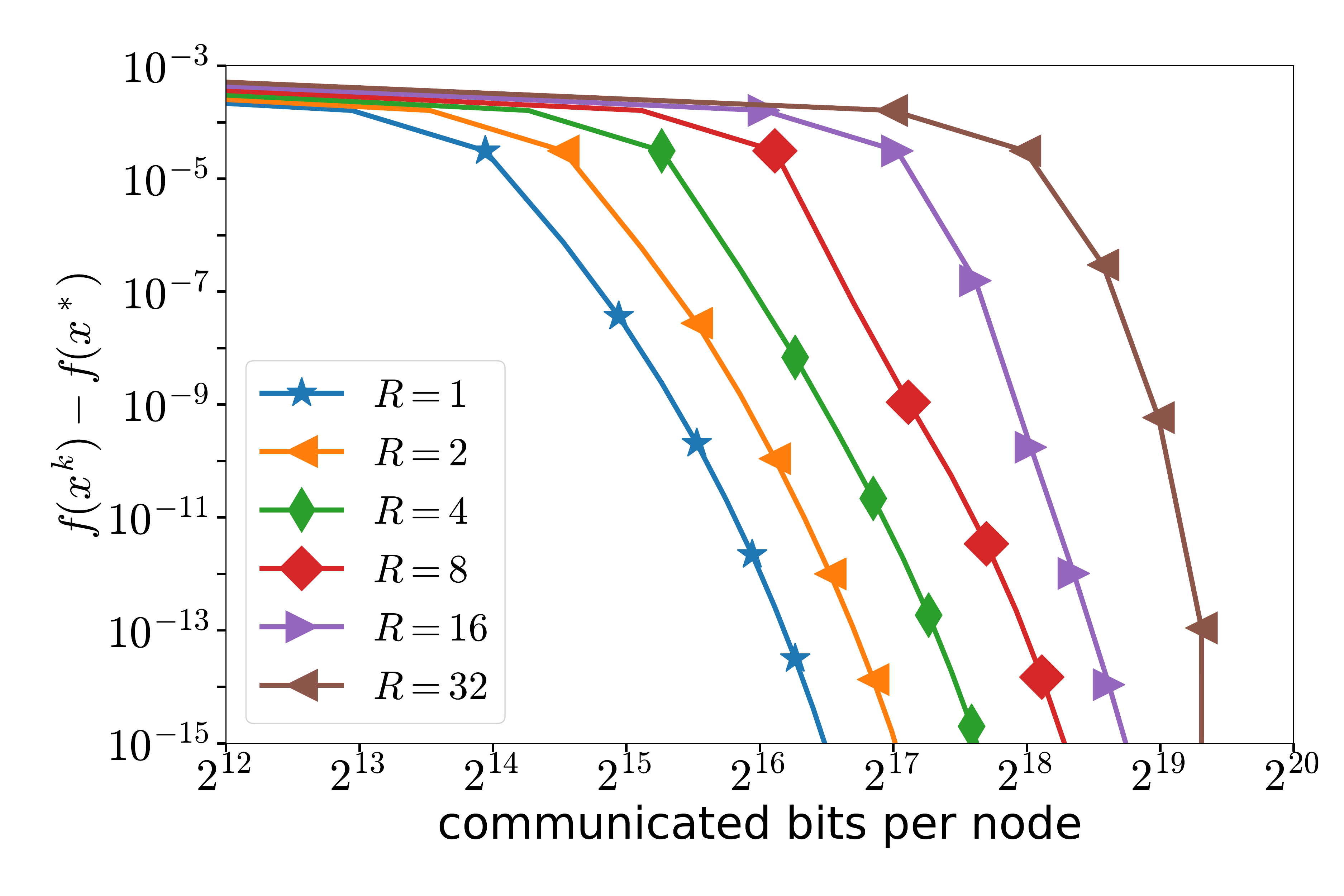}\\
            \dataname{a1a}, $\lambda=10^{-3}$ &
            \dataname{a1a}, $\lambda=10^{-4}$ &
            \dataname{a9a}, $\lambda=10^{-3}$ &
            \dataname{a9a}, $\lambda=10^{-4}$\\
            \includegraphics[width=0.22\linewidth]{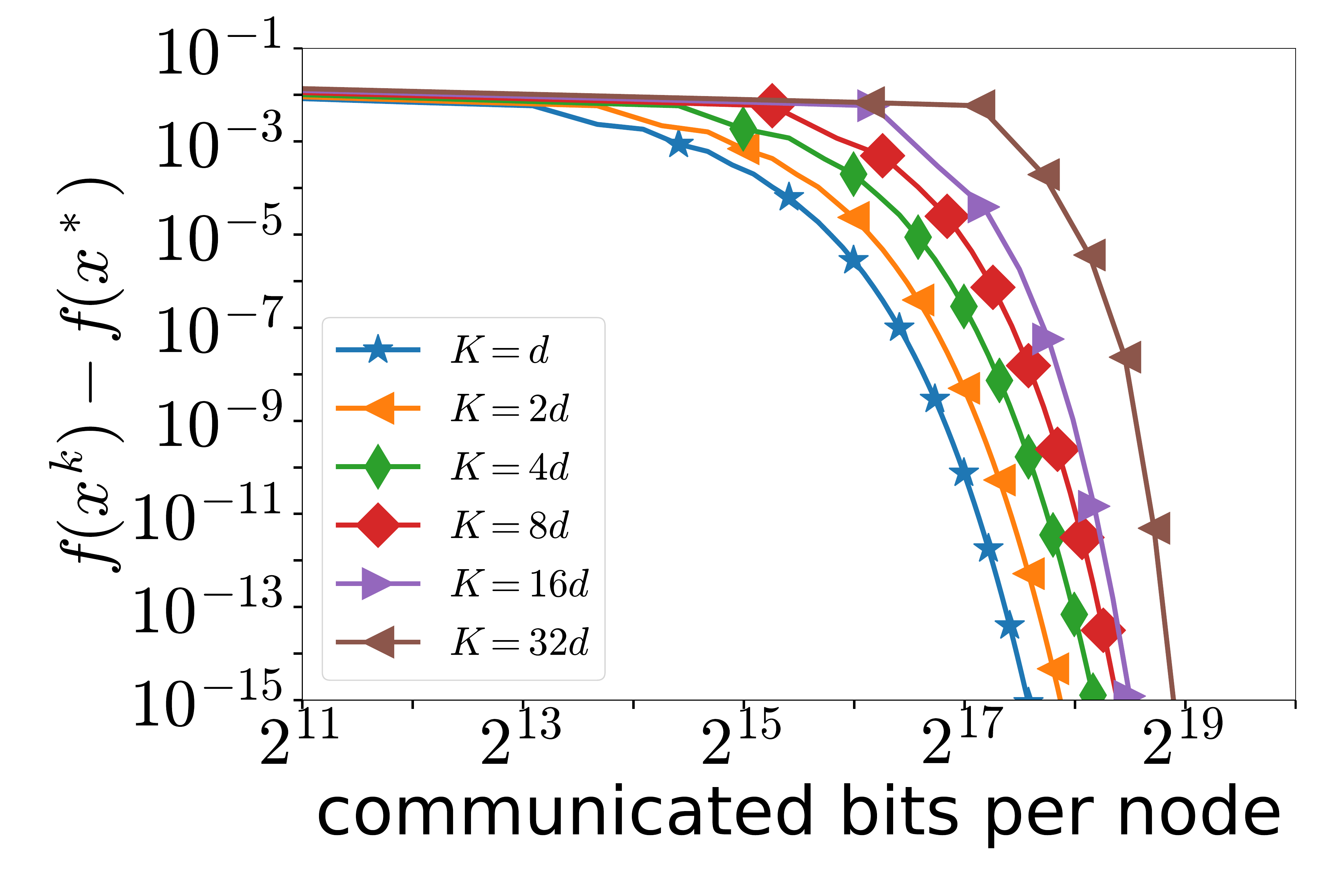} & 
            \includegraphics[width=0.22\linewidth]{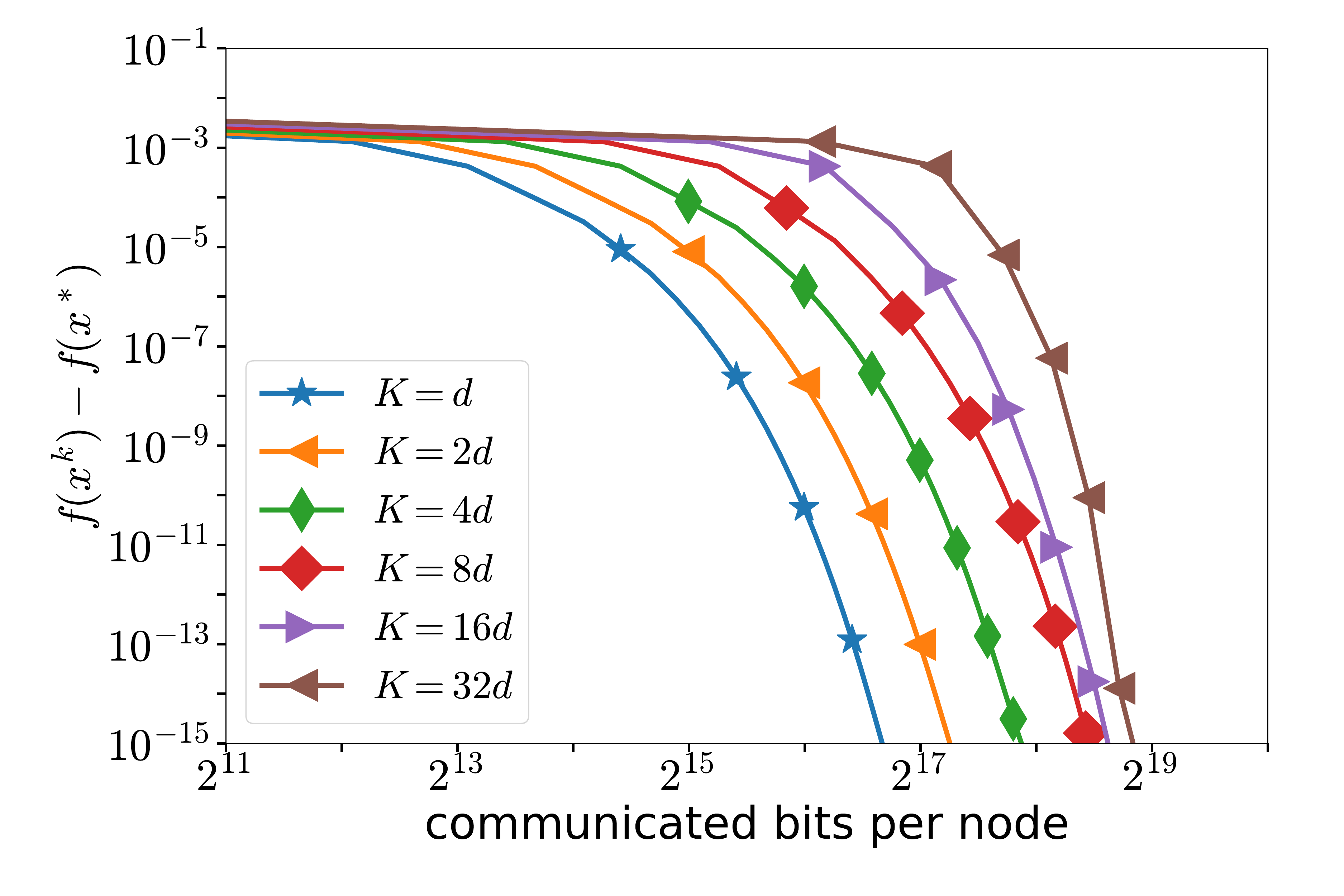} & 
            \includegraphics[width=0.22\linewidth]{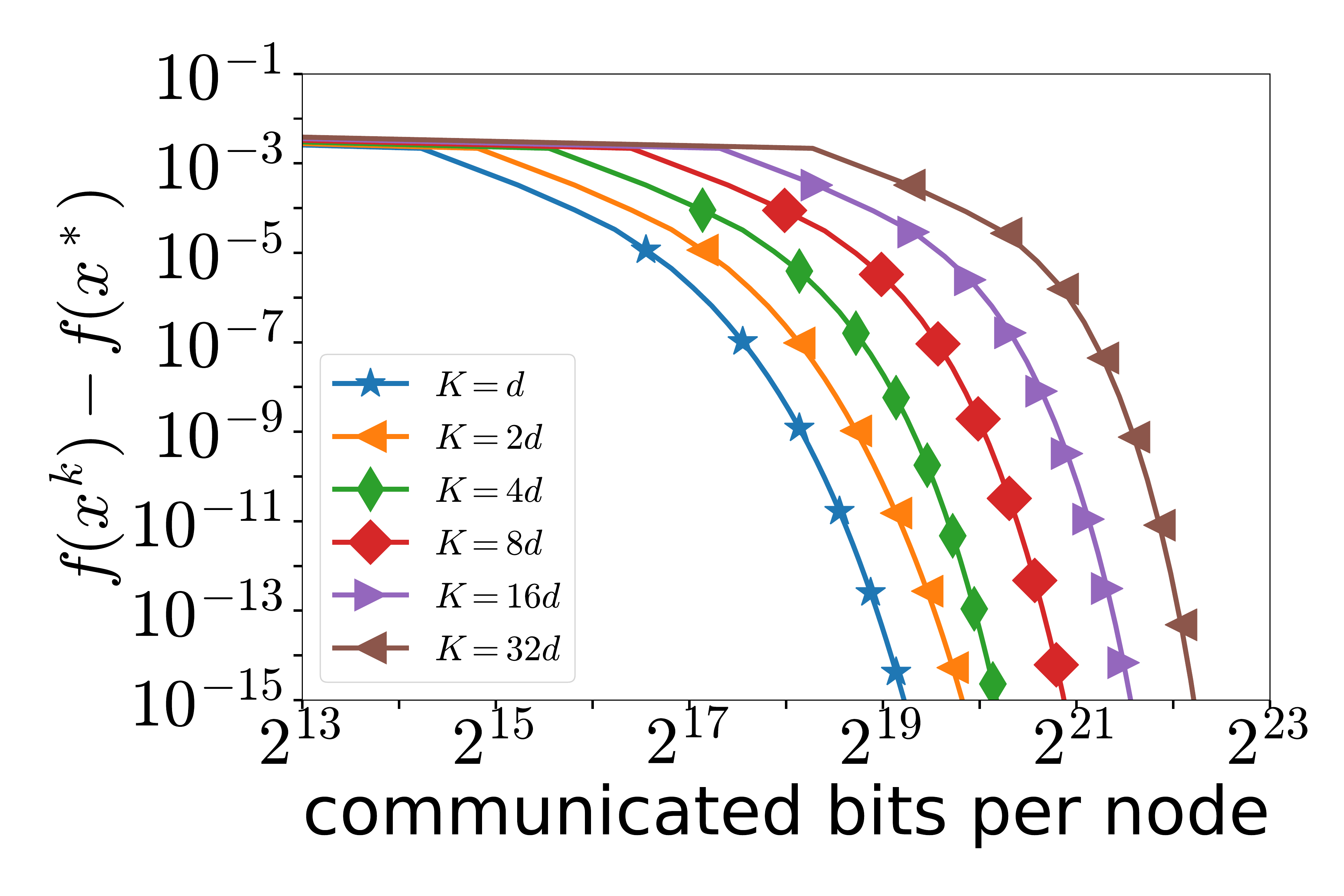} & 
            \includegraphics[width=0.22\linewidth]{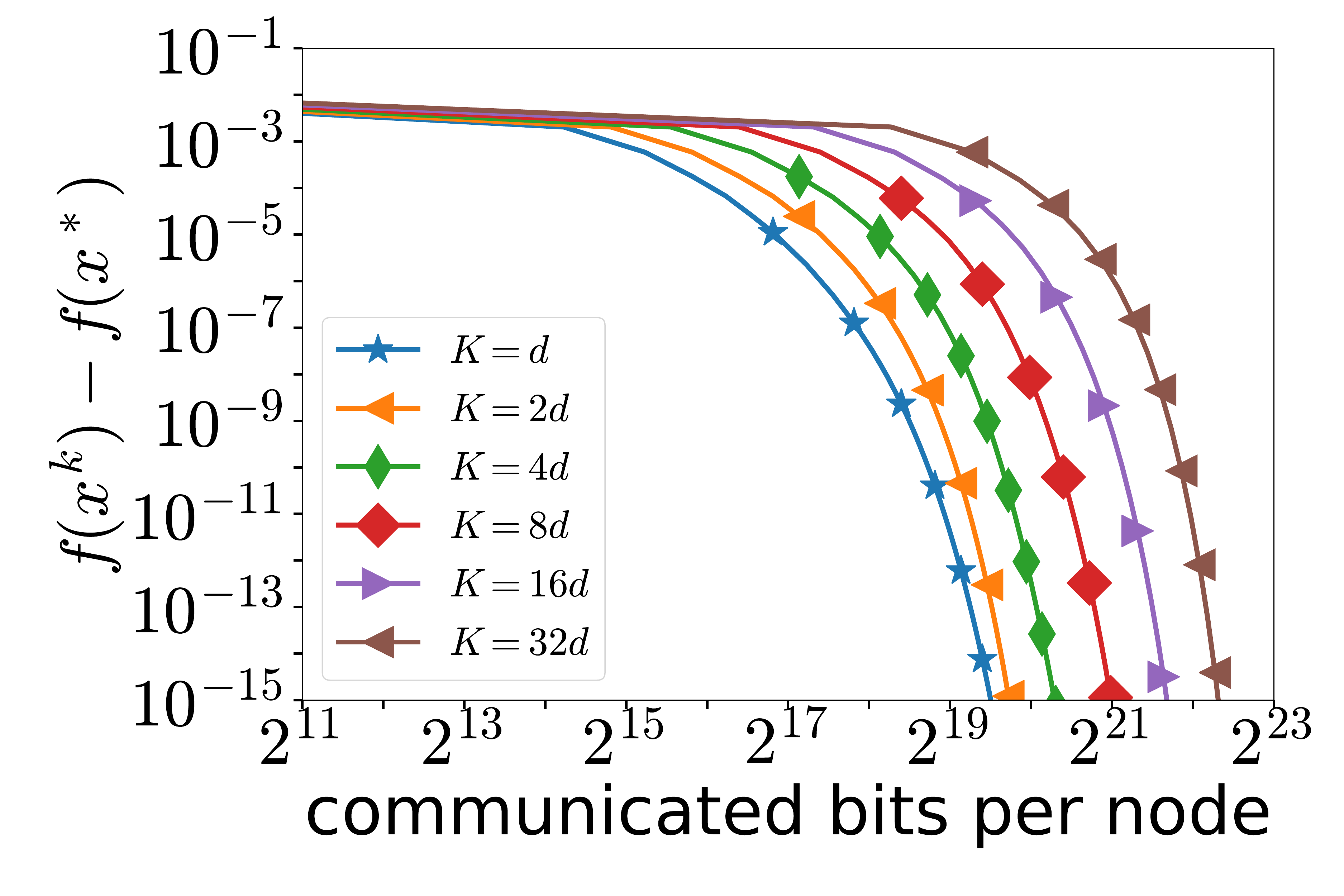}\\
            \dataname{phishing}, $\lambda=10^{-3}$ &
            \dataname{phishing}, $\lambda=10^{-4}$ &
            \dataname{w7a}, $\lambda=10^{-3}$ &
            \dataname{w8a}, $\lambda=10^{-3}$\\
            \includegraphics[width=0.22\linewidth]{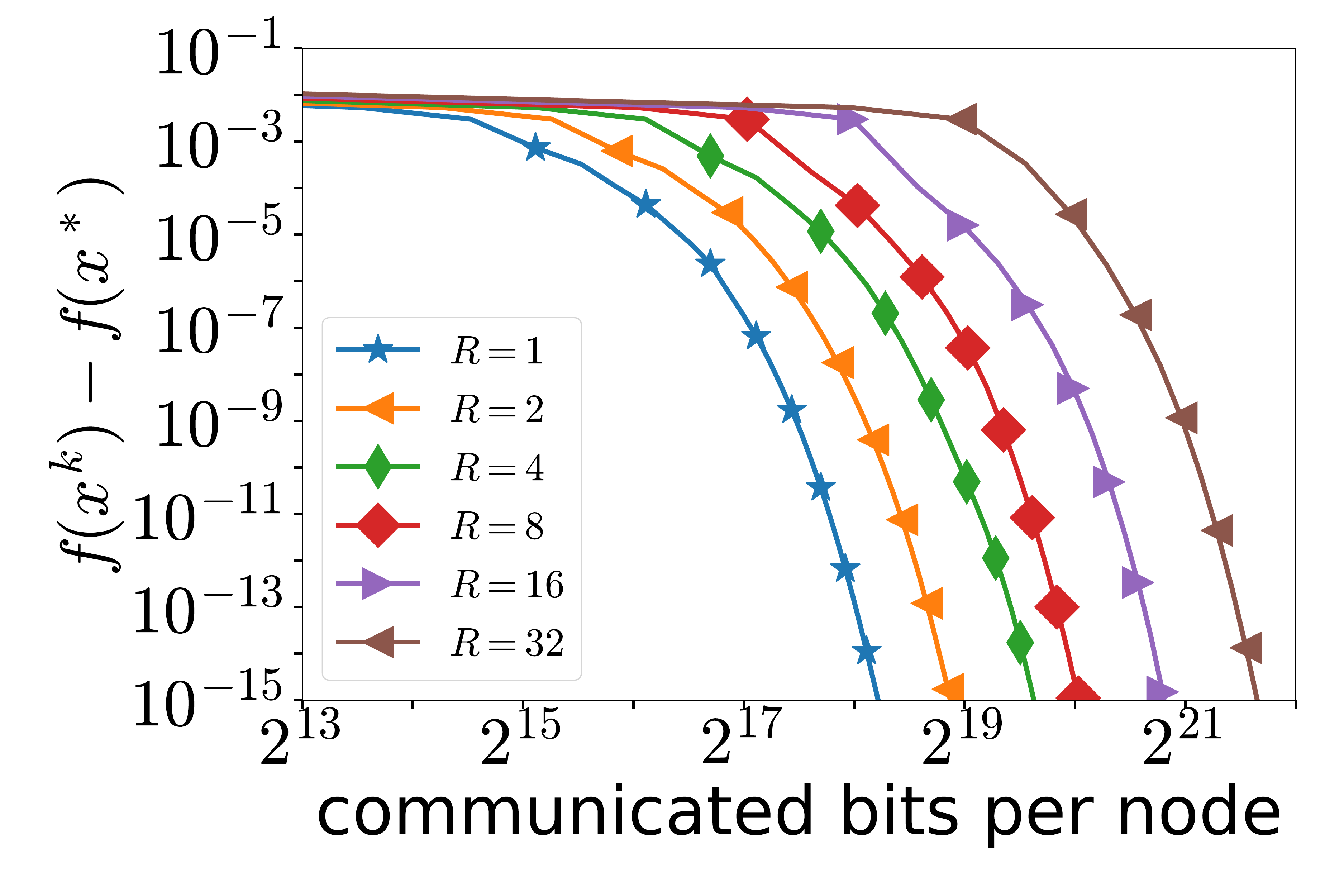} &
            \includegraphics[width=0.22\linewidth]{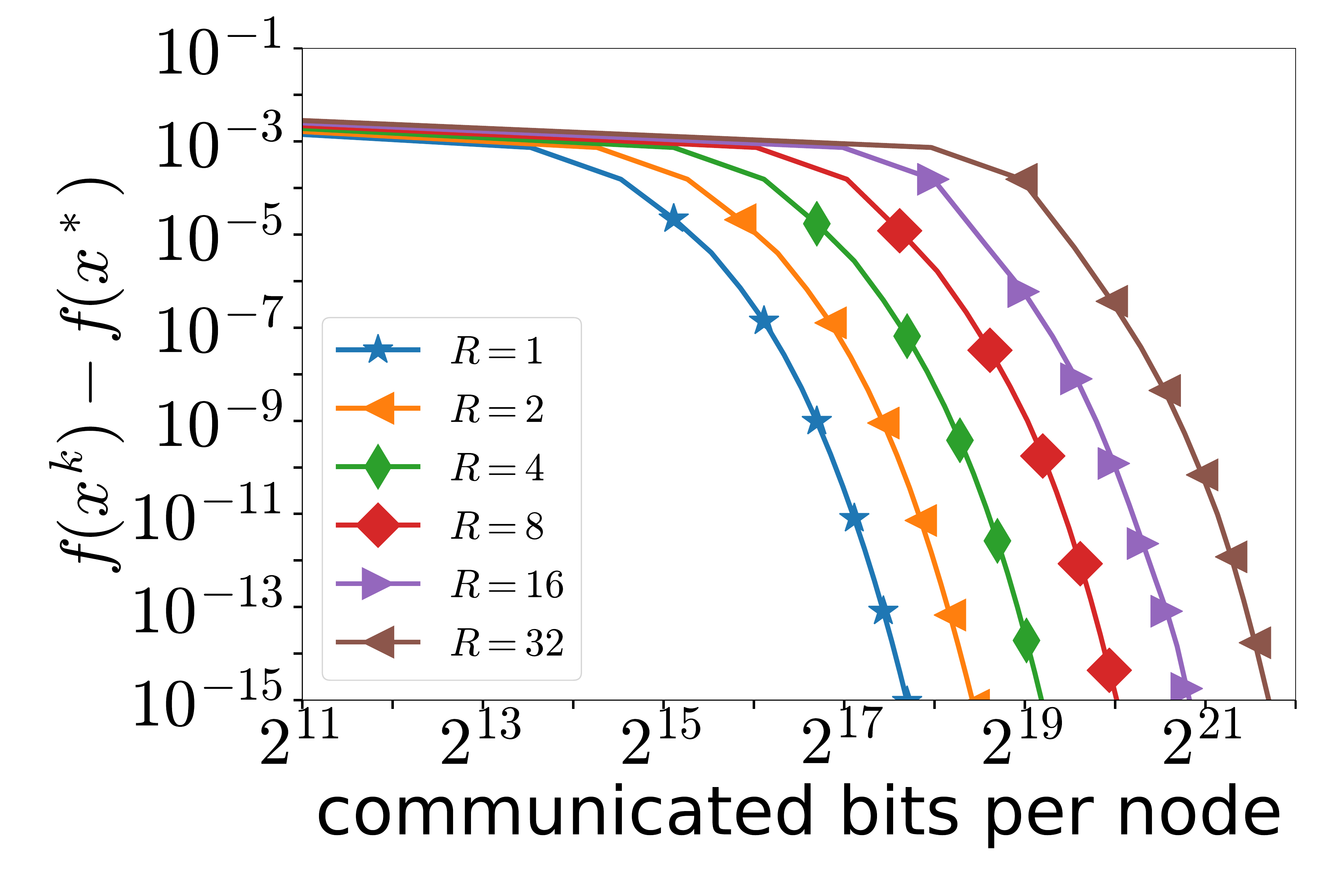} &
            \includegraphics[width=0.22\linewidth]{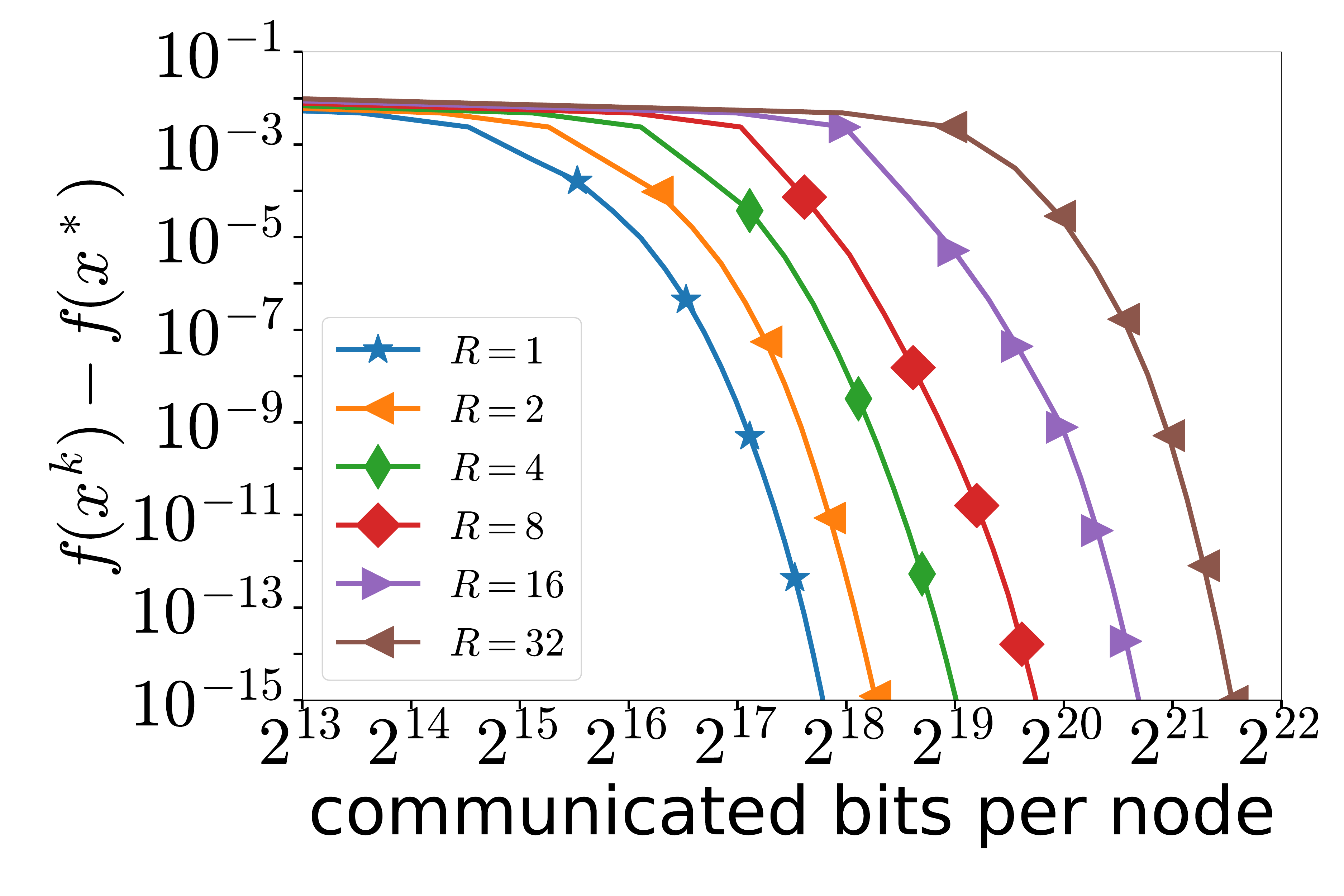} &
            \includegraphics[width=0.22\linewidth]{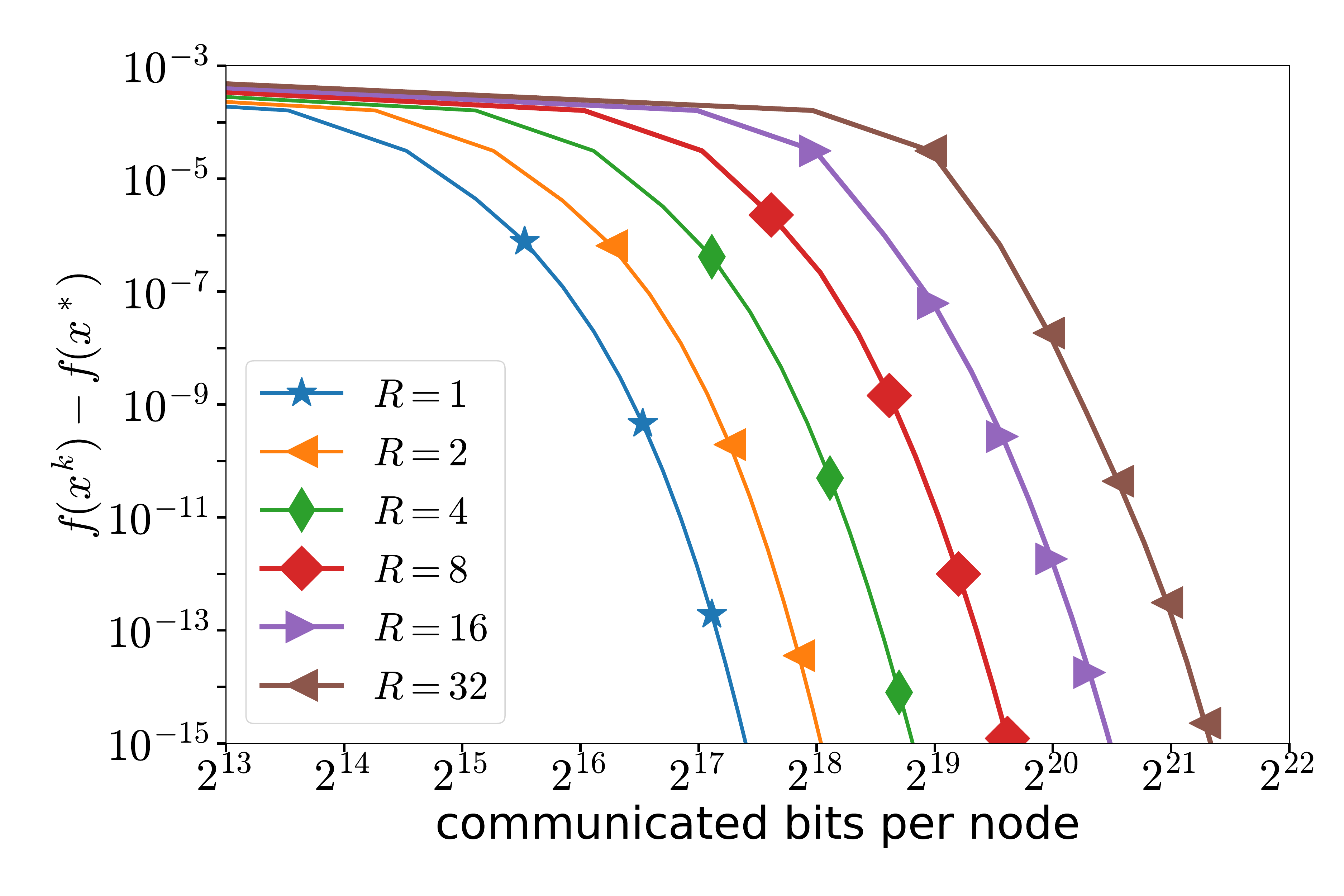}\\
            \dataname{a1a}, $\lambda=10^{-3}$ &
            \dataname{a1a}, $\lambda=10^{-4}$ &
            \dataname{a9a}, $\lambda=10^{-3}$ &
            \dataname{a9a}, $\lambda=10^{-4}$\\
        \end{tabular}
    \end{center}
    \caption{The performance of \algname{FedNL} with different types of compression operators: Rank-$R$ (first row); Top-$K$ (second row); PowerSGD of rank $R$ (third row) for several values of $R$ and $K$ in terms of communication complexity.}
    \label{fig:comp_operators}
\end{figure}

\subsection{Comparison of Options $1$ and $2$}

In our next experiment we investigate which Option ($1$ or $2$) for \algname{FedNL} with Rank-$R$ and stepsize $\alpha=1$ compressor demonstrates better results in terms of communication compexity. According to the results in Figure~\ref{fig:options}, we see that \algname{FedNL} with projection (Option $1$) is more communication effective than that with Option $2$. However, Option $1$ requires more computing resources.

\begin{figure}
    \begin{center}
        \begin{tabular}{cccc}
            \includegraphics[width=0.22\linewidth]{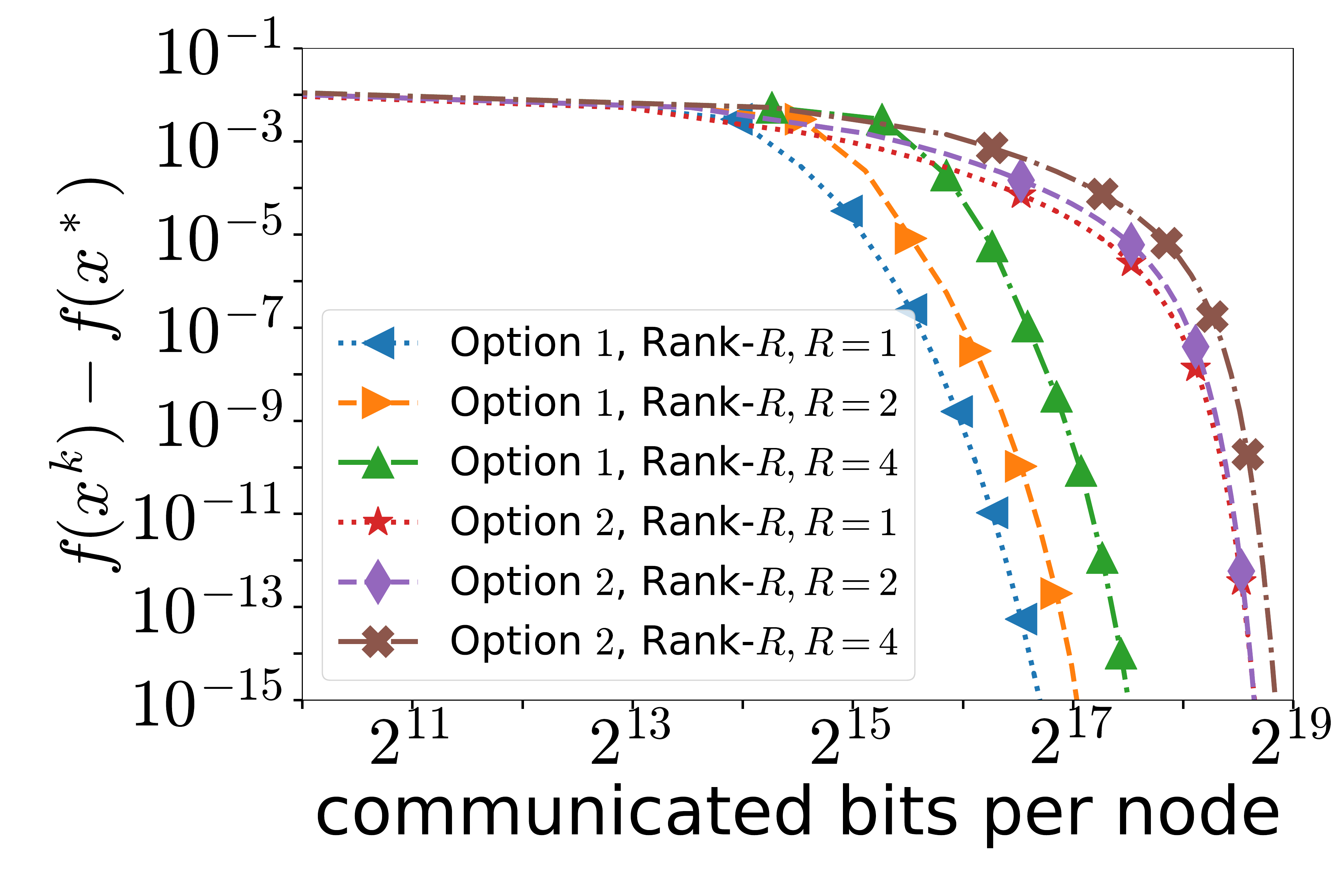} &
            \includegraphics[width=0.22\linewidth]{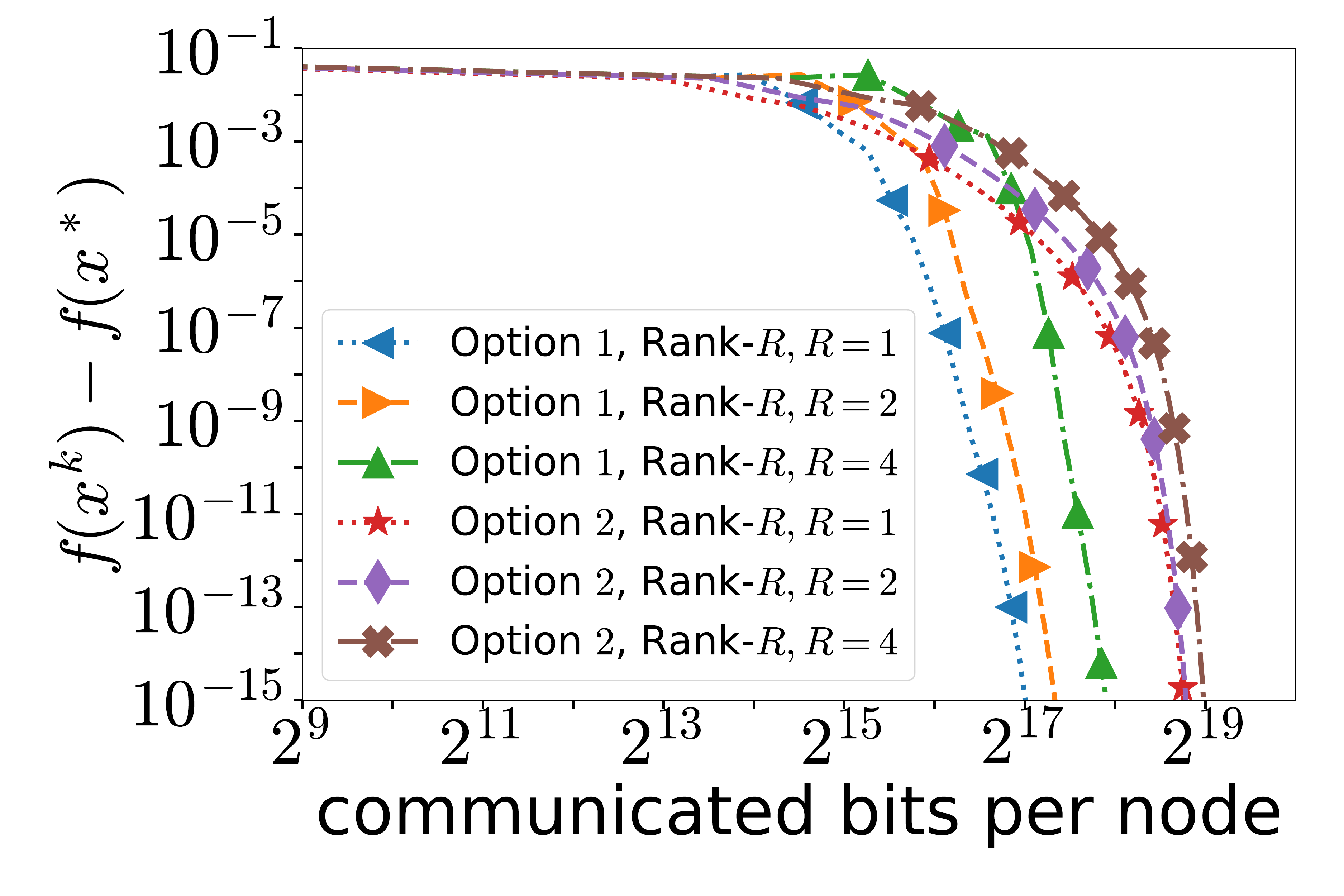} &
            \includegraphics[width=0.22\linewidth]{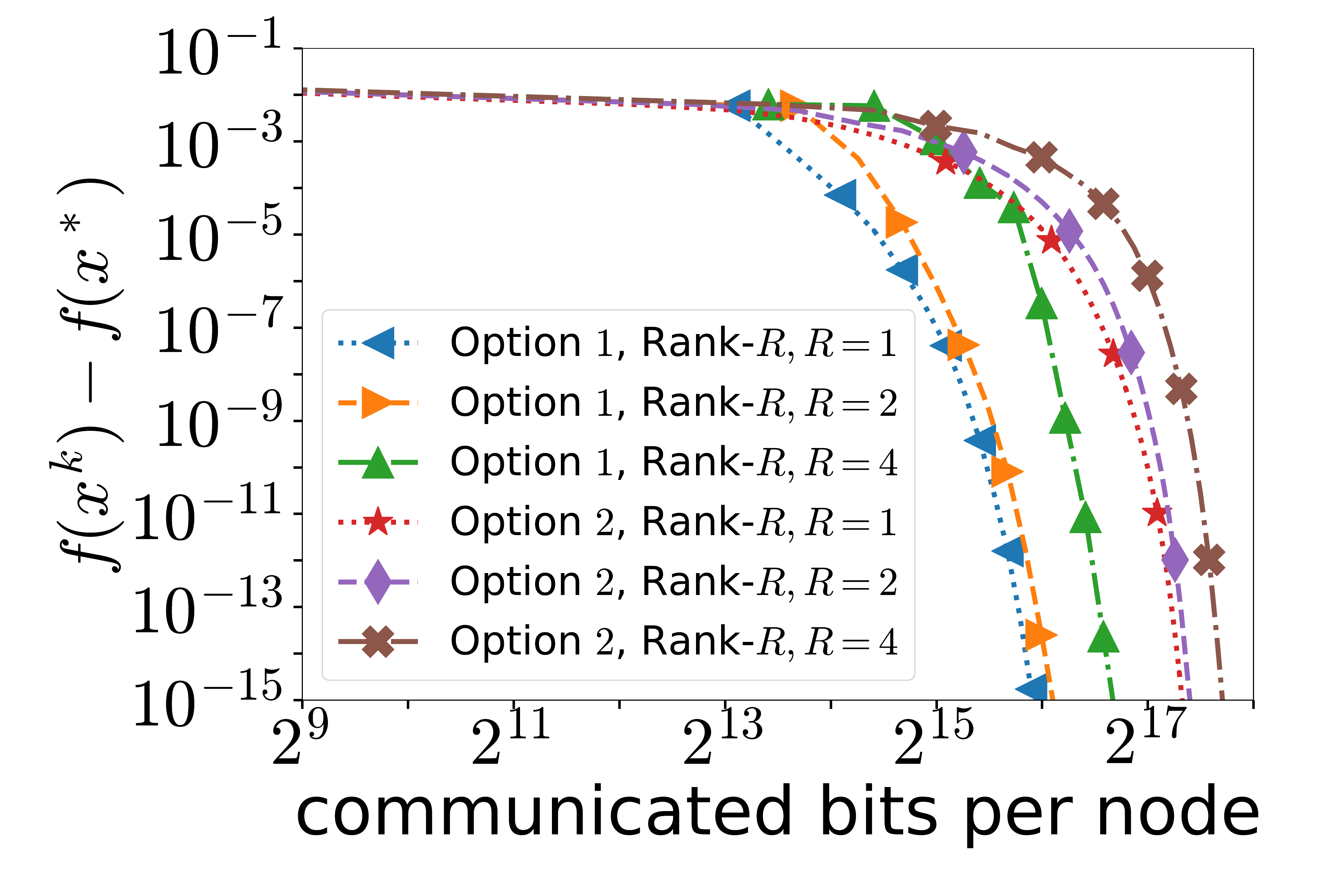} &
            \includegraphics[width=0.22\linewidth]{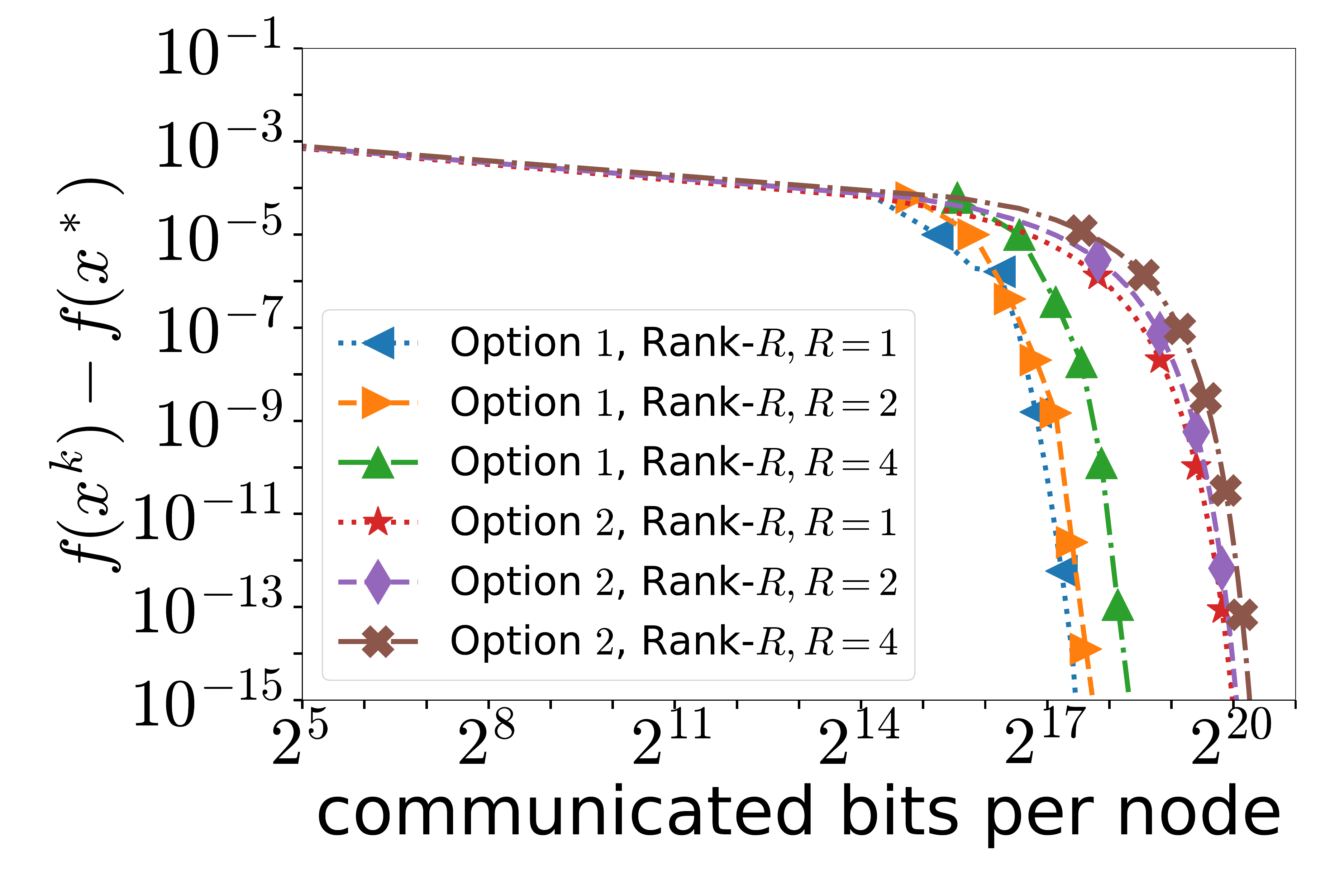}\\
            \dataname{a1a}, $\lambda=10^{-3}$ &
            \dataname{a9a}, $\lambda=10^{-3}$ &
            \dataname{phishing}, $\lambda=10^{-4}$ &
            \dataname{w7a}, $\lambda=10^{-4}$\\
        \end{tabular}
    \end{center}
    \caption{The performance of \algname{FedNL} with Options $1$ and $2$ in terms of communication complexity.}
    \label{fig:options}
\end{figure}

\subsection{Comparison of different compression operators}

Next, we study which compression operator is better in terms of communication complexity. Based on the results in Figure~\ref{fig:diff_compressors}, we can conclude that Rank-$R$ is the best compression operator; Top-$K$ and PowerSGD compressors can beat each other in different cases.

\begin{figure}
    \begin{center}
        \begin{tabular}{cccc}
            \includegraphics[width=0.22\linewidth]{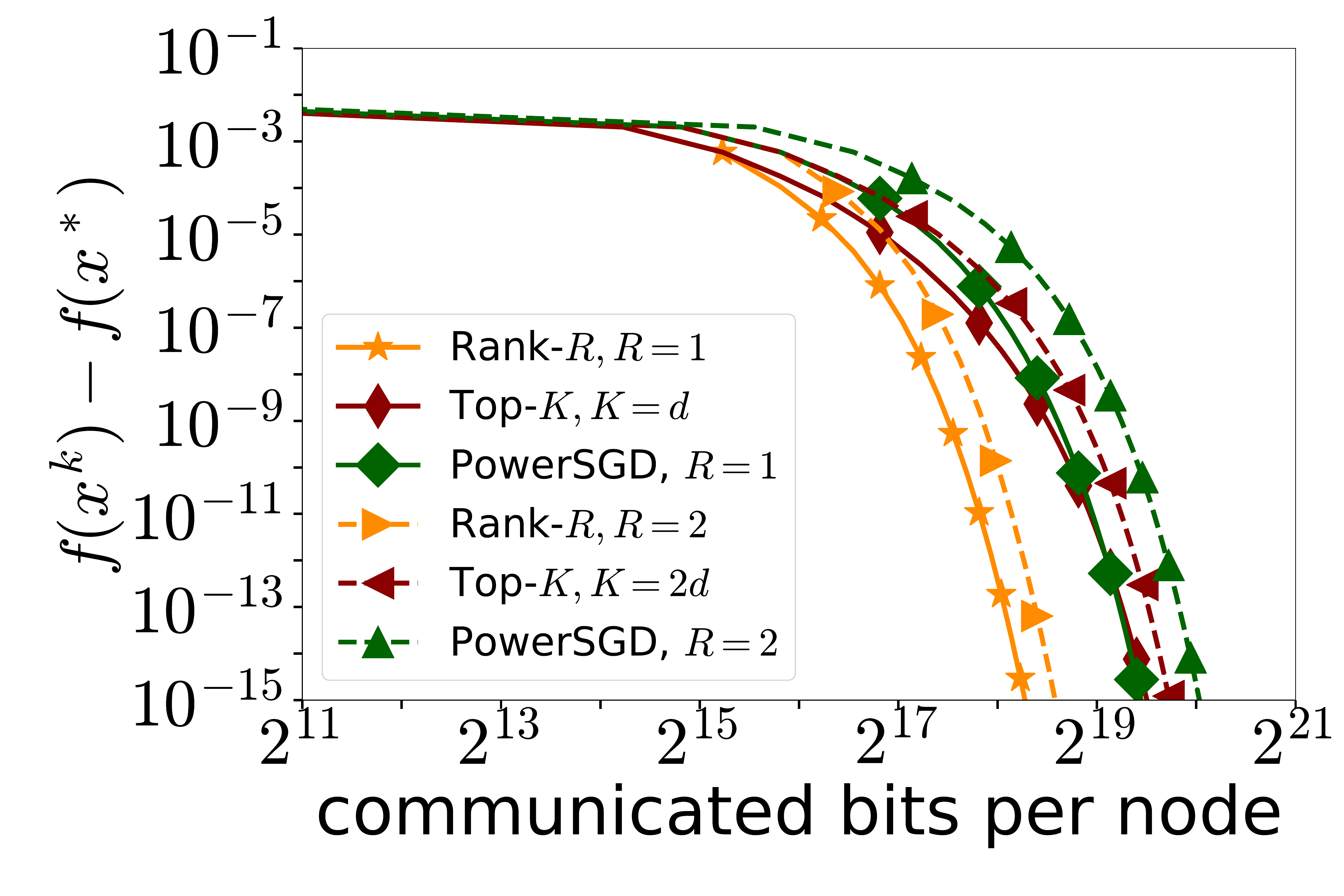} &
            \includegraphics[width=0.22\linewidth]{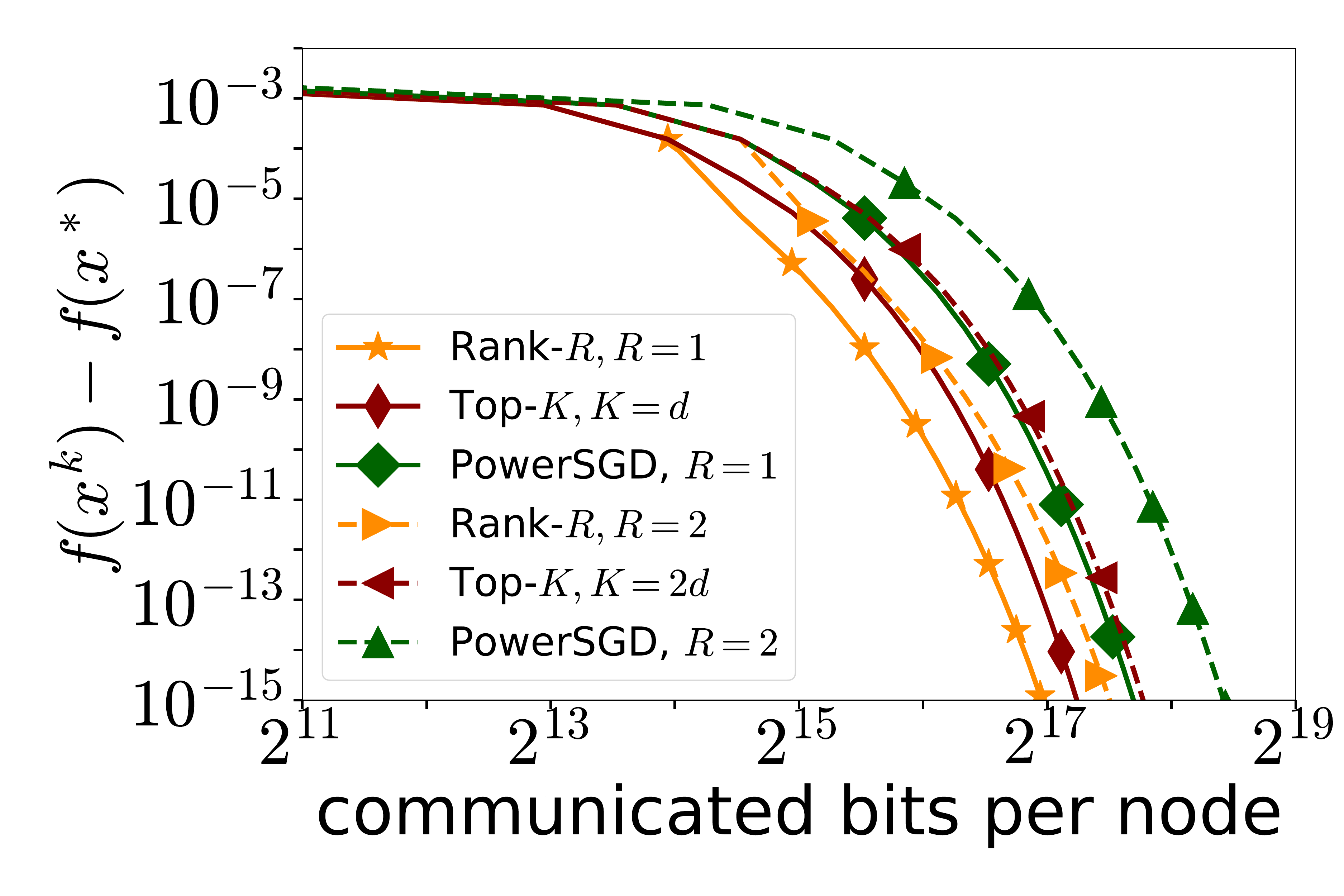} &
            \includegraphics[width=0.22\linewidth]{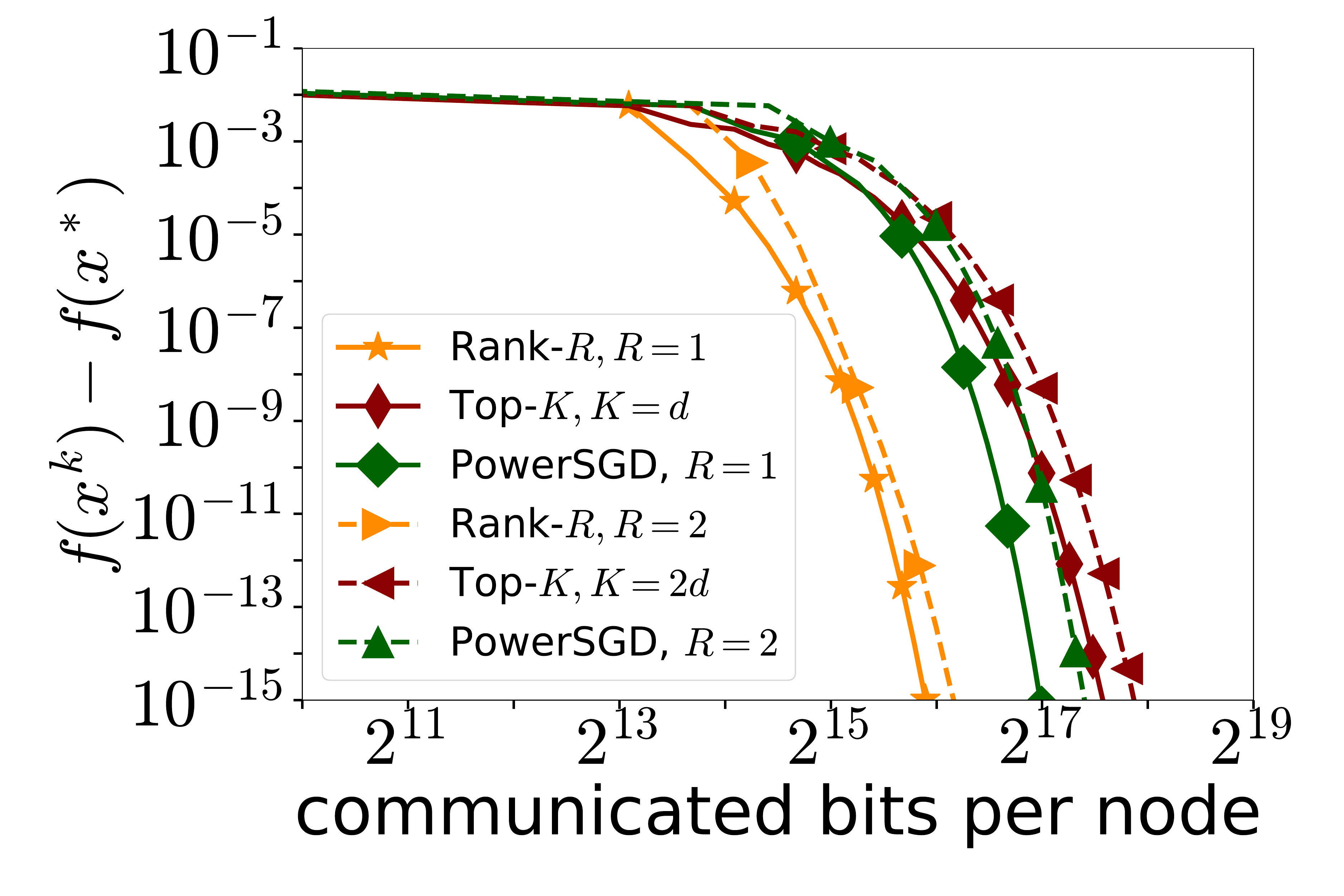} &
            \includegraphics[width=0.22\linewidth]{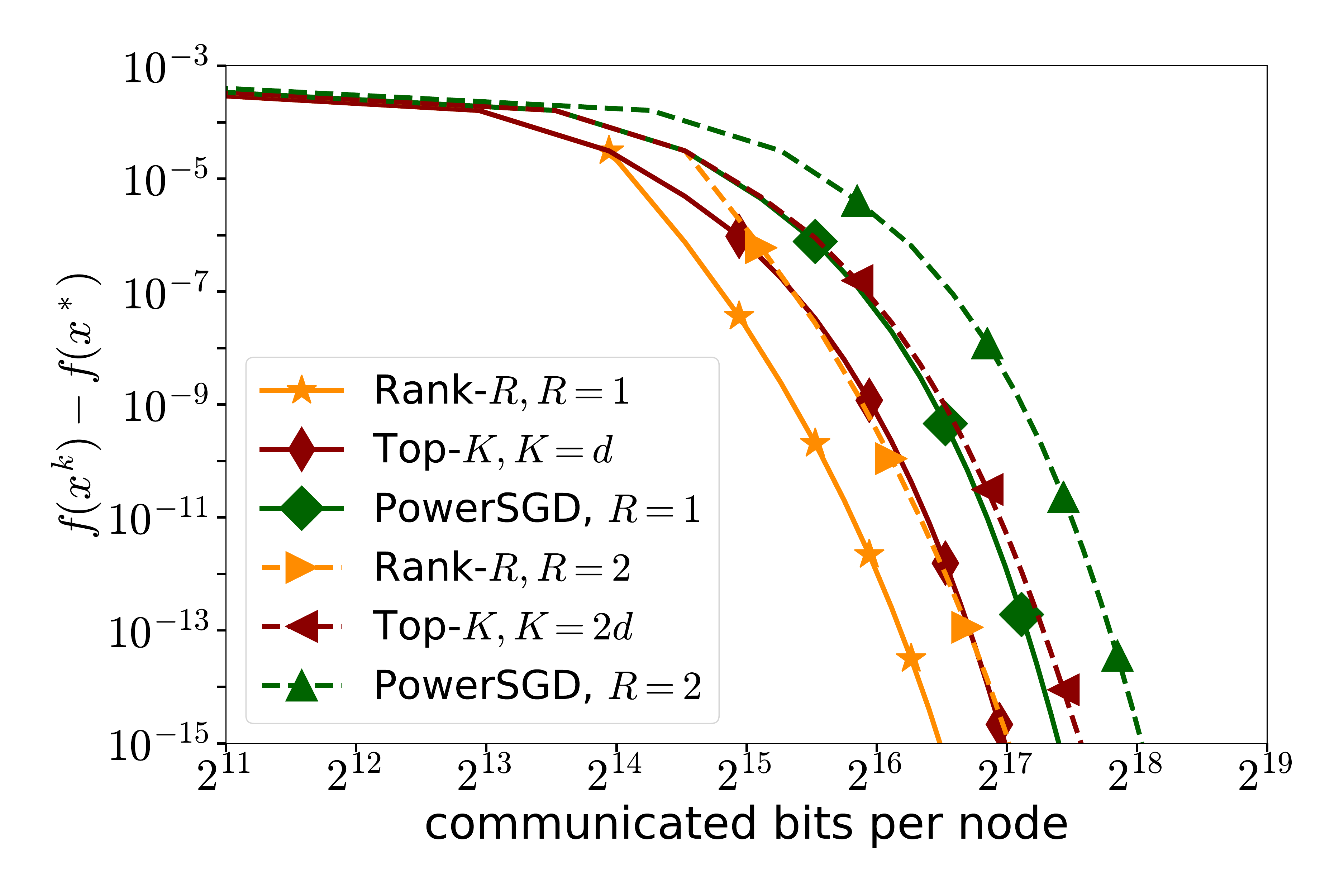}\\
            \dataname{w8a}, $\lambda=10^{-3}$ &
            \dataname{a1a}, $\lambda=10^{-4}$ &
            \dataname{phishing}, $\lambda=10^{-3}$ &
            \dataname{a9a}, $\lambda=10^{-4}$\\
        \end{tabular}
    \end{center}
    \caption{Comparison of the performance of \algname{FedNL} with different compression operators in terms of communication complexity.}
    \label{fig:diff_compressors}
\end{figure}

\subsection{Comparison of different update rules for Hessians}

On the following step we compare \algname{FedNL} with three update rules for Hessians in order to find the best one. They are biased Top-$K$ compression operator with stepsize $\alpha=1$ (Option $1$); biased Top-$K$ compression operator with stepsize $\alpha=1-\sqrt{1-\delta}$; unbiased Rand-$K$ compression operator with stepsize $\alpha=\frac{1}{\omega+1}$. The results of this experiment are presented in Figure~\ref{FIG:FedNL_3_upd_rules}. Based on them, we can make a conclusion that \algname{FedNL} with Top-$K$ compressor and stepsize $\alpha=1$ demonstrates the best performance. \algname{FedNL} with Rand-$K$ compressor and stepsize $\alpha=\frac{1}{\omega+1}$ performs a little bit better than that with Top-$K$ compressor and stepsize $\alpha=1-\sqrt{1-\delta}$. As a consequence, we will use biased compression operator with stepsize $\alpha=1$ for \algname{FedNL} in further experiments.

\begin{figure}
    \begin{center}
        \begin{tabular}{cccc}
            \includegraphics[width=0.22\linewidth]{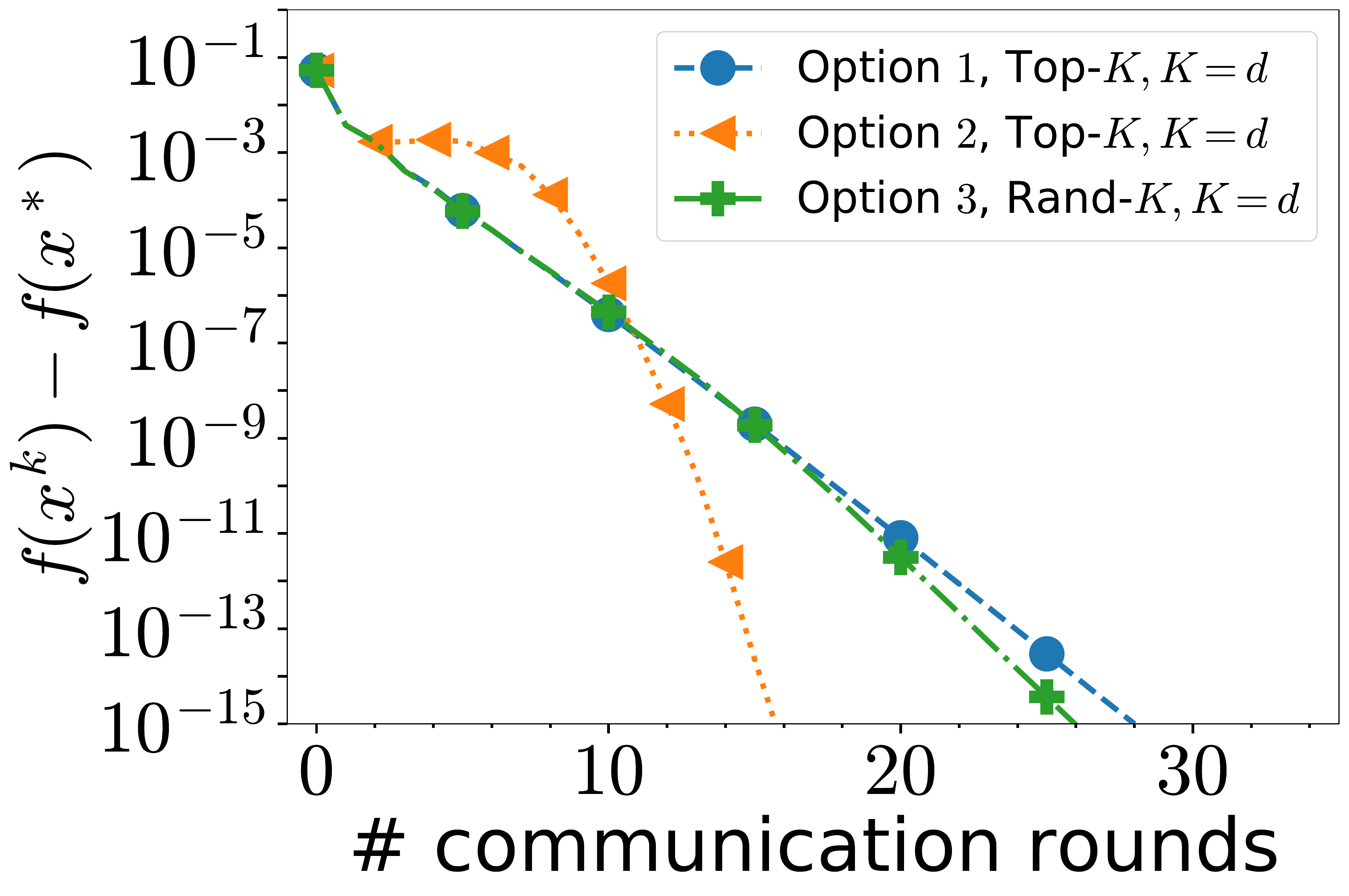} &
            \includegraphics[width=0.22\linewidth]{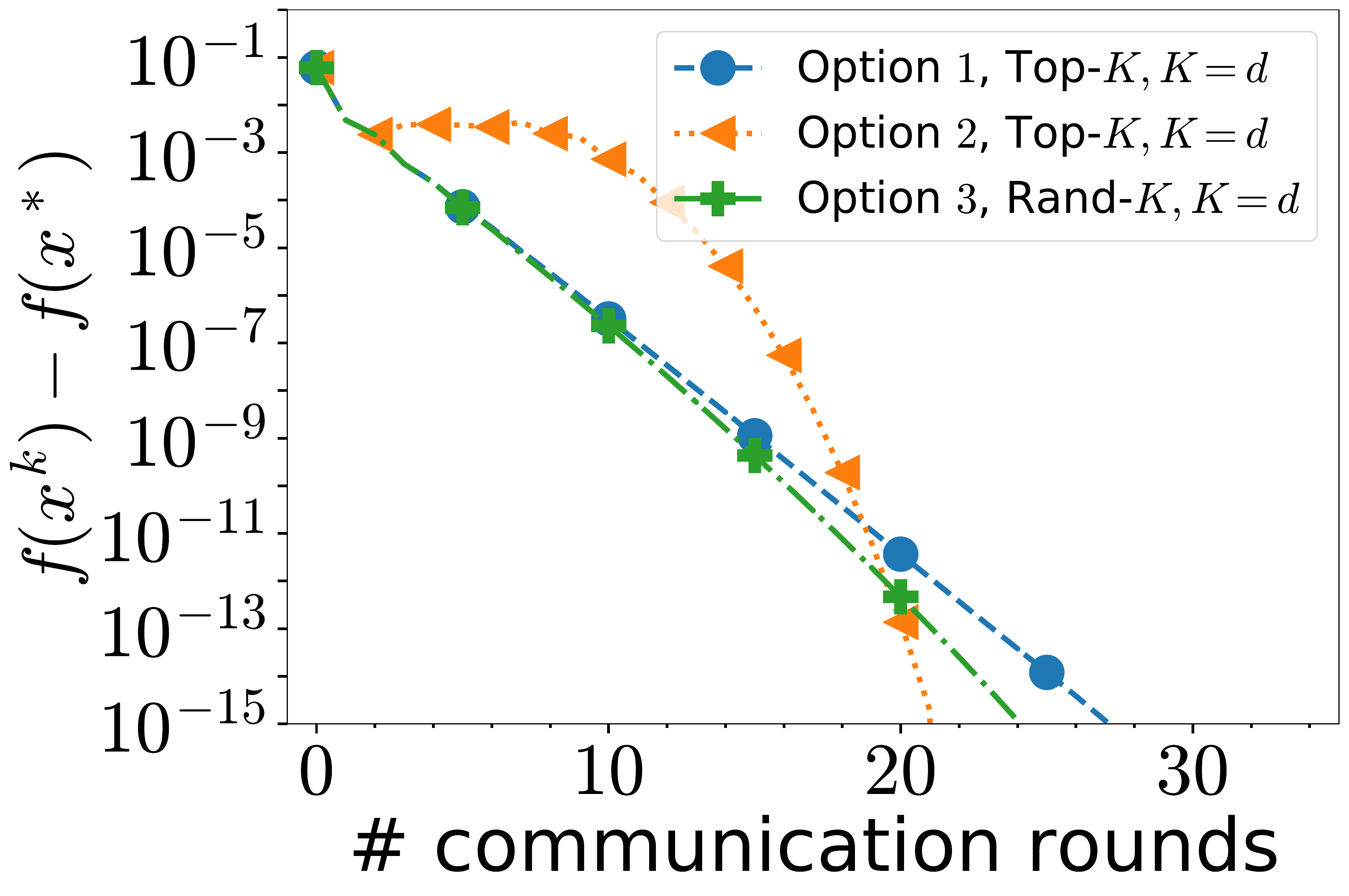} &
            \includegraphics[width=0.22\linewidth]{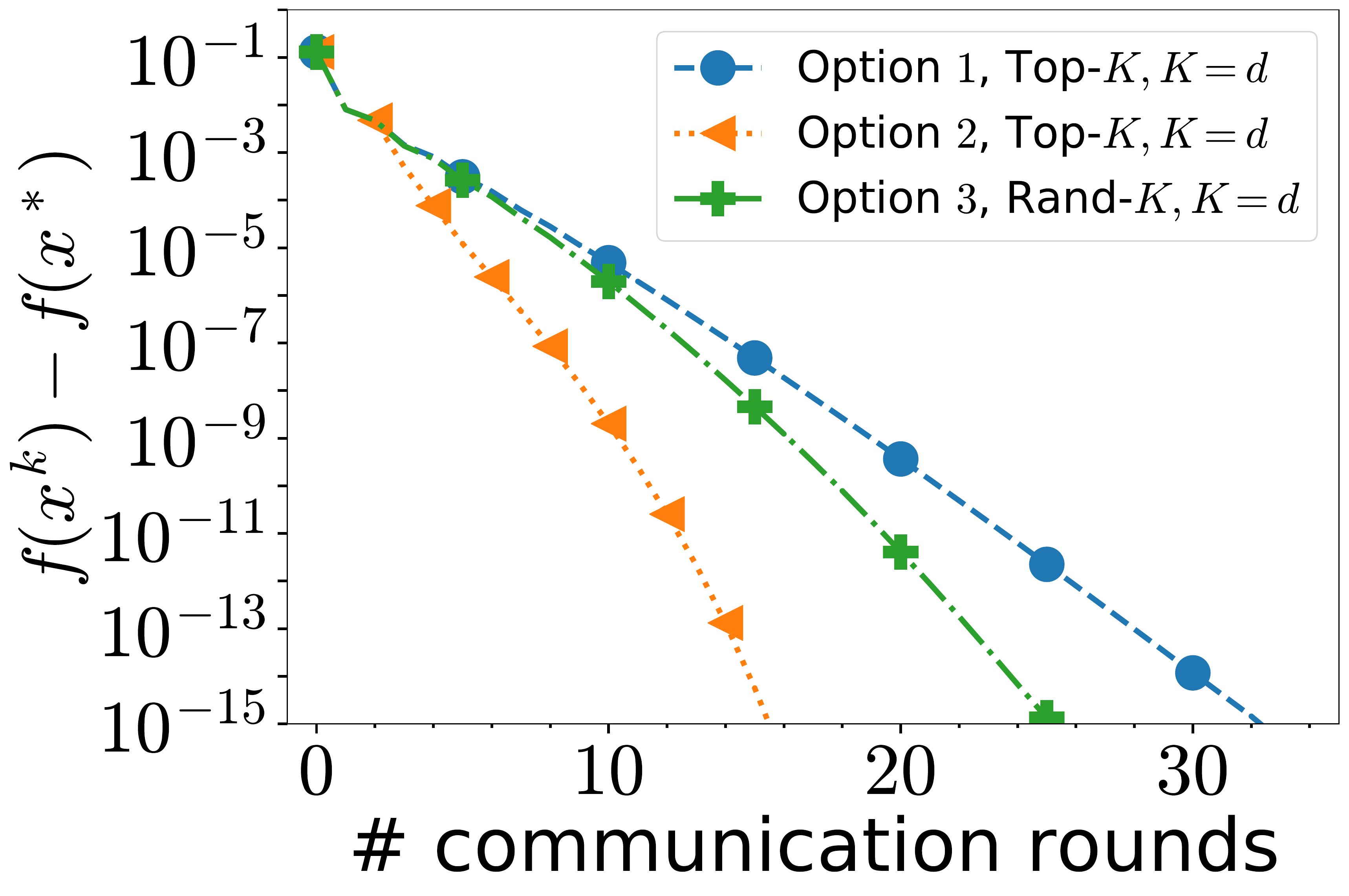} &
            \includegraphics[width=0.22\linewidth]{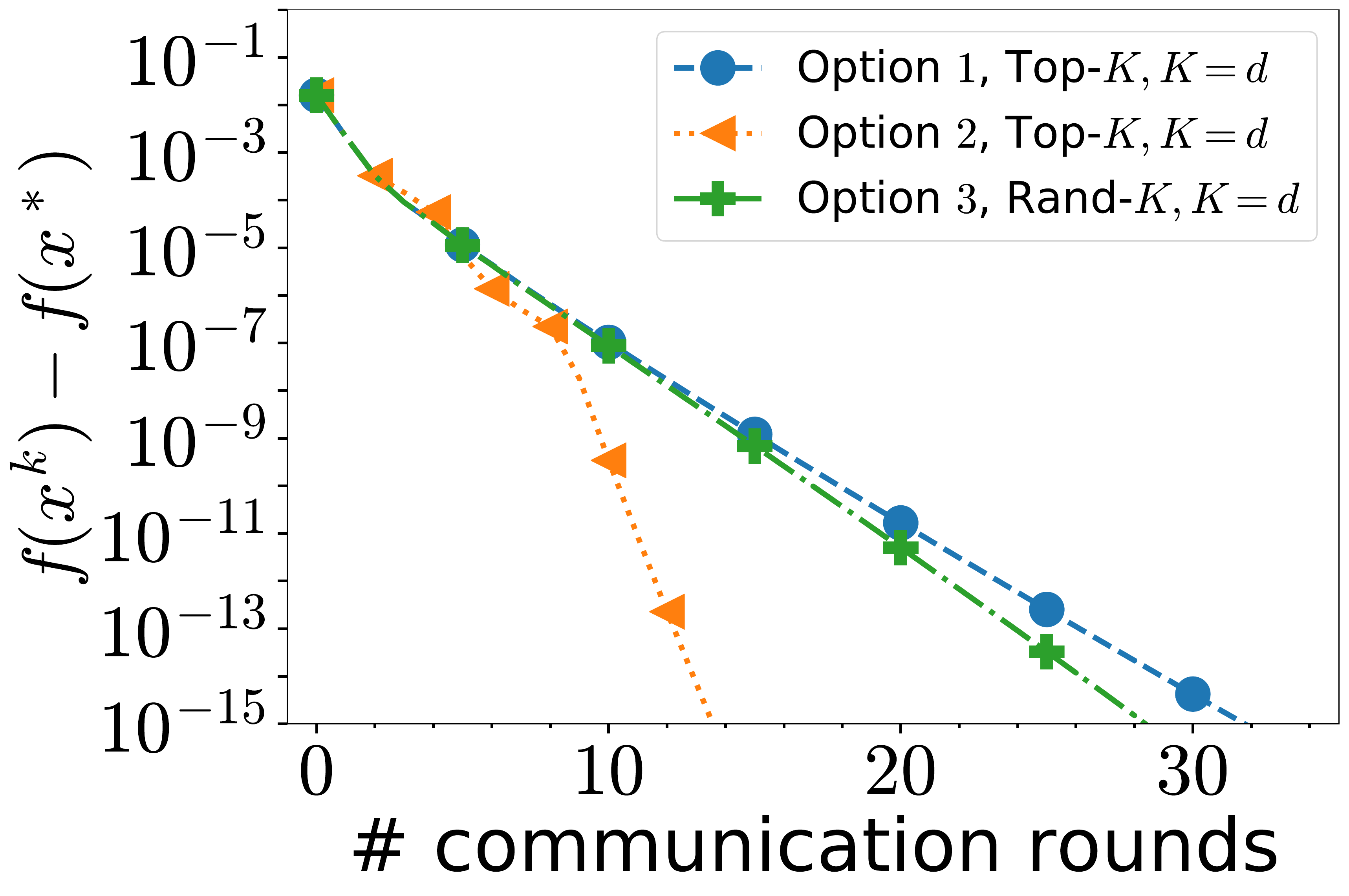}\\
            (a) \dataname{a1a}, $\lambda=10^{-3}$ &
            (b) \dataname{a9a}, $\lambda=10^{-3}$ &
            (c) \dataname{phishing}, $\lambda=10^{-3}$ &
            (d) \dataname{w7a}, $\lambda=10^{-3}$\\
        \end{tabular}
    \end{center}
    \caption{Comparison of \algname{FedNL} with three update rules: Top-$K, \alpha=1-\sqrt{1-\delta}$ (Option $1$); Top-$K, \alpha=1$ (Option $2$); Rand-$K, \alpha=\frac{1}{\omega+1}$ (Option $3$) in terms of iteration complexity.}
    \label{FIG:FedNL_3_upd_rules}
\end{figure}

\subsection{Bidirectional compression}

Now we study how the performance of \algname{FedNL-BC} (with Option $1$ and stepsize $\alpha=1$) is affected by the level of compression in Figure~\ref{FIG:FedNL_BC}. Here we use Top-$K$ compressor for Hessians and models, and broadcast gradients with probability $p$. In order to make the results more interpretable, we set $K$ to be $pd$, then we carry out experiments for several values of $p$. We clearly see that deep compression ($p=0.5; 0.6$) influences negatively the performance of \algname{FedNL-BC}. However, small compression ($p=0.9$) can be beneficial in some cases (see Figure~\ref{FIG:FedNL_BC}: (b), (d)), but this is not the case for Figure~\ref{FIG:FedNL_BC}: (a), (c), where the best performance is demonstrated by \algname{FedNL-BC} with $p=1$. We can conclude that only weak compression (the value of $p$ is close to $1$) can improve the performance of \algname{FedNL-BC}, but the improvement is relatively small.

We also compare \algname{FedNL-BC} (compression was described above, Option $2$ was used in the experiments) with \algname{DORE} method \citep{liu2019double}. This method applies bi-directional  compression on gradients (uplink compression) and models (downlink compression). All constants for this method were chosen according theoretical results in the paper. We use random dithering compressor in both directions ($s=\sqrt{d}$). Based on the numerical experiments in Figure~\ref{FIG:FedNL_vs_DORE}, we can conclude that \algname{FedNL-BC} is much more communication efficient method than \algname{DORE} by many orders in magnitude.

\begin{figure}
    \begin{center}
        \begin{tabular}{cccc}
            \includegraphics[width=0.22\linewidth]{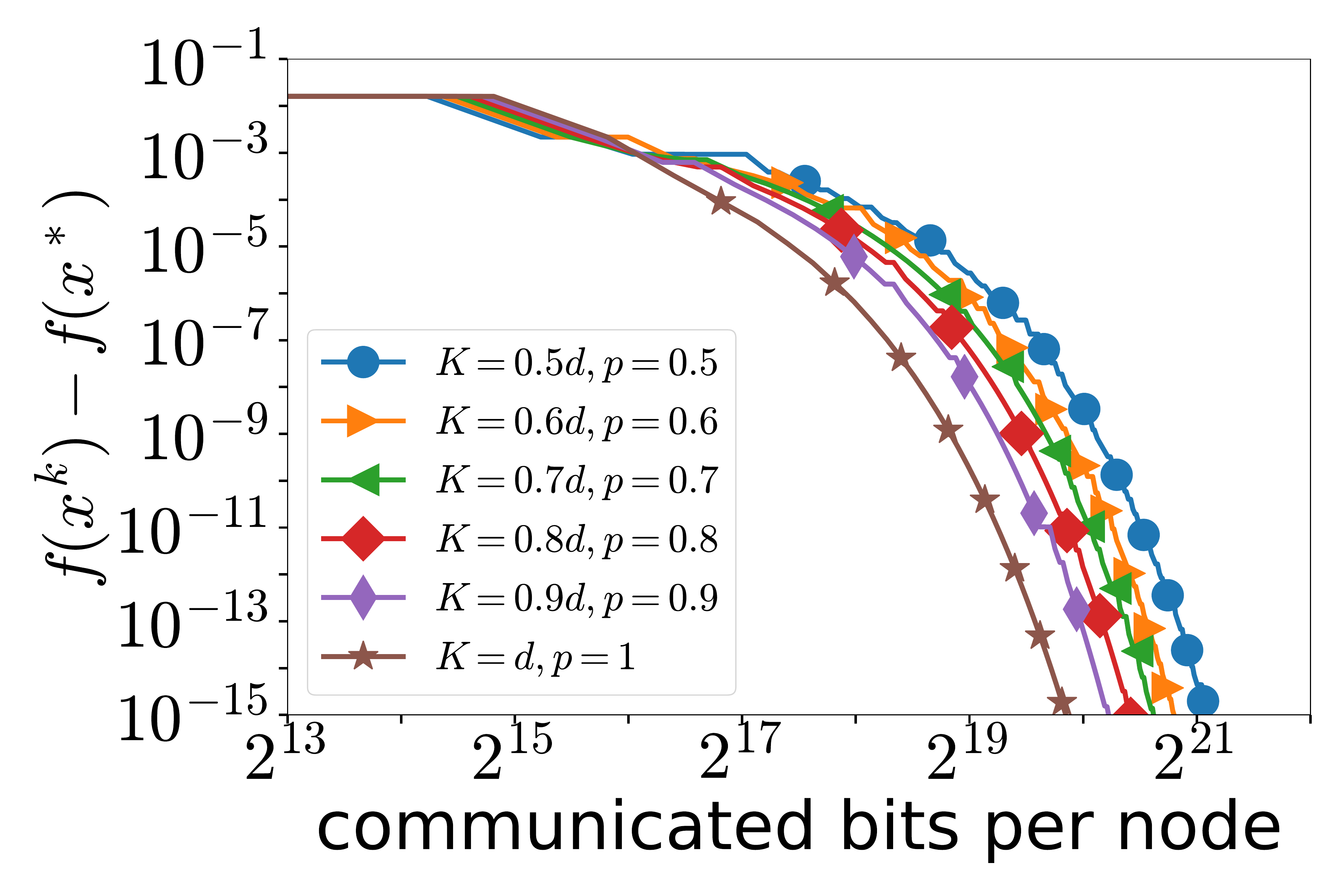} &
            \includegraphics[width=0.22\linewidth]{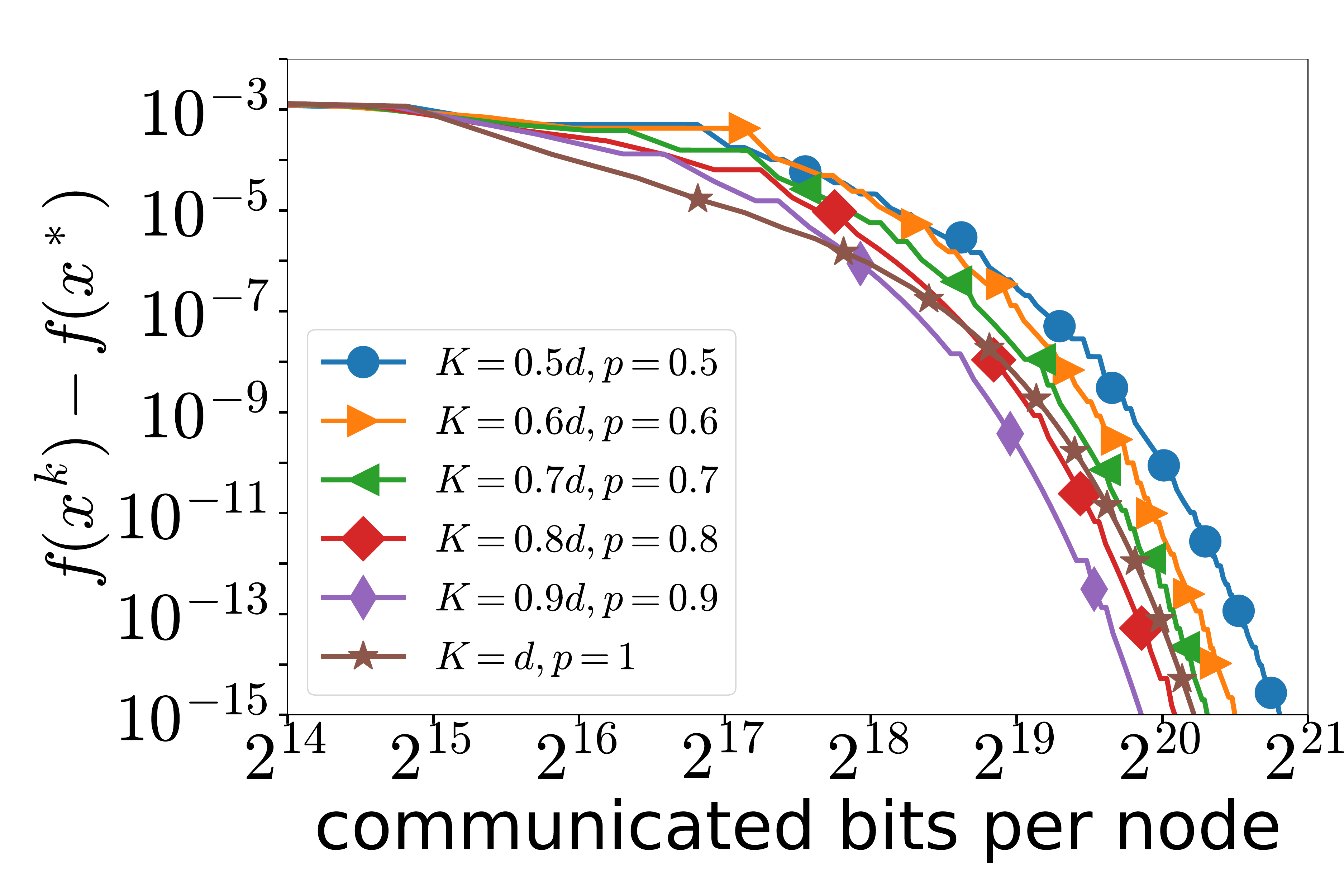} &
            \includegraphics[width=0.22\linewidth]{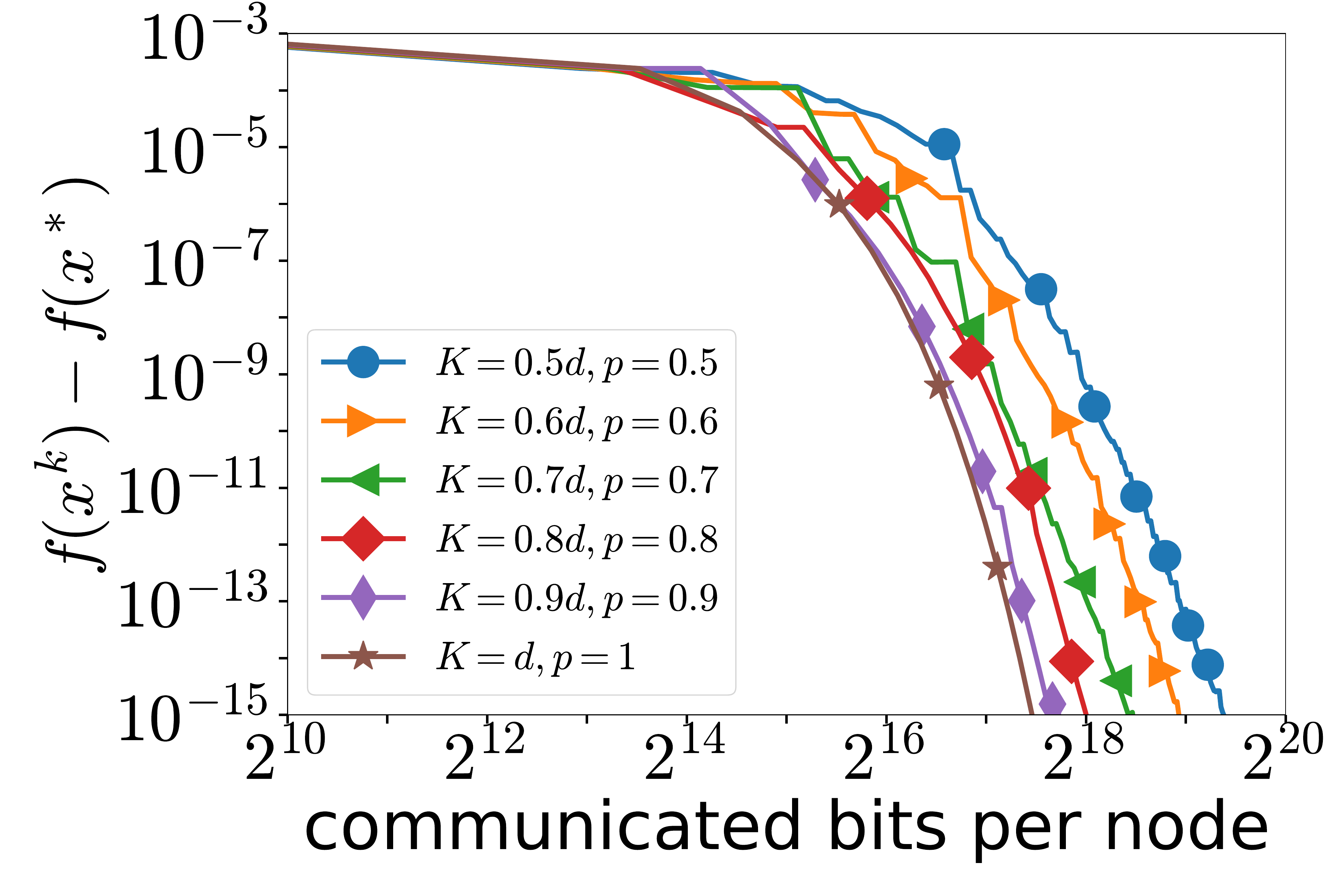} &
            \includegraphics[width=0.22\linewidth]{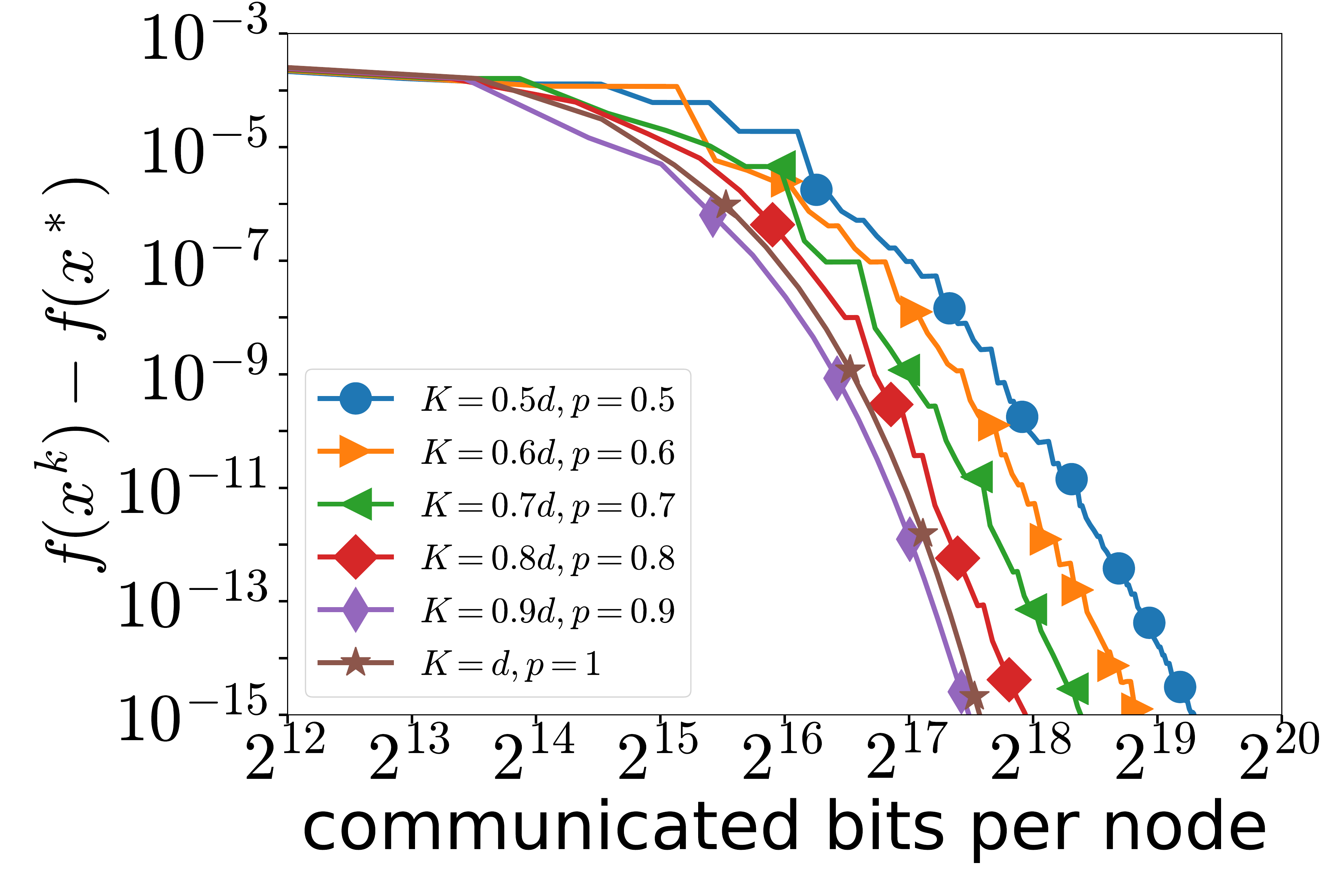}\\
            (a) \dataname{w7a}, $\lambda=10^{-3}$ &
            (b) \dataname{w8a}, $\lambda=10^{-3}$ &
            (c) \dataname{a1a}, $\lambda=10^{-4}$ &
            (d) \dataname{a9a}, $\lambda=10^{-4}$\\
        \end{tabular}
    \end{center}
    \caption{The performance of \algname{FedNL-BC} with Top-$K$ applied to Hessians and models ($K=pd$), and broadcasting gradients with probability $p$ for several values of $p$ in terms of communication complexity.}
    \label{FIG:FedNL_BC}
\end{figure}

\begin{figure}
    \begin{center}
        \begin{tabular}{cccc}
            \multicolumn{4}{c}{
                \includegraphics[width=0.8\linewidth]{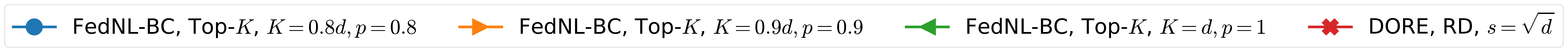}
            }\\
            \includegraphics[width=0.22\linewidth]{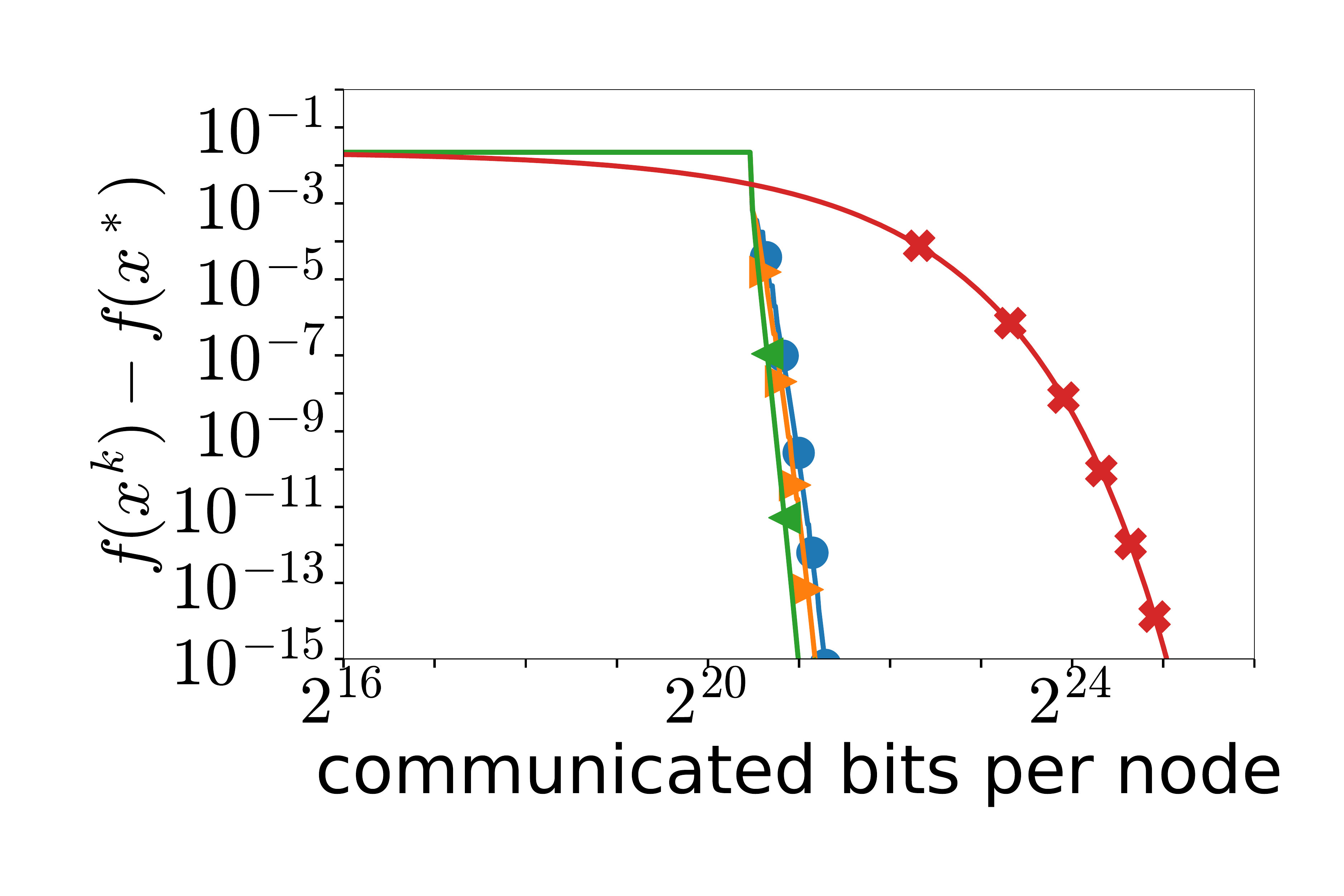} &
            \includegraphics[width=0.22\linewidth]{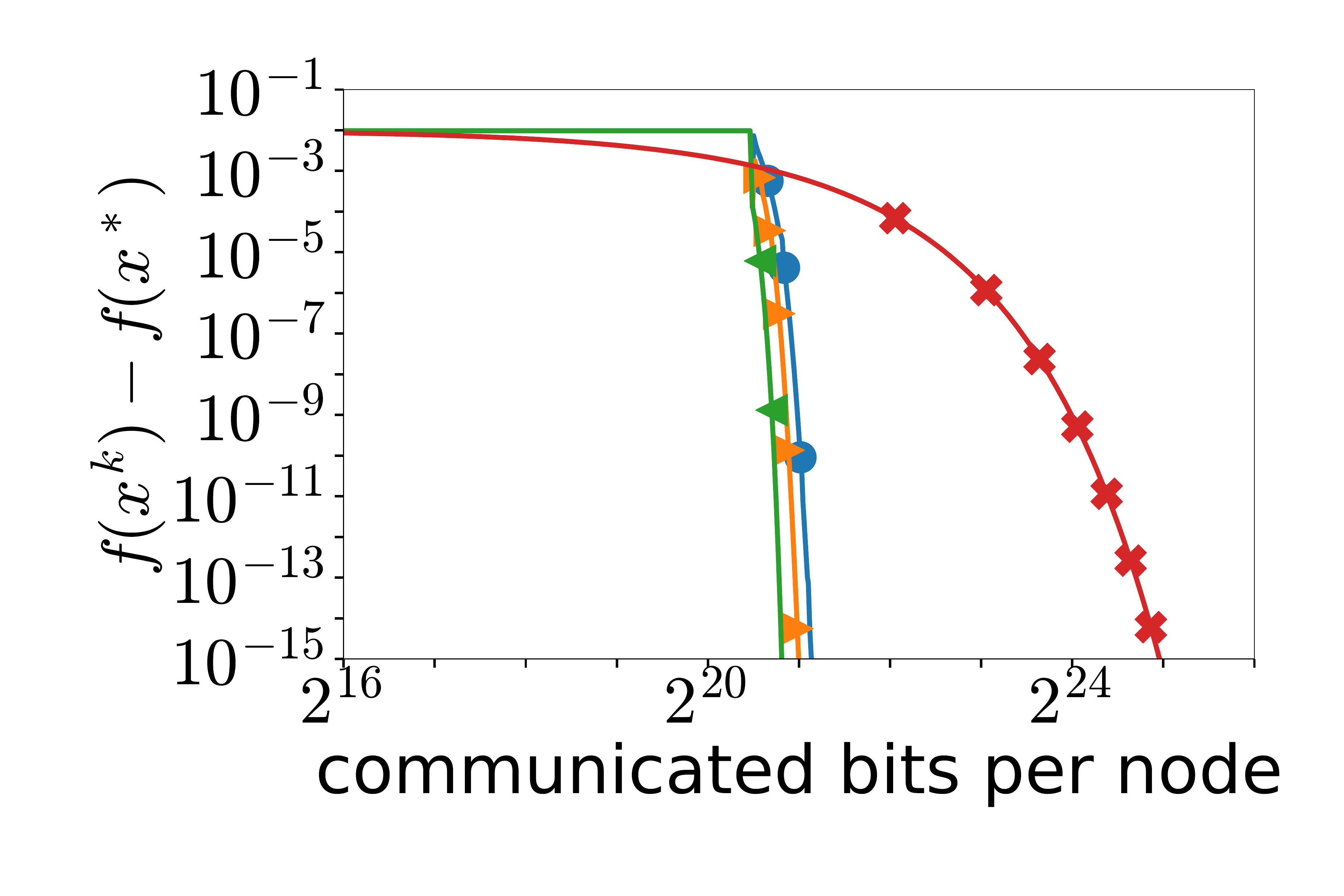} &
            \includegraphics[width=0.22\linewidth]{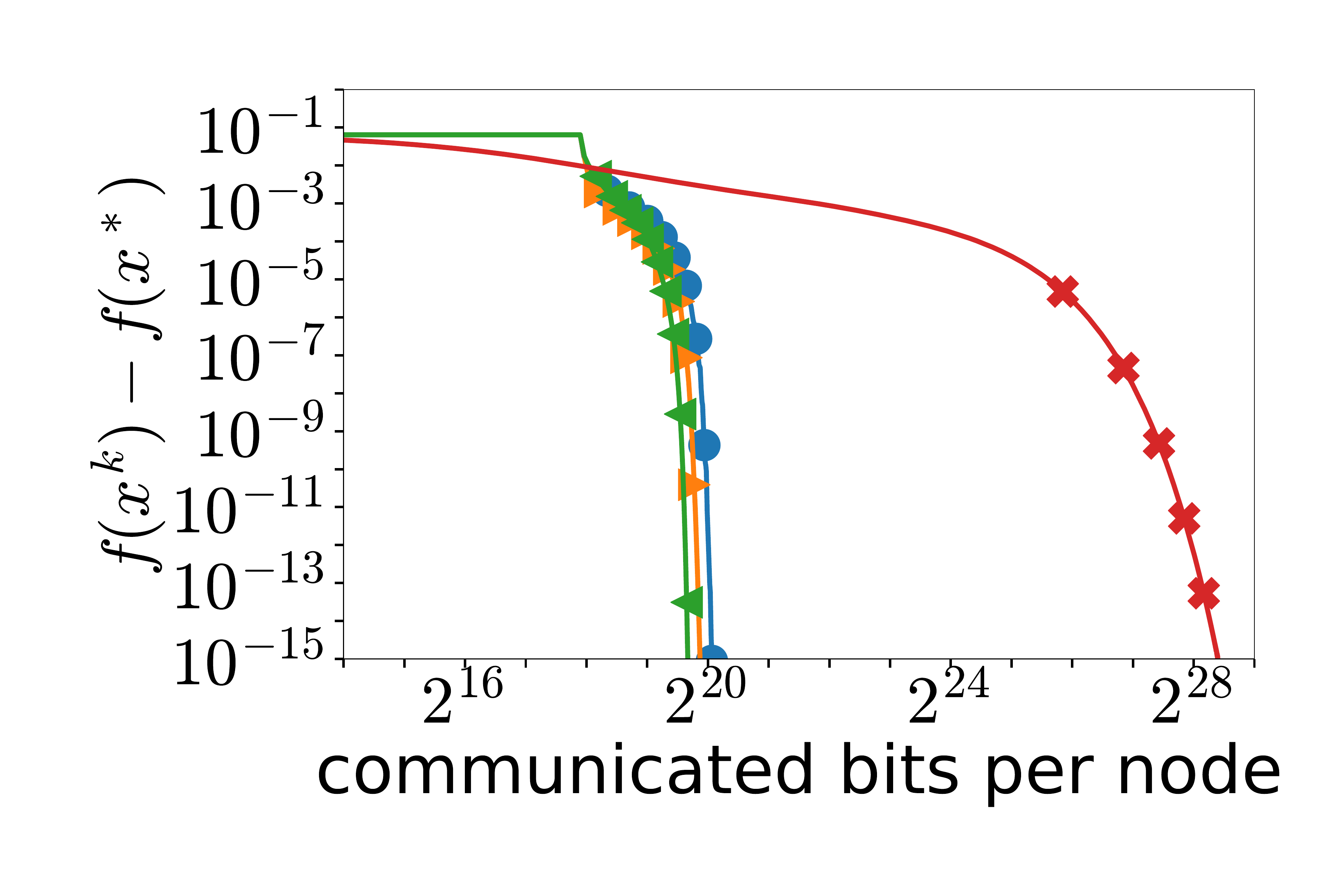} &
            \includegraphics[width=0.22\linewidth]{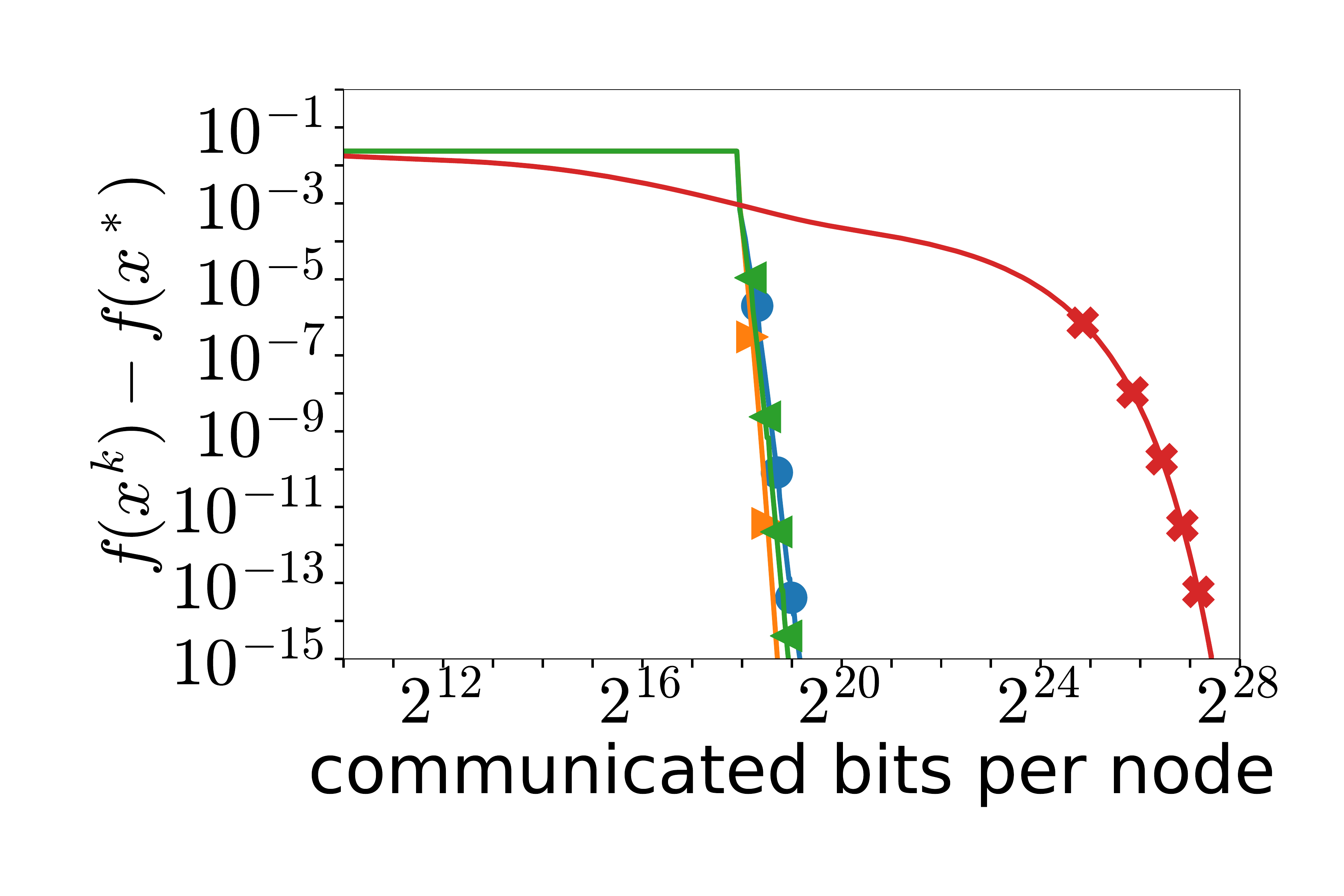}\\
            (a) \dataname{w7a}, $\lambda=10^{-3}$ &
            (b) \dataname{w8a}, $\lambda=10^{-3}$ &
            (c) \dataname{a1a}, $\lambda=10^{-4}$ &
            (d) \dataname{a9a}, $\lambda=10^{-4}$\\
        \end{tabular}
        
    \end{center}
    \caption{Comparison of \algname{FedNL-BC} with Top-$K$ applied to Hessians and models ($K=pd$), and broadcasting gradients with probability $p$ and \algname{DORE} in terms of communication complexity.}
    \label{FIG:FedNL_vs_DORE}
\end{figure}

\subsection{The performance of \algname{FedNL-PP}}

Now we deploy our \algname{FedNL-PP} method in order to study how the performance is inlfuenced by the value of active nodes $\tau$. We use \algname{FedNL-PP} with Rank-$1$ compression operator, and run the method for several values of $\tau$; see Figure~\ref{FIG:FedNL_PP}. As we can see, the smaller value of $\tau$ is, the worse performance of \algname{FedNL-PP} is, as it expected.

Now we compare \algname{FedNL-PP} with \algname{Artemis} \citep{philippenko2021bidirectional} which supports partial participation too. We use random sparsification compressor ($s=\sqrt{d}$) in uplink direction, and the server broadcasts descent direction to each node without compression. All contstants of the method were chosen according theory from the paper. Each node $i$ computes full local gradient $\nabla f_i(x^k)$. We conduct experiments for several number of active nodes: $\tau \in \{0.2n, 0.4n, 0.8n\}$, then we calculate the total number of transmitted bits received by the server from all active nodes. All results are presented in Figure~\ref{FIG:FedNL_vs_Artemis}. We clearly see that \algname{FedNL-PP} outperforms \algname{Artemis} by several orders in magnitude in terms of communication complexity. 

\begin{figure}
    \begin{center}
        \begin{tabular}{cccc}
            \includegraphics[width=0.22\linewidth]{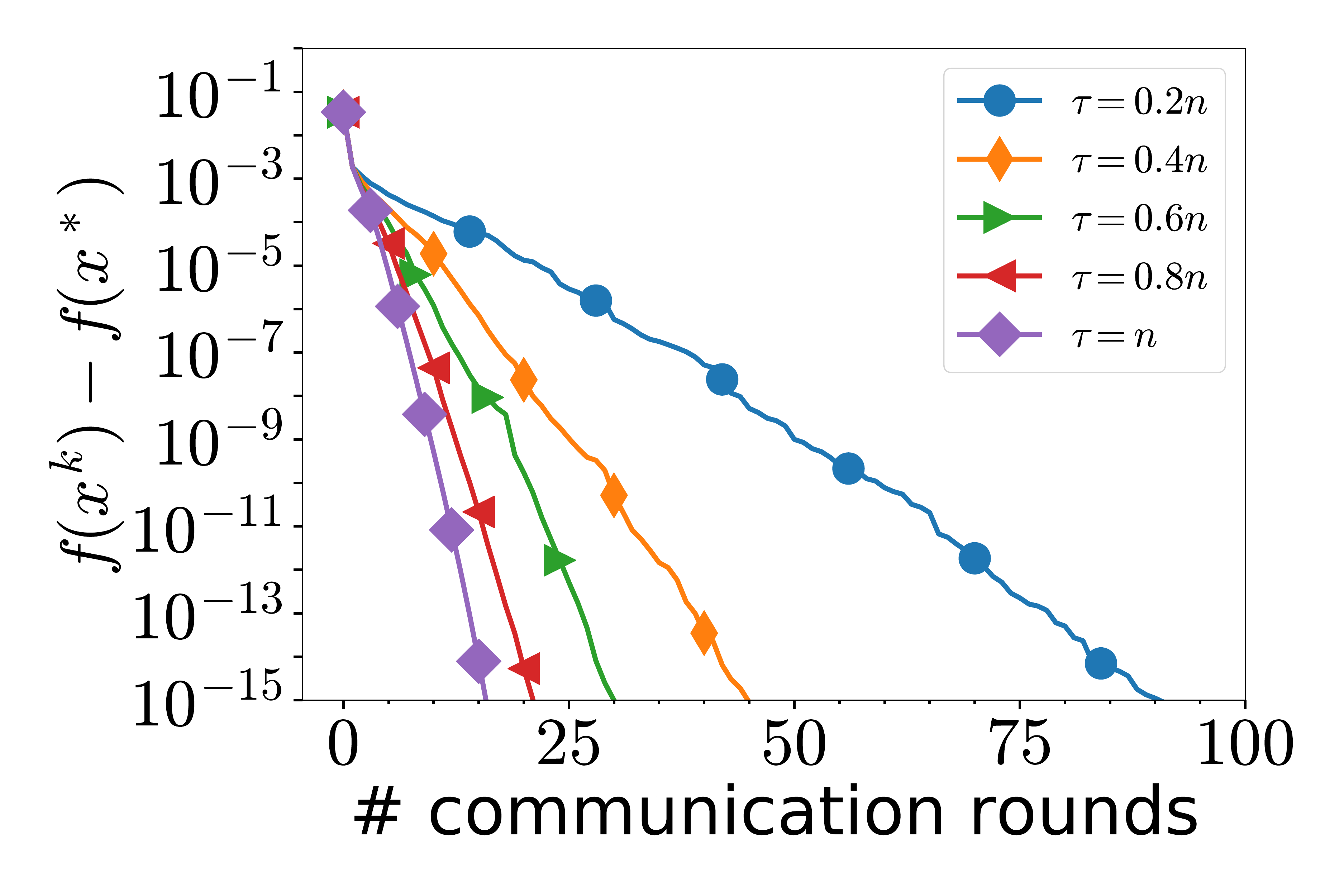} &
            \includegraphics[width=0.22\linewidth]{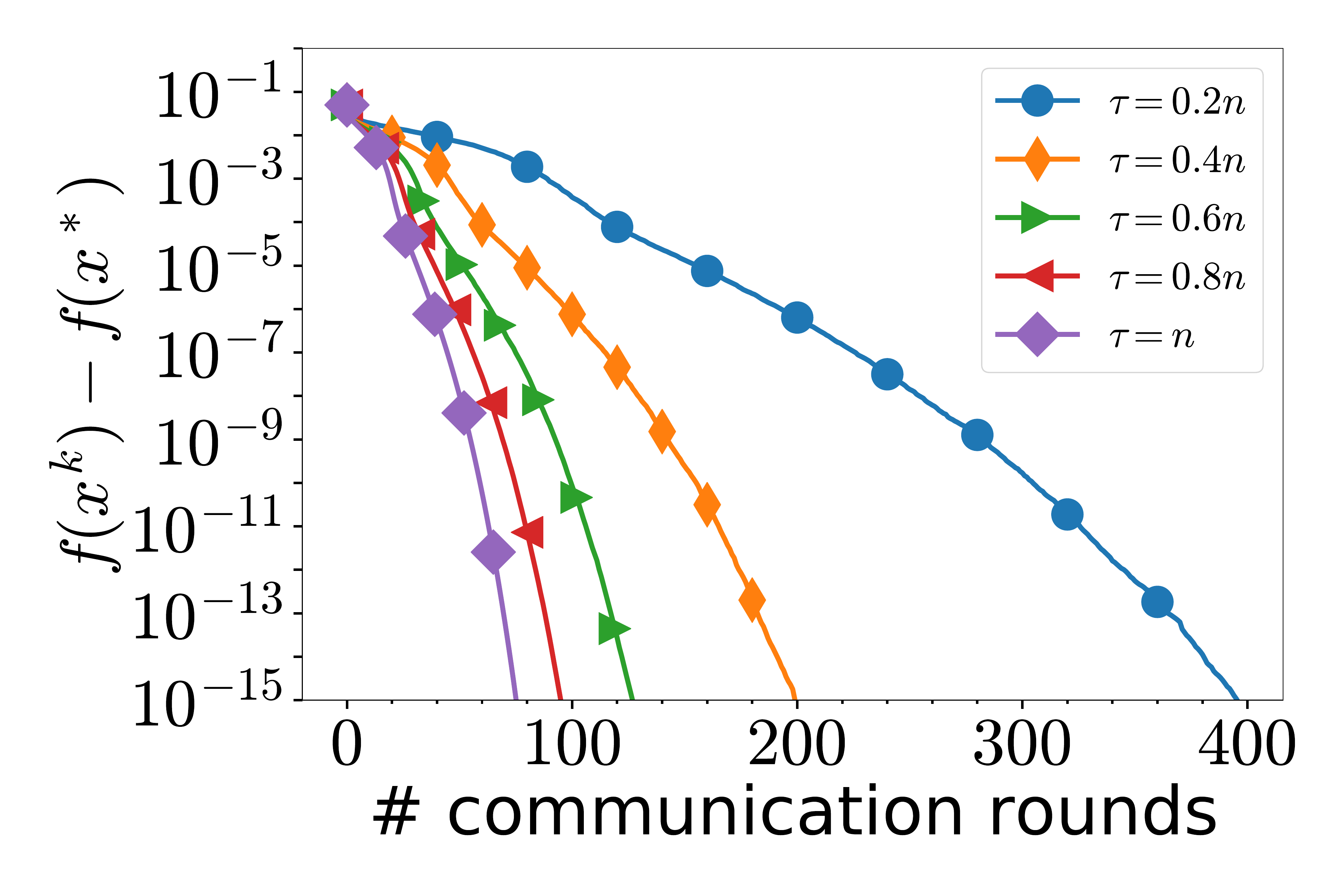} &
            \includegraphics[width=0.22\linewidth]{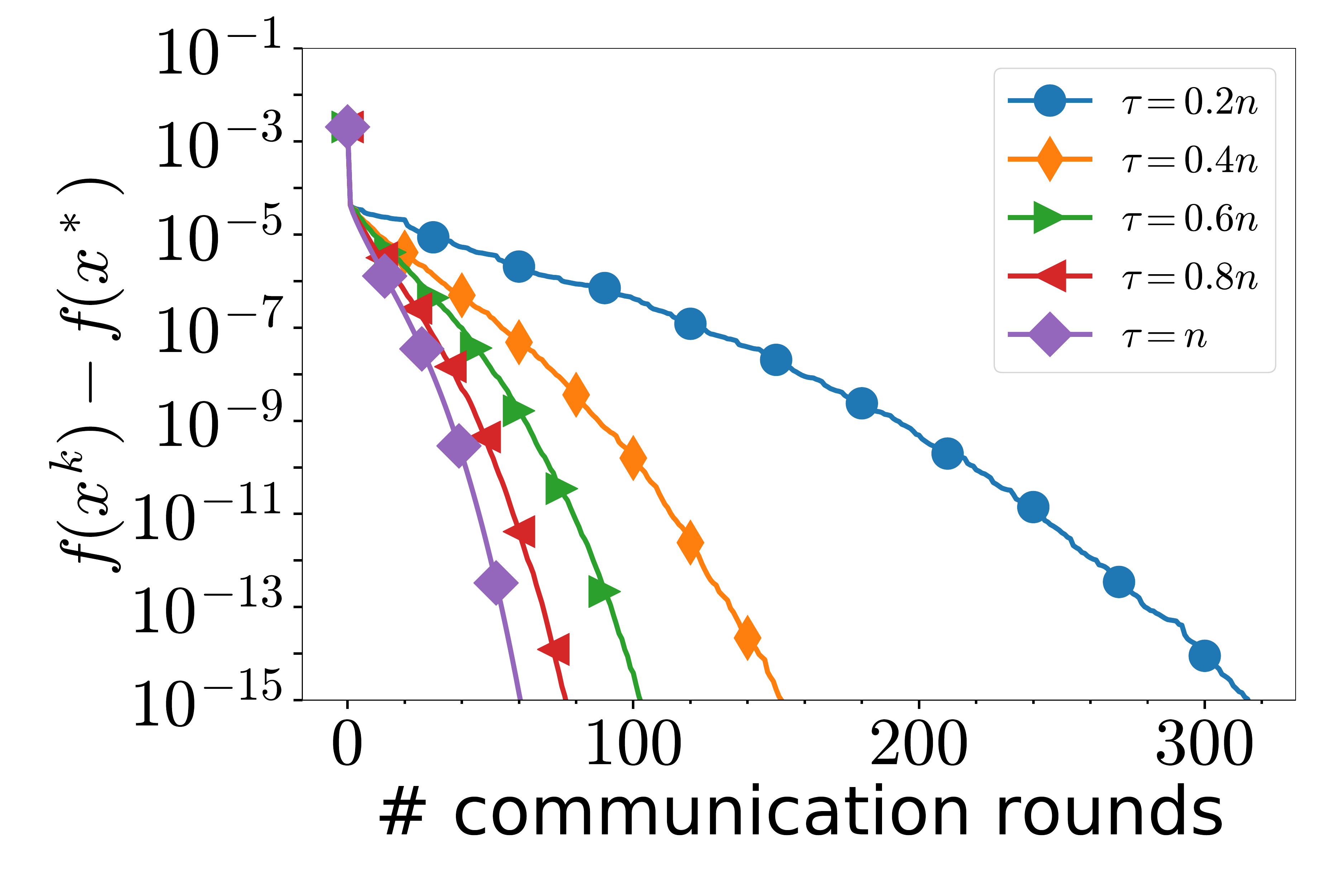} &
            \includegraphics[width=0.22\linewidth]{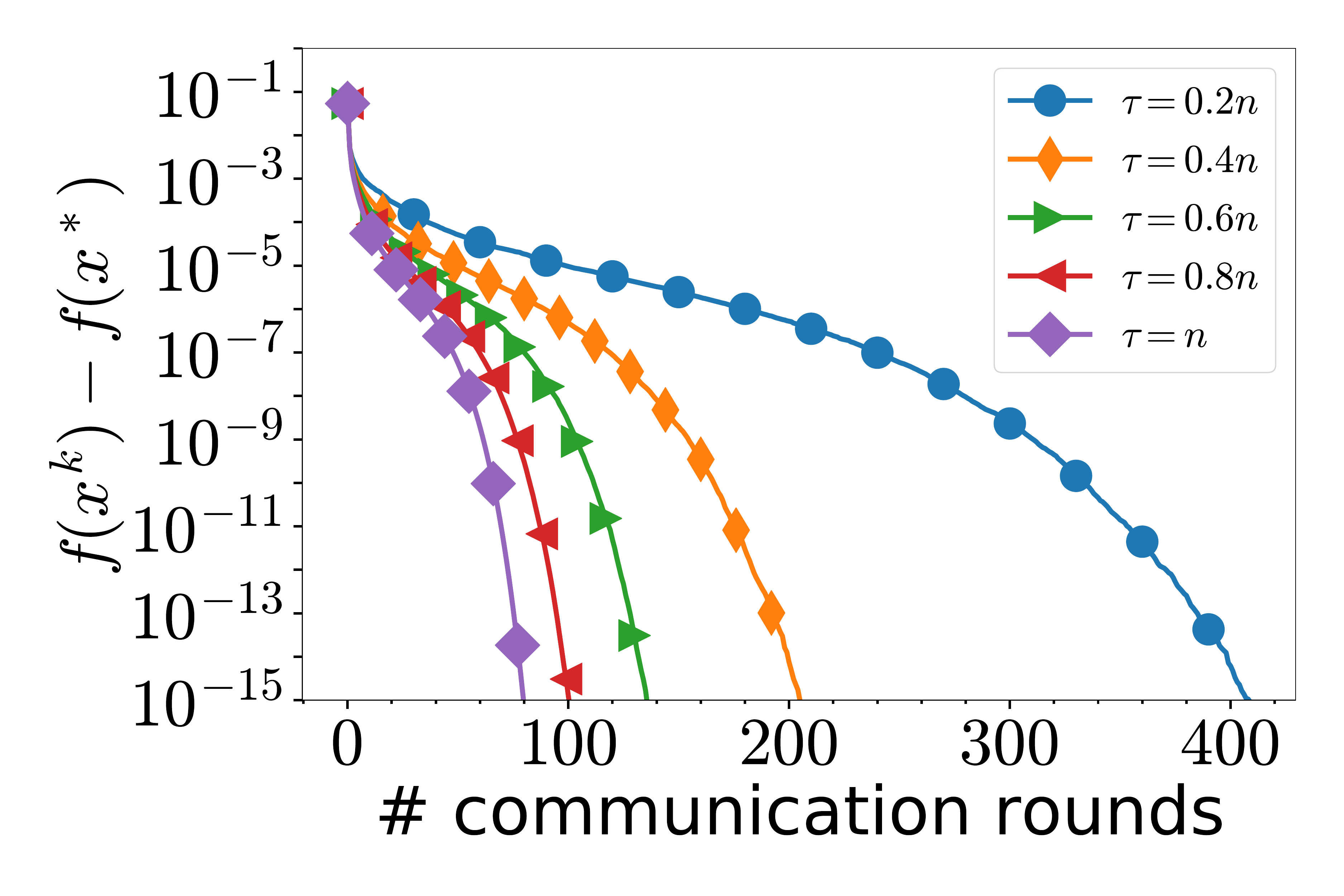}\\
            (a) \dataname{phishing}, $\lambda=10^{-3}$ &
            (b) \dataname{w8a}, $\lambda=10^{-3}$ &
            (c) \dataname{w7a}, $\lambda=10^{-4}$ &
            (d) \dataname{a9a}, $\lambda=10^{-4}$\\
        \end{tabular}
    \end{center}
    \caption{The performance of \algname{FedNL-PP} with Rank-$1$ compressor in terms of iteration complexity.}
    \label{FIG:FedNL_PP}
\end{figure}

\begin{figure*}[h]
    \begin{center}
        \begin{tabular}{cccc}
            \multicolumn{4}{c}{
                \includegraphics[width=1\linewidth]{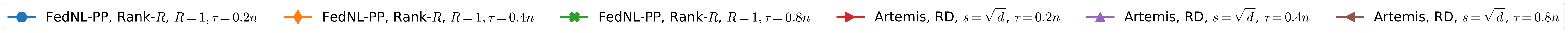}
            }\\
            \includegraphics[width=0.22\linewidth]{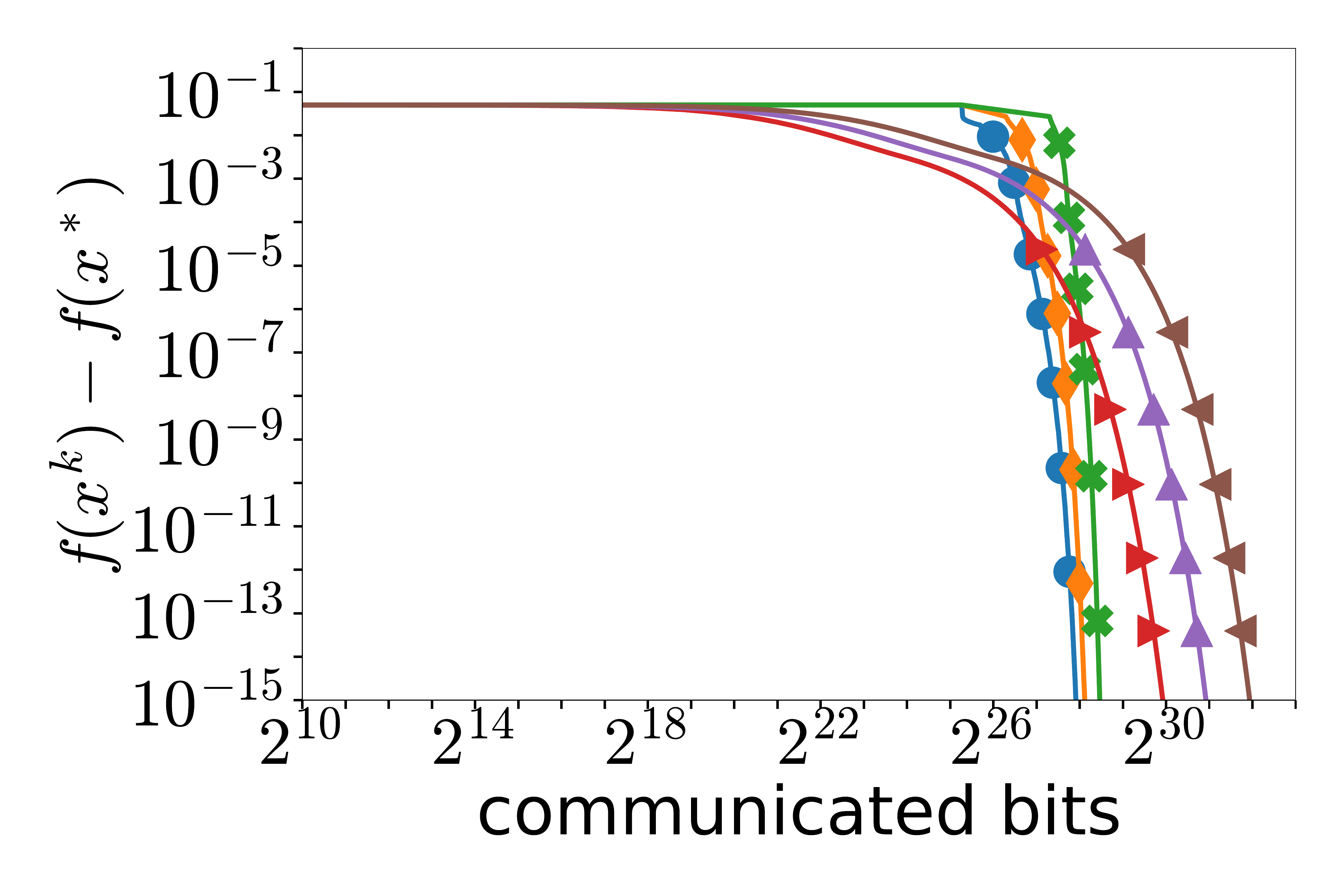} &
            \includegraphics[width=0.22\linewidth]{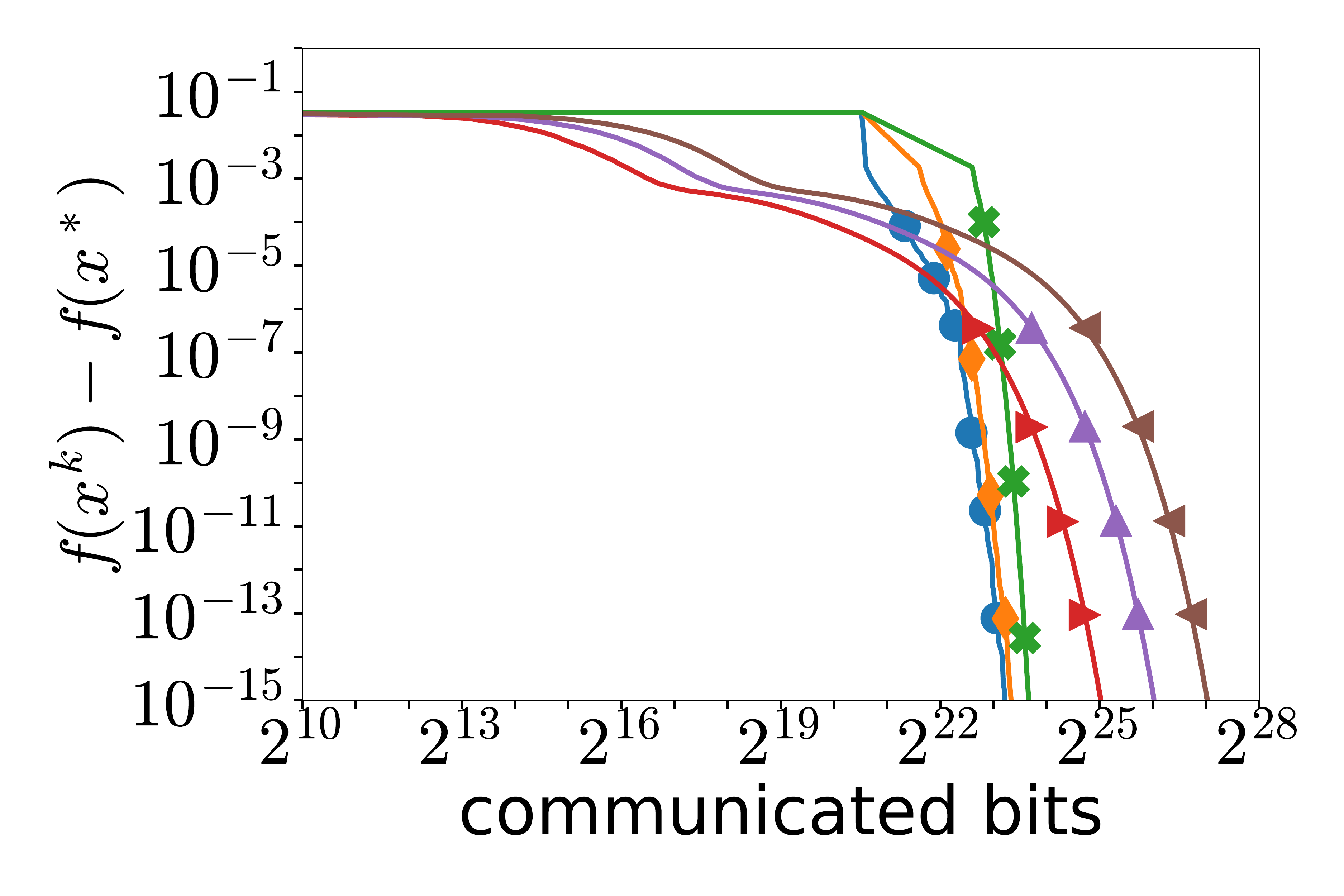} &
            \includegraphics[width=0.22\linewidth]{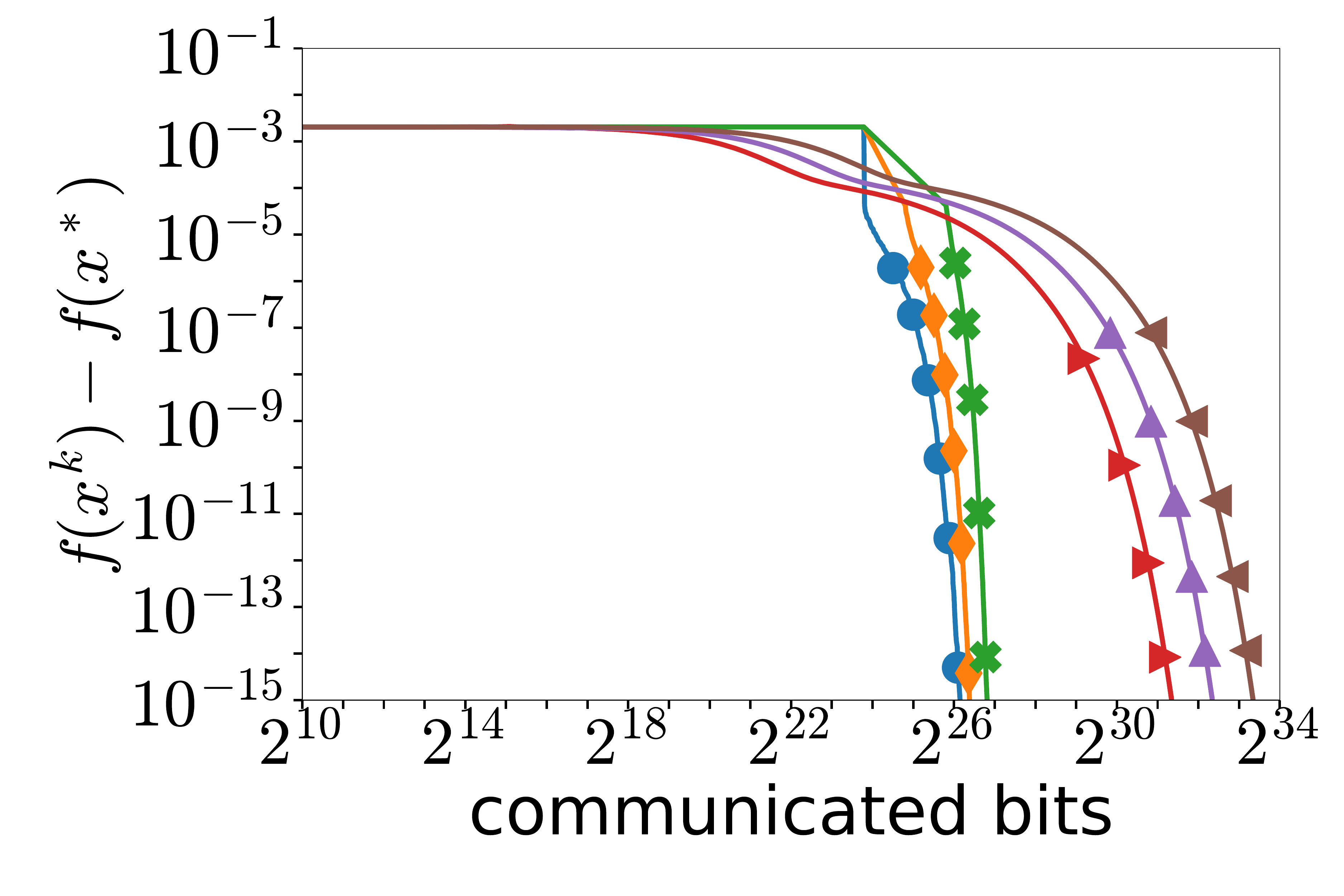} &
            \includegraphics[width=0.22\linewidth]{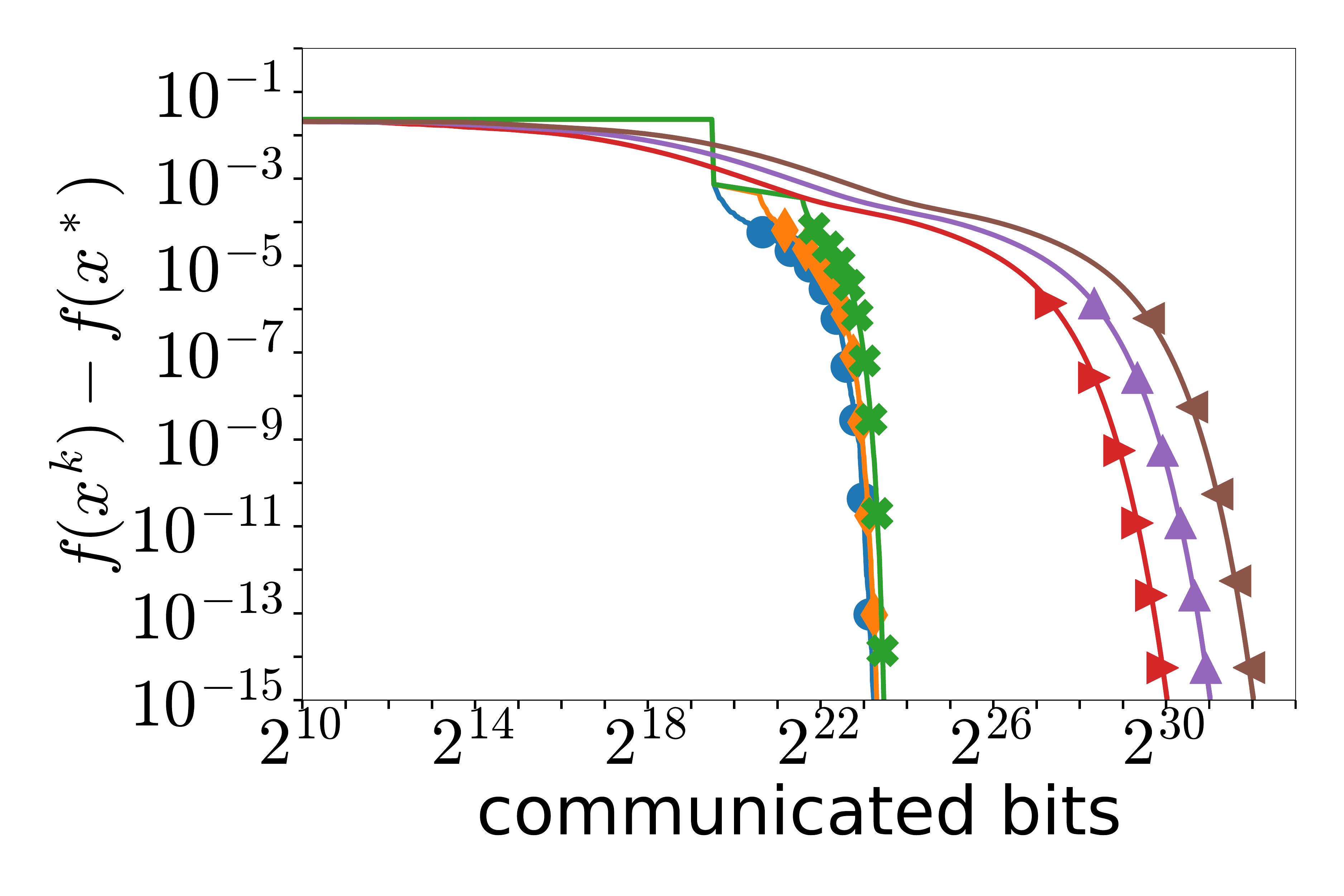}\\
            (a) \dataname{w8a}, $\lambda=10^{-3}$ &
            (b) \dataname{phishing}, $\lambda=10^{-3}$ &
            (c) \dataname{w7a}, $\lambda=10^{-4}$ &
            (d) \dataname{a1a}, $\lambda=10^{-4}$\\
        \end{tabular}
        
    \end{center}
    \caption{Comparison of \algname{FedNL-PP} with \algname{Artemis} in terms of communication complexity for several values of active nodes $\tau$.}
    \label{FIG:FedNL_vs_Artemis}
\end{figure*}

\subsection{Comparison with \algname{NL1}}

In our next experiment we compare \algname{FedNL} with three types of compression operators (Rank-$R$, Top-$K$, PowerSGD) and \algname{NL1}. As we can see in Figure~\ref{FIG:FedNL_vs_NL1}, \algname{FedNL} with Rank-$1$ are more communication efficient method in all cases. \algname{FedNL} with Top-$d$ and PowerSGD of rank $1$ compressors performs better or the same as \algname{NL1} in almost all cases, except Figure~\ref{FIG:FedNL_vs_NL1}: (c), where \algname{FedNL} with PowerSGD demonstrates a little bit worse results than \algname{NL1}. Based on these experiments, we can conclude that new compression mechanism for Hessians is more effective than that was introduced in \citep{Islamov2021NewtonLearn}.

\begin{figure}[h]
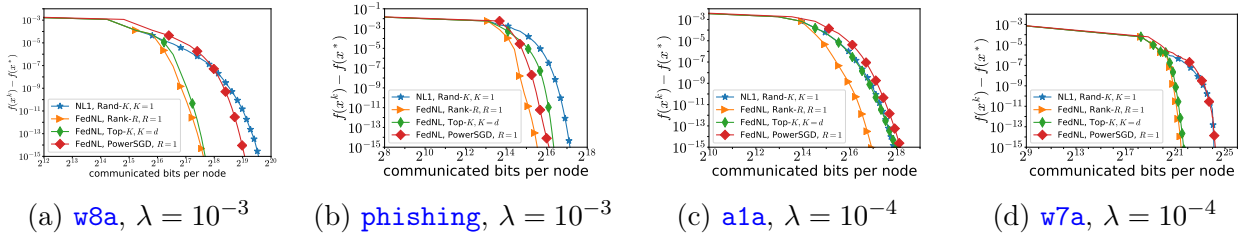

    \begin{center}
        \begin{tabular}{cccc}
            \includegraphics[width=0.23\linewidth]{Experiments/w8a/lmb=1e-3/old_Newton_vs_new_Newton/NL_old_vs_new_All_w8a_lmb=0.001_bits.pdf} &
            \includegraphics[width=0.23\linewidth]{Experiments/phishing/lmb=1e-3/old_Newton_vs_new_Newton/NL_old_vs_new_All_phishing_lmb=0.001_bits.pdf} &
            \includegraphics[width=0.23\linewidth]{Experiments/a1a/lmb=1e-4/old_Newton_vs_new_Newton/NL_old_vs_new_All_a1a_lmb=0.0001_bits.pdf} &
            \includegraphics[width=0.23\linewidth]{Experiments/w7a/lmb=1e-4/old_Newton_vs_new_Newton/NL_old_vs_new_All_w7a_lmb=0.0001_bits.pdf}\\
            (a) \dataname{w8a}, $\lambda=10^{-3}$ &
            (b) \dataname{phishing}, $\lambda=10^{-3}$ &
            (c) \dataname{a1a}, $\lambda=10^{-4}$ &
            (d) \dataname{w7a}, $\lambda=10^{-4}$\\
        \end{tabular}
    \end{center}
    \caption{Comparison of \algname{FedNL} with three types of compression and \algname{NL1} in terms of communication complexity.}
    \label{FIG:FedNL_vs_NL1}
\end{figure}

\subsection{Local comparison}

Now we compare \algname{FedNL} (Rank-$1$ compressor, $\alpha=1$) and \algname{N0} with first order methods: \algname{ADIANA} with random dithering (\algname{ADIANA}, RD, $s=\sqrt{d}$), \algname{DIANA} with random dithering (\algname{DIANA}, RD, $s=\sqrt{d}$), Shifted Local gradient descent (\algname{S-Local-GD}, $p=q=\frac{1}{n}$), and vanilla gradient descent (\algname{GD}). Here we set $x^0$ close to the solution $x^*$ in order to highlight fast local rates of \algname{FedNL} and \algname{N0} independent of the condition number. Moreover, we compare \algname{FedNL} (Rank-$1$ compressor, $\alpha=1$) against \algname{DINGO}. In order to make fair comparison we calculate transmitted bits in both directions, since \algname{DINGO} requires several expensive communication round per one iteration of the algorithm. All results are presented in Figure~\ref{FIG:FedNL_vs_others_locally}. We clearly see that \algname{FedNL} and \algname{N0} are more communication effective methods than gradient type ones. In some cases the difference is large; see Figure~\ref{FIG:FedNL_vs_others_locally}: (a), (d). In addition \algname{FedNL} is more effective than \algname{DINGO} in terms of communication complexity. 

\begin{figure}[h]
    \begin{center}
        \begin{tabular}{cccc}
            \includegraphics[width=0.23\linewidth]{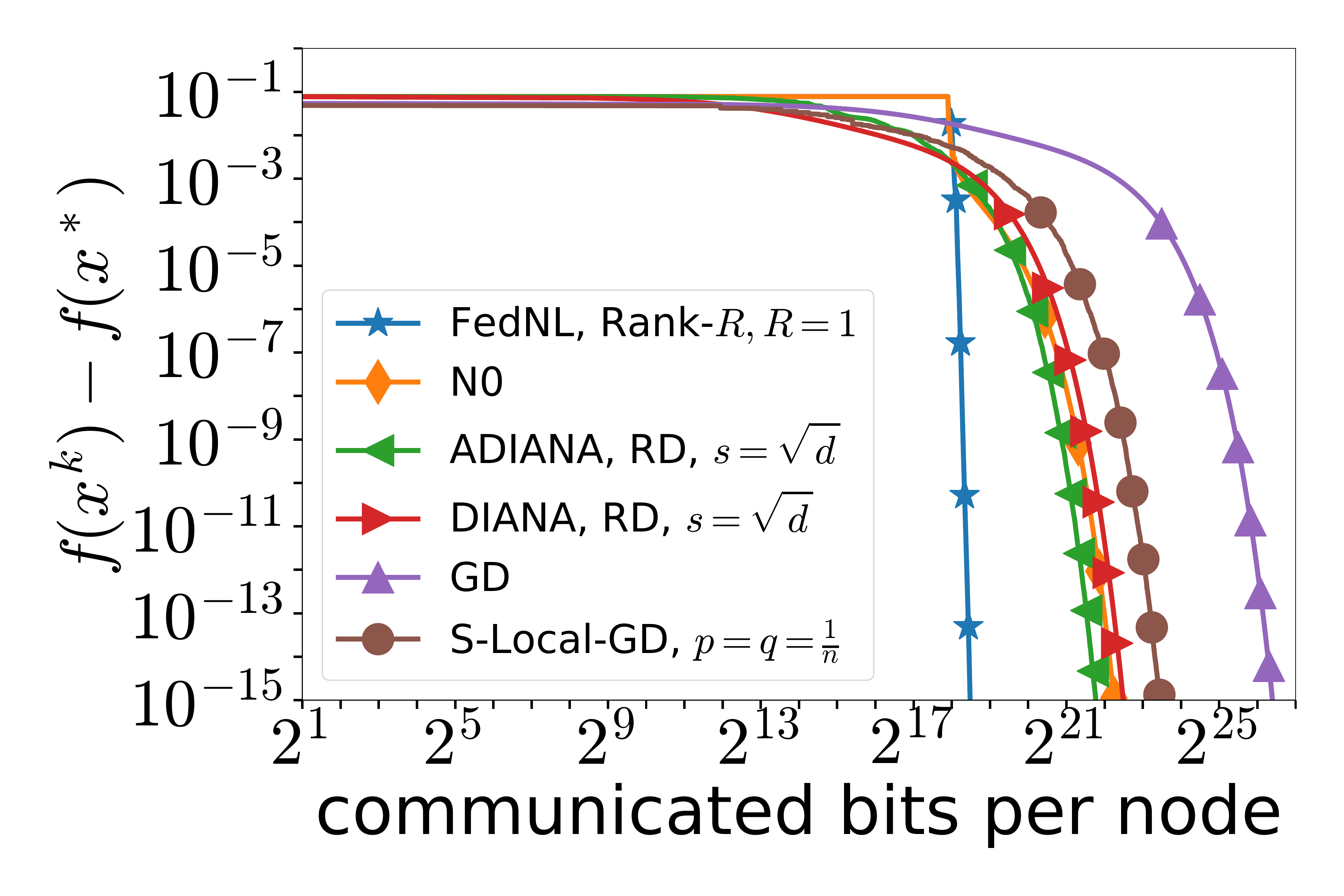} &
            \includegraphics[width=0.23\linewidth]{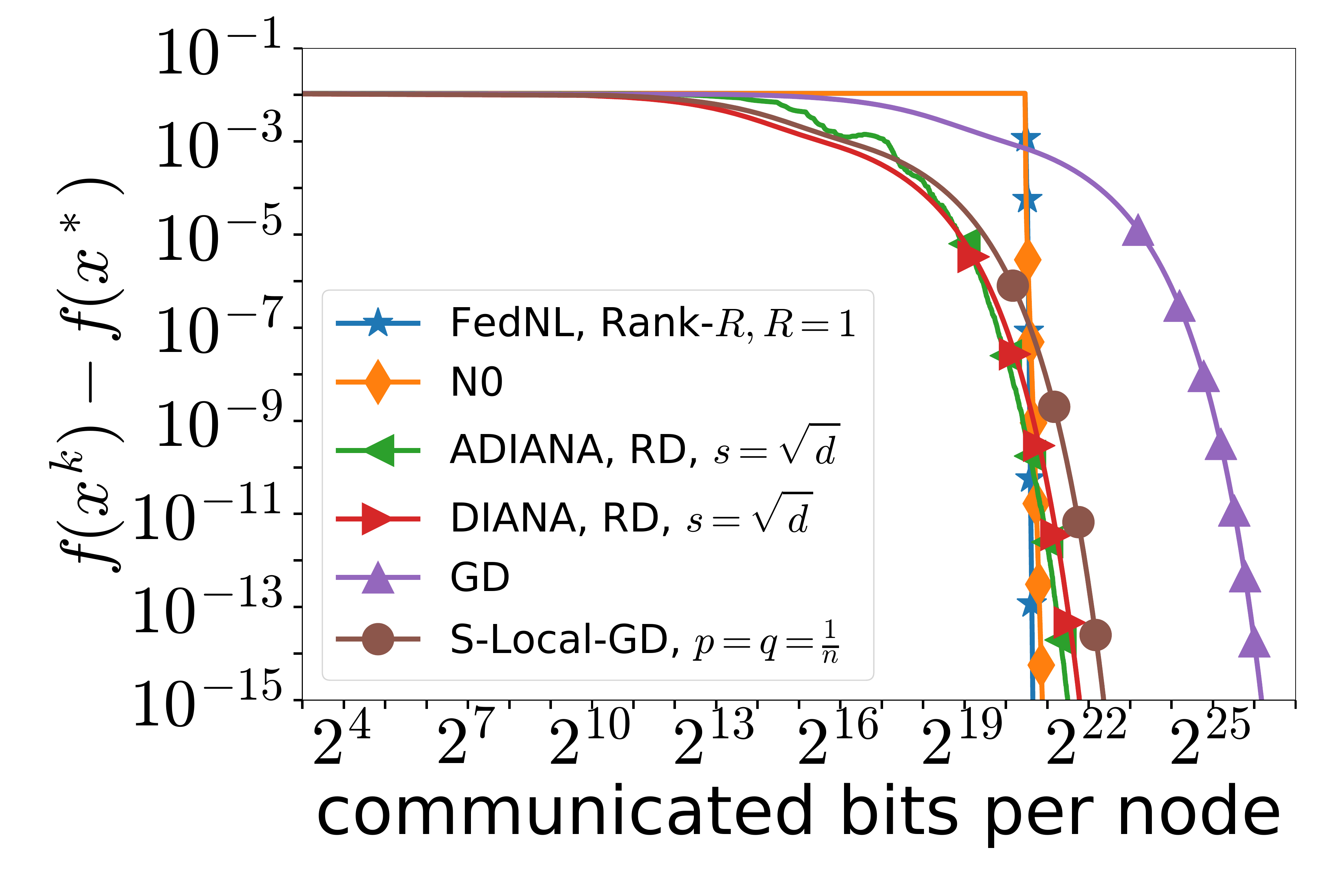} &
            \includegraphics[width=0.23\linewidth]{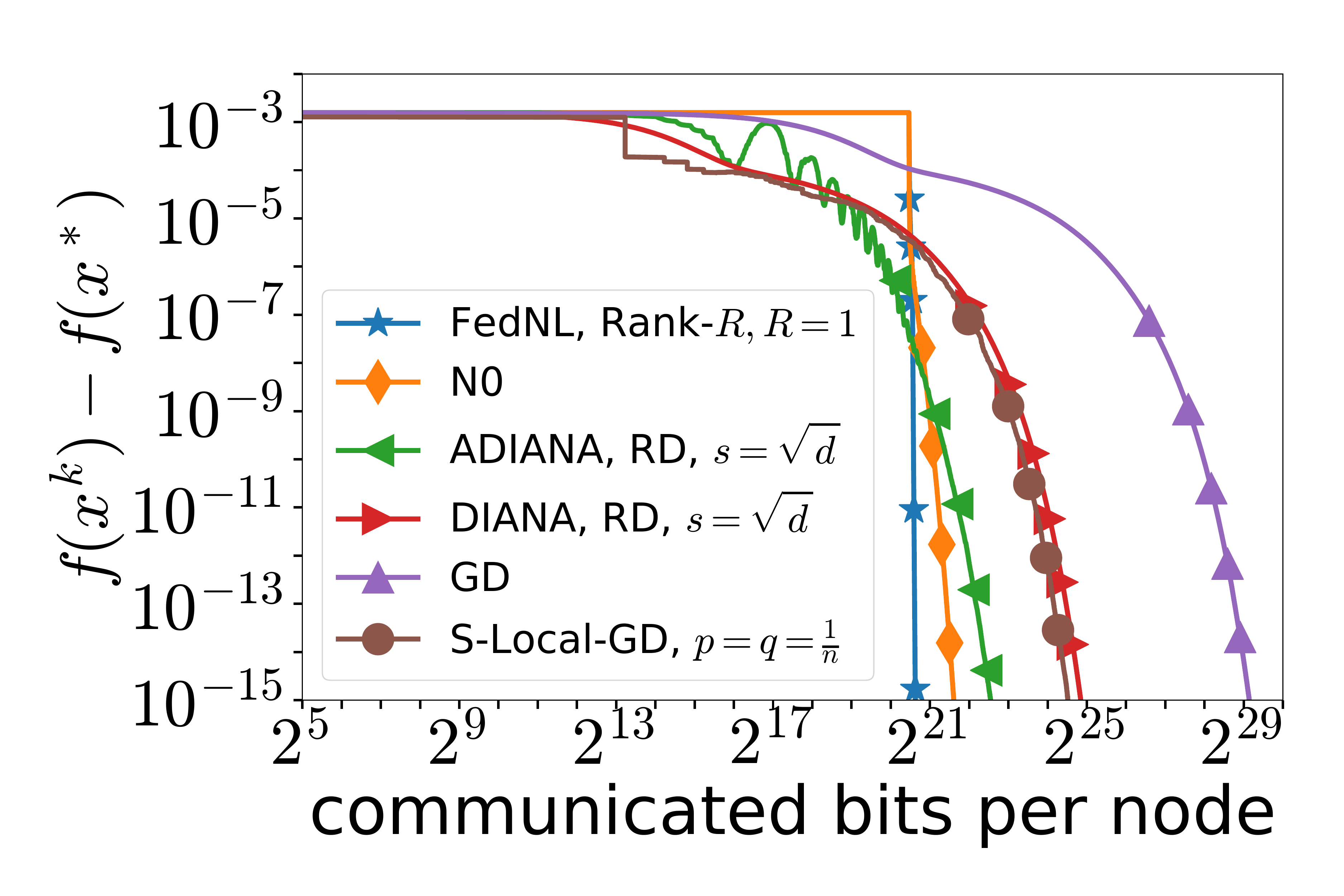} &
            \includegraphics[width=0.23\linewidth]{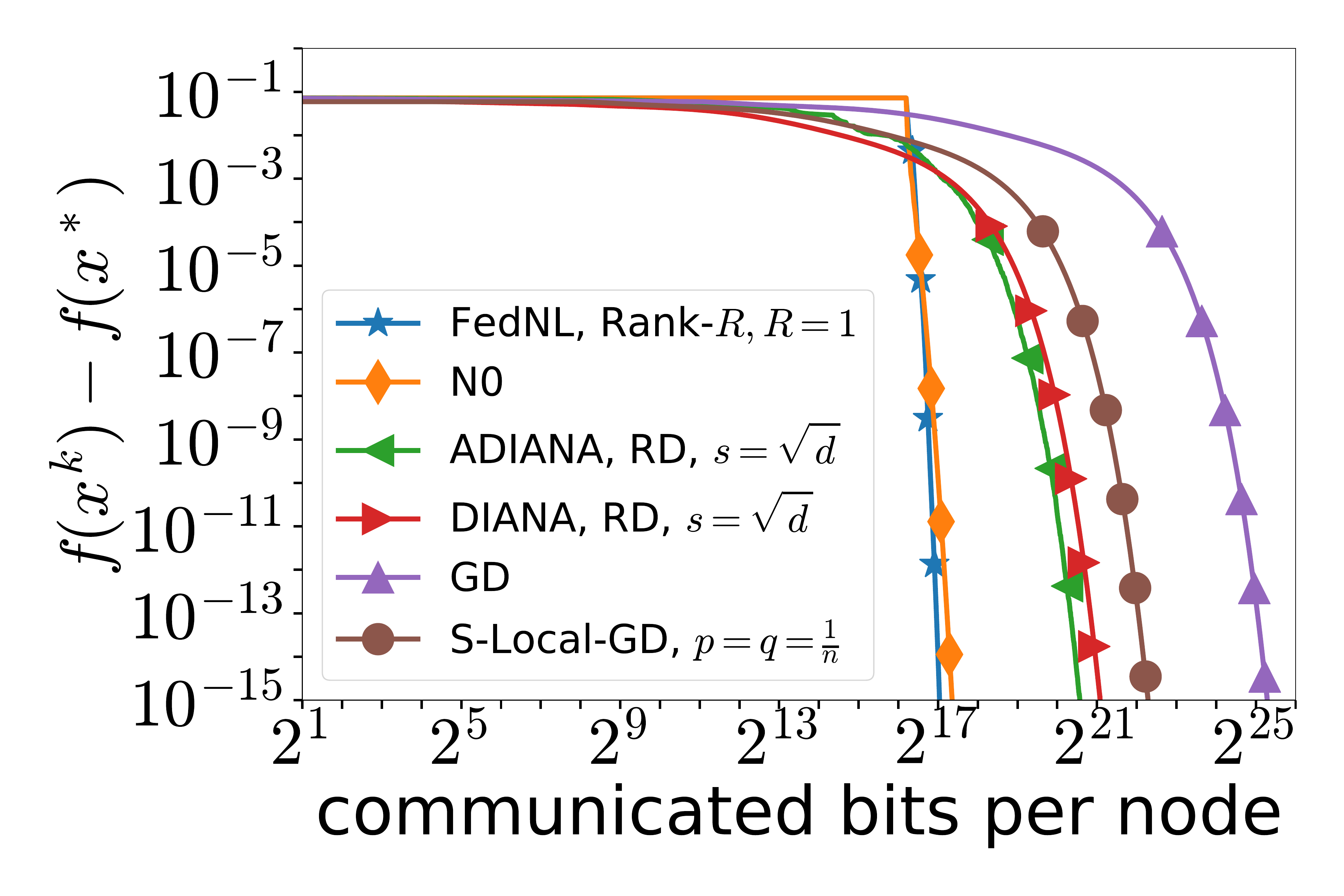}\\
            (a) \dataname{a1a}, $\lambda=10^{-3}$ &
            (b) \dataname{w8a}, $\lambda=10^{-3}$ &
            (c) \dataname{w7a}, $\lambda=10^{-4}$ &
            (d) \dataname{phishing}, $\lambda=10^{-4}$\\
            \includegraphics[width=0.23\linewidth]{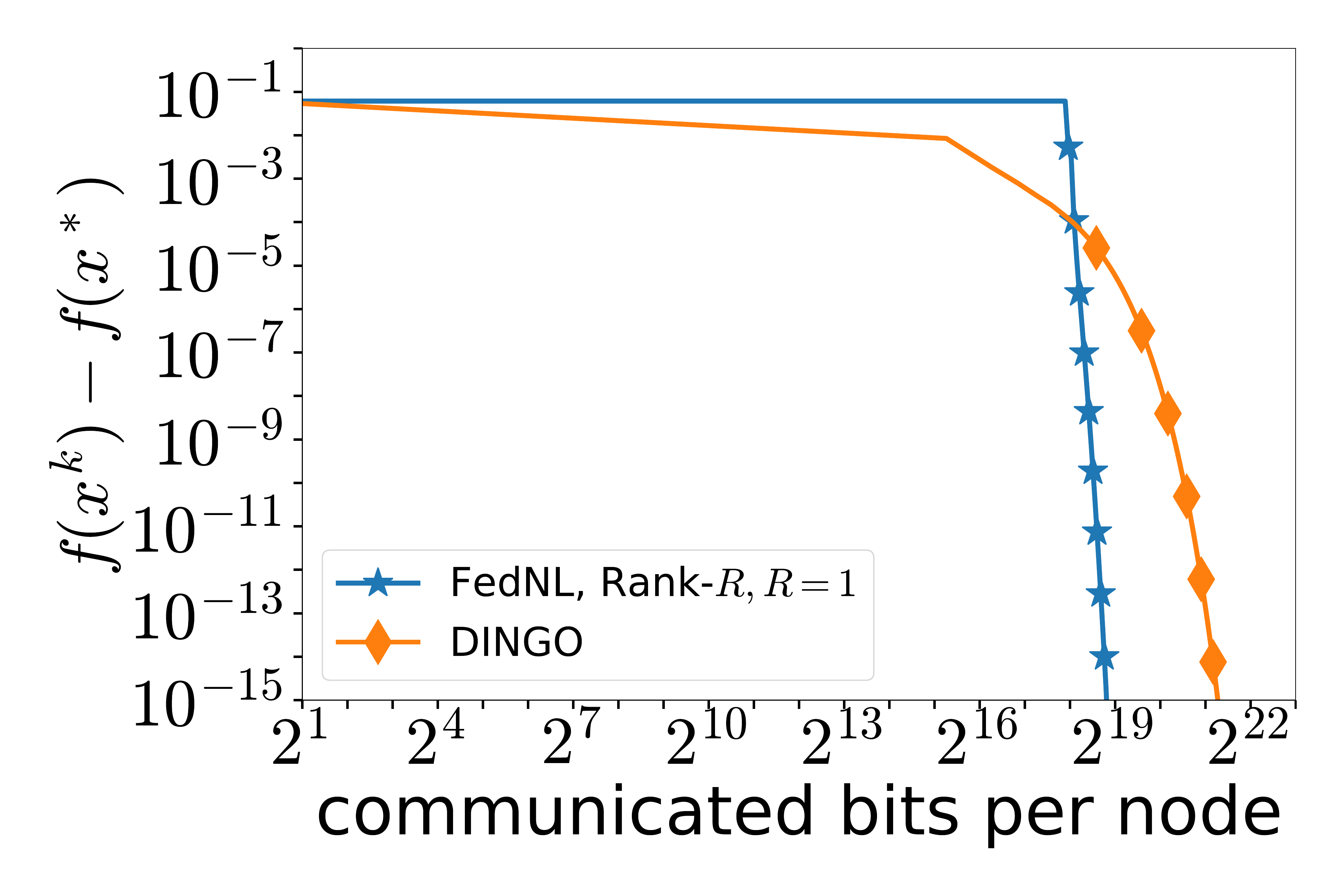} &
            \includegraphics[width=0.23\linewidth]{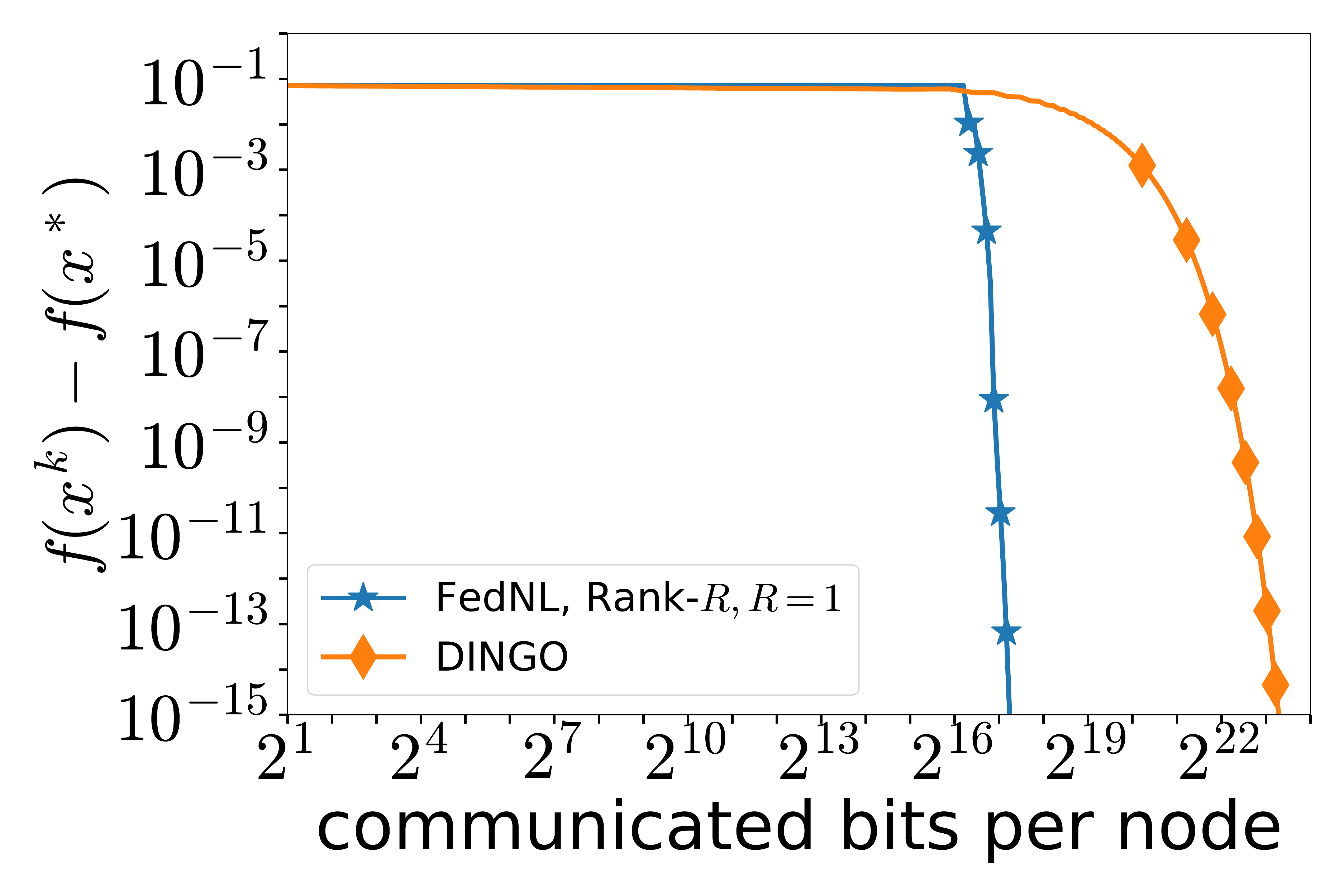} &
            \includegraphics[width=0.23\linewidth]{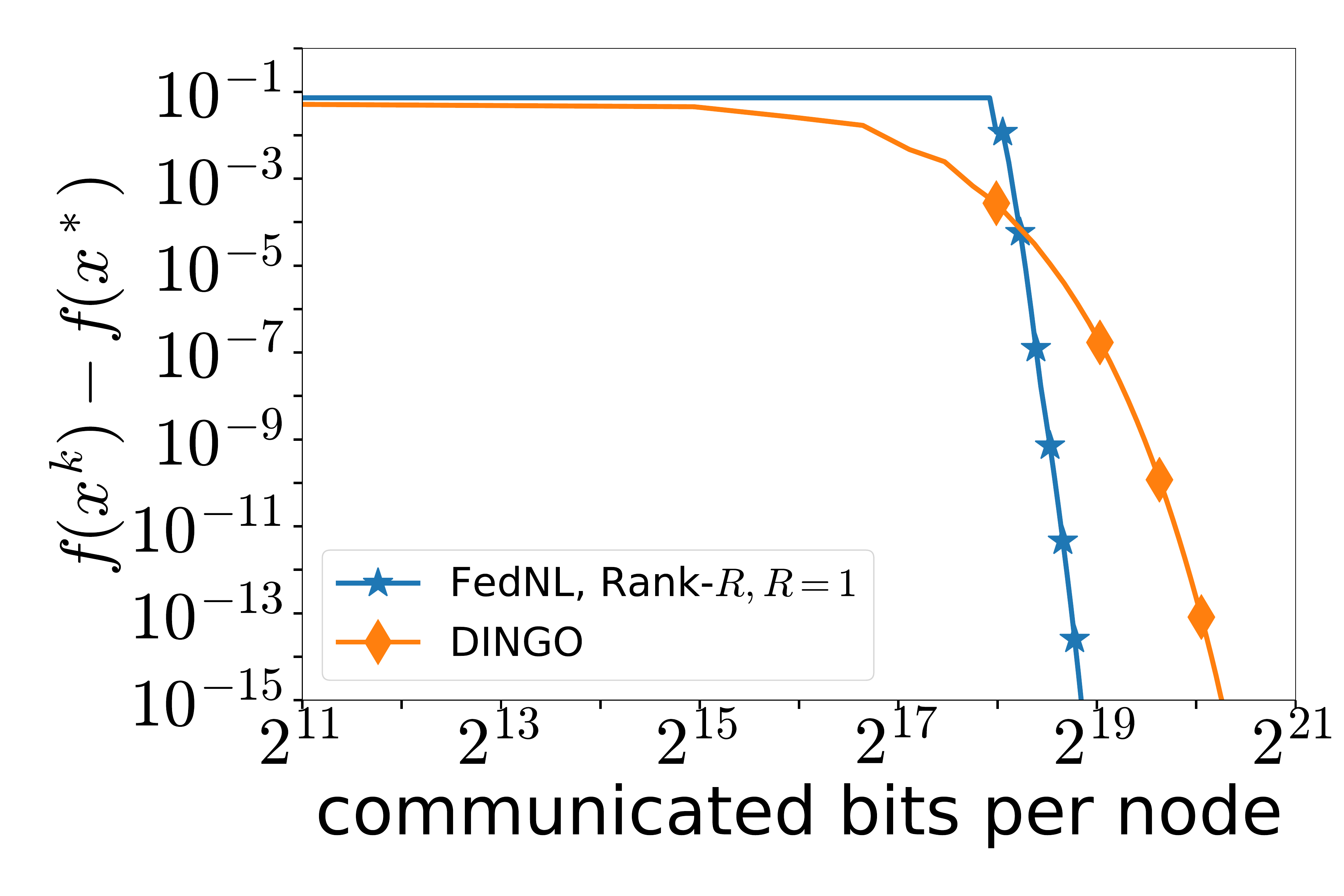} &
            \includegraphics[width=0.23\linewidth]{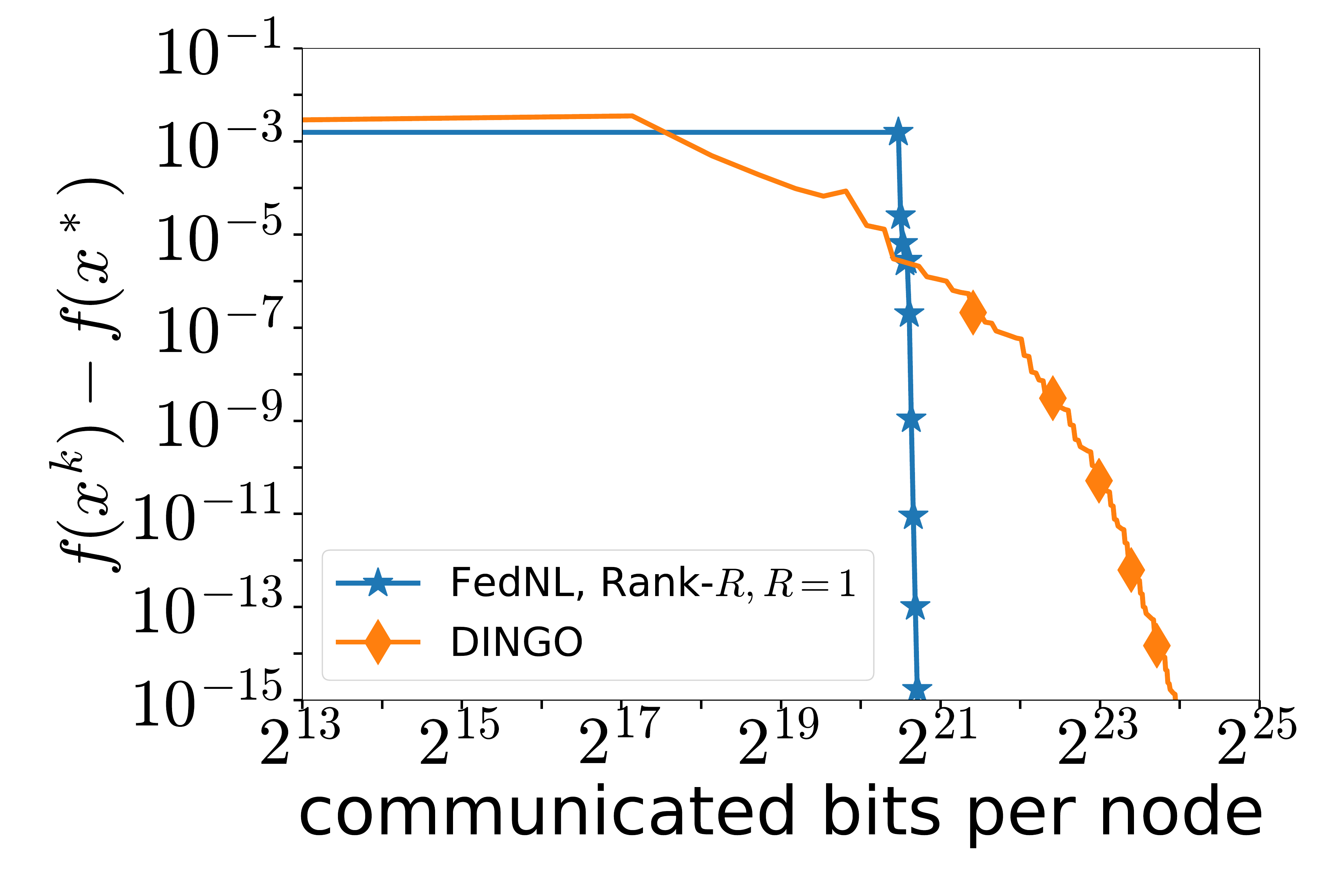}\\
            (a) \dataname{a1a}, $\lambda=10^{-3}$ &
            (b) \dataname{phishing}, $\lambda=10^{-3}$ &
            (c) \dataname{a9a}, $\lambda=10^{-4}$ &
            (d) \dataname{w7a}, $\lambda=10^{-4}$\\
        \end{tabular}
    \end{center}
    \caption{{\bf First row:} Local comparison of \algname{FedNL} and \algname{N0} with \algname{ADIANA}, \algname{DIANA}, \algname{S-Local-GD}, and \algname{GD} in terms of communication complexity. {\bf Second row:} Local comparison of \algname{FedNL} with \algname{DINGO} in terms of communication complexity.}
    \label{FIG:FedNL_vs_others_locally}
\end{figure} 

\subsection{Global compersion}

In our next test we compare \algname{FedNL-LS} (Rank-$1$ compressor, $\alpha=1$), \algname{N0-LS}, and \algname{FedNL-CR} (Rank-$1$ compressor, $\alpha=1$) with gradient type methods such as \algname{ADIANA} with random dithering (\algname{ADIANA}, RD, $s=\sqrt{d}$), \algname{DIANA} with random dithering (\algname{DIANA}, RD, $s=\sqrt{d}$),  Shifted Local gradient descent (\algname{S-Local-GD}, $p=q=\frac{1}{n}$), vanilla gradient descent (\algname{GD}), and gradient descent with line search (\algname{GD-LS}). Besides, we compare \algname{FedNL-LS} (Rank-$1$ compressor, $\alpha=1$) and \algname{FedNL-CR} (Rank-$1$ compressor, $\alpha=1$) with \algname{DINGO}. Since \algname{DINGO} requires several expensive communication round per iteration, we calculate transmitted bits in both directions to make fair comparison. According to numerical experiments, we can conclude that \algname{FedNL-LS} and \algname{N0-LS} are more communication effective methods than gradient type ones. In some cases (see Figure~\ref{FIG:FedNL_vs_others_globally}: (c), (d)) \algname{FedNL-CR} performs better or the same as \algname{DIANA}.

\begin{figure}[h]
    \begin{center}
        \begin{tabular}{cccc}
            \includegraphics[width=0.23\linewidth]{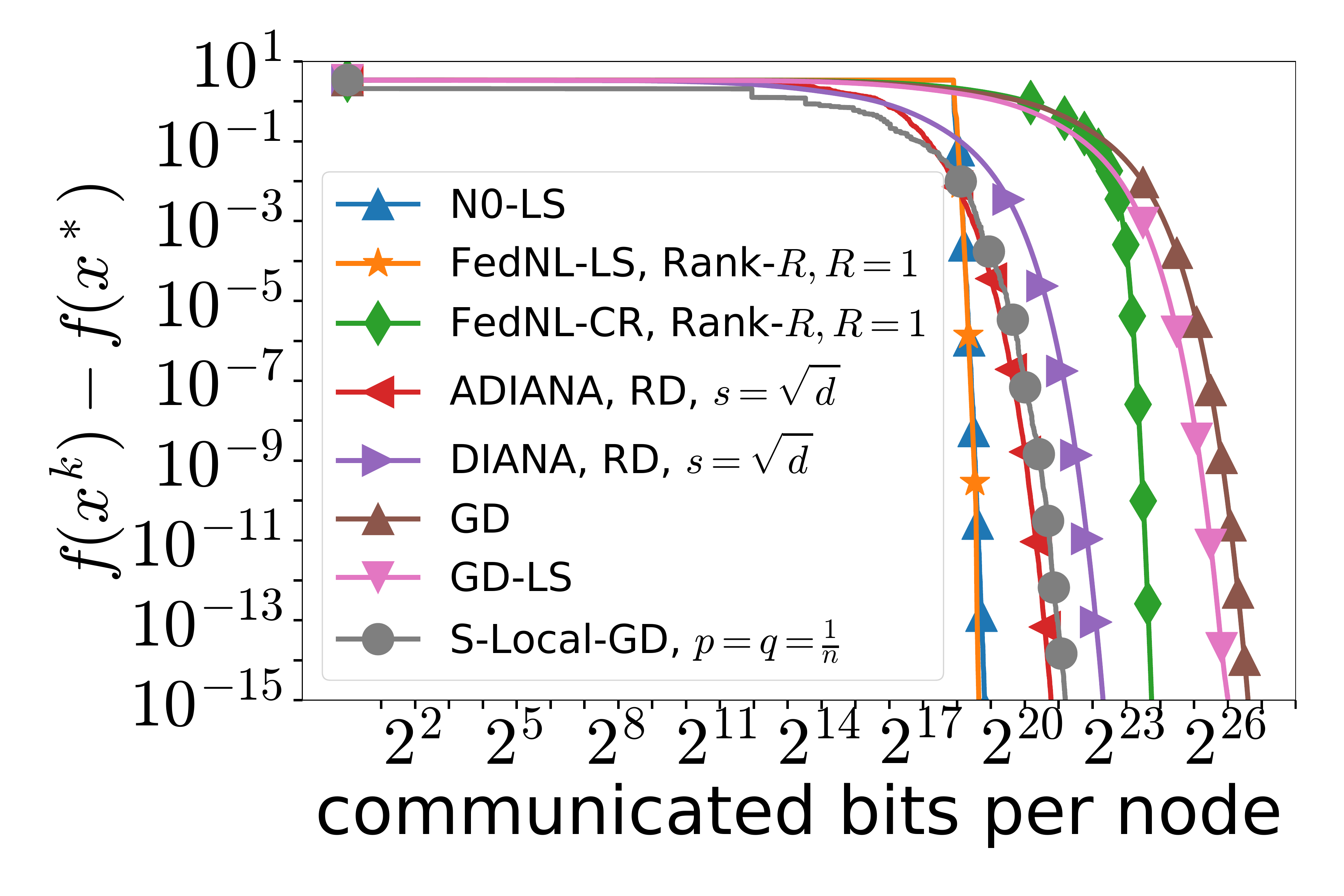} &
            \includegraphics[width=0.23\linewidth]{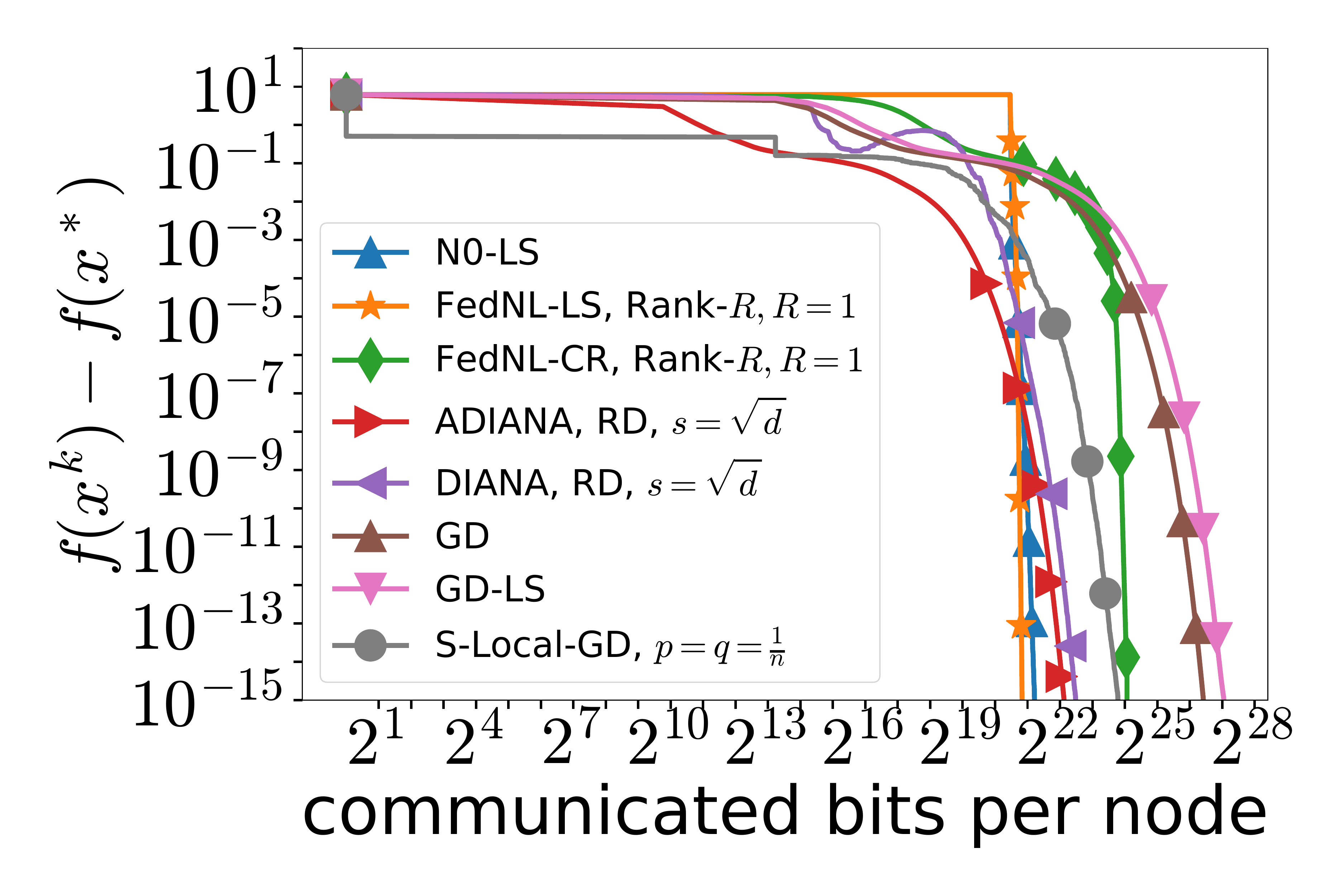} &
            \includegraphics[width=0.23\linewidth]{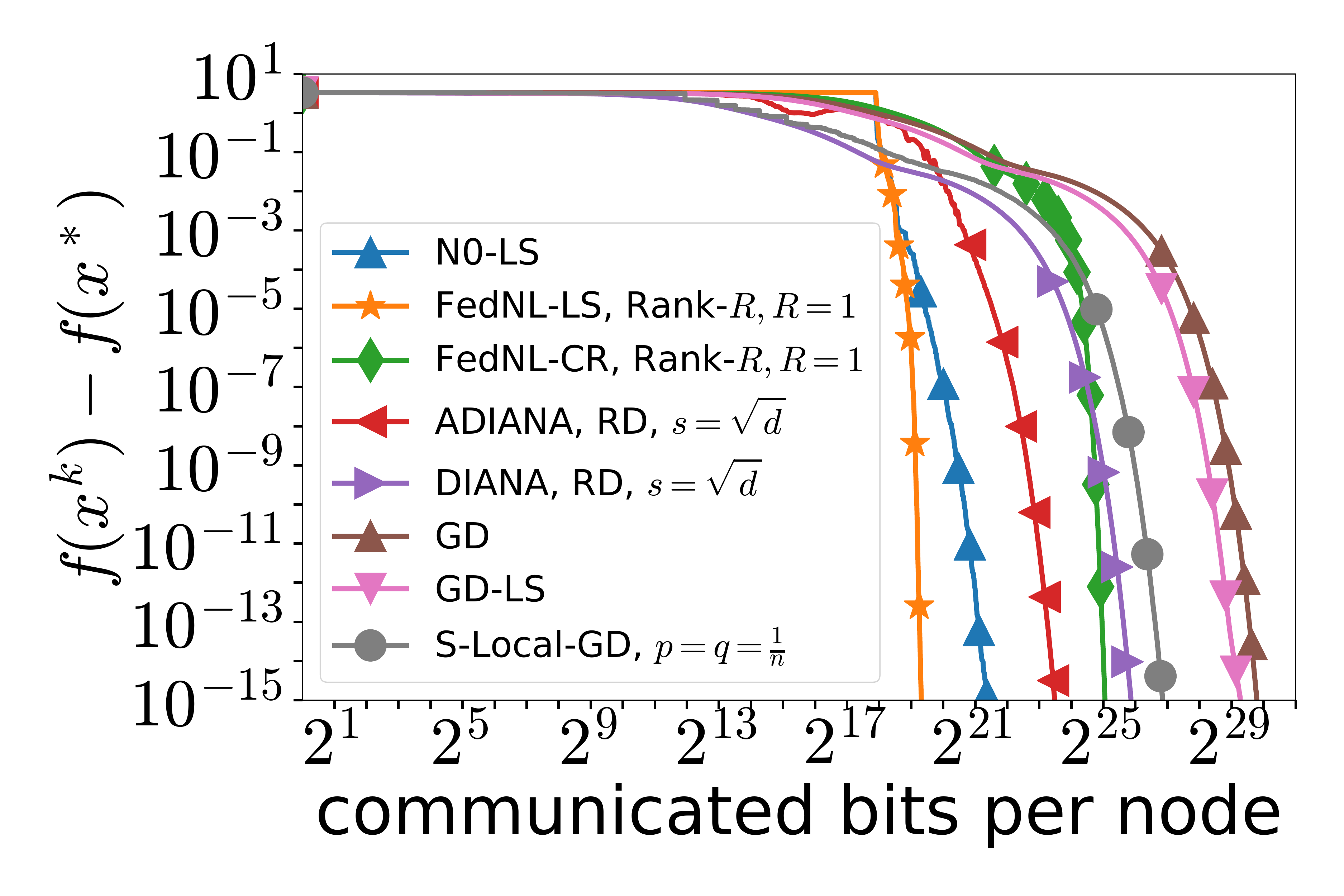} &
            \includegraphics[width=0.23\linewidth]{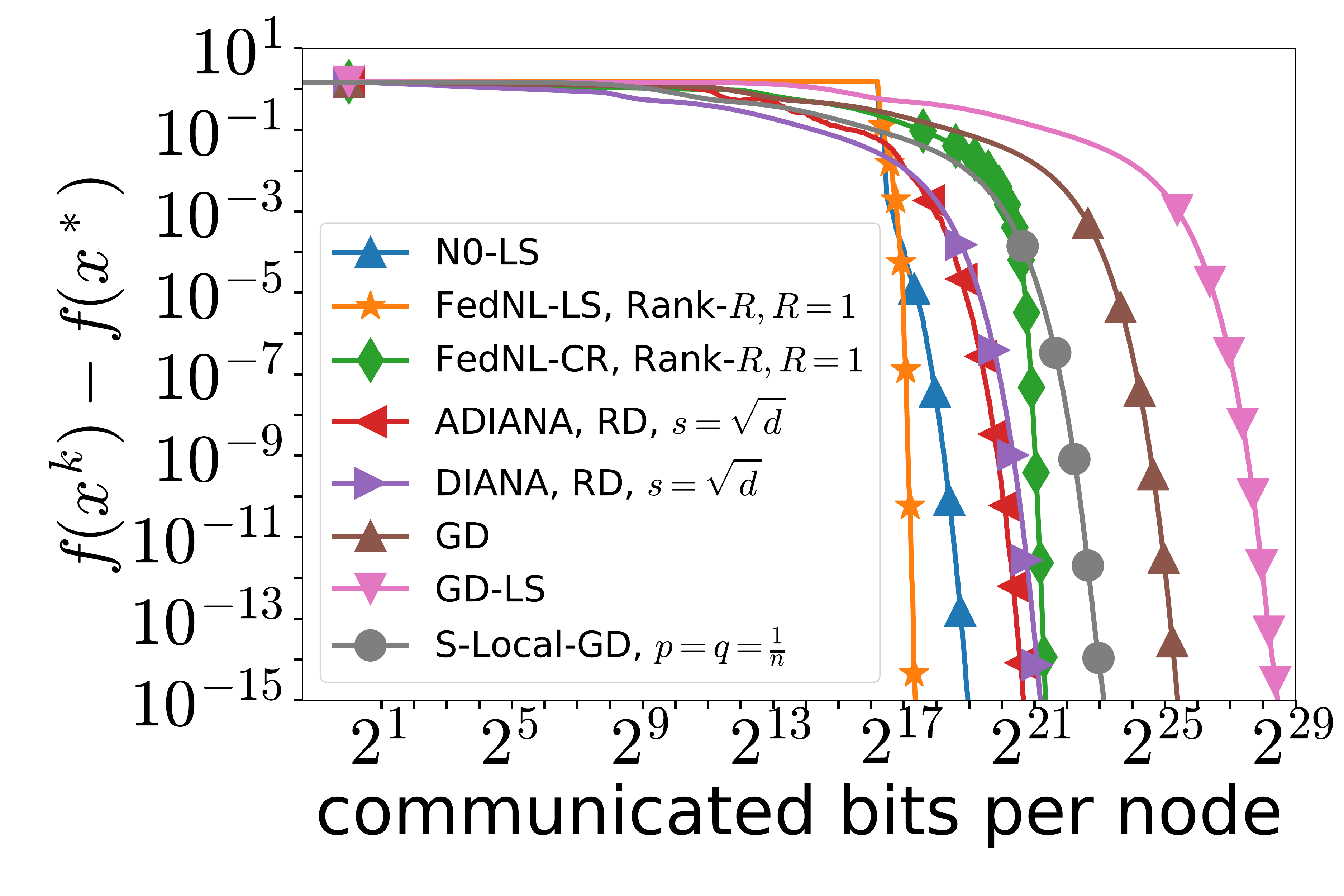}\\
            (a) \dataname{a9a}, $\lambda=10^{-3}$ &
            (b) \dataname{w7a}, $\lambda=10^{-3}$ &
            (c) \dataname{a1a}, $\lambda=10^{-4}$ &
            (d) \dataname{phishing}, $\lambda=10^{-4}$\\
            \includegraphics[width=0.23\linewidth]{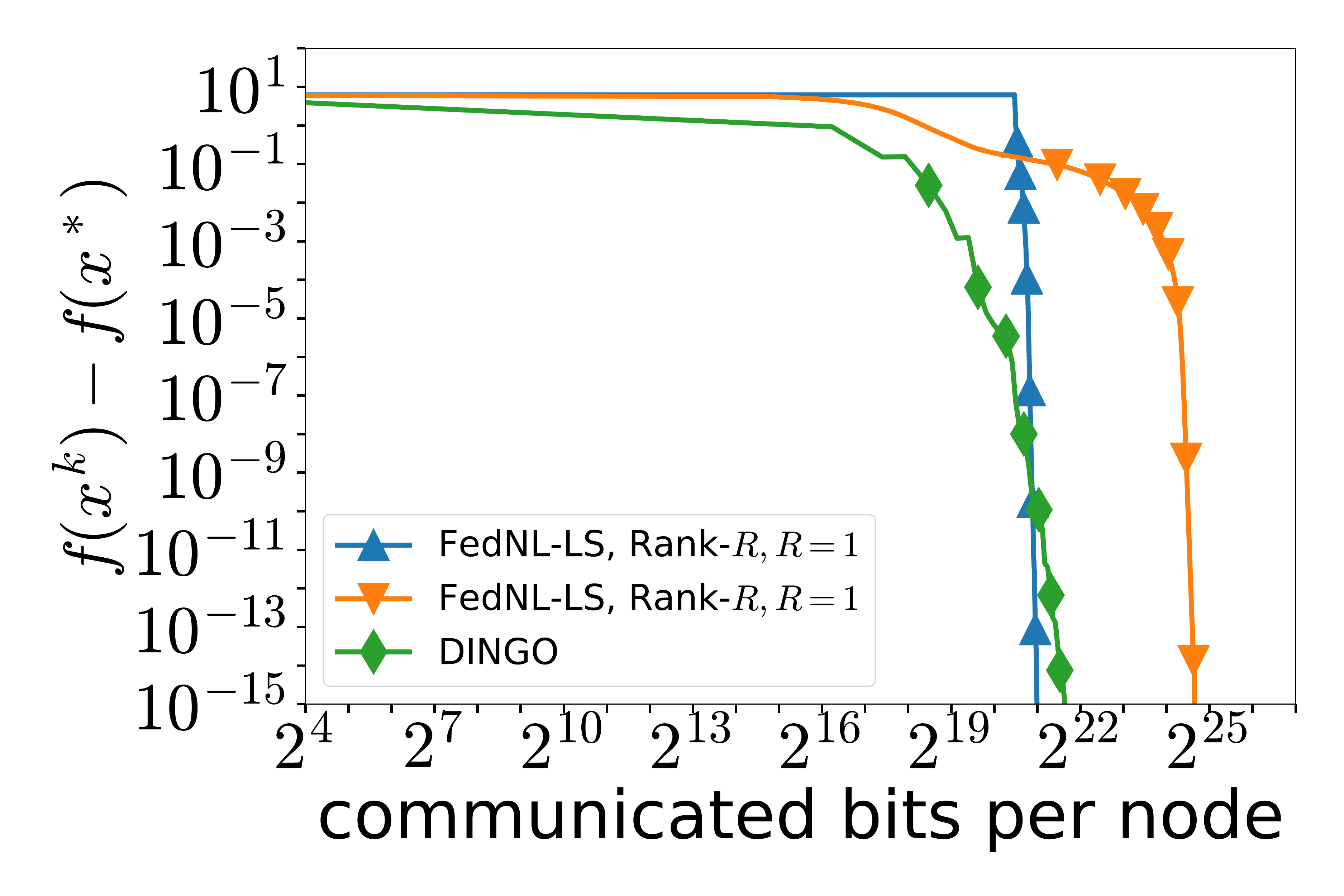} &
            \includegraphics[width=0.23\linewidth]{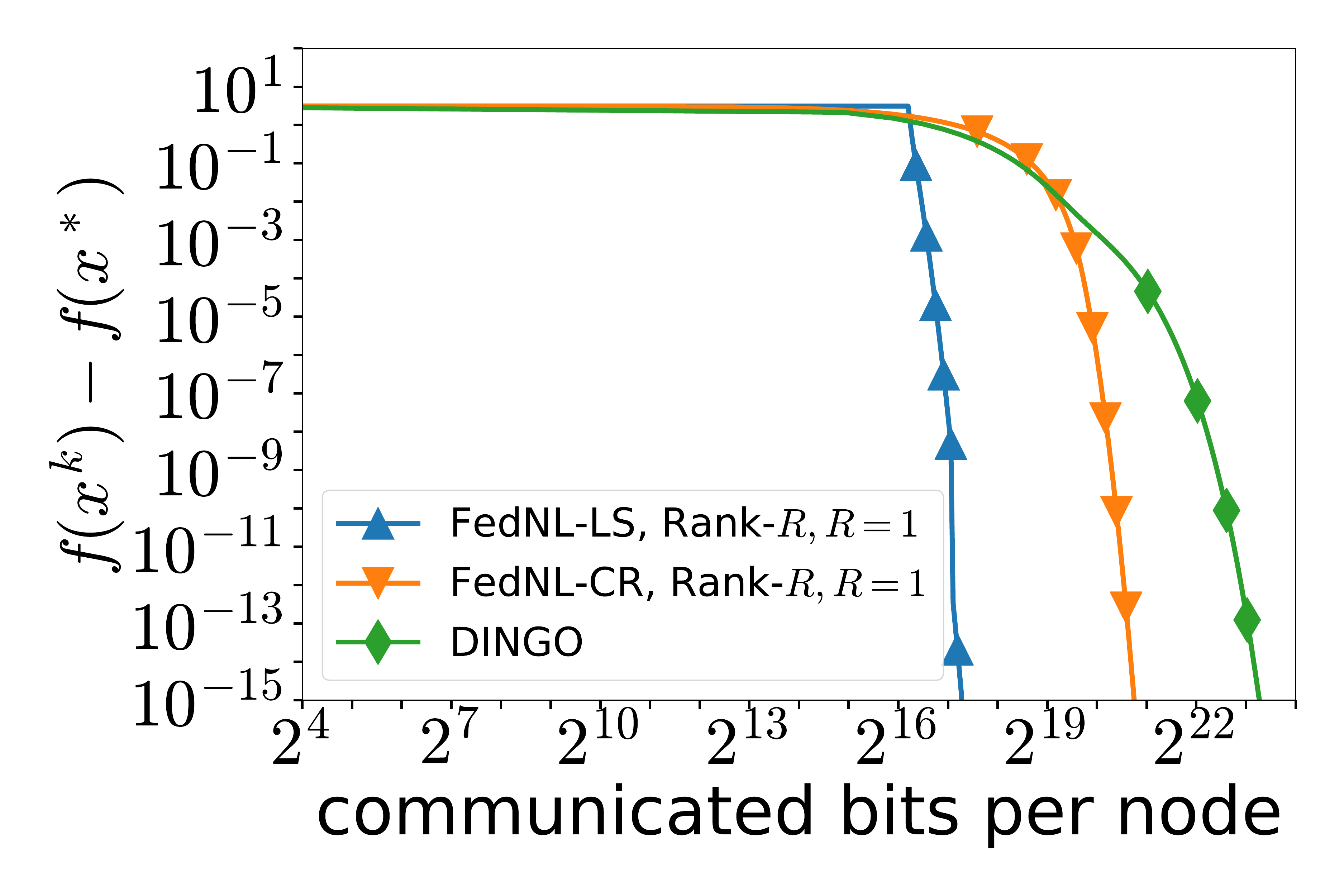} &
            \includegraphics[width=0.23\linewidth]{Experiments/a9a/lmb=1e-4/Global_comparison/Global_Comp_Dingo_a9a_lmb_0.0001_bits.pdf} &
            \includegraphics[width=0.23\linewidth]{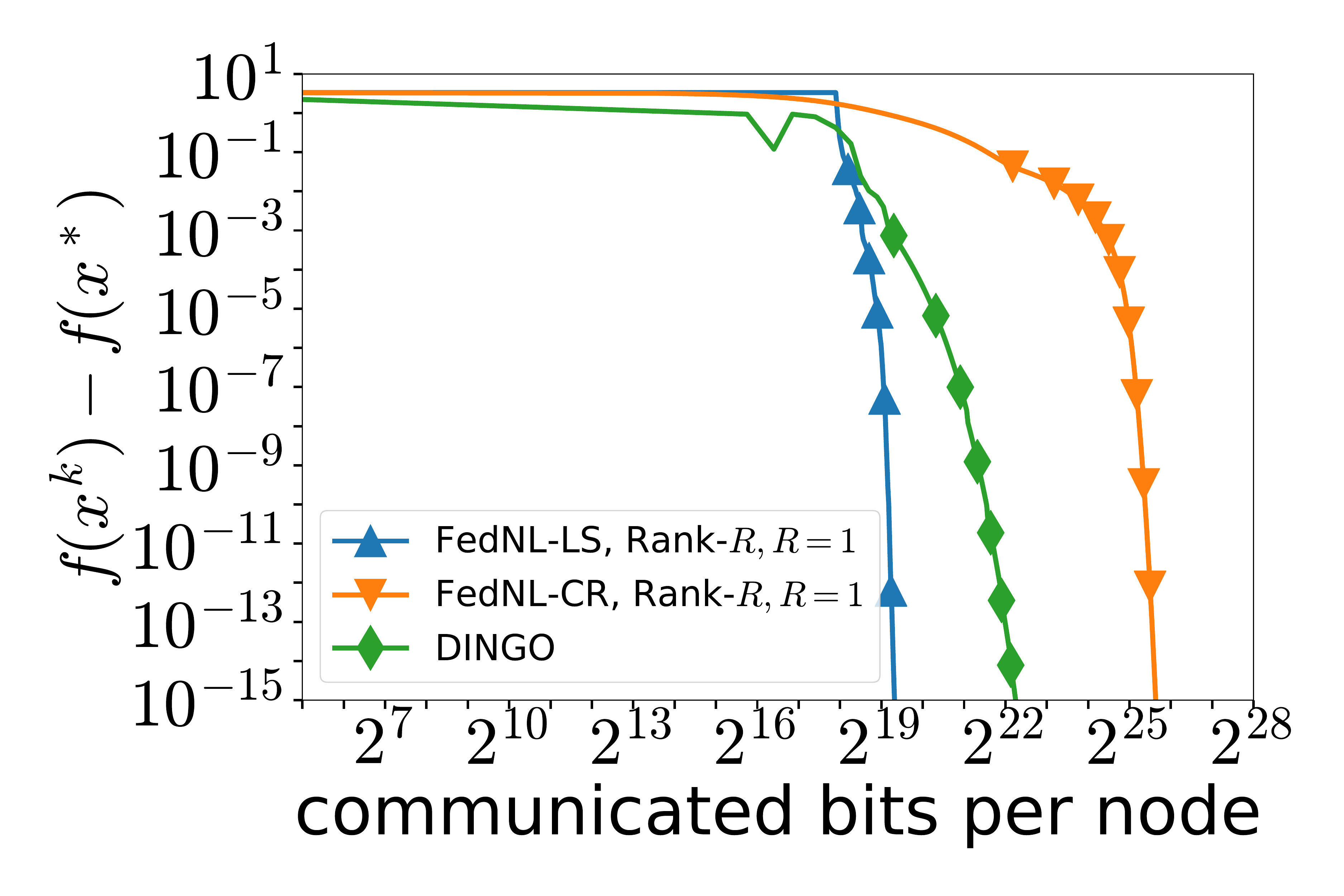}\\
            (a) \dataname{w7a}, $\lambda=10^{-3}$ &
            (b) \dataname{phishing}, $\lambda=10^{-3}$ &
            (c) \dataname{a9a}, $\lambda=10^{-4}$ &
            (d) \dataname{a1a}, $\lambda=10^{-4}$\\
        \end{tabular}
    \end{center}
    \caption{{\bf First row:} Global comparison of \algname{FedNL-LS}, \algname{N0-LS}, and \algname{FedNL-CR} with \algname{ADIANA}, \algname{DIANA}, \algname{S-Local-GD}, \algname{GD}, and \algname{GD-LS} in terms of communication complexity. {\bf Second row:} Global comparison of \algname{FedNL-LS} and \algname{FedNL-CR} with \algname{DINGO} in terms of communication complexity.}
    \label{FIG:FedNL_vs_others_globally}
\end{figure} 

\subsection{Effect of statistical heterogeneity}

In this set of experiments we investigate the performance of \algname{FedNL} under different level of heterogeneity of data. We generate synthethic data via rules as \citep{li2020federated} did. We set number of nodes $n=30$, the size of local data $m=200$, the dimension of the problem $d=100$, and regularization parameter $\lambda=10^{-3}$. 

The generation rules for non-IID synthetic data have two positive parameteres $\alpha, \beta$. For each node $i \in [n]$ let $B_i \sim \mathcal{N}(0,\beta)$. We use diagonal covariance matrix $\Sigma$ with $\Sigma_{j,j} = j^{-1.2}$, and mean vector $v_i$, each element of which is generated from $\mathcal{N}(B_i, 1)$ in order to get feature vector $a_{ij} \in \R^d$ from $\mathcal{N}(v_i, \Sigma)$. Let $u_i \sim \mathcal{N}(0, \alpha)$, $c_i \sim \mathcal{N}(u_i, 1)$, then we generate vector $w_i \in \R^d$ each entire of which is sampled from $\mathcal{N}(u_i, 1)$. Let $p_{ij}=\sigma\(w_i^\top a_{ij} + c_i\)$, where $\sigma(\cdot)$ is a sigmoid function. Finally the label $b_{ij}$ is equal to $-1$ with probability $p_{ij}$, and is equal to $+1$ with probability $1-p_{ij}$. We denote the data which is generated following the rules above as \dataname{Synthetic}$(\alpha, \beta)$. 

In addition, we generate IID data where $w \sim \mathcal{N}(0, 1)$ and $c \sim \mathcal{N}(0, 1)$ are sampled only once and used for each node $i$. Feature vectors $a_{ij}$ is generated from $\mathcal{N}(v_i, \Sigma)$, where each element of $v_i$ is equal to $B_i \sim \mathcal{N}(0,\beta)$. The label $b_{ij}$ is equal to $-1$ with probability $p_{ij} = \sigma(w^\top a_{ij} + c)$, and $+1$ otherwise. We denote such data as \dataname{IID}. 

Using generated synthetic datasets we compare local performance of \algname{FedNL} (Rank-$1$ compressor, $\alpha=1$, Option $2$),  \algname{ADIANA} with random dithering (\algname{ADIANA}, RD, $s=\sqrt{d}$), \algname{DIANA} with random dithering (\algname{DIANA}, RD, $s=\sqrt{d}$), Shifted Local gradient descent (\algname{S-Local-GD}, $p=q=\frac{1}{n}$), and vanilla gradient descent (GD) in terms of communication complexit; see Figure~\ref{FIG:data_heterogeneity} (first row). Besides, we compare \algname{FedNL} and \algname{DINGO}; see Figure~\ref{FIG:data_heterogeneity} (second row). According to the results, we see that the difference between \algname{FedNL} and gradient type methods is getting larger, when the local data is becoming more heterogeneous; \algname{FedNL} outperforms other methods by several orders in magnitude.  \algname{FedNL} is more stable varying data heterogeneity than \algname{DINGO}. The difference between these two methods on \dataname{IID} data is small; when data is becoming more heterogeneous, the difference is increasing dramatically. 


\begin{figure*}[h]
    \begin{center}
        \begin{tabular}{cccc}
            \includegraphics[width=0.23\linewidth]{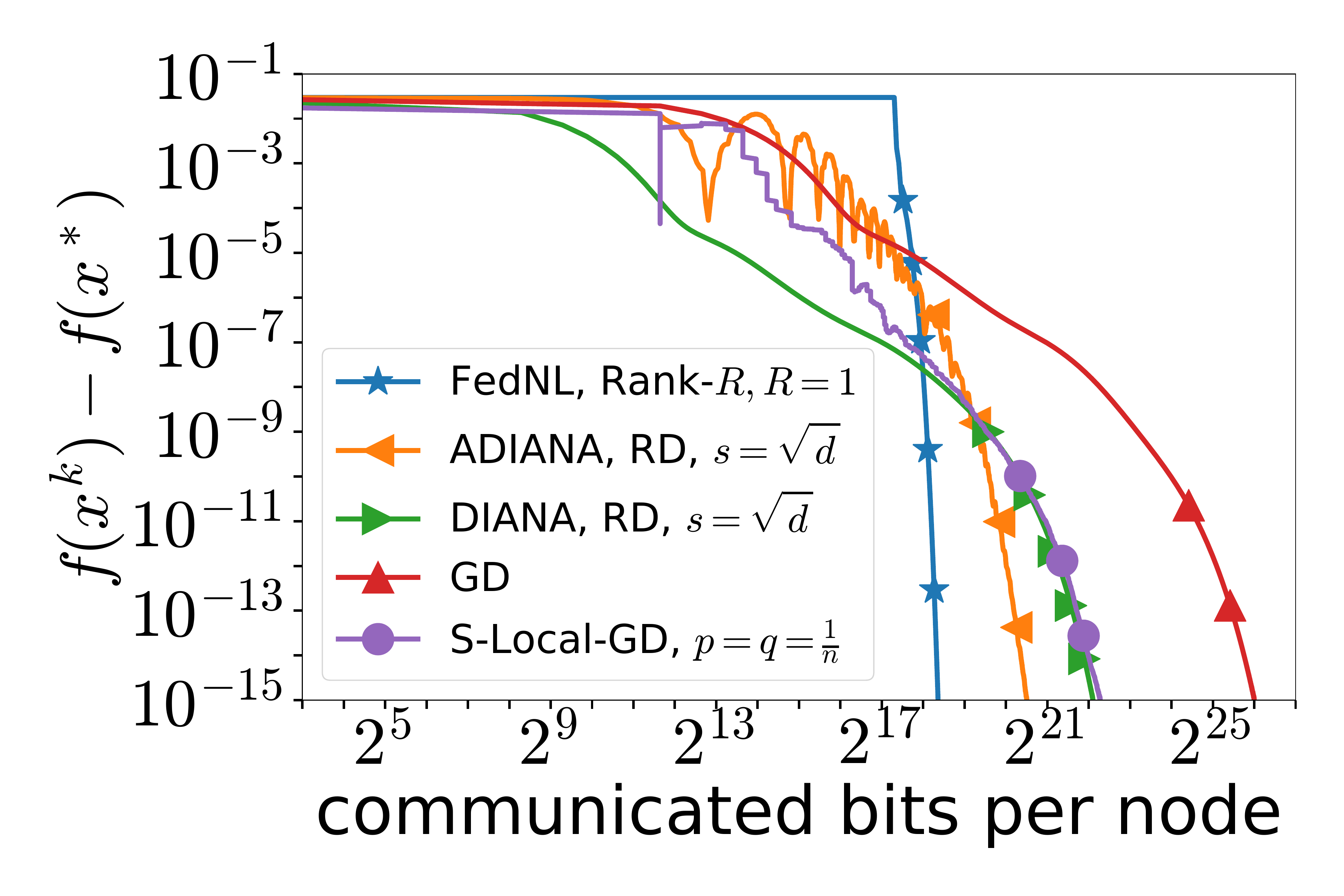} &
            \includegraphics[width=0.23\linewidth]{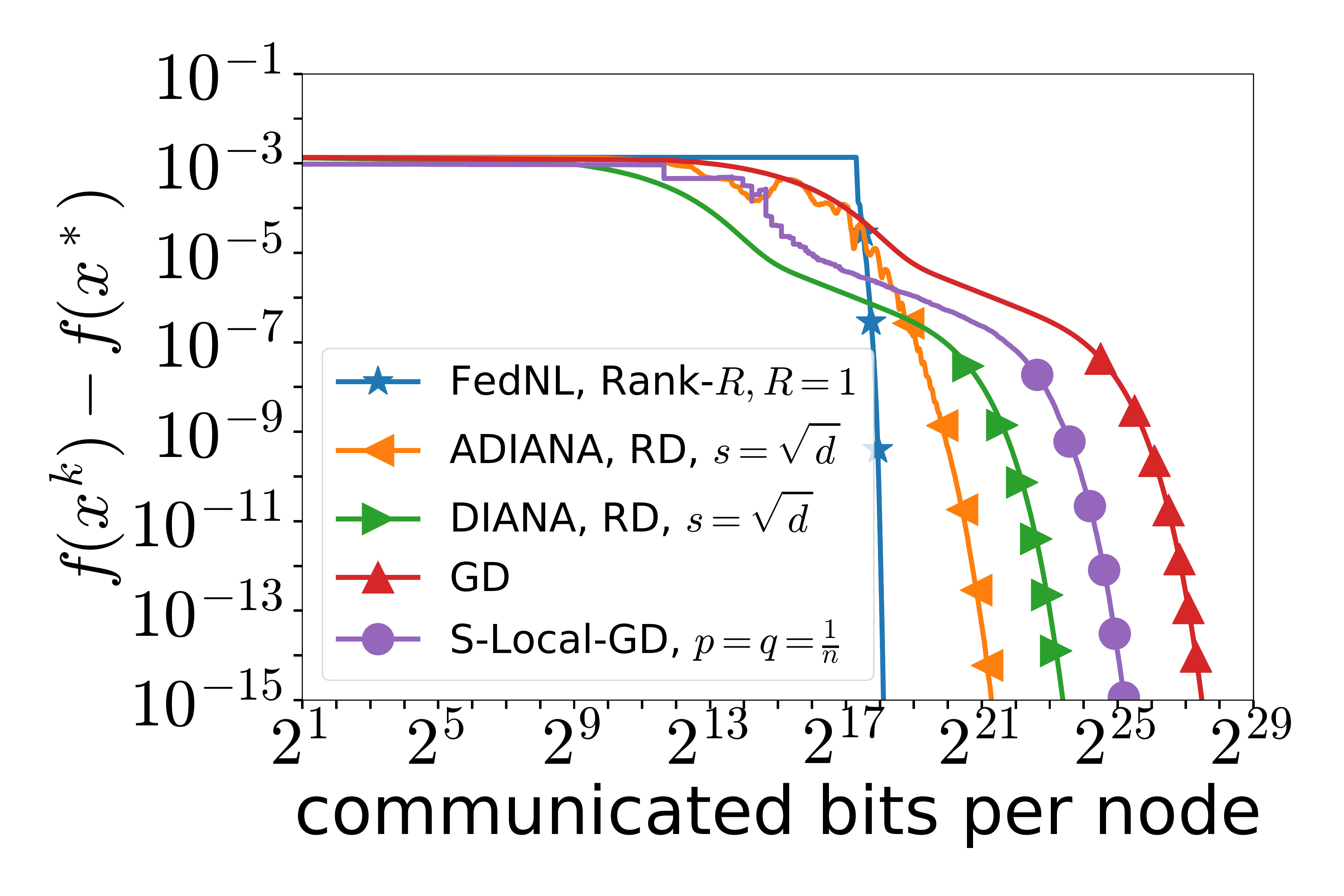} &
            \includegraphics[width=0.23\linewidth]{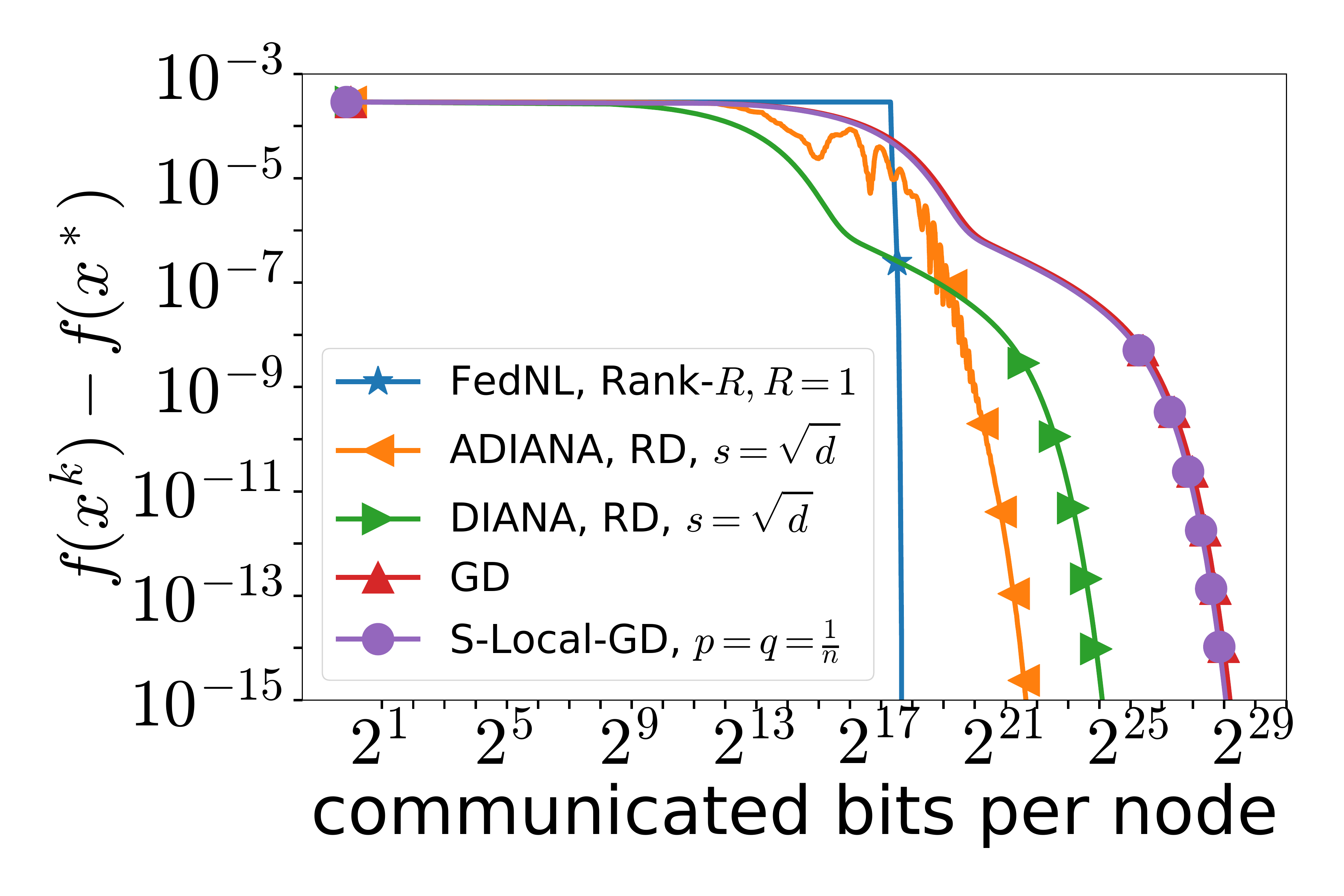} &
            \includegraphics[width=0.23\linewidth]{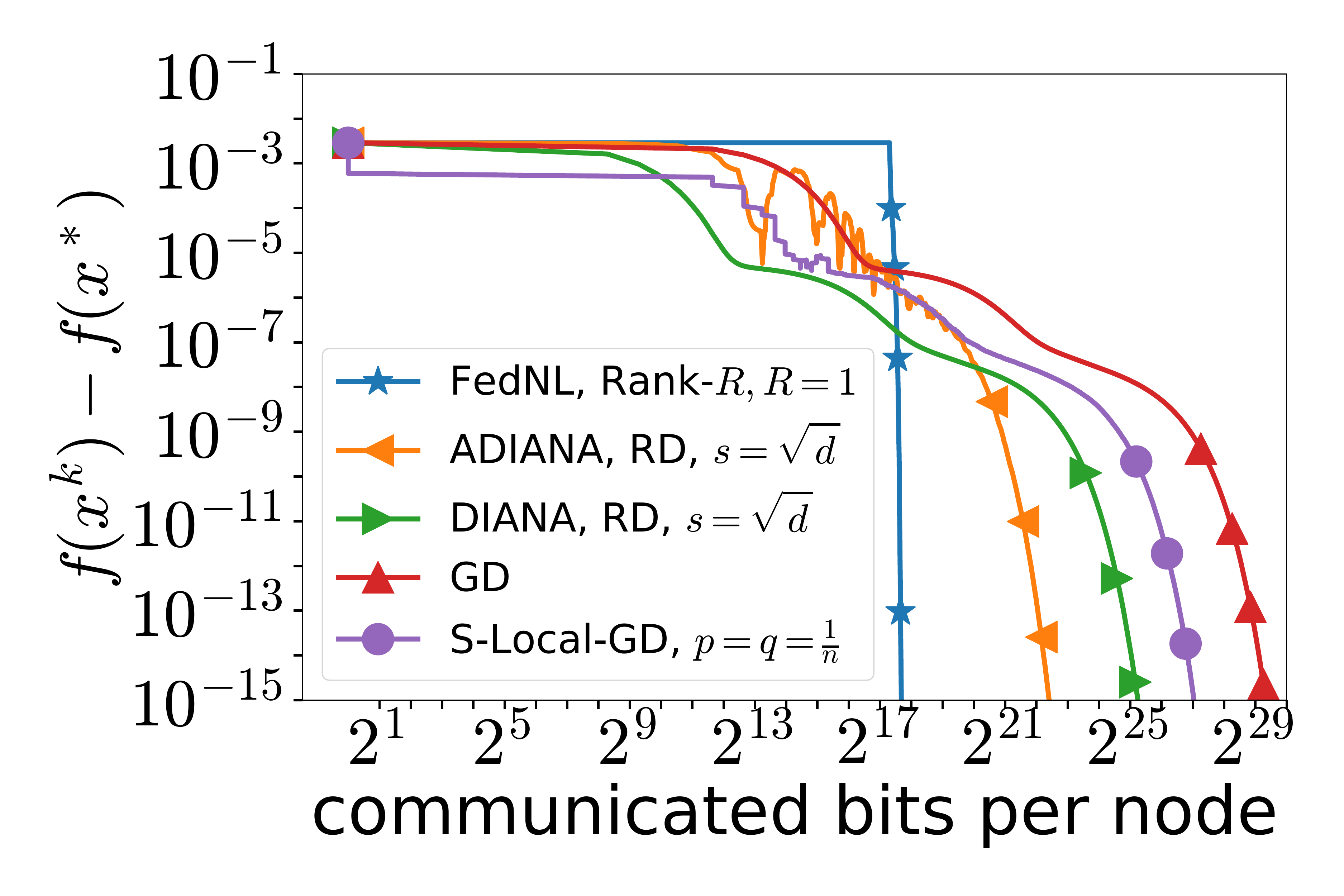}\\
            (a) \dataname{IID} &
            (b) \dataname{Synthetic}($0.5, 0.5$) &
            (c) \dataname{Synthetic}($0.75, 0.75$) &
            (d) \dataname{Syntethic}($1, 1$)\\
            \includegraphics[width=0.23\linewidth]{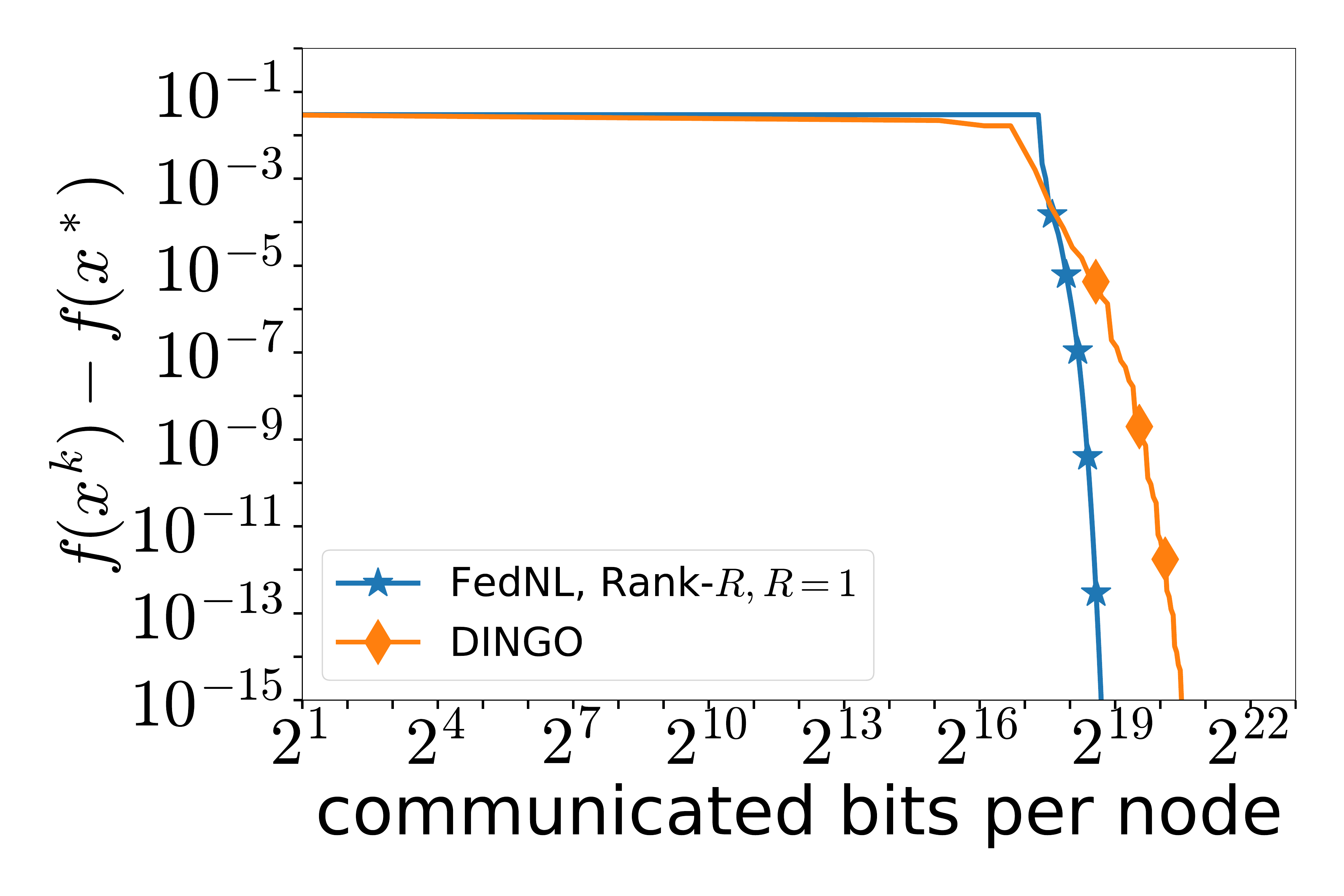} &
            \includegraphics[width=0.23\linewidth]{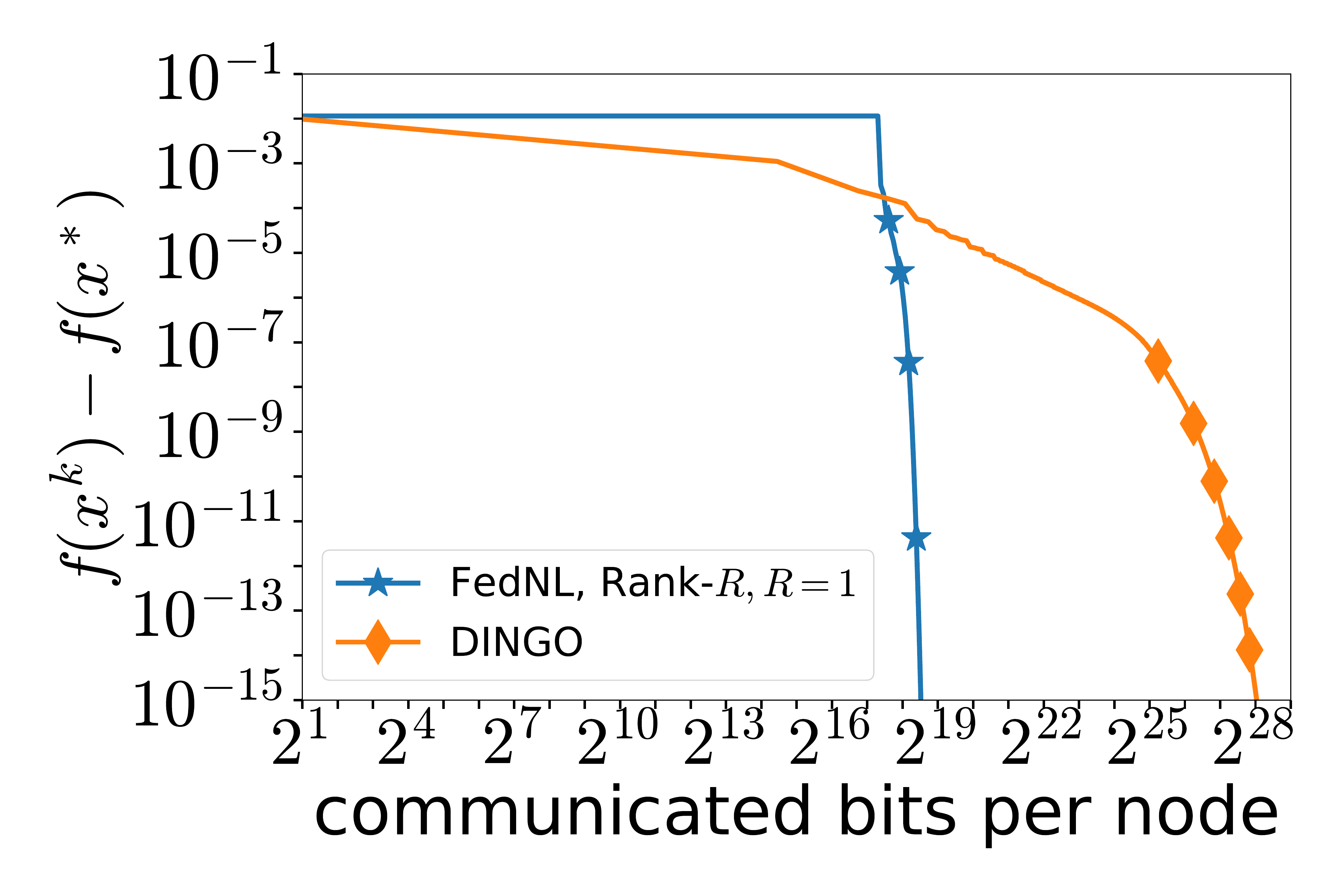} &
            \includegraphics[width=0.23\linewidth]{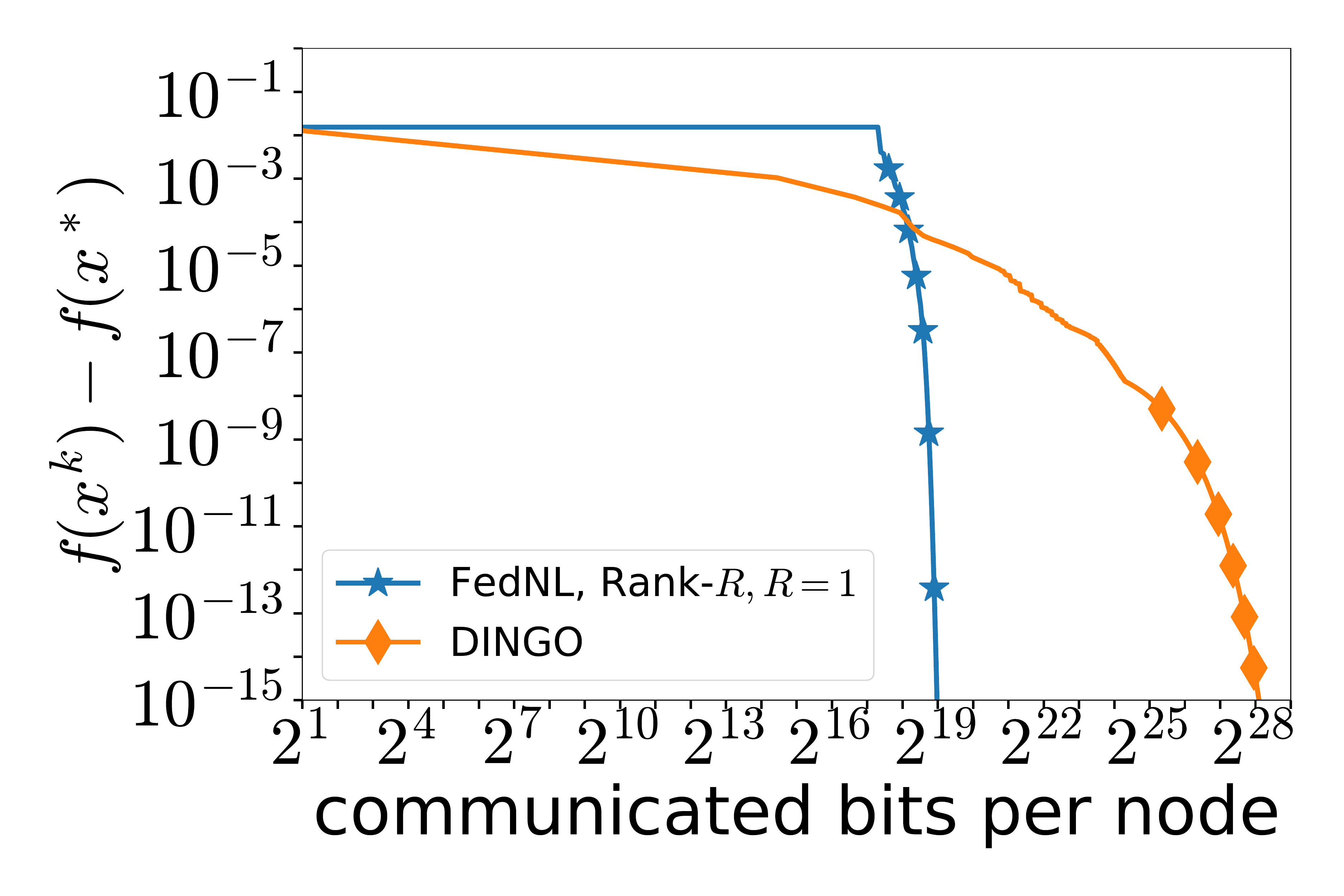} &
            \includegraphics[width=0.23\linewidth]{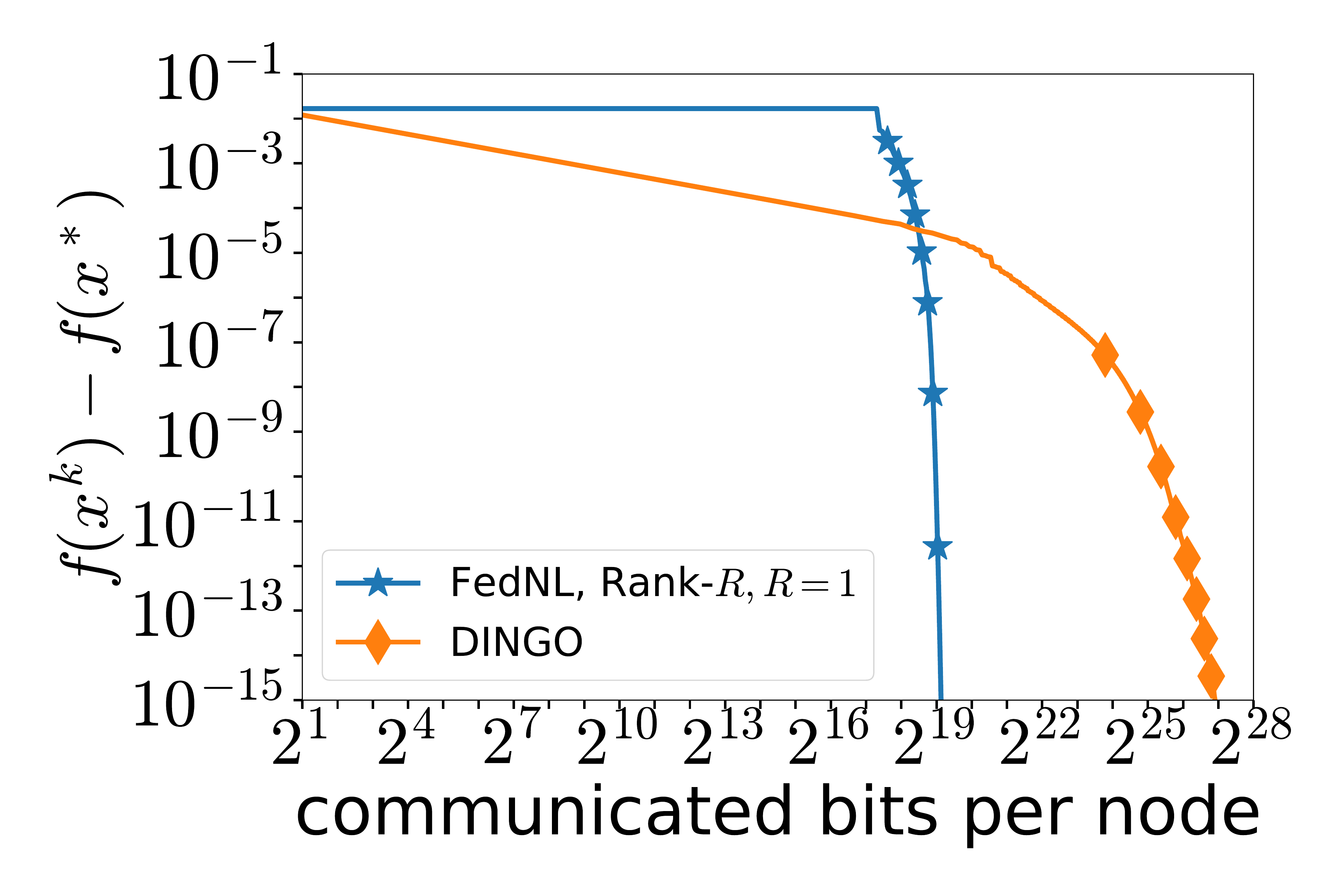}\\
            (a) \dataname{IID} &
            (b) \dataname{Synthetic}($0.5, 0.5$) &
            (c) \dataname{Synthetic}($0.75, 0.75$) &
            (d) \dataname{Syntethic}($1, 1$)\\
        \end{tabular}
    \end{center}
    \caption{{\bf First row:} Local comparison of \algname{FedNL},  with \algname{ADIANA}, \algname{DIANA}, \algname{S-Local-GD}, and \algname{GD} in terms of communication complexity. {\bf Second row:} Local comparison of \algname{FedNL} with \algname{DINGO} in terms of communication complexity.}
    \label{FIG:data_heterogeneity}
\end{figure*}

\section{Proofs of Results from Section~\ref{sec:FedNL-main-paper}}
\subsection{Auxiliary lemma}

Denote by $\E_k\[\cdot\]$ the conditional expectation given $k^{th}$ iterate $x^k$. We first develop a lemma to handle different cases of compressors for $\mathbb{E}_k \|\mH_i^k + \alpha \cC_i^k(\nabla^2 f_i(y) - \mH_i^k) - \nabla^2 f_i(z)\|^2_{\rm F}$, where $\mathbb{E}_k[y] = y$ and $\mathbb{E}_k[z] = z$.

\begin{lemma}\label{lm:threecomp}
  For any $y$, $z\in \R^d$ such that $\mathbb{E}_k[y] = y$ and $\mathbb{E}_k[z] = z$, we have the following results in different cases. 

  (i) If $\cC_i^k \in \B(\omega)$ and $\alpha \leq \frac{1}{\omega+1}$, then
  $$
  \E_k \[\|\mH_i^k + \alpha \cC_i^k(\nabla^2 f_i(y) - \mH_i^k) - \nabla^2 f_i(z)\|^2_{\rm F}\]
  \leq (1-\alpha) \|\mH_i^k - \nabla^2 f_i(z) \|_{\rm F}^2 + \alpha \HF^2 \|y-z\|^2. 
  $$
  (ii) If $\cC_i^k \in \bC(\delta)$ and $\alpha = 1 - \sqrt{1-\delta}$, then
  $$
  \E_k \[\|\mH_i^k + \alpha \cC_i^k(\nabla^2 f_i(y) - \mH_i^k) - \nabla^2 f_i(z)\|^2_{\rm F}\]
  \leq (1-\alpha^2)\left\|\mH_i^{k} - \nabla^2 f_i(z)\right\|^2_{\rm F} + \alpha \HF^2\|y-z\|^2. 
  $$
  (iii) If $\cC_i^k \in \bC(\delta)$ and $\alpha=1$, then
  $$
  \E_k \[\|\mH_i^k + \alpha \cC_i^k(\nabla^2 f_i(y) - \mH_i^k) - \nabla^2 f_i(z)\|^2_{\rm F}\]
  \leq \left(1-\frac{\delta}{4} \right) \|\mH_i^k - \nabla^2 f_i(z)\|^2_{\rm F} + \left(  \frac{6}{\delta} - \frac{7}{2}  \right) \HF^2 \| y-z\|^2. 
  $$ 
  Using the notation from \eqref{ABCD}, we can unify the above three cases into
  $$
  \E_k \[\|\mH_i^k + \alpha \cC_i^k(\nabla^2 f_i(y) - \mH_i^k) - \nabla^2 f_i(z)\|^2_{\rm F}\]
  \leq \left(1-A\right) \|\mH_i^k - \nabla^2 f_i(z)\|^2_{\rm F} + B \HF^2 \| y-z\|^2. 
  $$ 
\end{lemma}

\begin{proof}

Let $$LHS \eqdef  \E_k \[\|\mH_i^k + \alpha \cC_i^k(\nabla^2 f_i(y) - \mH_i^k) - \nabla^2 f_i(z)\|^2_{\rm F}\]$$ be the left hand side appearing in these inequalities.
  
  {\it (i).} If $\cC_i^k \in \B(\omega)$, then
  \begin{eqnarray*}
LHS  &=& \|\mH_i^k - \nabla^2 f_i(z) \|_{\rm F}^2 + 2\alpha \langle \mH_i^k - \nabla^2 f_i(z), \nabla^2 f_i(y) - \mH_i^k \rangle +  \alpha^2 \mathbb{E}_k \|\cC_i^k(\nabla^2 f_i(y) - \mH_i^k)\|_{\rm F}^2\\ 
  &\leq& \|\mH_i^k - \nabla^2 f_i(z) \|_{\rm F}^2 + 2\alpha \langle \mH_i^k - \nabla^2 f_i(z), \nabla^2 f_i(y) - \mH_i^k \rangle + \alpha^2(\omega+1) \| \mH_i^k - \nabla^2 f_i(y)\|_{\rm F}^2. 
  \end{eqnarray*}
  
  Using the stepsize restriction $\alpha \leq \frac{1}{\omega+1}$, we can bound $\alpha^2(\omega+1)\leq \alpha$. Plugging this back to the above inequality and using the identity $2\<\mA, \mB\>_{\rm F} + \|\mB\|^2_{\rm F} = - \|\mA\|_{\rm F}^2 + \|\mA+\mB\|_{\rm F}^2$, we get 
  \begin{eqnarray*}
LHS  &\leq& \|\mH_i^k - \nabla^2 f_i(z) \|_{\rm F}^2 + 2\alpha \langle \mH_i^k - \nabla^2 f_i(z), \nabla^2 f_i(y) - \mH_i^k \rangle + \alpha \| \mH_i^k - \nabla^2 f_i(y)\|_{\rm F}^2\\ 
  &=& (1-\alpha) \|\mH_i^k - \nabla^2 f_i(z) \|_{\rm F}^2 + \alpha \|\nabla^2 f_i(y) - \nabla^2 f_i(z)\|_{\rm F}^2 \\ 
  &\leq& (1-\alpha) \|\mH_i^k - \nabla^2 f_i(z) \|_{\rm F}^2 + \alpha \HF^2 \|y-z\|^2. 
  \end{eqnarray*}

  {\it (ii).} Let $\cC_i^k \in \bC(\delta)$ and $\alpha = 1 - \sqrt{1-\delta}$. Denote
  $$
  \mA \eqdef \mH_i^k - \nabla^2 f_i(z), \quad \mB \eqdef \nabla^2 f_i(y) - \mH_i^k.
  $$
  
  Then
  \begin{eqnarray}
LHS  &=&  \left\|\mA + \alpha \cC_i^k(\mB) \right\|^2_{\rm F} \notag \\
  &=&  \left\|\mA\right\|^2_{\rm F}
  + 2\alpha\<\mA, \cC_i^k(\mB)\>_{\rm F}
  + \alpha^2\|\cC_i^k(\mB)\|^2_{\rm F} \notag \\
  &\le&\left\|\mA\right\|^2_{\rm F}
  + 2\alpha\<\mA, \mB\>_{\rm F}
  + 2\alpha\<\mA, \cC_i^k(\mB) - \mB\>_{\rm F}
  + \alpha^2\|\mB\|^2_{\rm F} \notag \\
  &\le&\left\|\mA\right\|^2_{\rm F}
  + 2\alpha\<\mA, \mB\>_{\rm F}
  + 2\alpha \|\mA\|_{\rm F} \|\cC_i^k(\mB) - \mB\|_{\rm F}
  + \alpha^2\|\mB\|^2_{\rm F} \notag \\
  &\le&\left\|\mA\right\|^2_{\rm F}
  + 2\alpha\<\mA, \mB\>_{\rm F}
  + 2\alpha\sqrt{1-\delta} \|\mA\|_{\rm F} \|\mB\|_{\rm F}
  + \alpha^2\|\mB\|^2_{\rm F} \notag \\
  &\le&\left\|\mA\right\|^2_{\rm F}
  + 2\alpha\<\mA, \mB\>_{\rm F}
  + \alpha\sqrt{1-\delta} \( \|\mA\|_{\rm F}^2 + \|\mB\|_{\rm F}^2 \)
  + \alpha^2\|\mB\|^2_{\rm F} \notag \\
  &\le&(1+\alpha\sqrt{1-\delta})\left\|\mA\right\|^2_{\rm F}
  + 2\alpha\<\mA, \mB\>_{\rm F}
  + (\alpha\sqrt{1-\delta} + \alpha^2)\|\mB\|^2_{\rm F} \notag.
  \end{eqnarray}
  Since $\alpha = 1 - \sqrt{1-\delta}$, we have $\alpha\sqrt{1-\delta} + \alpha^2 = \alpha$. Using the identity $2\<\mA, \mB\>_{\rm F} + \|\mB\|^2_{\rm F} = - \|\mA\|_{\rm F}^2 + \|\mA+\mB\|_{\rm F}^2$, we get
  \begin{eqnarray}
LHS  &\le&(1+\alpha\sqrt{1-\delta})\left\|\mA\right\|^2_{\rm F}
  + 2\alpha\<\mA, \mB\>_{\rm F}
  + \alpha\|\mB\|^2_{\rm F} \notag \\
  &=&  (1+\alpha\sqrt{1-\delta} - \alpha)\left\|\mA\right\|^2_{\rm F}
  + \alpha\|\mA+\mB\|^2_{\rm F} \notag \\
  &=&  (1 - \alpha^2)\left\|\mA\right\|^2_{\rm F}
  + \alpha\|\mA+\mB\|^2_{\rm F} \notag \\
  &=&  (1-\alpha^2)\left\|\mH_i^{k} - \nabla^2 f_i(z)\right\|^2_{\rm F}
  + \alpha\|\nabla^2 f_i(y) - \nabla^2 f_i(z)\|_{\rm F}^2 \notag \\
  &\le&  (1-\alpha^2)\left\|\mH_i^{k} - \nabla^2 f_i(z)\right\|^2_{\rm F}
  + \alpha \HF^2\|y-z\|^2 \notag.
  \end{eqnarray}

  {\it (iii).} If $\cC_i^k \in \bC(\delta)$ and $\alpha=1$, we have 
  \begin{eqnarray*}
LHS  &=&  \|\mH_i^k + \cC_i^k(\nabla^2 f_i(y) - \mH_i^k) - \nabla^2 f_i(z)\|^2_{\rm F} \\ 
  &=& \|\mH_i^k - \nabla^2 f_i(y) + \cC_i^k(\nabla^2 f_i(y) - \mH_i^k) + \nabla^2 f_i(y) - \nabla^2 f_i(z)\|^2_{\rm F}\\ 
  &\leq& (1 + \beta)  \|\mH_i^k - \nabla^2 f_i(y) + \cC_i^k(\nabla^2 f_i(y) - \mH_i^k) \|^2_{\rm F} + \left(  1 + \frac{1}{\beta}  \right)  \| \nabla^2 f_i(y) - \nabla^2 f_i(z)\|^2_{\rm F} \\ 
  &\leq& (1+\beta)(1-\delta) \|\mH_i^k - \nabla^2 f_i(y)\|^2_{\rm F} +  \left(  1 + \frac{1}{\beta}  \right)  \| \nabla^2 f_i(y) - \nabla^2 f_i(z)\|^2_{\rm F}, 
  \end{eqnarray*}
  where we use Young's inequality in the first inequality for some $\beta>0$, and use the contraction property in the last inequality. By choosing $\beta = \frac{\delta}{2(1-\delta)}$ when $0<\delta<1$, we can get 
  \begin{eqnarray*}
LHS  &\leq& \left(1-\frac{\delta}{2} \right) \|\mH_i^k - \nabla^2 f_i(y)\|^2_{\rm F} + \left(  \frac{2}{\delta} -1  \right)  \| \nabla^2 f_i(y) - \nabla^2 f_i(z)\|^2_{\rm F} \\ 
  &\leq& \left(1-\frac{\delta}{2} \right) \|\mH_i^k - \nabla^2 f_i(y)\|^2_{\rm F} + \left(  \frac{2}{\delta} -1  \right) \HF^2 \| y-z\|^2. 
  \end{eqnarray*}
  When $\delta=1$, $$ LHS =  \| \nabla^2 f_i(y) - \nabla^2 f_i(z)\|^2_{\rm F} \leq \HF^2 \| y-z\|^2.$$ Overall, for any $0<\delta\leq 1$ we have 
  \begin{eqnarray*}
LHS  &\leq&  \left(1-\frac{\delta}{2} \right) \|\mH_i^k - \nabla^2 f_i(y)\|^2_{\rm F} + \left(  \frac{2}{\delta} -1  \right) \HF^2 \| y-z\|^2 \\ 
  &\leq& (1+\beta)\left(1-\frac{\delta}{2} \right) \|\mH_i^k - \nabla^2 f_i(z)\|^2_{\rm F} + \left(  1 + \frac{1}{\beta}  \right)\left(1-\frac{\delta}{2} \right) \|\nabla^2 f_i(y) - \nabla^2 f_i(z)\|^2_{\rm F} \\
  && \quad +  \left(  \frac{2}{\delta} -1  \right) \HF^2 \| y-z\|^2. 
  \end{eqnarray*}
  By choosing $\beta = \frac{\delta}{4-2\delta}$, we arrive at 
  \begin{eqnarray*}
LHS  &\leq& \left(1-\frac{\delta}{4} \right) \|\mH_i^k - \nabla^2 f_i(z)\|^2_{\rm F} + \left(  \frac{4}{\delta} + \frac{\delta}{2} -3 + \frac{2}{\delta} -1  \right)\HF^2 \| y-z\|^2\\ 
  &\leq& \left(1-\frac{\delta}{4} \right) \|\mH_i^k - \nabla^2 f_i(z)\|^2_{\rm F} + \left(  \frac{6}{\delta} - \frac{7}{2} \right)\HF^2 \| y-z\|^2. 
  \end{eqnarray*}
\end{proof}

\subsection{Proof of Theorem \ref{th:NLU}} 

We derive recurrence relation for $\|x^k-x^*\|^2$ covering both options of updating the global model. If {\em Option 1.} is used in \algname{FedNL}, then

\begin{eqnarray}
\|x^{k+1} - x^*\|^2
&=&   \left\|x^k-x^* - \[\mH^{k}_{\mu}\]^{-1} \nabla f(x^k) \right\|^2 \notag \\
&\le& \left\| \[\mH^{k}_{\mu}\]^{-1} \right\|^2 \left\|\mH^{k}_{\mu}(x^k-x^*) - \nabla f(x^k))\right\|^2 \notag \\
&\le& \frac{2}{\mu^2}\( \left\|\(\mH_{\mu}^{k} - \nabla^2 f(x^*)\)(x^k-x^*) \right\|^2 + \left\|\nabla^2 f(x^*)(x^k-x^*) - \nabla f(x^k) + \nabla f(x^*) \right\|^2\) \notag \\
&=& \frac{2}{\mu^2}\( \left\|\(\mH_{\mu}^{k} - \nabla^2 f(x^*)\)(x^k-x^*) \right\|^2 + \left\| \nabla f(x^k) - \nabla f(x^*) - \nabla^2 f(x^*)(x^k-x^*) \right\|^2\) \notag \\
&\le& \frac{2}{\mu^2}\(
\left\|\mH_{\mu}^{k} - \nabla^2 f(x^*)\right\|^2 \|x^k-x^*\|^2
+ \frac{\HS^2}{4}\|x^k-x^*\|^4
\) \notag \\
&=&   \frac{2}{\mu^2}\|x^k-x^*\|^2 \(
\left\|\mH_{\mu}^{k} - \nabla^2 f(x^*)\right\|^2
+ \frac{\HS^2}{4}\|x^k-x^*\|^2
\) \notag \\
&\le& \frac{2}{\mu^2}\|x^k-x^*\|^2 \(
\left\|\mH^{k} - \nabla^2 f(x^*)\right\|^2
+ \frac{\HS^2}{4}\|x^k-x^*\|^2
\) \notag \\
&\le& \frac{2}{\mu^2}\|x^k-x^*\|^2 \(
\left\|\mH^{k} - \nabla^2 f(x^*)\right\|^2_{\rm F}
+ \frac{\HS^2}{4}\|x^k-x^*\|^2
\)  \notag, 
\end{eqnarray}
where we use $\mH_\mu^k \succeq \mu \mI$ in the second inequality, and $\nabla^2 f(x^*) \succeq \mu \mI$ in the fourth inequality. From the convexity of $\|\cdot \|^2_{\rm F}$, we have 
$$
\|\mH^k - \nabla^2 f(x^*)\|^2_{\rm F} = \left\| \frac{1}{n}\sum_{i=1}^n \left(  \mH_i^k - \nabla^2 f_i(x^*)  \right) \right\|^2_{\rm F} \leq \frac{1}{n}\sum_{i=1}^n \|\mH_i^k - \nabla^2 f_i(x^*)\|^2_{\rm F} = {\cal H}^k. 
$$

Thus, 
\begin{equation}\label{eq:xk+1option1}
\|x^{k+1} - x^*\|^2 \leq \frac{2}{\mu^2}\|x^k-x^*\|^2 {\cal H}^k + \frac{\HS^2}{2\mu^2} \|x^k-x^*\|^4. 
\end{equation}

If {\em Option 2.} is used in \algname{FedNL}, then as $\mH^k + l^k\mI \succeq \nabla^2 f(x^k) \succeq \mu \mI$ and $\nabla f(x^*) = 0$, we have 
\begin{align*}
\|x^{k+1} - x^*\| &= \|x^k - x^* - [\mH^k + l^k\mI]^{-1} \nabla f(x^k) \| \\
& \leq \|[\mH^k + l^k \mI]^{-1}\| \cdot \|(\mH^k + l^k \mI) (x^k-x^*) - \nabla f(x^k) + \nabla f(x^*)\| \\ 
& \leq \frac{1}{\mu} \|(\mH^k + l^k \mI - \nabla^2 f(x^*))(x^k-x^*)\| + \frac{1}{\mu} \|\nabla f(x^k) - \nabla f(x^*) - \nabla^2 f(x^*) (x^k-x^*)\| \\ 
& \leq \frac{1}{\mu} \|\mH^k + l^k\mI - \nabla^2 f(x^*)\| \|x^k-x^*\| + \frac{\HS}{2\mu}\|x^k-x^*\|^2 \\ 
& \leq \frac{1}{n\mu} \sum_{i=1}^n \|\mH_i^k + l_i^k\mI - \nabla^2 f_i(x^*)\| \|x^k-x^*\| + \frac{\HS}{2\mu}\|x^k-x^*\|^2 \\ 
& \leq \frac{1}{n\mu} \sum_{i=1}^n (\|\mH_i^k - \nabla^2 f_i(x^*)\| + l_i^k )\|x^k-x^*\| +  \frac{\HS}{2\mu}\|x^k-x^*\|^2. 
\end{align*}

From the definition of $l_i^k$, we have 
$$
l_i^k = \|\mH_i^k - \nabla^2 f_i(x^k)\|_{\rm F} \leq \|\mH_i^k - \nabla^2 f_i(x^*)\|_{\rm F} + \HF \|x^k-x^*\|. 
$$
Thus, 
$$
\|x^{k+1} - x^*\|  \leq \frac{2}{n\mu} \sum_{i=1}^n \|\mH_i^k - \nabla^2 f_i(x^*)\|_{\rm F} \|x^k-x^*\| + \frac{\HS+2\HF}{2\mu}\|x^k-x^*\|^2. 
$$
From Young's inequality, we further have 
\begin{align}
\|x^{k+1} - x^*\|^2 & \leq \frac{8}{\mu^2} \left(  \frac{1}{n} \sum_{i=1}^n \|\mH_i^k - \nabla^2 f_i(x^*)\|_{\rm F} \|x^k-x^*\|   \right)^2 + \frac{(\HS+2\HF)^2}{2\mu^2} \|x^k-x^*\|^4 \nonumber \\ 
& \leq \frac{8}{\mu^2} \|x^k-x^*\|^2 \left(  \frac{1}{n} \sum_{i=1}^n \|\mH_i^k - \nabla^2 f_i(x^*)\|^2_{\rm F}  \right) +  \frac{(\HS+2\HF)^2}{2\mu^2} \|x^k-x^*\|^4 \nonumber \\ 
& = \frac{8}{\mu^2} \|x^k-x^*\|^2 {\cal H}^k + \frac{(\HS+2\HF)^2}{2\mu^2} \|x^k-x^*\|^4,  \label{eq:xk+1option2}
\end{align}
where we use the convexity of $\|\cdot\|^2_{\rm F}$ in the second inequality. 

Thus, from (\ref{eq:xk+1option1}) and (\ref{eq:xk+1option2}), we have the following unified bound for both {\em Option 1} and {\em Option 2}:
\begin{equation}\label{eq:xk+1U}
\|x^{k+1} - x^*\|^2 \leq \frac{C}{\mu^2} \|x^k-x^*\|^2 {\cal H}^k + \frac{D}{2\mu^2} \|x^k-x^*\|^4. 
\end{equation}

Assume $\|x^0-x^*\|^2 \leq \frac{\mu^2}{2D}$ and ${\cal H}^k \leq \frac{\mu^2}{4C}$ for all $k\geq 0$. Then we show that $\|x^k-x^*\|^2 \leq \frac{\mu^2}{2D}$ for all $k\geq 0$ by induction. Assume  $\|x^k-x^*\|^2 \leq \frac{\mu^2}{2D}$ for all $k \leq K$. Then from (\ref{eq:xk+1U}), we have 
\begin{align*}
\|x^{K+1} - x^*\|^2 & \leq \frac{1}{4}\|x^K-x^*\|^2 + \frac{1}{4}\|x^K-x^*\|^2 \leq \frac{\mu^2}{2D}. 
\end{align*} 
Thus we have $\|x^k-x^*\|^2 \leq \frac{\mu^2}{2D}$ and ${\cal H}^k \leq \frac{\mu^2}{4C}$ for $k\geq 0$. Using (\ref{eq:xk+1U}) again, we obtain 
\begin{equation}\label{eq:xk+1Ufix}
\|x^{k+1} - x^*\|^2 \leq \frac{1}{2} \|x^k-x^*\|^2. 
\end{equation}

Choosing $y=x^k$ and $z=x^*$ in Lemma \ref{lm:threecomp}, we get
$$
\mathbb{E}_k \|\mH_i^k + \alpha \cC_i^k(\nabla^2 f_i(x^k) - \mH_i^k) - \nabla^2 f_i(x^*)\|^2_{\rm F} \leq (1-A) \|\mH_i^k - \nabla^2 f_i(x^*) \|_{\rm F}^2 + B \HF^2 \|x^k-x^*\|^2. 
$$
Then by $\mH_i^{k+1} = \mH_i^k + \alpha \cC_i^k(\nabla^2 f_i(x^k) - \mH_i^k)$, we have 
$$
\mathbb{E}_k[{\cal H}^{k+1}] \leq (1-A) {\cal H}^k + B\HF^2 \|x^k-x^*\|^2. 
$$

Using the above inequality and (\ref{eq:xk+1Ufix}), for Lyapunov function $\Phi^k$ we deduce
\begin{align*}
\mathbb{E}_k[\Phi^{k+1}] & \leq (1-A) {\cal H}^k + B\HF^2 \|x^k-x^*\|^2 + 3B\HF^2 \|x^k-x^*\|^2 \\ 
& =  (1-A) {\cal H}^k + \left(  1 - \frac{1}{3}  \right)6B\HF^2 \|x^k-x^*\|^2 \\ 
& \leq \left(  1 - \min\left\{  A, \frac{1}{3}  \right\}  \right) \Phi^k. 
\end{align*}

Hence $\mathbb{E}[\Phi^k] \leq \left(  1 - \min\left\{  A, \frac{1}{3}  \right\}  \right)^k \Phi^0$. We further have $\mathbb{E}[{\cal H}^k] \leq \left(  1 - \min\left\{  A, \frac{1}{3}  \right\}  \right)^k \Phi^0$ and $\mathbb{E}[\|x^k-x^*\|^2] \leq \frac{1}{6B\HF^2} \left(  1 - \min\left\{  A, \frac{1}{3}  \right\}  \right)^k \Phi^0$ for $k\geq 0$. Assume $x^k\neq x^*$ for all $k$. Then from (\ref{eq:xk+1U}), we have 
$$
\frac{\|x^{k+1}-x^*\|^2}{\|x^k-x^*\|^2} \leq \frac{C}{\mu^2}{\cal H}^k + \frac{D}{2\mu^2}\|x^k-x^*\|^2, 
$$
and by taking expectation, we have 
\begin{align*}
\mathbb{E} \left[  \frac{\|x^{k+1}-x^*\|^2}{\|x^k-x^*\|^2}  \right] & \leq \frac{C}{\mu^2} \mathbb{E}[{\cal H}^k] + \frac{D}{2\mu^2} \mathbb{E}[\|x^k-x^*\|^2] \\ 
& \leq  \left(  1 - \min\left\{  A, \frac{1}{3}  \right\}  \right)^k \left(  C + \frac{D}{12B\HF^2}  \right) \frac{\Phi^0}{\mu^2}. 
\end{align*}

\subsection{Proof of Lemma \ref{lm:boundforbiased}} 

We prove this by induction. Assume $\|\mH_i^k - \nabla^2 f_i(x^*)\|^2_{\rm F}  \leq \frac{\mu^2}{4C}$ and $\|x^k-x^*\|^2 \leq  \min\{  \frac{A\mu^2}{4BC\HF^2}, \frac{\mu^2}{2D}  \}$ for $k\leq K$. Then we also have ${\cal H}^k \leq \frac{\mu^2}{4C}$ for $k\leq K$. From (\ref{eq:xk+1U}), we can get 
\begin{align*}
\|x^{K+1} - x^*\|^2 & \leq \frac{C}{\mu^2} \|x^K-x^*\|^2 {\cal H}^K + \frac{D}{2\mu^2} \|x^K-x^*\|^4 \\ 
& \leq \frac{1}{4}\|x^K-x^*\|^2 + \frac{1}{4} \|x^K-x^*\|^2 \\ 
& \leq \min \left\{  \frac{A\mu^2}{4BC\HF^2}, \frac{\mu^2}{2D} \right \}. 
\end{align*}

From Lemma \ref{lm:threecomp}, by choosing $y=x^k$ and $z=x^*$, for all $i\in [n]$, we have 
\begin{align*}
 \|\mH_i^{K+1} - \nabla^2 f_i(x^*)\|^2_{\rm F} 
& = \mathbb{E}_k \|\mH_i^K + \alpha \cC_i^k(\nabla^2 f_i(x^K) - \mH_i^K) - \nabla^2 f_i(x^*)\|^2_{\rm F} \\ 
& \leq (1-A) \|\mH_i^K - \nabla^2 f_i(x^*) \|_{\rm F}^2 + B \HF^2 \|x^K-x^*\|^2 \\ 
& \leq (1-A) \frac{\mu^2}{4C} + B\HF^2 \cdot \frac{A\mu^2}{4BC\HF^2} \\ 
& = \frac{\mu^2}{4C}. 
\end{align*}

\subsection{Proof of Lemma \ref{lm:boundforspar}} 

Notice that Assumption \ref{asm:comp-2} implies $\mH_i^0 = \nabla^2 f_i(x^0)$, from which we have 
$$
\|\mH_i^0 - \nabla^2 f_i(x^*)\|^2_{\rm F}  = \sum_{j, l} |(\nabla^2 f_i(x^0) - \nabla^2 f_i(x^*))_{jl}|^2 \leq d^2 \HM^2 \frac{\mu^2}{D + 4Cd^2\HM^2} \leq \frac{\mu^2}{4C}, 
$$
which implies ${\cal H}^0 \leq \frac{\mu^2}{4C}$. Next we prove $\|x^k - x^*\|^2 \leq \frac{\mu^2}{D + 4Cd^2\HM^2}$ for all $k\geq 0$ by induction. Assume $\|x^k - x^*\|^2 \leq \frac{\mu^2}{D + 4Cd^2\HM^2}$ for $k\leq K$. Since $(\mH_i^k)_{jl}$ is a convex combination of $\{  (\nabla^2 f_i(x^0))_{jl}, ..., (\nabla^2 f_i(x^k))_{jl}  \}$ for all $i\in [n]$, $j, l\in [d]$, from the convexity of $|\cdot|^2$, we have 
$$
|(\mH_i^k - \nabla^2 f_i(x^*))_{jl}|^2 \leq \HM^2 \cdot \frac{\mu^2}{D + 4Cd^2\HM^2} \leq \frac{\mu^2}{4Cd^2}, 
$$
for $k\leq K$. Then we can get $\|\mH_i^k - \nabla^2 f_i(x^*)\|^2_{\rm F} \leq \frac{\mu^2}{4C}$ and thus ${\cal H}^k \leq \frac{\mu^2}{4C}$ for $k\leq K$. From (\ref{eq:xk+1U}), we have 
\begin{align*}
\|x^{K+1} - x^*\|^2 &\leq \frac{C}{\mu^2} \|x^K-x^*\|^2 {\cal H}^K + \frac{D}{2\mu^2}\|x^K-x^*\|^4 \\ 
& \leq \frac{1}{4}\|x^K-x^*\|^2 + \frac{1}{2}\|x^K-x^*\|^2 \\ 
& \leq \frac{\mu^2}{D + 4Cd^2\HM^2}. 
\end{align*}


\section{Extension: Partial Participation (\algname{FedNL-PP})}\label{apx:FedNL-PP}

Our first extension to the vanilla \algname{FedNL} is to handle partial participation: a setup when in each iteration only randomly selected clients participate. This is important when the number $n$ of devices is very large.

\begin{algorithm}[h!]
  \caption{\algname{FedNL-PP} (Federated Newton Learn with {\color{blue}Partial Participation})}
  \label{alg:FedNL-PP}
  \begin{algorithmic}[1]
    \STATE {\bfseries Parameters:} Hessian learning rate $\alpha>0$; compression operators $\{\cC_1^k, \dots,\cC_n^k\}$; {\color{blue}number of participating devices $\tau \in \{1,2,\dots,n\}$}
    \STATE {\bfseries Initialization:}
    For all $i\in [n]$: $w^0_i = x^0 \in \R^d$; $\mH_i^0 \in \R^{d\times d}$; $l_i^0 = \|\mH_i^{0} - \nabla^2 f_i(w_i^{0})\|_{\rm F}$; $g_i^0 = (\mH_i^{0} + l_i^{0} \mI)w_i^{0} - \nabla f_i(w_i^{0})$; Moreover: $\mH^0 = \frac{1}{n} \sum_{i=1}^n \mH_i^0$; $l^0 = \frac{1}{n} \sum_{i=1}^n l_i^0$; $g^0 = \frac{1}{n} \sum_{i=1}^n g_i^0$
    \STATE \textbf{on} server
    \STATE ~~~ $x^{k+1} = \left(  \mH^k + l^k\mI  \right)^{-1} g^k$ \hfill { \scriptsize Main step: Update the global model}
    \STATE ~~~ {\color{blue} Choose a subset $S^{k} \subseteq \{1,\dots, n\}$ of devices of cardinality $\tau$, uniformly at random}
    \STATE ~~~ Send $x^{k+1}$ to {\color{blue} the selected devices $i\in S^k$} \hfill { \scriptsize Communicate to selected clients}
    \FOR{each device $i = 1, \dots, n$ in parallel}
    \STATE {\color{blue} {\bf for participating devices} $i \in S^k$ {\bf do} }
    \STATE $w_i^{k+1} = x^{k+1}$ \hfill { \scriptsize Update local model}
    \STATE $\mH_i^{k+1} = \mH_i^k + \alpha \cC_i^k(\nabla^2 f_i(w_i^{k+1}) - \mH_i^k)$ \hfill { \scriptsize Update local Hessian estimate}
    \STATE $l_i^{k+1} = \|\mH_i^{k+1} - \nabla^2 f_i(w_i^{k+1})\|_{\rm F}$ \hfill { \scriptsize Compute local Hessian error}
    \STATE $g_i^{k+1} = (\mH_i^{k+1} + l_i^{k+1} \mI)w_i^{k+1} - \nabla f_i(w_i^{k+1})$ \hfill { \scriptsize Compute Hessian-corrected local gradient}
    \STATE Send $\cC_i^k(\nabla^2 f_i(w_i^{k+1}) - \mH_i^k)$,\; $l_i^{k+1} - l_i^k$ and $g_i^{k+1} - g_i^k$ to server \hfill { \scriptsize Communicate to server}
    \STATE {\color{blue} {\bf for non-participating devices} $i \notin S^k$ {\bf do} }
    \STATE $w_i^{k+1} = w_i^k$, $\mH_i^{k+1} = \mH_i^k$, $l_i^{k+1} = l_i^k$, $g_i^{k+1} = g_i^k$  \hfill { \scriptsize Do nothing}
    \ENDFOR

    \STATE \textbf{on} server
    \STATE ~~~ $g^{k+1} = g^k + \frac{1}{n}\sum_{i\in S^k} \left(  g_i^{k+1} - g_i^k  \right)$  \hfill { \scriptsize Maintain the relationship $g^k = \frac{1}{n} \sum_{i=1}^n g_i^k$}
    \STATE ~~~ $\mH^{k+1} = \mH^k + \frac{\alpha}{n}\sum_{i\in S^k} \cC_i^k(\nabla^2 f_i(w_i^{k+1}) - \mH_i^k)$   \hfill { \scriptsize Update  the Hessian estimate on the server}

    \STATE ~~~ $l^{k+1} = l^k + \frac{1}{n}\sum_{i\in S^k} \left(  l_i^{k+1} - l_i^k  \right)$ \hfill { \scriptsize Maintain the relationship $l^k = \frac{1}{n} \sum_{i=1}^n l_i^k$}
  \end{algorithmic}
\end{algorithm}

\subsection{Hessian corrected local gradients $g_i^k$}
The key technical novelty in \algname{FedNL-PP} is the structure of local gradients $$g_i^{k} = (\mH_i^{k} + l_i^{k} \mI)w_i^{k} - \nabla f_i(w_i^{k})$$ (see line 12 of Algorithm~\ref{alg:FedNL-PP}). The intuition behind this form is as follows. Because of the partial participation, some devices might remain inactive for several rounds. As a consequence, each device $i$ holds a local model $w_i^k$, which is a stale global model (true global model of the last round client $i$ participated) when the device is inactive. This breaks the analysis of \algname{FedNL} and requires an additional trick to handle stale global models of inactive clients.
The trick is to apply some form of Newton-type step locally and then update the global model at the server in communication efficient manner. In particular, clients use their corrected learned local Hessian estimates $\mH_i^k + l_i^k\mI$ to do Newton-type step from $w_i^k$ to $w_i^k - \[\mH_i^k + l_i^k\mI\]^{-1}\nabla f_i(w_i^k)$, which can be transformed into $$\(\mH_i^k + l_i^k\mI\)^{-1}\[\(\mH_i^k + l_i^k\mI\) w_i^k - \nabla f_i(w_i^k)\] = \(\mH_i^k + l_i^k\mI\)^{-1} g_i^k.$$ Next, all active clients communicate compressed differences $\cC_i^k(\nabla^2 f_i(w_i^{k+1}) - \mH_i^k)$, $l_i^{k+1} - l_i^k$ and $g_i^{k+1} - g_i^k$ to the sever, which then updates global estimates $g^{k+1},\; \mH^{k+1},\; l^{k+1}$ (see lines 18, 19, 20) and the global model $x^{k+1}$ (see line 4).

\subsection{Importance of compression errors $l_i^k$}
Notice that, unlike \algname{FedNL}, here we have only one option to update the global model at the sever (this corresponds to {\em Option~2} of \algname{FedNL}). Although, it is possible to extend the theory also for {\em Option~1}, it would require strong practical requirements. Indeed, in order to carry out the analysis with {\em Option~1}, either all active clients have to compute projected estimates $\[\mH_i^k\]_{\mu}^{-1}$ or the central server needs to maintain this for all clients in each iteration. Although implementable, both variants seem to be too much restrictive from the practical point of view. Compression errors $l_i^k$ mitigate the storage and computation requirements by the cost of sending an extra float per active client.


\subsection{Local convergence theory} We prove three local rates for \algname{FedNL-PP}: for the squared distance of the global model $x^k$ to the solution $\|x^k-x^*\|^2$, averaged squared distance of stale (due to partial participation) local models $w_i^k$ to the solution ${\cal W}^k \eqdef \frac{1}{n} \sum_{i=1}^n \|w_i^k-x^*\|^2$, and for the Lyapunov function $$\Psi^k \eqdef {\cal H}^k + B\HF^2{\cal W}^k.$$

\begin{theorem}\label{th:NL-pp}
  Let Assumption \ref{asm:main} holds and further assume that all loss functions $f_i$ are $\mu$-convex. Suppose $\|x^0 - x^*\|^2 \leq \frac{\mu^2}{4(\HS+2\HF)^2}$ and ${\cal H}^k \leq \frac{\mu^2}{64}$ for all $k\geq 0$. Then, global model $x^k$ and all local models $w_i^k$ of \algname{FedNL-PP} (Algorithm \ref{alg:FedNL-PP}) converge linearly as follows
  \begin{equation*}
  \|x^{k+1} - x^*\|^2 \le {\cal W}^k, \qquad \mathbb{E}\left[{\cal W}^k \right] \leq \left(  1 - \frac{3\tau}{4n}  \right)^k {\cal W}^0.
  \end{equation*}
  Moreover, depending on the choice \eqref{ABCD} of compressors $\cC_i^k$ and step-size $\alpha$, we have linear rates
  \begin{equation}\label{rate-linear-NL-PP}
  \mathbb{E}\left[\Psi^k\right] \leq \left(  1 -  \frac{\tau}{n}\min\left\{  A, \frac{1}{2}  \right\}  \right)^k \Psi^0,
  \end{equation}
  \begin{align}\label{rate-faster-linear-NL-PP}
  \mathbb{E}\left[  \frac{\|x^{k+1} -x^*\|^2}{{\cal W}^k}  \right]  \leq \left(  1 -  \min\left\{  A, \frac{1}{2}  \right\}  \right)^k \left(  \frac{(\HS+2\HF)^2}{2B\HF^2} + 8 \right) \frac{\Psi^0}{\mu^2}. 
  \end{align}
\end{theorem}

Similar to Theorem \ref{th:NLU}, we assumed ${\cal H}^k \leq \frac{\mu^2}{64}$ holds for all iterates $k\geq 0$. Below, we prove that this inequality holds, using the initial conditions only.

\begin{lemma}\label{lm:boundforbiased-pp}
  Let Assumption \ref{asm:comp-1} holds. Assume $\|x^0 - x^*\|^2 \leq e_3 \eqdef \min\{  \frac{A\mu^2}{16B\HF^2}, \frac{\mu^2}{4(\HS+2\HF)^2}  \}$ and $\|\mH_i^0 - \nabla^2 f_i(x^*)\|^2_{\rm F} \leq \frac{\mu^2}{64}$.
  Then $\|x^k-x^*\|^2 \leq  e_3$ and $\|\mH_i^k - \nabla^2 f_i(x^*)\|^2_{\rm F}  \leq \frac{\mu^2}{64}$ for all $k\geq 1$. 
\end{lemma}

\begin{lemma}\label{lm:boundforspar-pp}
  Let Assumption \ref{asm:comp-2} holds and assume $\|x^0 - x^*\|^2 \leq \frac{\mu^2}{(\HS+2\HF)^2 + 64d^2\HM^2}$. Then ${\cal H}^k \leq \frac{\mu^2}{64}$ for all $k\geq 0$. 
\end{lemma}

In the upcoming three subsections we provide the proofs of Theorem~\ref{th:NL-pp}, Lemma~\ref{lm:boundforbiased-pp} and \ref{lm:boundforspar-pp}.

\subsection{Proof of Theorem \ref{th:NL-pp}}

From 
$$
x^{k+1} = \left(  \mH^k + l^k\mI \right)^{-1} g^k =  \left(  \mH^k + l^k \mI  \right)^{-1} \left[  \frac{1}{n} \sum_{i=1}^n (\mH^k_i + l_i^k\mI) w_i^k - \nabla f_i(w_i^k)     \right] , 
$$
and 
$$
x^* = \left(  \mH^k + l^k \mI  \right)^{-1} \left[  (\mH^k + l^k \mI) x^* - \nabla f(x^*)  \right] =  \left(  \mH^k + l^k \mI  \right)^{-1} \left[  \frac{1}{n} \sum_{i=1}^n ( \mH^k_i + l_i^k\mI) x^* - \nabla f_i(x^*)     \right], 
$$
we can obtain 
$$
x^{k+1} - x^* =  \left(  \mH^k + l^k \mI \right)^{-1} \left[  \frac{1}{n} \sum_{i=1}^n  (\mH^k_i + l_i^k\mI) (w_i^k - x^*) - (\nabla f_i(w_i^k) - \nabla f_i(x^*) )    \right]. 
$$
As all functions $f_i$ are $\mu$-convex, we get $\mH^k + l^k\mI \succeq \frac{1}{n} \sum_{i=1}^n \nabla^2 f_i(w_i^k) \succeq  \mu \mI$. Using the triangle inequality, we have
\begin{align*}
\|x^{k+1} - x^*\| &\leq \frac{1}{\mu n} \sum_{i=1}^n \left\| \nabla f_i(w_i^k) - \nabla f_i(x^*) - (\mH^k_i + l_i^k\mI) (w_i^k - x^*)    \right\| \\ 
& \leq \frac{1}{\mu n} \sum_{i=1}^n \left\| \nabla f_i(w_i^k) - \nabla f_i(x^*) - \nabla^2 f_i(x^*) (w_i^k - x^*)  \right\| \\ 
& \quad + \frac{1}{\mu n} \sum_{i=1}^n \left\| (\mH_i^k + l_i^k\mI - \nabla^2 f_i(x^*) (w_i^k - x^*)) \right\| \\ 
& \leq \frac{\HS}{2\mu n} \sum_{i=1}^n \|w_i^k - x^*\|^2 + \frac{1}{\mu n} \sum_{i=1}^n \|\mH_i^k + l_i^k\mI - \nabla^2 f_i(x^*)\| \cdot \|w_i^k - x^*\| \\ 
& \leq \frac{\HS}{2\mu} {\cal W}^k + \frac{1}{\mu n} \sum_{i=1}^n \left(  \|\mH_i^k - \nabla^2 f_i(x^*)\| + l_i^k \right) \cdot \|w_i^k - x^*\|. 
\end{align*}

Recall that 
\begin{eqnarray*}
l_i^k
&=& \|\mH_i^k - \nabla^2 f_i(w_i^k)\|_{\rm F} \\
&\leq& \|\mH_i^k - \nabla^2 f_i(x^*)\|_{\rm F} + \|\nabla^2 f_i(x^*)- \nabla^2 f_i(w_i^k)\|_{\rm F} \\
&\leq& \|\mH_i^k - \nabla^2 f_i(x^*)\|_{\rm F} + \HF \|w_i^k - x^*\|. 
\end{eqnarray*}
Then we arrive at 
\begin{eqnarray*}
&& \|x^{k+1} - x^*\| \\
&\leq& \frac{\HS}{2\mu} {\cal W}^k + \frac{1}{\mu n} \sum_{i=1}^n \left(  \|\mH_i^k - \nabla^2 f_i(x^*)\| + \|\mH_i^k - \nabla^2 f_i(x^*)\|_{\rm F} + \HF \|w_i^k - x^*\| \right) \cdot \|w_i^k - x^*\| \\ 
&\leq& \frac{\HS + 2\HF}{2\mu} {\cal W}^k + \frac{2}{\mu n} \sum_{i=1}^n \|\mH_i^k - \nabla^2 f_i(x^*)\|_{\rm F} \cdot \|w_i^k - x^*\|. 
\end{eqnarray*}

We further use Young's inequality to bound $\|x^{k+1} - x^*\|^2$ as 
\begin{align}
\|x^{k+1} - x^*\|^2 & \leq \frac{(\HS+2\HF)^2}{2\mu^2} ({\cal W}^k)^2 + \frac{8}{\mu^2 n^2} \left(   \sum_{i=1}^n \|\mH_i^k - \nabla^2 f_i(x^*)\|_{\rm F} \cdot \|w_i^k - x^*\|  \right)^2 \nonumber \\ 
& \leq \frac{(\HS+2\HF)^2}{2\mu^2} ({\cal W}^k)^2 + \frac{8}{\mu^2} \left(  \frac{1}{n} \sum_{i=1}^n \|\mH_i^k - \nabla^2 f_i(x^*)\|_{\rm F}^2  \right) {\cal W}^k \nonumber \\ 
& = \frac{(\HS+2\HF)^2}{2\mu^2} ({\cal W}^k)^2 + \frac{8}{\mu^2} {\cal H}^k {\cal W}^k, \label{eq:xk+1-pp}
\end{align}
where we use Cauchy-Schwarz inequality in the second inequality and use ${\cal H}^k =  \frac{1}{n} \sum_{i=1}^n \|\mH_i^k - \nabla^2 f_i(x^*)\|_{\rm F}^2$ in the last equality.
From the update rule of $w_i^k$, we have 
\begin{align}
\mathbb{E}_k [{\cal W}^{k+1}] & = \frac{\tau}{n} \mathbb{E}_k \left[  \|x^{k+1} - x^*\|^2  \right] + \left(  1 - \frac{\tau}{n}  \right) {\cal W}^k \nonumber \\ 
& \leq \frac{\tau}{n}{\cal W}^k \left(  \frac{(\HS+2\HF)^2}{2\mu^2} {\cal W}^k +  \frac{8}{\mu^2} {\cal H}^k  \right)  + \left(  1 - \frac{\tau}{n}  \right) {\cal W}^k \label{eq:Wk+1-lk}. 
\end{align}

From the assumptions we have $\|w_i^0 - x^*\|^2 = \|x^0-x^*\|^2 \leq \frac{\mu^2}{4(\HS+2\HF)^2}$ and ${\cal H}^k \leq \frac{\mu^2}{64}$ for all $k\geq 0$. Next we show that $\|x^k-x^*\|^2 \leq  \frac{\mu^2}{4(\HS+2\HF)^2}$ for all $k\geq 1$ by mathematical induction. First, we have ${\cal W}^0 \leq  \frac{\mu^2}{4(\HS+2\HF)^2}$. Then from (\ref{eq:xk+1-pp}) we have 
\begin{align*}
\|x^{1} - x^*\|^2 & \leq \frac{(\HS+2\HF)^2}{2\mu^2} ({\cal W}^0)^2 + \frac{8}{\mu^2} {\cal H}^0 {\cal W}^0 \\
& \leq \frac{1}{8} {\cal W}^0 +  \frac{1}{8} {\cal W}^0 \\ 
& \leq \frac{\mu^2}{4(\HS+2\HF)^2}. 
\end{align*} 

Assume $\|x^k-x^*\|^2 \leq \frac{\mu^2}{4(\HS+2\HF)^2}$ for $k\leq K$. Then ${\cal W}^k \leq \min \{  \frac{\mu^2}{(\HS+2\HF)^2}, M\}$ for $k\leq K$, and from (\ref{eq:xk+1-pp}) and the assumption that ${\cal H}^k \leq \frac{\mu^2}{64}$ for $k\geq 0$, we have 
\begin{align*}
\|x^{K+1} - x^*\|^2 & \leq \frac{(\HS+2\HF)^2}{2\mu^2} ({\cal W}^K)^2 + \frac{8}{\mu^2} {\cal H}^K {\cal W}^K \\
& \leq \frac{1}{8} {\cal W}^K +  \frac{1}{8} {\cal W}^K \\ 
& \leq  \frac{\mu^2}{4(\HS+2\HF)^2}. 
\end{align*}

This indicates that $ \frac{(\HS+2\HF)^2}{2\mu^2} {\cal W}^k + \frac{8}{\mu^2} {\cal H}^k \leq \frac{1}{4}$ for all $k\geq 0$. Then from (\ref{eq:Wk+1-lk}), we can obtain 
\begin{equation}\label{eq:Wk+1-pp}
\mathbb{E}_k [{\cal W}^{k+1}] \leq \left(  1 - \frac{3\tau}{4n}  \right) {\cal W}^k. 
\end{equation}
By applying the tower property, we have $\mathbb{E}[{\cal W}^{k+1}] \leq \left(  1 - \frac{3\tau}{4n}  \right) \mathbb{E}[{\cal W}^k]$. Unrolling the recursion, we can get $\mathbb{E}[{\cal W}^k] \leq \left(  1 - \frac{3\tau}{4n}  \right)^k {\cal W}^0$. 
Since at each step, each worker makes update with probability $\frac{\tau}{n}$, we have 
\begin{eqnarray*}
&& \mathbb{E}_k \|\mH_i^{k+1} - \nabla^2 f_i(x^*)\|^2_{\rm F} \\
&=& \left(  1 - \frac{\tau}{n}  \right) \mathbb{E}_k\left[ \|\mH_i^{k+1} - \nabla^2 f_i(x^*)\|^2_{\rm F} | i\notin S^k \right] + \frac{\tau}{n}  \mathbb{E}_k\left[ \|\mH_i^{k+1} - \nabla^2 f_i(x^*)\|^2_{\rm F} | i\in S^k \right] \\ 
&=& \left(  1 - \frac{\tau}{n}  \right) \|\mH_i^{k} - \nabla^2 f_i(x^*)\|^2_{\rm F} + \frac{\tau}{n} \mathbb{E}_k \|\mH_i^k + \alpha \cC_i^k(\nabla^2 f_i(x^{k+1}) - \mH_i^k) - \nabla^2 f_i(x^*)\|^2_{\rm F}. 
\end{eqnarray*}

Then since $\mathbb{E}_k[x^{k+1}] = x^{k+1}$ and $\mathbb{E}_k[x^*] = x^*$, by choosing $z=x^*$ and $y = x^{k+1}$ in Lemma \ref{lm:threecomp}, we have 
\begin{eqnarray*}
&& \mathbb{E}_k \|\mH_i^{k+1} - \nabla^2 f_i(x^*)\|^2_{\rm F} \\ 
&\leq& \left(  1 - \frac{\tau}{n}  \right) \|\mH_i^{k} - \nabla^2 f_i(x^*)\|^2_{\rm F} + \frac{\tau}{n} (1-A) \|\mH_i^k - \nabla^2 f_i(x^*)\|^2_{\rm F} + \frac{\tau}{n} B\HF^2\|x^{k+1} - x^*\|^2 \\ 
&=& \left(  1 - \frac{A \tau}{n}  \right) \|\mH_i^{k} - \nabla^2 f_i(x^*)\|^2_{\rm F} + \frac{\tau B\HF^2}{n}\|x^{k+1} - x^*\|^2. 
\end{eqnarray*}

Summing up the above inequality from $i=1$ to $n$ and multiplying $\frac{1}{n}$, we can obtain 
$$
\mathbb{E}_k [{\cal H}^{k+1}] \leq  \left(  1 - \frac{A \tau}{n}  \right) {\cal H}^k + \frac{\tau B\HF^2}{n}\mathbb{E}_k\|x^{k+1} - x^*\|^2. 
$$

Recall that $ \frac{(\HS+2\HF)^2}{2\mu^2} {\cal W}^k + \frac{8}{\mu^2} {\cal H}^k \leq \frac{1}{4}$ for all $k\geq 0$, from (\ref{eq:xk+1-pp}), we have 
$$
\|x^{k+1} - x^*\|^2 \leq \frac{1}{4}{\cal W}^k, 
$$
which implies that 
\begin{equation}\label{eq:Hk+1}
\mathbb{E}_k[{\cal H}^{k+1}] \leq  \left(  1 - \frac{A \tau}{n}  \right) {\cal H}^k + \frac{\tau B\HF^2}{4n}{\cal W}^k. 
\end{equation}

Then from (\ref{eq:Wk+1-pp}) and (\ref{eq:Hk+1}), we have the following recurrence relation for the Lyapunov function $\Psi$:
\begin{eqnarray*}
\mathbb{E}_k[\Psi^{k+1}] &=& \mathbb{E}_k[{\cal H}^{k+1}] + B\HF^2 \mathbb{E}_k[{\cal W}^{k+1}] \\ 
&\leq& \left(  1 - \frac{A \tau}{n}  \right) {\cal H}^k + \frac{\tau B\HF^2}{4n}{\cal W}^k + \left(  1 - \frac{3\tau}{4n}  \right)B\HF^2 {\cal W}^k \\ 
&=& \left(  1 - \frac{A \tau}{n}  \right) {\cal H}^k + \left(  1 - \frac{\tau}{2n}  \right)B\HF^2 {\cal W}^k \\ 
&\leq& \left(  1 -  \frac{\tau}{n}\min\left\{  A, \frac{1}{2}  \right\}  \right) \Psi^k. 
\end{eqnarray*}

By applying the tower property, we have $\mathbb{E}[\Psi^{k+1}] \leq \left(  1 -  \frac{\tau}{n}\min\left\{  A, \frac{1}{2}  \right\}  \right) \mathbb{E}[\Psi^k]$. Unrolling the recursion, we can obtain $\mathbb{E}[\Psi^k] \leq \left(  1 -  \frac{\tau}{n}\min\left\{  A, \frac{1}{2}  \right\}  \right)^k \Psi^0$. 
We further have $\mathbb{E}[{\cal H}^k] \leq \left(  1 -  \frac{\tau}{n}\min\left\{  A, \frac{1}{2}  \right\}  \right)^k \Psi^0$ and $\mathbb{E}[{\cal W}^k] \leq \frac{1}{B\HF^2} \left(  1 -  \frac{\tau}{n}\min\left\{  A, \frac{1}{2}  \right\}  \right)^k \Psi^0$, which applied on (\ref{eq:xk+1-pp}) gives
\begin{eqnarray*}
\E\[ \frac{\|x^{k+1} -x^*\|^2}{{\cal W}^k} \]
&\leq& \frac{(\HS+2\HF)^2}{2\mu^2} \mathbb{E}[{\cal W}^k] + \frac{8}{\mu^2} \mathbb{E}[{\cal H}^k] \\ 
&\leq& \left(  1 -  \frac{\tau}{n}\min\left\{  A, \frac{1}{2}  \right\}  \right)^k \left(  \frac{(\HS+2\HF)^2}{2B\HF^2} + 8 \right) \frac{\Psi^0}{\mu^2}. 
\end{eqnarray*}

\subsection{Proof of Lemma \ref{lm:boundforbiased-pp}}

First, we have ${\cal W}^0 \leq  \min\{  \frac{A\mu^2}{16B\HF^2}, \frac{\mu^2}{4(\HS+2\HF)^2}  \}$ and ${\cal H}^0 \leq \frac{\mu^2}{64}$. Then from (\ref{eq:xk+1-pp}) we can get 
$$
\|x^1-x^*\|^2 \leq \frac{1}{4}{\cal W}^0. 
$$
For each $i$, either $\mH_i^1 = \mH_i^0$, or by Lemma \ref{lm:threecomp} 
\begin{align*}
\|\mH_i^1 - \nabla^2 f_i(x^*)\|^2_{\rm F} &= \|\mH_i^0 + \alpha\cC_i^0(\nabla^2 f_i(x^1) - \mH_i^0) - \nabla^2 f_i(x^*)\|^2_{\rm F} \\ 
& \leq (1-A) \|\mH_i^0 - \nabla^2 f_i(x^*)\|^2_{\rm F} + B\HF^2 \|x^{1}-x^*\|^2 \\ 
& \leq (1-A) \|\mH_i^0 - \nabla^2 f_i(x^*)\|^2_{\rm F} + A \cdot \frac{1}{4A}B\HF^2 {\cal W}^0\\ 
& \leq (1-A)\frac{\mu^2}{64} + A \cdot \frac{\mu^2}{64} \\ 
& \leq \frac{\mu^2}{64}. 
\end{align*}

We assume $\|\mH_i^k - \nabla^2 f_i(x^*)\|^2_{\rm F}  \leq \frac{\mu^2}{64}$ and $\|x^k-x^*\|^2 \leq  \min\{  \frac{A\mu^2}{16B\HF^2}, \frac{\mu^2}{4(\HS+2\HF)^2}  \}$ for all $k\leq K$. Then we have ${\cal H}^k \leq \frac{\mu^2}{64}$ and ${\cal W}^k \leq \min\{  \frac{A\mu^2}{16B\HF^2}, \frac{\mu^2}{4(\HS+2\HF)^2}  \}$ for all $k\leq K$. Then from (\ref{eq:xk+1-pp}) we can get 
$$
\|x^{K+1}-x^*\|^2 \leq \frac{1}{4}{\cal W}^k \leq \min\{  \frac{A\mu^2}{16B\HF^2}, \frac{\mu^2}{4(\HS+2\HF)^2}  \}. 
$$

For each $i$, either $\mH_i^{K+1} = \mH_i^K$, or by Lemma \ref{lm:threecomp} 
\begin{align*}
\|\mH_i^{K+1} - \nabla^2 f_i(x^*)\|^2_{\rm F} &= \|\mH_i^K + \alpha\cC_i^K(\nabla^2 f_i(x^{K+1}) - \mH_i^K) - \nabla^2 f_i(x^*)\|^2_{\rm F} \\ 
& \leq (1-A) \|\mH_i^K - \nabla^2 f_i(x^*)\|^2_{\rm F} + B\HF^2 \|x^{K+1}-x^*\|^2 \\ 
& \leq (1-A) \|\mH_i^K - \nabla^2 f_i(x^*)\|^2_{\rm F} + A \cdot \frac{1}{4A}B\HF^2 {\cal W}^K\\ 
& \leq (1-A)\frac{\mu^2}{64} + A \cdot \frac{\mu^2}{64} \\ 
& \leq \frac{\mu^2}{64}. 
\end{align*}

\subsection{Proof of Lemma \ref{lm:boundforspar-pp}}

First, since $\mH_i^0 = \nabla^2 f_i(w_i^0)$, we have 
$$
\|\mH_i^0 - \nabla^2 f_i(x^*)\|^2_{\rm F}  = \sum_{j, l} |(\nabla^2 f_i(w_i^0) - \nabla^2 f_i(x^*))_{jl}|^2 \leq d^2 \HM^2 \frac{\mu^2}{(H+2\HF)^2 + 64d^2\HM^2} \leq \frac{\mu^2}{64}, 
$$
which implies ${\cal H}^0 \leq \frac{\mu^2}{64}$. Then from (\ref{eq:xk+1-pp}), we have 
$$
\|x^1 - x^*\|^2 \leq {\cal W}^0 \leq \frac{\mu^2}{(\HS+2\HF)^2 + 64d^2\HM^2}. 
$$
Next we prove $\|x^k - x^*\|^2 \leq \frac{\mu^2}{(\HS+2\HF)^2 + 64d^2\HM^2}$ for all $k\geq 1$ by induction. \\ Assume $\|x^k - x^*\|^2 \leq \frac{\mu^2}{(\HS+2\HF)^2 + 64d^2\HM^2}$ for $k\leq K$. Then since $(\mH_i^k)_{jl}$ is a convex combination of $\{  (\nabla^2 f_i(w_i^0))_{jl}, (\nabla^2 f_i(x^1))_{jl}, ..., (\nabla^2 f_i(x^k))_{jl}  \}$, from the convexity of $|\cdot|^2$, we have 
$$
|(\mH_i^k - \nabla^2 f_i(x^*))_{jl}|^2 \leq \HM^2 \cdot \frac{\mu^2}{(\HS+2\HF)^2 + 64d^2\HM^2} \leq \frac{\mu^2}{64d^2}, 
$$
for $k\leq K$. Therefore, $\|\mH_i^k - \nabla^2 f_i(x^*)\|^2_{\rm F} \leq \frac{\mu^2}{64}$ and ${\cal H}^k \leq \frac{\mu^2}{64}$ for $k\leq K$. Furthermore, from ${\cal W}^k \leq \frac{\mu^2}{(\HS+2\HF)^2 + 64d^2\HM^2}$ for all $k\leq K$ and (\ref{eq:xk+1-pp}), we can also obtain 
$$
\|x^{K+1} - x^*\|^2 \leq {\cal W}^K \leq \frac{\mu^2}{(\HS+2\HF)^2 + 64d^2\HM^2}. 
$$


\section{Extension: Globalization via Line Search (\algname{FedNL-LS})} \label{apx:FedNL-LS}

Next two extensions of \algname{FedNL} is to incorporate globalization strategy. Our first globalization technique is based on backtracking line search described in \algname{FedNL-LS} below.

\begin{algorithm}[H]
  \caption{\algname{FedNL-LS} (Federated Newton Learn with {\color{blue}Line Search)}}
  \label{alg:FedNL-LS}
    \begin{algorithmic}[1]
    \STATE \textbf{Parameters:} Hessian learning rate $\alpha\ge0$; compression operators $\{\cC_1^k, \dots,\cC_n^k\}$; {\color{blue}line search parameters $c \in (0,\nicefrac{1}{2}]$ and $\gamma \in (0,1)$}
    \STATE \textbf{Initialization:} $x^0\in\R^d$; $\mH_1^0, \dots, \mH_n^0 \in \R^{d\times d}$ and $\mH^0 \eqdef \frac{1}{n}\sum_{i=1}^n \mH_i^0$
    \FOR{each device $i = 1, \dots, n$ in parallel} 
        \STATE Get $x^k$ from the server; compute $f_i(x^k)$,\; $\nabla f_i(x^k)$ and $\nabla^2 f_i(x^k)$
        \STATE Send $f_i(x^k)$,\; $\nabla f_i(x^k)$ and $\mS_i^k \eqdef \cC_i^k(\nabla^2 f_i(x^k) - \mH_i^k)$ to the server
        \STATE Update local Hessian shifts $\mH_i^{k+1} = \mH_i^k + \alpha\mS_i^k$
    \ENDFOR
        \STATE \textbf{on} server
        \STATE \quad Get $f_i(x^k)$,\; $\nabla f_i(x^k)$ and $\mS_i^k$ from all devices $i\in[n]$
        \STATE \quad $f(x^k) = \frac{1}{n}\sum_{i=1}^n f_i(x^k), \; \nabla f(x^k) = \frac{1}{n}\sum_{i=1}^n \nabla f_i(x^k), \; \mS^k = \frac{1}{n}\sum_{i=1}^n \mS_i^k$
        \STATE \quad {\color{blue}Compute search direction $d^k = -\[\mH^k\]_\mu^{-1}\nabla f(x^k)$}
        \STATE \quad {\color{blue}Find the smallest integer $s\ge0$ satisfying $f(x^k+\gamma^sd^k) \le f(x^k) + c\gamma^s \<\nabla f(x^k), d^k\>$}
        \STATE \quad Update global model to $x^{k+1} = x^k + \gamma^sd^k$
    \STATE \quad Update global Hessian shift to $\mH^{k+1} = \mH^k + \alpha\mS^k$
    \end{algorithmic}
\end{algorithm}

\subsection{Line search procedure} In contrast to the vanilla \algname{FedNL}, here we do not follow the direction $d^k = -\[\mH^k\]_\mu^{-1}\nabla f(x^k)$ with unit step size. Instead, \algname{FedNL-LS} aims to select some step size which would guarantee sufficient decrease in the empirical loss. Thus, we fix the direction $d_k$ (see line 11 of Algorihtm \ref{alg:FedNL-LS}) of next iterate $x^{k+1}$, but want to adjust the step size along that direction. With parameters $c \in (0,\nicefrac{1}{2}]$ and $\gamma \in (0,1)$, we choose the largest step size of the form $\gamma^s$, which leads to a sufficient decrease in the loss $f(x^k+\gamma^sd^k) \le f(x^k) + c\gamma^s \<\nabla f(x^k), d^k\>$ (see line 12).
Note that this procedure requires computation of local functions $f_i$ for all devices $i\in[n]$ in order to do the step in line 12. One the other hand, communication cost of line search procedure is extremely cheap compared to communication cost of gradients and Hessians.

\subsection{Local convergence theory}
We provide global linear convergence analysis for \algname{FedNL-LS}. Despite the fact that theoretical rate is slower than the rate of \algname{GD}, it shows excellent results in experiments. By $L$-smoothness we assume Lipschitz continuity of gradients with Lipschitz constant $L$.
\begin{theorem}\label{th:NL-ls}
Let Assumption \ref{asm:main} hold, function $f$ be $L$-smooth and assume $\widetilde{L} \eqdef \sup_{k\ge0} \|\mH^k\|$ is finite. Then convergence of \algname{FedNL-LS} is linear with the following rate
\begin{equation}\label{rate:global-linear-fval-ls}
f(x^{k+1}) - f(x^*) \le \( 1 - \frac{\mu}{L}\min\left\{\frac{\mu}{\widetilde{L}},1\right\} \)^k \( f(x^{0}) - f(x^*) \)
\end{equation}
\end{theorem}

Next, we provide upper bounds for $\widetilde{L}$, which was assumed to be finite in Theorem~\ref{th:NL-ls}.
\begin{lemma}\label{lm:sup_tilde-L}
If Assumption \ref{asm:comp-1} holds, then $\widetilde{L} \le \|\nabla^2 f(x^*)\| + \|\mH_i^{0} - \nabla^2 f_i(x^*)\|_{\rm F} + \sqrt{\frac{B}{A}} \HF R$. If Assumption \ref{asm:comp-2} holds, then $\widetilde{L} \le d\HM R + \|\nabla^2 f(x^*)\|$.
%
\end{lemma}

\subsection{Proof of Theorem \ref{th:NL-ls}}

Denote $\kappa \eqdef \frac{L}{\mu}$. Using $L$-smoothness of $f$ we get
\begin{eqnarray*}
f\(x^k+\frac{1}{\kappa}d^k\)
&\le& f(x^k)+\frac{1}{\kappa}\langle \nabla f(x^k), d^k\rangle + \frac{L}{2\kappa^2}\norm{d^k}^2\notag\\
&=&   f(x^k)-\frac{1}{\kappa}\langle \nabla f(x^k), \[\mH^k_{\mu}\]^{-1}\nabla f(x^k)\rangle + \frac{L}{2\kappa^2}\norm{\[\mH^k_{\mu}\]^{-1}\nabla f(x^k)}^2\notag\\
&\le& f(x^k)-\frac{1}{\kappa}\langle \nabla f(x^k), \[\mH^k_{\mu}\]^{-1}\nabla f(x^k)\rangle
      + \frac{L}{2\mu\kappa^2}\langle \nabla f(x^k), \[\mH^k_{\mu}\]^{-1}\nabla f(x^k)\rangle \\ \notag
&=&   f(x^k)-\frac{1}{2\kappa}\langle \nabla f(x^k), \[\mH^k_{\mu}\]^{-1}\nabla f(x^k)\rangle. \notag
\end{eqnarray*}
From this we conclude that, if $c=\gamma=\frac{1}{2}$, then line search procedure needs at most $s\le\log_2\kappa$ steps. To continue the above chain of derivations, we need to upper bound shifts $\mH_{\mu}^k$ in spectral norm.

Notice that if $\mH^k$ has at least on eigenvalue larger than $\mu$, then clearly $\|\mH_{\mu}^k\| = \|\mH^k\|$. Otherwise, if all eigenvalues do not exceed $\mu$, then projection gives $\mH_{\mu}^k = \mu\mI$. Thus, in both cases we can state that $\|\mH_{\mu}^k\| \le \max\{\|\mH^k\|, \mu\} \le \max\{\widetilde{L}, \mu\}$. Hence
\begin{eqnarray*}
f\(x^k+\frac{1}{\kappa}d^k\)
&\le&   f(x^k)-\frac{1}{2\kappa}\langle \nabla f(x^k), \[\mH^k_{\mu}\]^{-1}\nabla f(x^k)\rangle \notag \\
&\le&   f(x^k)-\frac{1}{2\kappa}\frac{1}{\max\{\widetilde{L}, \mu\}} \|\nabla f(x^k)\|^2 \notag \\
&=&     f(x^k)-\frac{1}{2\kappa}\frac{1}{\max\{\widetilde{L}, \mu\}} \|\nabla f(x^k) - \nabla f(x^*)\|^2 \notag \\
&\le&   f(x^k)-\frac{1}{\kappa}\frac{\mu}{\max\{\widetilde{L}, \mu\}} \(f(x^k) - f(x^*)\) \notag.
\end{eqnarray*}
Taking $x^{k+1} = x^k+\frac{1}{\kappa}d^k$, subtracting both sides by $f(x^*)$ and unraveling the above recurrence, we get \eqref{rate:global-linear-fval-ls}.

\subsection{Proof of Lemma \ref{lm:sup_tilde-L}}

Recall that $R = \sup\{\|x-x^*\| \colon f(x) \le f(x^0)\}$. It follows from the line search procedure that function values are non-increasing, namely $f(x^{k+1}) \le f(x^k) \le f(x^0)$. Hence $\|x^k-x^*\| \le R$ for all $k\ge0$. Denote
$$
\tilde{l}^k \eqdef \frac{1}{n}\sum_{i=1}^n \tilde{l}_i^k, \quad \tilde{l}_i^k \eqdef \|\mH_i^k - \nabla^2 f_i(x^*)\|_{\rm F}.$$

Consider the case when compressors $\cC_i^k\in\bC(\delta)$ and the learning rate is either $\alpha = 1-\sqrt{1-\delta}$ or $\alpha=1$. Using Lemma \ref{lm:threecomp} with $y=x^k$ and $z=x^*$, for both cases we get
\begin{equation}\label{eq:04-1}
\|\mH_i^{k+1} - \nabla^2 f_i(x^*)\|_{\rm F}^2
\le (1-A) \|\mH_i^{k} - \nabla^2 f_i(x^*)\|_{\rm F}^2 + B \HF^2\|x^k-x^*\|^2.
\end{equation}
Reusing \eqref{eq:04-1} multiple times we get
\begin{eqnarray*}
\|\mH_i^{k+1} - \nabla^2 f_i(x^*)\|_{\rm F}^2
&\le& (1-A) \|\mH_i^{k} - \nabla^2 f_i(x^*)\|_{\rm F}^2 + B \HF^2R^2 \\
&\le& (1-A)^2 \|\mH_i^{k-1} - \nabla^2 f_i(x^*)\|_{\rm F}^2 + \[1+(1-A)\] B \HF^2 R^2 \\
&\le& (1-A)^{k+1} \|\mH_i^{0} - \nabla^2 f_i(x^*)\|_{\rm F}^2 + B \HF^2 R^2 \sum_{t=0}^{\infty}(1-A)^t \\
&\le& \|\mH_i^{0} - \nabla^2 f_i(x^*)\|_{\rm F}^2 + \frac{B}{A}\HF^2R^2,
\end{eqnarray*}
which implies boundedness of $\tilde{l}^k$:
$$
\tilde{l}^{k}
=
\frac{1}{n}\sum_{i=1}^n \tilde{l}_i^{k}
\le
\frac{1}{n}\sum_{i=1}^n \sqrt{\|\mH_i^{0} - \nabla^2 f_i(x^*)\|_{\rm F}^2 + \frac{B}{A}\HF^2R^2}
\le
\|\mH_i^{0} - \nabla^2 f_i(x^*)\|_{\rm F} + \sqrt{\frac{B}{A}} \HF R.
$$

From this we also conclude boundedness of $\widetilde{L}$ as follows
\begin{eqnarray*}
\|\mH^k\|
&\le& \|\mH^k - \nabla^2 f(x^*)\| + \|\nabla^2 f(x^*)\| \\
&\le& \left\|\frac{1}{n}\sum_{i=1}^n (\mH_i^k - \nabla^2 f_i(x^*)) \right\|_{\rm F} + \|\nabla^2 f(x^*)\| \\
&\le& \frac{1}{n}\sum_{i=1}^n \|\mH_i^k - \nabla^2 f_i(x^*)\|_{\rm F} + \|\nabla^2 f(x^*)\| \\
&\le& \|\nabla^2 f(x^*)\| + \|\mH_i^{0} - \nabla^2 f_i(x^*)\|_{\rm F} + \sqrt{\frac{B}{A}} \HF R.
\end{eqnarray*}

Consider the case when compressors $\cC_i^k\in\B(\omega)$ and the learning rate $\alpha\le\frac{1}{\omega+1}$. As we additionally assume that $(\mH_i^k)_{jl}$ is a convex combination of past Hessians $\{(\nabla^2 f_i(x^0))_{jl}, \dots, (\nabla^2 f_i(x^k))_{jl}\}$, we get
$$
|(\mH_i^k - \nabla^2 f_i(x^*))_{jl}|^2 \le \HM^2 \max_{0\le t\le k}\|x^t-x^*\|^2 \le \HM^2 R^2.
$$
Therefore
$$
\|\mH_i^k - \nabla^2 f_i(x^*)\|^2_{\rm F} \le d^2 \HM^2 R^2,
$$
from which
\begin{eqnarray*}
\|\mH^k\|
\le \frac{1}{n}\sum_{i=1}^n \|\mH_i^k - \nabla^2 f_i(x^*)\|_{\rm F} + \|\nabla^2 f(x^*)\|
\le d\HM R + \|\nabla^2 f(x^*)\|.
\end{eqnarray*}

\section{Extension: Globalization via Cubic Regularization (\algname{FedNL-CR})}\label{apx:FedNL-CR}

Our next extension to \algname{FedNL} providing global convergence guarantees is cubic regularization.

\begin{algorithm}[H]
  \caption{\algname{FedNL-CR} (Federated Newton Learn with {\color{blue}Cubic Regularization)}}
  \label{alg:FedNL-CR}
  \begin{algorithmic}[1]
    \STATE \textbf{Parameters:} Hessian learning rate $\alpha\ge0$; compression operators $\{\cC_1^k, \dots,\cC_n^k\}$; {\color{blue}Lipschitz constant $H\ge0$ for Hessians}
    \STATE \textbf{Initialization:} $x^0\in\R^d$; $\mH_1^0, \dots, \mH_n^0 \in \R^{d\times d}$ and $\mH^0 \eqdef \frac{1}{n}\sum_{i=1}^n \mH_i^0$
    \FOR{each device $i = 1, \dots, n$ in parallel} 
    \STATE Get $x^k$ from the server and compute local gradient $\nabla f_i(x^k)$ and local Hessian $\nabla^2 f_i(x^k)$
    \STATE Send $\nabla f_i(x^k)$,\; $\mS_i^k \eqdef \cC_i^k(\nabla^2 f_i(x^k) - \mH_i^k)$ and $l_i^k \eqdef \|\mH_i^k - \nabla^2 f_i(x^k)\|_{\rm F}$ to the server
    \STATE Update local Hessian shift to $\mH_i^{k+1} = \mH_i^k + \alpha\mS_i^k$
    \ENDFOR
    \STATE \textbf{on} server
    \STATE \quad Get $\nabla f_i(x^k),\; \mS_i^k$ and $l_i^k$ from all devices $i\in[n]$
    \STATE \quad $\nabla f(x^k) = \frac{1}{n}\sum_{i=1}^n \nabla f_i(x^k), \quad \mS^k = \frac{1}{n}\sum_{i=1}^n \mS_i^k, \quad l^k = \frac{1}{n}\sum_{i=1}^n l_i^k$
    \STATE \quad {\color{blue}$h^k = \arg\min_{h\in\R^d}T_k(h)$, where $T_k(h) \eqdef \<\nabla f(x^k),h\> + \frac{1}{2}\<(\mH^k+l^k\mI)h,h\> + \frac{\HS}{6}\|h\|^3$}
    \STATE \quad Update global model to $x^{k+1} = x^k + h^k$
    \STATE \quad Update global Hessian shift to $\mH^{k+1} = \mH^k + \alpha\mS^k$
  \end{algorithmic}
\end{algorithm}


\subsection{Cubic regularization} Adding third order regularization term $\frac{\HS}{6}\|h\|^3$ is a well known technique to guarantee global convergence for Newton-type methods. Basically, this term provides means to upper bound the loss function globally, which ultimately leads to global convergence. Notice that, without this term \algname{FedNL-CR} reduces to \algname{FedNL} with {\em Option 2}. However, cubic regularization alone does not provide us global upper bounds as the second order information, the Hessians, are compressed, and thus upper bounds might be violated.

\subsection{Solving the subproblem} In each iteration, the sever needs to solve the subproblem in line 11 in order to compute $h^k$. Although it does not admit a closed form solution, the server can solve it by reducing to certain one-dimensional nonlinear equation. For more details, see section C.1 of \citep{Islamov2021NewtonLearn}. 

\subsection{Importance of compression errors $l_i^k$} Unlike \algname{FedNL} and \algname{FedNL-PP}, compression errors are the only option for \algname{FedNL-CR} to update the global model. The reason is that to get a cubic upper bound for $f$ we need to upper bound current true Hessians $\nabla^2 f_i(x^k)$ in the matrix order. Neither current learned Hessian $\mH_i^k$ nor the projected matrix $\[\mH_i^k\]_{\mu}$ does not guarantee upper bound for $\nabla^2 f_i(x^k)$. Meanwhile, from $l_i^k \eqdef \|\mH_i^k - \nabla^2 f_i(x^k)\|_{\rm F}$, we have $\nabla^2 f_i(x^k) \preceq \mH_i^k + l_i^k\mI$.

\subsection{Global and local convergence theory} We prove two global rates (covering convex and strongly convex cases) and the same three local rates of \algname{FedNL}.

\begin{theorem}\label{th:NL-cr}
Let Assumption \ref{asm:main} hold and assume $l \eqdef \sup_{k\ge0} l^k$ is finite. Then if $f(x)$ is convex (i.e., $\mu=0$), we have global sublinear rate
\begin{equation}\label{rate:global-sublinear-fval}
f(x^k) - f(x^*) \le \frac{9l R^2}{k} + \frac{9\HS R^3}{k^2} + \frac{3\(f(x^{0}) - f(x^*)\)}{k^3},
\end{equation}
where $R\eqdef\{\|x-x^*\|\colon f(x)\le f(x^0)\}$. Moreover, if $f(x)$ is $\mu$-convex with $\mu>0$, then convergence becomes linear with respect to function sub-optimality, i.e., $f(x^k) - f(x^*)\le\varepsilon$ is guaranteed after
\begin{equation}\label{rate:global-linear-fval}
\cO\( \(\frac{l}{\mu} + \sqrt{\frac{\HS R}{\mu}} +1\) \log\frac{f(x^0)-f(x^*)}{\varepsilon}\)
\end{equation}
iterations. Furthermore, if $\|x^0-x^*\|^2 \le \frac{\mu^2}{20(\HS^2+8\HF^2)}$ and $\cH^k\le\frac{\mu^2}{160}$ for all $k\ge0$, then we have the same local rates \eqref{rate:local-linear-iter}, \eqref{rate:local-linear-Lyapunov} and \eqref{rate:local-superlinear-iter}.
\end{theorem}

Next, we provide upper bounds for $l$, which was assumed to be finite in the theorem.

\begin{lemma}\label{lm:sup_l}
If Assumption \ref{asm:comp-1} holds, then $l\le \sqrt{\cH^0} + \(1+\sqrt{\frac{B}{A}}\)\HF R$. If Assumption \ref{asm:comp-2} holds, then $l\le (d\HM + \HF)R$.
%
\end{lemma}

\subsection{Proof of Theorem \ref{th:NL-cr}}

{\bf Global rate for general convex case ($\mu=0$).}
First, from $\HS$-Lipschitzness of the Hessian of $f$ we get
\begin{eqnarray}
f(x^{k+1})
&\le& f(x^k) + \<\nabla f(x^k), h^k\> + \frac{1}{2}\<\nabla^2 f(x^k)h^k, h^k\> + \frac{\HS}{6}\|h^k\|^3 \notag \\
&\le& f(x^k) + \<\nabla f(x^k), h^k\> + \frac{1}{2}\<(\mH^k+l^k\mI)h^k, h^k\> + \frac{\HS}{6}\|h^k\|^3 \notag \\
&=&   f(x^k) + \min_{h\in\R^d}T_k(h) \label{eq:01} \\
&\le& f(x^k) + T_k(y-x^k) \notag \\
&\le& f(x^k) + \<\nabla f(x^k), y-x^k\> + \frac{1}{2}\<(\mH^k+l^k\mI)(y-x^k), y-x^k\> + \frac{\HS}{6}\|y-x^k\|^3 \notag \\
&\le& f(x^k) + \<\nabla f(x^k), y-x^k\> + \frac{1}{2}\<\nabla^2 f(x^k)(y-x^k), y-x^k\> + \frac{\HS}{6}\|y-x^k\|^3 \notag \\
&&\;\;+\;\; \frac{1}{2}\left\|\mH^k - \nabla^2 f(x^k)\right\| \|y-x^k\|^2 + \frac{1}{2}l^k\|y-x^k\|^2 \notag \\
&\le& f(x^k) + \<\nabla f(x^k), y-x^k\> + \frac{1}{2}\<\nabla^2 f(x^k)(y-x^k), y-x^k\> + l^k \|y-x^k\|^2 + \frac{\HS}{6}\|y-x^k\|^3 \notag \\
&\le& f(y) + \frac{\HS}{6}\|y-x^k\|^3 + l^k \|y-x^k\|^2 + \frac{\HS}{6}\|y-x^k\|^3 \notag \\
&\le& f(y) + l\|y-x^k\|^2 + \frac{\HS}{3}\|y-x^k\|^3. \label{eq:02}
\end{eqnarray}

Denote $a_k \eqdef k^2$ and
$$A_k \eqdef 1 + \sum_{i=1}^k a_i = 1 + \sum_{i=1}^k i^2 = 1 + \frac{k(k+1)(2k+1)}{6} \ge 1 + \frac{k^3}{3}.$$
Let $\sigma_k = \frac{a_{k+1}}{A_{k+1}} \in (0,1)$. Then we get $1-\sigma_k = \frac{A_k}{A_{k+1}}$. Now we choose $y = \sigma_k x^* + (1-\sigma_k)x^k = x^k + \sigma_k(x^*-x^k)$. Using convexity of $f$, we get
\begin{eqnarray}
f(x^{k+1})
&\le& f(y) + l \|y-x^k\|^2 + \frac{\HS}{3}\|y-x^k\|^3 \notag \\
&\le& \sigma_k f(x^*) + (1-\sigma_k)f(x^k) + l \sigma_k^2 \|x^k-x^*\|^2 + \frac{\HS}{3}\sigma_k^3\|x^k-x^*\|^3 \label{eq:05} \\
&\le& \sigma_k f(x^*) + (1-\sigma_k)f(x^k) + l \sigma_k^2 R^2 + \frac{\HS}{3}\sigma_k^3R^3. \notag 
\end{eqnarray}
Using the definition of $\sigma_k$ and subtracting both sides by $A_kf(x^*)$ we get
\begin{eqnarray}
A_{k+1}\(f(x^{k+1}) - f(x^*)\)
&\le& A_k \(f(x^{k}) - f(x^*)\) + l R^2 \frac{a_{k+1}^2}{A_{k+1}} + \frac{\HS R^3}{3} \frac{a_{k+1}^3}{A_{k+1}^2}, \notag 
\end{eqnarray}
repeated application of which provides us the following bound
\begin{eqnarray}
A_{k}\(f(x^{k}) - f(x^*)\)
&\le& A_0 \(f(x^{0}) - f(x^*)\) + l R^2 \sum_{t=1}^k\frac{a_{t}^2}{A_{t}} + \frac{\HS R^3}{3} \sum_{t=0}^k\frac{a_{t}^3}{A_{t}^2}. \label{eq:06} 
\end{eqnarray}
Next we upper bound the above two sums:
$$
\sum_{t=1}^k\frac{a_{t}^2}{A_{t}} \le \sum_{t=1}^k\frac{t^4}{1 + \frac{t^3}{3}} \le 3k^2, \quad
\sum_{t=1}^k\frac{a_{t}^3}{A_{t}^2} \le \sum_{t=1}^k\frac{t^6}{\(1 + \frac{t^3}{3}\)^2} \le 9k.
$$
Hence the bound \eqref{eq:06} can be transformed into
\begin{eqnarray}
f(x^{k}) - f(x^*)
&\le& \frac{1}{A_k}\[ \(f(x^{0}) - f(x^*)\) + 3k^2 \cdot l R^2 + 3k \cdot \HS R^3 \] \notag \\
&\le&  \frac{9l R^2}{k} + \frac{9\HS R^3}{k^2} + \frac{3\(f(x^{0}) - f(x^*)\)}{k^3}. \notag
\end{eqnarray}

Thus, we have shown $\cO(\frac{1}{k})$ rate for convex functions and it holds for any $k\ge1$.

{\bf Global rate for strongly convex case ($\mu>0$).}
We can turn this rate into a linear rate using strong convexity of $f$. Namely, in this case we have $R^2 \le \frac{2}{\mu}(f(x^0)-f(x^*))$ and therefore
\begin{eqnarray}
f(x^{k}) - f(x^*)
&\le&  \[\frac{18l}{k\mu} + \frac{18\HS R}{k^2\mu} + \frac{3}{k^3}\]\(f(x^{0}) - f(x^*)\)
\le     \frac{1}{2}\(f(x^{0}) - f(x^*)\), \notag
\end{eqnarray}
if $k \ge K_1 \eqdef \max\(\frac{108l}{\mu},\sqrt{\frac{108\HS R}{\mu}},3\)$. In other words, we half the error $f(x^k)-f(x^*)$ after $K_1$ steps. This implies the following linear rate
$$
\cO\( \(\frac{l}{\mu} + \sqrt{\frac{\HS R}{\mu}} +1\)\log\frac{1}{\varepsilon} \).
$$

{\bf Local rate for strongly convex case ($\mu>0$).}

From the definition of $h^k$ direction, we have
$$
\nabla T_k(h^k) = \nabla f(x^k) + (\mH^k + l^k\mI)h^k + \frac{\HS}{2}\|h^k\|h^k = 0,
$$
which implies the following equivalent update rule
\begin{eqnarray*}
x^{k+1}
&=& x^k + h^k \\
&=& x^k - \[\mH^k + l^k\mI + \frac{\HS}{2}\|x^{k+1}-x^k\|\]^{-1}\nabla f(x^k).
\end{eqnarray*}
Then, using $\mu\mI\preceq\nabla f(x^k) \preceq \mH^k+l^k\mI$, we have
\begin{eqnarray}
&& \|x^{k+1} - x^*\|^2 \notag \\
&=&   \left\|x^k-x^* - \[\mH^k + l^k\mI + \frac{\HS}{2}\|x^{k+1}-x^k\|\]^{-1}\nabla f(x^k) \right\|^2 \notag \\
&\le& \frac{1}{\mu^2} \left\|\[\mH^k + l^k\mI + \frac{\HS}{2}\|x^{k+1}-x^*\| + \frac{\HS}{2}\|x^{k}-x^*\|\](x^k-x^*) - \nabla f(x^k)\right\|^2 \notag \\
&\le& \frac{5}{\mu^2}\left(
      \left\|\nabla^2 f(x^k)(x^k-x^*) - \nabla f(x^k) + \nabla f(x^*) \right\|^2
      + \frac{\HS^2}{4}\|x^k-x^*\|^4
      + \frac{\HS^2}{4}\|x^{k+1}-x^*\|^2\|x^k-x^*\|^2 \right. \notag \\
      &&\; + \left. \left\|\(\mH^{k} - \nabla^2 f(x^k)\)(x^k-x^*) \right\|^2
           + \[l^k\]^2\|x^k-x^*\|^2 \right) \notag \\
&\le& \frac{5}{\mu^2}\left(
      \frac{\HS^2}{4}\|x^k-x^*\|^4
      + \frac{\HS^2}{4}\|x^k-x^*\|^4
      + \frac{\HS^2}{4}\|x^{k+1}-x^*\|^2\|x^k-x^*\|^2 \right. \notag \\
      &&\; + \left. \left\|\mH^{k} - \nabla^2 f(x^k)\right\|^2 \|x^k-x^*\|^2
           + \frac{1}{n}\sum_{i=1}^n\left\|\mH_i^k - \nabla f_i(x^k)\right\|_{\rm F}^2 \|x^k-x^*\|^2 \right) \notag \\
&\le& \frac{5}{\mu^2}\left(
      \frac{\HS^2}{2}\|x^k-x^*\|^4
      + \frac{\HS^2}{4}\|x^{k+1}-x^*\|^2\|x^k-x^*\|^2
           + \frac{2}{n}\sum_{i=1}^n\left\|\mH_i^k - \nabla f_i(x^k)\right\|_{\rm F}^2 \|x^k-x^*\|^2 \right) \notag \\
&\le& \frac{5\HS^2}{2\mu^2}\|x^k-x^*\|^4
      + \frac{5\HS^2}{4\mu^2}\|x^{k+1}-x^*\|^2\|x^k-x^*\|^2
           + \frac{20}{n\mu^2}\sum_{i=1}^n\left\|\mH_i^k - \nabla f_i(x^*)\right\|_{\rm F}^2 \|x^k-x^*\|^2 \notag \\
           &&\; + \; \frac{20}{n\mu^2}\sum_{i=1}^n\left\|\nabla f_i(x^k) - \nabla f_i(x^*)\right\|_{\rm F}^2 \|x^k-x^*\|^2 \notag \\
&\le& \frac{5\HS^2}{2\mu^2}\|x^k-x^*\|^4
      + \frac{5\HS^2}{4\mu^2}\|x^{k+1}-x^*\|^2\|x^k-x^*\|^2
           + \frac{20}{\mu^2} \|x^k-x^*\|^2 \cH^k + \frac{20\HF^2}{\mu^2} \|x^k-x^*\|^4 \notag \\
&\le& \frac{5\HS^2}{4\mu^2}\|x^{k+1}-x^*\|^2 \|x^k-x^*\|^2
           + \frac{20}{\mu^2} \|x^k-x^*\|^2 \cH^k + \frac{5(\HS^2 + 8\HF^2)}{2\mu^2} \|x^k-x^*\|^4 \label{eq:cr-rec}.
\end{eqnarray}

Using the assumptions we show that $\|x^k-x^*\|^2 \le \frac{\mu^2}{20(\HS^2+8\HF^2)}$ for all $k\ge0$. We prove this again by induction on $k$. From $\|x^k-x^*\|^2 \le \frac{\mu^2}{20(\HS^2+8\HF^2)} \le \frac{2\mu^2}{5\HS^2} $ and $\cH^k \le \frac{\mu^2}{160}$, it follows
\begin{eqnarray*}
&& \|x^{k+1}-x^*\|^2 \\
&\le& \frac{5\HS^2}{4\mu^2} \|x^k-x^*\|^2 \|x^{k+1}-x^*\|^2
           + \frac{20}{\mu^2} \cH^k \|x^k-x^*\|^2
           + \frac{5(\HS^2 + 8\HF^2)}{2\mu^2} \|x^k-x^*\|^2 \|x^k-x^*\|^2 \\
&\le& \frac{1}{2} \|x^{k+1}-x^*\|^2 + \frac{1}{8} \|x^{k}-x^*\|^2 + \frac{1}{8} \|x^{k}-x^*\|^2 \\
&\le& \frac{1}{2} \|x^{k+1}-x^*\|^2 + \frac{1}{4} \|x^{k}-x^*\|^2.
\end{eqnarray*}
Hence
\begin{equation}\label{eq:1/2-rate}
\|x^{k+1}-x^*\|^2 \le \frac{1}{2} \|x^{k}-x^*\|^2 \le \frac{\mu^2}{20(\HS^2+8\HF^2)}.
\end{equation}
By this we complete the induction and also derived the local linear rate for iterates. Moreover, \eqref{eq:cr-rec} and \eqref{eq:1/2-rate} imply
\begin{equation}\label{eq:cr-rec-1}
\|x^{k+1}-x^*\|^2
\le  \frac{20}{\mu^2} \|x^k-x^*\|^2 \cH^k + \frac{3\HS^2 + 20\HF^2}{\mu^2} \|x^k-x^*\|^4.
\end{equation}

Choosing $y=x^k$ and $z=x^*$ in Lemma \ref{lm:threecomp}, and noting that $\mH_i^{k+1} = \mH_i^k + \alpha \cC_i^k(\nabla^2 f_i(x^k) - \mH_i^k)$, we get
$$
\mathbb{E}_k \left[{\cal H}^{k+1} \right] \leq (1-A) {\cal H}^k + B\HF^2 \|x^k-x^*\|^2. 
$$

Using the same Lyapunov function $\Phi^k = {\cal H}^k + 6B\HF^2 \|x^k-x^*\|^2$, from the above inequality and (\ref{eq:1/2-rate}), we arrive at 
\begin{align*}
\mathbb{E}_k \left[\Phi^{k+1}\right] & \leq (1-A) {\cal H}^k + B\HF^2 \|x^k-x^*\|^2 + 3B\HF^2 \|x^k-x^*\|^2 \\ 
& =  (1-A) {\cal H}^k + \left(  1 - \frac{1}{3}  \right)6B\HF^2 \|x^k-x^*\|^2 \\ 
& \leq \left(  1 - \min\left\{  A, \frac{1}{3}  \right\}  \right) \Phi^k. 
\end{align*}

Hence $\mathbb{E}[\Phi^k] \leq \left(  1 - \min\left\{  A, \frac{1}{3}  \right\}  \right)^k \Phi^0$. We further have $\mathbb{E}[{\cal H}^k] \leq \left(  1 - \min\left\{  A, \frac{1}{3}  \right\}  \right)^k \Phi^0$ and $\mathbb{E}[\|x^k-x^*\|^2] \leq \frac{1}{6B\HF^2} \left(  1 - \min\left\{  A, \frac{1}{3}  \right\}  \right)^k \Phi^0$ for $k\geq 0$. Assume $x^k\neq x^*$ for all $k$. Then from (\ref{eq:cr-rec-1}), we have 
$$
\frac{\|x^{k+1}-x^*\|^2}{\|x^k-x^*\|^2} \leq \frac{20}{\mu^2}{\cal H}^k + \frac{3\HS^2 + 20\HF^2}{\mu^2}\|x^k-x^*\|^2, 
$$
and by taking expectation, we have 
\begin{align*}
\mathbb{E} \left[  \frac{\|x^{k+1}-x^*\|^2}{\|x^k-x^*\|^2}  \right] & \leq \frac{20}{\mu^2} \mathbb{E}[{\cal H}^k] + \frac{3\HS^2 + 20\HF^2}{\mu^2} \mathbb{E}[\|x^k-x^*\|^2] \\ 
& \leq  \left(  1 - \min\left\{  A, \frac{1}{3}  \right\}  \right)^k \left( 20 + \frac{3\HS^2 + 20\HF^2}{6B\HF^2}  \right) \frac{\Phi^0}{\mu^2}. 
\end{align*}

To conclude, \algname{FedNL-CR} method provably provides global rates (both for convex and strongly convex cases) and recovers the same local rates \eqref{rate:local-linear-iter}, \eqref{rate:local-linear-Lyapunov} and \eqref{rate:local-superlinear-iter} that we showed for \algname{FedNL}. Note that constants $A$ and $B$ are the same, while $C$ and $D$ differ from \eqref{rate:local-superlinear-iter}.

\subsection{Proof of Lemma \ref{lm:sup_l}}

Recall that $R = \sup\{\|x-x^*\| \colon f(x) \le f(x^0)\}$. Since $T_k(0)=0$, from \eqref{eq:01} we can show that $f(x^{k+1}) \le f(x^k) \le f(x^0)$, and hence $\|x^k-x^*\| \le R$ for all $k\ge0$. Denote
$$
\tilde{l}^k \eqdef \frac{1}{n}\sum_{i=1}^n \tilde{l}_i^k, \qquad \tilde{l}_i^k \eqdef \|\mH_i^k - \nabla^2 f_i(x^*)\|_{\rm F}.$$
Notice that
\begin{eqnarray}
l_i^k
&=& \|\nabla^2 f_i(x^k) - \mH_i^k\|_{\rm F} \notag \\
&\le& \|\mH_i^k - \nabla^2 f_i(x^*)\|_{\rm F} + \|\nabla^2 f_i(x^k) - \nabla^2 f_i(x^*)\|_{\rm F} \notag \\
&\le& \tilde{l}_i^k + \HF\|x^k-x^*\| \notag \\
&\le& \tilde{l}_i^k + \HF R. \label{eq:03}
\end{eqnarray}
Consider the case when compressors $\cC_i^k\in\bC(\delta)$ and the learning rate is either $\alpha = 1-\sqrt{1-\delta}$ or $\alpha=1$. Using Lemma \ref{lm:threecomp} with $y=x^k$ and $z=x^*$, for both cases we get
\begin{equation}\label{eq:04}
\|\mH_i^{k+1} - \nabla^2 f_i(x^*)\|_{\rm F}^2
\le (1-A) \|\mH_i^{k} - \nabla^2 f_i(x^*)\|_{\rm F}^2 + B \HF^2\|x^k-x^*\|^2.
\end{equation}
Reusing \eqref{eq:04} multiple times we get
\begin{eqnarray*}
\|\mH_i^{k+1} - \nabla^2 f_i(x^*)\|_{\rm F}^2
&\le& (1-A) \|\mH_i^{k} - \nabla^2 f_i(x^*)\|_{\rm F}^2 + B \HF^2R^2 \\
&\le& (1-A)^2 \|\mH_i^{k-1} - \nabla^2 f_i(x^*)\|_{\rm F}^2 + \[1+(1-A)\] B \HF^2 R^2 \\
&\le& (1-A)^{k+1} \|\mH_i^{0} - \nabla^2 f_i(x^*)\|_{\rm F}^2 + B \HF^2 R^2 \sum_{t=0}^{\infty}(1-A)^t \\
&\le& \|\mH_i^{0} - \nabla^2 f_i(x^*)\|_{\rm F}^2 + \frac{B}{A}\HF^2R^2,
\end{eqnarray*}
which implies boundedness of $\tilde{l}_i^k$:
$$
\tilde{l}_i^{k}
\le
\sqrt{\|\mH_i^{0} - \nabla^2 f_i(x^*)\|_{\rm F}^2 + \frac{B}{A}\HF^2R^2}
\le
\tilde{l}_i^{0} + \sqrt{\frac{B}{A}} \HF R.
$$

From this we also conclude boundedness of $l^k$ as follows
\begin{eqnarray*}
l^k
= \frac{1}{n}\sum_{i=1}^n l_i^k
\overset{\eqref{eq:03}}{\le} \frac{1}{n}\sum_{i=1}^n \tilde{l}_i^k + \HF R
\le \tilde{l}^0 + \(1+\sqrt{\frac{B}{A}}\)\HF R.
\end{eqnarray*}
We can further upper bound $\tilde{l}^0 \le \sqrt{\cH^0}$ and conclude $l\le \sqrt{\cH^0} + \(1+\sqrt{\frac{B}{A}}\)\HF R$.

Consider the case when compressors $\cC_i^k\in\B(\omega)$ and the learning rate $\alpha\le\frac{1}{\omega+1}$. As we additionally assume that $(\mH_i^k)_{jl}$ is a convex combination of past Hessians $\{(\nabla^2 f_i(x^0))_{jl}, \dots, (\nabla^2 f_i(x^k))_{jl}\}$, we get
$$
|(\mH_i^k - \nabla^2 f_i(x^*))_{jl}|^2 \le \HM^2 \max_{0\le t\le k}\|x^t-x^*\|^2 \le \HM^2 R^2.
$$
Therefore
$$
\[\tilde{l}_i^k\]^2 = \|\mH_i^k - \nabla^2 f_i(x^*)\|^2_{\rm F} \le d^2 \HM^2 R^2,
$$
from which
\begin{eqnarray*}
l^k
= \frac{1}{n}\sum_{i=1}^n l_i^k
\overset{\eqref{eq:03}}{\le} \frac{1}{n}\sum_{i=1}^n \tilde{l}_i^k + \HF R
\le d \HM R + \HF R = (d\HM + \HF)R.
\end{eqnarray*}

\section{Extension: Bidirectional Compression (\algname{FedNL-BC})}\label{apx:FedNL-BC}

Finally, we extend the vanilla \algname{FedNL} to allow for an even more severe level of compression that can't be attained by compressing the Hessians only. This is achieved by compressing the gradients (uplink) and the model (downlink), in a ``smart'' way. Thus, in \algname{FedNL-BC} (Algorithm~\ref{alg:FedNL-BC}) described below, both directions of communication are fully compressed.

\begin{algorithm}[H]
  \caption{\algname{FedNL-BC} (Federated Newton Learn with {\color{blue}Bidirectional Compression)}}
  \label{alg:FedNL-BC}
  \begin{algorithmic}[1]
    \STATE \textbf{Parameters:} Hessian learning rate $\alpha\ge0$; {\color{blue}model learning rate $\eta\ge0$}; {\color{blue}gradient compression probability $p\in(0,1]$}; compression operators $\{\cC_1^k,\dots,\cC_n^k\}$ and {\color{blue}$\cC^k_{\rm M}$}
    \STATE \textbf{Initialization:} $x^0\;{\color{blue}=w^0=z^0}\in\R^d$; $\mH_1^0, \dots, \mH_n^0 \in \R^{d\times d}$ and $\mH^0 \eqdef \frac{1}{n}\sum_{i=1}^n \mH_i^0$; {\color{blue}$\xi^0 = 1$}
    \FOR{each device $i = 1, \dots, n$ in parallel}
    \STATE {\color{blue}Get $\xi^k$ from the server}
    \STATE \textbf{if} $\xi^k = 1$
    \STATE \quad Compute local gradient $\nabla f_i(z^k)$ and send to the server
    \STATE \quad $g^k_i = \nabla f_i(z^k), \quad w^{k+1} = z^{k}$
    \STATE \textbf{if} {\color{blue}$\xi^k = 0$}
    \STATE \quad {\color{blue}$g^k_i = \mH^k_i(z^k-w^k) + \nabla f_i(w^k), \quad w^{k+1} = w^k$}
    \STATE Compute local Hessian $\nabla^2 f_i(z^k)$
    \STATE Send $\mS_i^k \eqdef \cC_i^k(\nabla^2 f_i(z^k) - \mH_i^k)$ and $l_i^k \eqdef \|\nabla^2 f_i(z^k) - \mH_i^k\|_{\rm F}$ to the server
    \STATE Update local Hessian shift to $\mH_i^{k+1} = \mH_i^k + \alpha\mS_i^k$
    \ENDFOR
    \STATE \textbf{on} server
    \STATE \quad $g^k = \frac{1}{n}\sum\limits_{i=1}^n g^k_i, \quad \mS^k = \frac{1}{n}\sum_{i=1}^n \mS_i^k, \quad l^k = \frac{1}{n}\sum_{i=1}^n l_i^k$
    \STATE \quad \textit{Option 1:} $x^{k+1} = z^k - \[\mH^k\]_{\mu}^{-1} g^k$
    \STATE \quad \textit{Option 2:} $x^{k+1} = z^k - \[\mH^k + l^k\mI\]^{-1} g^k$
    \STATE \quad Update global Hessian shifts $\mH^{k+1} = \mH^k + \alpha\mS^k$
    \STATE \quad {\color{blue}Send $s^k \eqdef \cC^k_{\rm M}(x^{k+1}-z^k)$ to all devices $i\in[n]$}
    \STATE \quad {\color{blue}Update the model $z^{k+1} = z^k+\eta s^k$}
    \STATE \quad {\color{blue}Sample $\xi^{k+1} \sim \text{Bernoulli}(p)$ and send to all devices $i\in[n]$}
    \FOR{each device $i = 1, \dots, n$ in parallel}
    \STATE {\color{blue}Get $s^k$ from the server and update the model $z^{k+1}=z^k+\eta s^k$}
    \ENDFOR
  \end{algorithmic}
\end{algorithm}

\subsection{Model learning technique}
In \algname{FedNL-BC} we introduced ``smart" model learning technique, which is similar to the proposed Hessian learning technique. As in Hessian learning technique, the purpose of the model learning technique is learn the optimal model $x^*$ in a communication efficient manner. This is achieved by maintaining and progressively updating global model estimates $z^k$ for all nodes $i\in[n]$ and for the sever. Thus, the goal is to make updates from $z^k$ to $z^{k+1}$ easy to communicate and to induce $z^k\to x^*$ throughout the training process. Similar to the Hessian learning technique, the server operates its own compressors $\cC^k_{\rm M}$ and updates the model estimates $z^k$ via the rule $z^{k+1} = z^k+\eta s^k$, where $s^k = \cC^k_{\rm M}(x^{k+1}-z^k)$ and $\eta>0$ is the learning rate. Again, we reduce the communication cost by explicitly requiring the server to send compressed model information $s^k$ to all clients.

\subsection{Hessian corrected local gradients}
The second key technical novelty in \algname{FedNL-PP} is another structure of Hessian corrected local gradients $$g^k_i = \mH^k_i(z^k-w^k) + \nabla f_i(w^k)$$ (see line 9 of Algorithm~\ref{alg:FedNL-BC}). The intuition behind this form is as follows. Uplink gradient compression is done by Bernoulli compression synchronized by the server: namely, if the Bernoulli trial $\xi^k\sim\text{Bernoulli}(p)$ is successful (i.e., $\xi^k=1$, see line 5), then all clients compute and communicate the current true local gradients $\nabla f_i(z^k)$, otherwise (i.e., $\xi^k=0$, see line 8) devices do not even compute the local gradient. In the latter case, devices approximate current local gradient $\nabla f_i(z^k)$ based on stale local  gradient $\nabla f_i(w^k)$ and current Hessian estimate $\mH_i^k$ via the rule $g^k_i = \mH^k_i(z^k-w^k) + \nabla f_i(w^k)$, $w^k$ is the last learned global model when Bernoulli trial was successful and local gradients are sent to the server. 

To further motivate the structure of $g_i^k$, consider for a moment the case when $\mH_i^k = \nabla^2 f_i(x^k)$. Then $g^k_i = \nabla^2 f_i(x^k)(z^k-w^k) + \nabla f_i(w^k)$ is, indeed, approximates $\nabla f(z^k)$ as $$\|\nabla f(z^k) - \nabla f_i(w^k) - \nabla^2 f_i(x^k)(z^k-w^k)\| \le \frac{\HS}{2}\|z^k-w^k\|^2 \le \HS\|z^k - x^*\|^2 + \HS\|w^k - x^*\|^2.$$

\subsection{Local convergence theory} Similar to Assumptions \ref{asm:comp-1} and \ref{asm:comp-2}, we need one of the following assumptions related to the compression done by the master.
\begin{assumption}\label{asm:comp-3}
Compressors $\cC_{\rm M}^k\in\bC(\delta_{\rm M})$ and learning rate (i) $\eta=1-\sqrt{1-\delta_{\rm M}}$ or (ii) $\delta_{\rm M}=1$.
\end{assumption}

\begin{assumption}\label{asm:comp-4}
Compressors $\cC_{\rm M}^k\in\B(\omega_{\rm M})$, learning rate $0< \eta \le \frac{1}{\omega_{\rm M}+1}$. Moreover, for all $j\in [d]$,  each entry $(z^k)_{j}$ is a convex combination of $\{ (x^t)_{j} \}_{t=0}^k$ for any $k\geq 0$.
\end{assumption}

Note that Assumption \ref{asm:comp-2} assumes that $(\mH_i^k)_{jl}$ is a convex combination of $\{  (\nabla^2 f_i(x^t))_{jl}  \}_{t=0}^k$ as the Hessian learning technique is based on exact Hessians $\nabla^2 f_i(x^k)$ at $x^k$. However, in \algname{FedNL-BC}, the Hessian learning technique is based on Hessians $\nabla^2 f_i(z^k)$ at $z^k$. Hence, it makes sense to adapt Assumption \ref{asm:comp-2} and assume that $(\mH_i^k)_{jl}$ is a convex combination of $\{  (\nabla^2 f_i(z^t))_{jl}  \}_{t=0}^k$.

Moreover, we need alternatives to constants $A,\;B,\;C,\;D$ in this case, which we denote by $A_{\rm M},\;B_{\rm M},\;C_{\rm M},\;D_{\rm M}$ and define as follows
\begin{equation}\label{AB_M}
(A_{\rm M}, B_{\rm M}) \eqdef
\begin{cases}
(\eta^2, \eta)   & \;\text{if Assumption}\; \ref{asm:comp-3} \text{(i) holds} \phantom{i}  \\
(\tfrac{\delta_{\rm M}}{4}, \tfrac{6}{\delta_{\rm M}} - \tfrac{7}{2})   & \;\text{if Assumption}\; \ref{asm:comp-3} \text{(ii) holds}  \\
(\eta, \eta) & \;\text{if Assumption}\; \ref{asm:comp-4}\; \text{holds} \phantom{00}
\end{cases}
\end{equation}
\begin{equation}\label{CD_M}
(C_{\rm M}, D_{\rm M}) \eqdef
\begin{cases}
(24, 8\HF^2 + \nicefrac{9}{4}\HS^2) & \;\text{if {\it Option 1} is used in \algname{FedNL-BC}}  \\
(32, 16\HF^2 + \nicefrac{9}{4}\HS^2)   & \;\text{if {\it Option 2} is used in \algname{FedNL-BC}} \\
\end{cases}.
\end{equation}

Following the same steps of Lemma \ref{lm:threecomp}, one can show the following lemma for different compressors applied by the master to handle $\mathbb{E}_k \[\|z^k + \eta \cC_{\rm M}^k(u - z^k) - v\|^2\]$ term, where $\mathbb{E}_k[u] = u$ and $\mathbb{E}_k[v] = v$.

\begin{lemma}\label{lm:threecomp-master}
  For any $u$, $v\in \R^d$ such that $\mathbb{E}_k[u] = u$ and $\mathbb{E}_k[v] = v$, we have the following result combining three different cases from \eqref{AB_M}:
  $$
  \mathbb{E}_k \|z^k + \eta \cC_{\rm M}^k(u - z^k) - v\|^2
  \leq (1-A_{\rm M}) \|z^k - v\|^2 + B_{\rm M} \|u-v\|^2. 
  $$
\end{lemma}

The proof of Lemma \ref{lm:threecomp-master} can be obtained by repeating the proof of Lemma \ref{lm:threecomp} with small modifications. Denote
$$
r_k \eqdef \norm{x^k-x^*}^2, \quad \nu_k \eqdef \norm{w^k-x^*}^2, \quad \gamma_k = \norm{z^k-x^*}^2.
$$
$$
E_1 \eqdef 16\HF^2,\; E_2 \eqdef 16,\; E_3 \eqdef 16\HF^2 + 8\HS^2.
$$

We prove local linear rate for Lyapunov function $\Phi^k \eqdef \|z^k - x^*\|^2 + \frac{A_{\rm M}}{3p}\|w^k - x^*\|^2$. As a result, we show that both $z^k\to x^*$ and $w^k\to x^*$ converge locally linearly. 

\begin{theorem}\label{th:FedNL-BC}
Let Assumption \ref{asm:main} hold and assume that $\cH^k \le \frac{A_{\rm M}}{B_{\rm M}}\frac{\mu^2}{9C_{\rm M}}$ and $\|z^k - x^*\|^2 \le \frac{A_{\rm M}}{B_{\rm M}}\frac{\mu^2}{9E_3}$ for all $k\ge0$. Then, we have the following linear rate for \algname{FedNL-BC}:
\begin{eqnarray*}
\E\[\Phi^{k}\]
\le \(1 - \min\left\{\frac{A_{\rm M}}{3}, \frac{p}{2} \right\}\)^k \Phi^0.
\end{eqnarray*}
\end{theorem}

We assumed inequalities $\cH^k \le \frac{A_{\rm M}}{B_{\rm M}}\frac{\mu^2}{9C_{\rm M}}$ and $\|z^k - x^*\|^2 \le \frac{A_{\rm M}}{B_{\rm M}}\frac{\mu^2}{9E_3}$ hold for all $k\ge0$. Next we prove these inequalities using initial conditions only.

\begin{lemma}\label{lm:bc-13}
Let Assumptions \ref{asm:comp-1} and \ref{asm:comp-3} hold. If
$$
\cH^0 \le \frac{A_{\rm M}}{B_{\rm M}}\frac{\mu^2}{9C_{\rm M}},
\quad
\|z^0 - x^*\| \le \min\left\{ \frac{A_{\rm M}}{B_{\rm M}}\frac{\mu^2}{9E_3}, \frac{A}{B\HF^2}\frac{A_{\rm M}}{B_{\rm M}}\frac{\mu^2}{9C_{\rm M}} \right\},
$$
then the same upper bounds hold for all $k\ge 0$, i.e.,
$$
\cH^k \le \frac{A_{\rm M}}{B_{\rm M}}\frac{\mu^2}{9C_{\rm M}},
\quad
\|z^k - x^*\| \le \min\left\{ \frac{A_{\rm M}}{B_{\rm M}}\frac{\mu^2}{9E_3}, \frac{A}{B\HF^2}\frac{A_{\rm M}}{B_{\rm M}}\frac{\mu^2}{9C_{\rm M}} \right\}.
$$
\end{lemma}

\begin{lemma}\label{lm:bc-24}
Let Assumptions \ref{asm:comp-2} and \ref{asm:comp-4} hold. If
$$
\|x^0 - x^*\| \le \min\left\{ \frac{\mu^2}{9d^2E_3}, \frac{\mu^2}{9C_{\rm M}d^4\HM^2} \right\},
$$
then the following upper bounds hold for all $k\ge0$, i.e.,
\begin{equation}\label{locality-bc}
\cH^k \le \frac{\mu^2}{9dC_{\rm M}},
\quad
\|z^k - x^*\| \le \min\left\{ \frac{\mu^2}{9dE_3}, \frac{\mu^2}{9C_{\rm M}d^3\HM^2} \right\}.
\end{equation}
\end{lemma}

\subsection{Proof of Theorem \ref{th:FedNL-BC}}

Consider {\it Option 1} first and expand $\|x^{k+1} - x^*\|^2$: 
\begin{eqnarray*}
  \norm{x^{k+1}-x^*}^2 =  \norm{\[\mH^k_\mu\]^{-1}(\mH^k_\mu(z^k-x^*) - g^k)}^2 \leq \frac{1}{\mu^2}\norm{\mH^k_\mu(z^k-x^*) - g^k}^2.
\end{eqnarray*}

Then we decompose the term $\mH^k_\mu (z^k -x^*) - g^k$ as follows
\begin{eqnarray}
  && \mH^k_\mu (z^k -x^*) - g^k\notag\\
  &=& (\mH^k_\mu-\nabla^2 f(z^k))(z^k-x^*) + \[\nabla^2f(z^k) - \nabla f(z^k) + \nabla f(x^*)\] + \[\nabla f(z^k) - g^k\]\notag\\
  &=& (\mH^k_\mu - \nabla^2f(z^k))(z^k-x^*) + \[\nabla^2f(z^k)(z^k-x^*) - \nabla f(z^k) + \nabla f(x^*)\]\notag\\
  && \qquad + \[\nabla f(z^k)-\nabla f(w^k) - \mH^k(z^k-w^k)\]\notag\\
  &=&  (\mH^k_\mu - \nabla^2f(z^k))(z^k-x^*) + \[\nabla^2f(z^k)(z^k-x^*) - \nabla f(z^k) + \nabla f(x^*)\]\notag\\
  && \qquad + \[\nabla f(z^k)-\nabla f(w^k) - \nabla^2 f(z^k)(z^k-w^k)\] + (\nabla^2 f(z^k)-\mH^k)(z^k-w^k) \label{eq:decompose-error}
\end{eqnarray}
and apply back to the previous inequality
\begin{eqnarray}
  &&\norm{x^{k+1}-x^*}^2\notag\\
  &\leq&\frac{4}{\mu^2}\left(\norm{(\mH^k_\mu-\nabla^2f(z^k))(z^k-x^*)}^2 + \norm{\nabla^2f(z^k)(z^k-x^*) - \nabla f(z^k) + \nabla f(x^*)}^2\right.\notag\\
  &&  \quad + \left.\norm{\nabla f(z^k) - \nabla f(w^k) - \nabla^2f(z^k)(z^k-w^k)}^4 + \norm{(\mH^k-\nabla f(w^k))(z^k-w^k)}^4\right)\notag\\
  & \leq & \frac{4}{\mu^2}\left(\fronorm{\mH^k - \nabla ^2f(z^k)}^2\norm{z^k-x^*}^2 + \fronorm{\mH^k-\nabla^2 f(w^k)}\norm{z^k-w^k}^2\right.\notag\\
  && \quad + \left.\frac{\HS^2}{4}\norm{z^k-x^*}^2 + \frac{\HS^2}{4}\norm{z^k-w^k}^2\right)\notag\\
  &\leq & \frac{4}{\mu^2}\left(2\fronorm{\mH^k-\nabla^2 f(x^*)}^2\norm{z^k-x^*}^2+2\fronorm{\nabla^2f(z^k)-\nabla^2f(x^*)}^2\norm{z^k-x^*}^2\right. \notag \\
  && \quad + 4\[\fronorm{\mH^k-\nabla^2f(x^*)}^2+\fronorm{\nabla^2f(w^k)-\nabla^2f(x^*)}^2\]\[\norm{z^k-x^*}^2 + \norm{w^k-x^*}^2\]\notag\\
  &&\quad + \left.\frac{\HS^2}{4}\norm{z^k-x^*}^4 + \frac{\HS^2}{4}\norm{z^k-w^k}^4\right)\notag\\
  &\leq & \frac{4}{\mu^2}\left(2\fronorm{\mH^k-\nabla^2 f(x^*)}^2\norm{z^k-x^*}^2 + 2\HF^2\norm{z^k-x^*}^4\right. \notag\\
  && \quad + 4\[\fronorm{\mH^k-\nabla^2f(x^*)}^2 + \HF^2\norm{w^k-x^*}^2\]\[\norm{z^k-x^*}^2+\norm{w^k-x^*}^2\]\notag\\
  &&\quad + \left.\frac{\HS^2}{4}\norm{z^k-x^*}^4 + 2\HS^2\norm{z^k-x^*}^4+2H^2\norm{w^k-x^*}^4\right) \notag\\
  &\leq& \frac{4}{\mu^2}\left(2\cH^k \gamma_k + 2\HF^2 \gamma_k^2 + 4\[\cH^k + \HF^2\nu_k\](\gamma_k+\nu_k)+\frac{\HS^2}{4}\gamma_k^2+2H^2\gamma_k^2 + 2\HS^2\nu_k^2 \right)\notag\\
  &=&\gamma_k\(\frac{24}{\mu^2}\cH^k + \frac{8\HF^2 + \nicefrac{9}{4}\HS^2}{\mu^2}\gamma_k + \frac{16\HF^2}{\mu^2}\nu_k\)+\frac{16}{\mu^2}\cH^k\nu_k + \frac{16\HF^2 + 8\HS^2}{\mu^2}\nu_k^2 \label{eq:rec-bc-1},
\end{eqnarray}
where
\begin{eqnarray}
  \cH^k &\eqdef& \frac{1}{n}\sum\limits_{i=1}^n\fronorm{\mH^k_i-\nabla^2f_i(x^*)}^2,\notag\\
  r_k &\eqdef& \norm{x^k-x^*}^2, \quad \nu_k \eqdef \norm{w^k-x^*}^2, \quad \gamma_k = \norm{z^k-x^*}^2\notag.
\end{eqnarray}

For {\it Option 2} we have similar bound with different constants. Recall that $\mu\mI \preceq \nabla^2 f(x^k) \preceq \mH^k + l^k\mI$.
\begin{eqnarray*}
\norm{x^{k+1}-x^*}
=  \norm{\[\mH^k + l^k\mI\]^{-1}\(\[\mH^k + l^k\mI\](z^k-x^*) - g^k\)}
\leq \frac{1}{\mu}\norm{\[\mH^k + l^k\mI\](z^k-x^*) - g^k}.
\end{eqnarray*}

Then we decompose the term $\[\mH^k + l^k\mI\](z^k-x^*) - g^k$ similar to \eqref{eq:decompose-error}:
\begin{eqnarray}
  && \[\mH^k + l^k\mI\](z^k -x^*) - g^k\notag\\
  &=&  (\mH^k - \nabla^2f(z^k))(z^k-x^*) + l^k(z^k-x^*) + \[\nabla^2f(z^k)(z^k-x^*) - \nabla f(z^k) + \nabla f(x^*)\]\notag\\
  && \qquad + \[\nabla f(z^k)-\nabla f(w^k) - \nabla^2 f(z^k)(z^k-w^k)\] + (\nabla^2 f(z^k)-\mH^k)(z^k-w^k)\notag
\end{eqnarray}
and apply back to the previous inequality
\begin{eqnarray}
  &&\norm{x^{k+1}-x^*}^2\notag\\
  &\leq&\frac{5}{\mu^2}\left(\norm{(\mH^k-\nabla^2f(z^k))(z^k-x^*)}^2 + \|l^k(z^k-x^*)\|^2 + \norm{\nabla^2f(z^k)(z^k-x^*) - \nabla f(z^k) + \nabla f(x^*)}^2\right.\notag\\
  &&  \quad + \left.\norm{\nabla f(z^k) - \nabla f(w^k) - \nabla^2f(z^k)(z^k-w^k)}^4 + \norm{(\mH^k-\nabla f(w^k))(z^k-w^k)}^4\right)\notag\\
  & \leq & \frac{5}{\mu^2}\left(\fronorm{\mH^k - \nabla ^2f(z^k)}^2\norm{z^k-x^*}^2 + \[l^k\]^2\|(z^k-x^*)\|^2 + \fronorm{\mH^k-\nabla^2 f(w^k)}\norm{z^k-w^k}^2\right.\notag\\
  && \quad + \left.\frac{\HS^2}{4}\norm{z^k-x^*}^2 + \frac{\HS^2}{4}\norm{z^k-w^k}^2\right)\notag\\
  & \leq & \frac{5}{\mu^2}\left( \frac{2}{n}\sum_{i=1}^n\fronorm{\mH_i^k - \nabla^2f_i(z^k)}^2\norm{z^k-x^*}^2 + \fronorm{\mH^k-\nabla^2 f(w^k)}\norm{z^k-w^k}^2\right.\notag\\
  && \quad + \left.\frac{\HS^2}{4}\norm{z^k-x^*}^2 + \frac{\HS^2}{4}\norm{z^k-w^k}^2\right)\notag\\
  &\leq & \frac{5}{\mu^2}\left(\frac{4}{n}\sum_{i=1}^n\fronorm{\mH_i^k - \nabla^2f_i(x^*)}^2\norm{z^k-x^*}^2 + \frac{4}{n}\sum_{i=1}^n\fronorm{\nabla^2f_i(z^k) - \nabla^2f_i(x^*)}^2\norm{z^k-x^*}^2 \right. \notag\\
  && \quad + 4\[\fronorm{\mH^k-\nabla^2f(x^*)}^2+\fronorm{\nabla^2f(w^k)-\nabla^2f(x^*)}^2\]\[\norm{z^k-x^*}^2 + \norm{w^k-x^*}^2\]\notag\\
  &&\quad + \left.\frac{\HS^2}{4}\norm{z^k-x^*}^4 + \frac{\HS^2}{4}\norm{z^k-w^k}^4\right)\notag\\
  &\leq & \frac{5}{\mu^2}\left(4\cH^k\norm{z^k-x^*}^2 + 4\HF^2\norm{z^k-x^*}^4\right. \notag\\
  && \quad + 4\[\fronorm{\mH^k-\nabla^2f(x^*)}^2 + \HF^2\norm{w^k-x^*}^2\]\[\norm{z^k-x^*}^2+\norm{w^k-x^*}^2\]\notag\\
  &&\quad + \left.\frac{\HS^2}{4}\norm{z^k-x^*}^4 + 2\HS^2\norm{z^k-x^*}^4+2H^2\norm{w^k-x^*}^4\right) \notag\\
  &\leq& \frac{5}{\mu^2}\left(4\cH^k \gamma_k + 4\HF^2 \gamma_k^2 + 4\[\cH^k + \HF^2\nu_k\](\gamma_k+\nu_k)+\frac{9\HS^2}{4}\gamma_k^2 + 2\HS^2\nu_k^2 \right)\notag\\
  &=&\gamma_k\(\frac{32}{\mu^2}\cH^k + \frac{16\HF^2 + \nicefrac{9}{4}\HS^2}{\mu^2}\gamma_k + \frac{16\HF^2}{\mu^2}\nu_k\)+\frac{16}{\mu^2}\cH^k\nu_k + \frac{16\HF^2 + 8\HS^2}{\mu^2}\nu_k^2 \label{eq:rec-bc-2}.
\end{eqnarray}

Combining \eqref{eq:rec-bc-1} and \eqref{eq:rec-bc-2} with \eqref{CD_M}, we have
\begin{equation}\label{eq:rec-iter}
r_{k+1}
\le \gamma_k\(\frac{C_{\rm M}}{\mu^2}\cH^k + \frac{D_{\rm M}}{\mu^2}\gamma_k + \frac{E_1}{\mu^2}\nu_k\)+\frac{E_2}{\mu^2}\cH^k\nu_k + \frac{E_3}{\mu^2}\nu_k^2,
\end{equation}
where $E_1 \eqdef 16\HF^2,\; E_2 \eqdef 16,\; E_3 \eqdef 16\HF^2 + 8\HS^2$.

Choosing $y=z^k$ and $z=x^*$ in Lemma \ref{lm:threecomp}, we get the following recurrence for $\cH^k$:
\begin{eqnarray}\label{rec-H}
  \E_k\[\cH^{k+1}\] \leq (1-A)\cH^k + B\HF^2\gamma_k.
\end{eqnarray}
Choosing $u=x^{k+1}$ and $v=x^*$ in Lemma \ref{lm:threecomp-master}, we get the following recurrence for $\gamma_k$:
\begin{eqnarray}
  && \E_k\[\gamma_{k+1}\] \\
  &\le& (1-A_{\rm M})\gamma_k + B_{\rm M}r_{k+1} \notag\\
  &\overset{\eqref{eq:rec-bc-2}}{\le}& (1-A_{\rm M})\gamma_k + \gamma_k\(\frac{B_{\rm M}C_{\rm M}}{\mu^2}\cH^k + \frac{B_{\rm M}D_{\rm M}}{\mu^2}\gamma_k + \frac{B_{\rm M}E_1}{\mu^2}\nu_k\)+\frac{B_{\rm M}E_2}{\mu^2}\cH^k\nu_k + \frac{B_{\rm M}E_3}{\mu^2}\nu_k^2 \notag.
\end{eqnarray}

Assume that $\cH^k \le \frac{A_{\rm M}}{B_{\rm M}}\frac{\mu^2}{\max(9C_{\rm M},12E_2)} = \frac{A_{\rm M}}{B_{\rm M}}\frac{\mu^2}{9C_{\rm M}}$ and $\gamma_k \le \frac{A_{\rm M}}{B_{\rm M}}\frac{\mu^2}{9\max(D_{\rm M}, E_1, E_3)} = \frac{A_{\rm M}}{B_{\rm M}}\frac{\mu^2}{9E_3}$ for all $k\ge0$. Then from the update rule of $w^k$ we also have $\nu_k \le \frac{\mu^2 A_{\rm M}}{9B_{\rm M}\max(D_{\rm M}, E_1, E_3)}$. Using this upper bounds we can simplify the recurrence relation for $\gamma_k$ to the following
\begin{equation}\label{rec-gamma}
\E_k\[\gamma_{k+1}\] \le \(1 - \frac{2A_{\rm M}}{3}\)\gamma_k + \frac{A_{\rm M}}{6}\nu_k.
\end{equation}

In addition, from the update rule of $w^k$ we imply
\begin{eqnarray*}
  \E_k\[\nu_{k+1}\] = (1-p)\nu_k + p \gamma_k.
\end{eqnarray*}

Finally, for the Lyapunov function
$$
\Phi^k = \gamma_k + \frac{A_{\rm M}}{3p}\nu_k,
$$
we have
\begin{eqnarray}
\E_k\[\Phi^{k+1}\]
&=& \E_k\[\gamma_{k+1}\] + \frac{A_{\rm M}}{3p}\E_k\[\nu_{k+1}\] \notag \\
&\le& \(1 - \frac{2A_{\rm M}}{3}\)\gamma_k + \frac{A_{\rm M}}{6}\nu_k + \frac{A_{\rm M}}{3p}\[ (1-p)\nu_k + p \gamma_k \] \notag \\
&=& \(1 - \frac{A_{\rm M}}{3}\)\gamma_k + \(1 - \frac{p}{2}\)\frac{A_{\rm M}}{3p}\nu_k \notag \\
&\le& \(1 - \min\left\{\frac{A_{\rm M}}{3}, \frac{p}{2} \right\}\) \Phi^k.
\end{eqnarray}

\subsection{Proof of Lemma \ref{lm:bc-13}}

We prove the lemma by induction. Let for some $k$ we have $\cH^k \le \frac{A_{\rm M}}{B_{\rm M}}\frac{\mu^2}{9C_{\rm M}}$ and $\gamma_k \le \min\left\{ \frac{A_{\rm M}}{B_{\rm M}}\frac{\mu^2}{9E_3}, \frac{A}{B\HF^2}\frac{A_{\rm M}}{B_{\rm M}}\frac{\mu^2}{9C_{\rm M}} \right\}$. Then, from the definition of $w^k$ we have $\nu_k \le \min\left\{ \frac{A_{\rm M}}{B_{\rm M}}\frac{\mu^2}{9E_3}, \frac{A}{B\HF^2}\frac{A_{\rm M}}{B_{\rm M}}\frac{\mu^2}{9C_{\rm M}} \right\}$. Since compressors $\cC_{\rm M}^k$ are deterministic (Assumption \ref{asm:comp-3}), from \eqref{rec-gamma} we conclude
$$
\gamma_{k+1} \le \(1 - \frac{2A_{\rm M}}{3}\)\gamma_k + \frac{A_{\rm M}}{6}\nu_k \le \max\{\gamma_k, \nu_k\}
\le \min\left\{ \frac{A_{\rm M}}{B_{\rm M}}\frac{\mu^2}{9E_3}, \frac{A}{B\HF^2}\frac{A_{\rm M}}{B_{\rm M}}\frac{\mu^2}{9C_{\rm M}} \right\}.
$$
Since compressors $\cC_i^k$ are deterministic (Assumption \ref{asm:comp-1}), from \eqref{rec-H} we conclude
$$
\cH_{k+1} \le (1-A)\cH^k + B\HF^2\gamma_k \le (1-A)\frac{A_{\rm M}}{B_{\rm M}}\frac{\mu^2}{9C_{\rm M}} + B\HF^2\frac{A}{B\HF^2}\frac{A_{\rm M}}{B_{\rm M}}\frac{\mu^2}{9C_{\rm M}}
= \frac{A_{\rm M}}{B_{\rm M}}\frac{\mu^2}{9C_{\rm M}}.
$$

\subsection{Proof of Lemma \ref{lm:bc-24}}

First note that in this case $A_{\rm M} = B_{\rm M} = \eta$ so that the ratio $\frac{A_{\rm M}}{B_{\rm M}} = 1$. From the Assumption \ref{asm:comp-4}, we have $\mH_i^0 = \nabla^2 f_i(z^0)$, from which we get
$$
\|\mH_i^0 - \nabla^2 f_i(x^*)\|_{\rm M}^2 \le \HM^2\|z^0 - x^*\|^2 \le \frac{\mu^2}{9dC_{\rm M}},
$$
which implies $\cH^0 \le \frac{\mu^2}{9dC_{\rm M}}$. Also notice that $x^0=z^0$ so that \eqref{locality-bc} holds for $k=0$. Next we do induction. Let
\begin{equation*}
\cH^k \le \frac{\mu^2}{9dC_{\rm M}},
\;
\|z^k - x^*\| \le \min\left\{ \frac{\mu^2}{9dE_3}, \frac{\mu^2}{9C_{\rm M}d^3\HM^2} \right\},
\;
\|x^k - x^*\| \le \min\left\{ \frac{\mu^2}{9d^2E_3}, \frac{\mu^2}{9C_{\rm M}d^4\HM^2} \right\}.
\end{equation*}
hold for all $k\le K$ and prove it for $k=K+1$. Using bounds $\cH^k\le\frac{\mu^2}{9dC_{\rm M}}$ and $\gamma_k\le\frac{\mu^2}{9dE_3}$ we deduce from \eqref{eq:rec-iter} that
$$
\|x^{K+1} - x^*\|^2 \le \frac{1}{3d}\gamma_K + \frac{1}{6d}\nu_K \le \frac{1}{d}\max\{\gamma_K,\nu_K\}
\le \min\left\{ \frac{\mu^2}{9d^2E_3}, \frac{\mu^2}{9C_{\rm M}d^4\HM^2} \right\}.
$$

Since $(z^{K+1})_{j}$ is a convex combination of $\{ (x^t)_{j} \}_{t=0}^{K+1}$, we get
\begin{eqnarray*}
\|z^{K+1}-x^*\|^2
&=& \sum_{j=1}^d |(z^{K+1} - x^*)_j|^2 \\
&\le& \sum_{j=1}^d \max_{0\le t\le K+1}|(x^t - x^*)_j|^2 \\
&\le& d \max_{0\le t\le K+1}\|x^t - x^*\|^2
\le \min\left\{ \frac{\mu^2}{9dE_3}, \frac{\mu^2}{9C_{\rm M}d^3\HM^2} \right\}.
\end{eqnarray*}
Since $(\mH_i^{K+1})_{jl}$ is a convex combination of $\{  (\nabla^2 f_i(z^t))_{jl}  \}_{t=0}^{K+1}$, we get
\begin{eqnarray*}
\|\mH_i^{K+1} - \nabla^2 f_i(x^*)\|^2
&=& \sum_{j,l=1}^d |(\mH_i^{K+1} - \nabla^2 f_i(x^*))_{jl}|^2 \\
&\le& d^2\HM^2 \max_{0\le t\le K+1}\|z^t - x^*\|^2
\le \frac{\mu^2}{9dC_{\rm M}}.
\end{eqnarray*}
The last three inequalities complete the induction step and we conclude the lemma.

\clearpage
\section{Local Quadratic Rate of \algname{NEWTON-STAR} for General Finite-Sum Problems}

In their recent work, \citet{Islamov2021NewtonLearn} proposed a novel Newton-type method, which does not update the Hessian estimator from iteration to iteration and, meanwhile, preserves fast local quadratic rate of convergence. The method can be describe with a single update rule preformed by the master:
\begin{equation}\label{eq:newton-star}
x^{k+1} = x^k - \[\nabla^2 f(x^*)\]^{-1} \nabla f(x^k), \quad k\ge0.
\end{equation}
Note that parallel nodes need to send the master only gradient information $\nabla f_i(x^k)$. Then master aggregates them, performs the update step \eqref{eq:newton-star} and sends new parameters $x^{k+1}$ to devices for the next round. While this scheme is very simple-looking, notice that the update rule \eqref{eq:newton-star} depends on the knowledge of $\nabla^2 f(x^*)$, where $x^*$ is the (unique) solution of \eqref{erm-prob}. As we do not know $x^*$ (otherwise there is no sense to do any training), this method, called \algname{NEWTON-STAR}, is practically useless and cannot be implemented. However, this method was quite useful in theory, since it led to a new practical method.

Now, the local quadratic rate of \algname{NEWTON-STAR} was shown using some special structure of local loss functions $f_i(x)$. Here we provide a very simple proof of local quadratic rate which works for any smooth losses and does not need special structure of $f_i(x)$.

\begin{theorem}\label{th:N*}
Assume that $f\colon\R^d\to\R$ has $\HS$-Lipschitz Hessian and the Hessian at the optimum $x^*$ is positive definite with parameter $\mu>0$. Then local convergence rate of \algname{NEWTON-STAR} \eqref{eq:newton-star} is quadratic, i.e., for any $k\ge0$ and initial point $x^0\in\R^d$ we have
\begin{eqnarray*}
\|x^{k+1} - x^*\| \le \frac{\HS}{2\mu}\|x^k-x^*\|^2.
\end{eqnarray*}\
\end{theorem}
\begin{proof}
As we do not have a regularization term in our ERM problem, we imply $\nabla f(x^*)=0$. Hence
\begin{eqnarray*}
\|x^{k+1} - x^*\|
&=& \left\|x^k-x^* - \[\nabla^2 f(x^*)\]^{-1} \nabla f(x^k)\right\| \\
&\le& \left\| \[\nabla^2 f(x^*)\]^{-1} \right\| \left\|\nabla^2 f(x^*)(x^k-x^*) - \nabla f(x^k) + \nabla f(x^*)\right\| \\
&\le& \frac{\HS}{2\mu}\|x^k-x^*\|^2,
\end{eqnarray*}
where we used positive definiteness $\nabla^2 f(x^*) \succeq \mu\mI$ and $\HS$-Lipschitzness of the Hessian $\nabla^2 f(x)$, namely
$$
\left\|\nabla^2 f(y)(x-y) - \nabla f(x) + \nabla f(y)\right\| \le \frac{\HS}{2}\|x-y\|^2, \quad x,y\in\R^d.
$$
\end{proof}

\clearpage
\section{Limitations}\label{sec:limitations}

Here we discuss main limitations of our approach and directions which are not explored in this work.

\begin{itemize}
\item Our theory covers general convex (the rate \eqref{rate:global-sublinear-fval}) and strongly convex (all other rates of this paper) loss functions. We do not consider non-convex objectives in this work.

\item All the proposed methods are analyzed in the regime when the exact local gradients and exact local Hessians of local loss functions are computed for all participating devices. We do not consider stochastic gradient or stochastic Hessian oracles of local loss functions in our analyses.
\item We present separate methods/extensions (\algname{FedNL}, \algname{FedNL-PP}, \algname{FedNL-CR}, etc) for each setup (compressed communication, partial participation, globalization, etc) to make our contributions clearer. For practical purposes, however, one might need to combine these extensions in order to get a method which supports compressed communication, partial participation, globalization, etc at the same time. We do not design a single master method containing all these extensions as special cases.
\item Finally, we do not provide strong (differential) privacy guarantees for our methods. Our privacy enhancement mechanism offers the most  rudimentary level of privacy only: we forbid the devices do directly send/reveal their training data  to the server.
\end{itemize}

\clearpage
\section{Table of Frequently Used Notation}

\begin{table}[!h]
\caption{Notation we use throughout the paper.}
\label{tbl:notation}
\begin{center}
{\scriptsize
\begin{tabular}{|c|l|c|}
\hline 
\multicolumn{3}{|c|}{{\bf Basic} }\\
\hline 
$d$ & number of the model parameters to be trained & \\
\hline
$n$ & number of the devices/workers/clients in distributed system & \\
\hline
$[n]$ & $\eqdef \{1, 2, \dots, n\}$ & \\
\hline
$f_i$ & local loss function associated with data stored on device $i\in[n]$ & \eqref{erm-prob}\\
\hline
$f$ & $\eqdef \frac{1}{n}\sum_{i=1}^n f_i(x)$, overall empirical loss/risk & \eqref{erm-prob}\\
\hline
$x^*$ & trained model, i.e., the optimal solution to \eqref{erm-prob} & \\
\hline
$\varepsilon$ & target accuracy & \\
\hline
$\R^{d\times d}$ & the set of $d\times d$ square matrices & \\
\hline
$(\mM)_{jl}$ & the element at $j^{th}$ row and $l^{th}$ column of matrix $\mM$ & \\
\hline
\hline 
\multicolumn{3}{|c|}{{\bf Standard} }\\
\hline 
$\mu$ & strong convexity parameter of $f$ & Asm~\ref{asm:main} \\
\hline
$L$ & Lipschitz constant of the gradient $\nabla f(x)$ w.r.t. the Euclidean norm & Thm~\ref{th:NL-ls} \\
\hline
$\HS$ & Lipschitz constant of the Hessian $\nabla^2 f(x)$ w.r.t. the spectral norm & Asm~\ref{asm:main} \\
\hline
$\HF$ & Lipschitz constant of the Hessian $\nabla^2 f(x)$ w.r.t. the Frobenius norm & Asm~\ref{asm:main} \\
\hline
$\HM$ & Lipschitz constant of the Hessian $\nabla^2 f(x)$ w.r.t. the $\max$ norm & Asm~\ref{asm:main} \\
\hline
$\cC$ & (possibly randomized) compression operator $\cC\colon\R^d\to\R^d$ & \eqref{class-unbiased}, \eqref{class-contractive}  \\
\hline
$\B(\omega)$ & class of unbiased compressors with bounded variance $\omega\ge0$ & Def~\ref{def:class-unbiased} \\
\hline
$\bC(\delta)$ & class of deterministic contractive compressors with contraction $\delta\in[0,1]$ & Def~\ref{def:class-contractive} \\
\hline
\hline 
\multicolumn{3}{|c|}{{\bf Algorithm names} }\\
\hline 
\algname{GD} & Gradient Descent & \\
\hline
\algname{GD-LS} & \algname{GD} with Line Search procedure & \\
\hline
\algname{DIANA} & Compressed \algname{GD} with variance reduction \citep{DIANA} & \\
\hline
\algname{ADIANA} & \algname{DIANA} with Nesterov's acceleration \citep{ADIANA} & \\
\hline
\algname{N} & classical Newton & \\
\hline
\algname{NS} & Newton Star & \eqref{eq:newton-star} \\
\hline
\algname{N0} & Newton Zero {\bf (new)} & \eqref{eq:N0} \\
\hline
\algname{N0-LS} & Newton Zero with Line Search procedure {\bf (new)} & \\
\hline
\algname{NL1, NL2} & Newton Learn methods of \citet{Islamov2021NewtonLearn} & \\
\hline
\algname{CNL} & Cubic Newton Learn \citep{Islamov2021NewtonLearn} & \\
\hline
\algname{DINGO} & Distributed Newton-type method of \cite{DINGO} & \\
\hline
\algname{FedNL} & Federated Newton Learn {\bf (new)} & Alg~\ref{alg:FedNL} \\
\hline
\algname{FedNL-PP} & Extension to \algname{FedNL}: Partial Participation {\bf (new)} & Alg~\ref{alg:FedNL-PP} \\
\hline
\algname{FedNL-LS} & Extension to \algname{FedNL}: Globalization via Line Search {\bf (new)} & Alg~\ref{alg:FedNL-LS} \\
\hline
\algname{FedNL-CR} & Extension to \algname{FedNL}: Globalization via Cubic Regularization {\bf (new)} & Alg~\ref{alg:FedNL-CR} \\
\hline
\algname{FedNL-BC} & Extension to \algname{FedNL}: Bidirectional Compression {\bf (new)} & Alg~\ref{alg:FedNL-BC} \\
\hline
\hline 
\multicolumn{3}{|c|}{{\bf Federated Newton Learn (\algname{FedNL})} }\\
\hline 
$\mH_i^k$ & estimate of the local optimal Hessian $\nabla^2 f_i(x^*)$ at client $i$ in iteration $k$ & \\
\hline
$\mH^k$ & estimate of the global optimal Hessian $\nabla^2 f(x^*)$ at the server in iteration $k$ & \\
\hline
$\alpha$ & Hessian learning rate & \\
\hline
$\cC_i^k$ & compression operator applied by the client $i$ in iteration $k$ & \\
\hline
$\mS_i^k$ & $\eqdef \cC_i^k(\nabla^2 f_i(x^k) - \mH_i^k)$ compressed second order information & \\
\hline
$l_i^k$ & $\eqdef \|\nabla^2 f_i(x^k) - \mH_i^k\|_{\rm F}$ compression error & \\
\hline
$A,\; B$ & constants depending on the choice of compressors $\cC_i^k$ and learning rate $\alpha$ & \eqref{ABCD} \\
\hline
$C,\; D$ & constants depending on which option is chosen for the global update & \eqref{ABCD} \\
\hline
\hline 
\multicolumn{3}{|c|}{{\bf Experiments} }\\
\hline 
$\{a_{ij}, b_{ij}\}$ & $j^{th}$ data point stored in device $i$ & \eqref{prob:log-reg} \\
\hline
$m$ & number of local training data points & \eqref{prob:log-reg} \\
\hline
$\lambda$ & regularization parameter & \eqref{prob:log-reg} \\
\hline
\end{tabular}
}
\end{center}
\end{table}

\end{document}